\RequirePackage{silence} %
\PassOptionsToPackage{table}{xcolor}
\documentclass[
    oneside,
    titlepage,
    numbers=noenddot,%
    headinclude,
    footinclude,
    cleardoublepage=empty,
    abstract=on,
    BCOR=5mm,
    paper=letter,
    fontsize=11pt,
    ]{scrreprt}
\PassOptionsToPackage{utf8}{inputenc}
  \usepackage{inputenc}

\PassOptionsToPackage{T1}{fontenc} %
  \usepackage{fontenc}

\PassOptionsToPackage{
  drafting=false,    %
  tocaligned=false, %
  dottedtoc=false,  %
  eulerchapternumbers=true, %
  linedheaders=false,       %
  floatperchapter=true,     %
  eulermath=true,  %
  beramono=true,    %
  palatino=true,    %
  style=arsclassica %
}{classicthesis}

\newcommand{\myTitle}{Learning Equivariant Representations}

\newcommand{\myName}{Carlos Esteves}

\newcommand{\myFaculty}{}

\newcommand{\myUni}{University of Pennsylvania}

\newcommand{\myVersion}{\classicthesis}

\providecommand{\mLyX}{L\kern-.1667em\lower.25em\hbox{Y}\kern-.125emX\@}

\PassOptionsToPackage{american}{babel} %
\usepackage{babel}

\usepackage{csquotes}
\PassOptionsToPackage{%
  uniquelist=false,
  mincitenames=1,
  maxcitenames=2,
  maxbibnames=99,
  backend=bibtex8,bibencoding=ascii,%
  language=auto,%
  style=numeric-comp,%
  sorting=nyt, %
 sortcites=false, %
  natbib=true %
}{biblatex}
    \usepackage{biblatex}

  \usepackage{amsmath}

\usepackage{graphicx} %
\usepackage{scrhack} %
\usepackage{xspace} %
\PassOptionsToPackage{printonlyused,smaller}{acronym}
  \usepackage{acronym} %
\usepackage{tabularx} %
\usepackage{subfig}
\usepackage{listings}
\usepackage{classicthesis}

\hypersetup{%
  colorlinks=true, linktocpage=true, pdfstartpage=3, pdfstartview=FitV,%
  breaklinks=true, pageanchor=true,%
  pdfpagemode=UseNone, %
  plainpages=false, bookmarksnumbered, bookmarksopen=true, bookmarksopenlevel=1,%
  hypertexnames=true, pdfhighlight=/O,%
  urlcolor=CTurl, linkcolor=CTlink, citecolor=CTcitation, %
  pdftitle={\myTitle},%
  pdfauthor={\textcopyright\ \myName, \myUni, \myFaculty},%
  pdfsubject={},%
  pdfkeywords={},%
  pdfcreator={pdfLaTeX},%
  pdfproducer={LaTeX with hyperref and classicthesis}%
}

\usepackage[
  papersize={8.5 in, 11 in},
  includeheadfoot,
  left=1.5 in,
  right = 1 in,
  vmargin= 1 in]{geometry}
\usepackage[doublespacing]{setspace} %

\lehead{\mbox{{\small\pagemark\kern2em}\headmark\hfil}}
\rohead{\mbox{\hfil{\headmark}{\small\kern2em\pagemark}}}

\counterwithout*{footnote}{chapter}

\renewcommand\formatchapter[1]{%
  \begin{minipage}[b]{0.15\linewidth}
    \chapterNumber
  \end{minipage}%
  \begin{minipage}[b]{0.7\linewidth}
    \raggedright\spacedallcaps{#1}
  \end{minipage}
}

\usepackage{tikz-cd}
\usepackage[capitalize]{cleveref}
\usepackage{amsmath,amssymb,amsthm} %
\usepackage{microtype}
\usepackage{siunitx}
\usepackage{booktabs}
\usepackage{multirow}
\usepackage{enumitem}
\usepackage{tensor}
\usepackage{nicefrac}
\newcolumntype{H}{>{\setbox0=\hbox\bgroup}c<{\egroup}@{}}
\newcolumntype{Z}{>{\setbox0=\hbox\bgroup}c<{\egroup}@{\hspace*{-\tabcolsep}}}

\robustify\cellcolor
\newcommand{\bgl}{\cellcolor[HTML]{DDDDDD}}
\newcommand{\bgd}{\cellcolor[HTML]{BBBBBB}}
\newcommand{\chaptersubtitle}[1]{{\raggedright\Large\itshape #1\par}}
\newcommand{\li}[2]{{_#1}{#2}}
\setlist[description]{labelindent=1cm}

\theoremstyle{definition} %
\newtheorem{definition}{Definition}

\theoremstyle{plain} %
\newtheorem{theorem}{Theorem}

\theoremstyle{remark} %

\theoremstyle{theorem}
\newtheorem{proposition}{Proposition}
\newtheorem{lemma}{Lemma}
\theoremstyle{remark}
\newtheorem*{remark}{Remark}
\newtheorem*{example}{Example}
\newtheorem*{examples}{Examples}

\usepackage{tocloft}
\renewcommand{\contentsname}{Table of Contents}

\renewcommand{\listfigurename}{List of Illustrations}
\usepackage{mathtools}

\newcommand{\twobytwo}[4]{\begin{pmatrix}#1 & #2 \\#3 & #4\end{pmatrix}}
\newcommand{\twobyone}[2]{\begin{pmatrix}#1 \\ #2 \end{pmatrix}}
\newcommand{\threebythree}[9]{\begin{pmatrix}#1 & #2 & #3 \\#4 & #5 & #6\\#7 & #8 & #9 \end{pmatrix}}

\newcommand{\norm}[1]{\left\lVert#1\right\rVert}
\newcommand{\fun}[3]{\ensuremath{#1\colon #2\to #3}}

\DeclarePairedDelimiter\inner{\langle}{\rangle}
\DeclarePairedDelimiter\abs{\lvert}{\rvert}

\newcommand{\eps}{\epsilon}
\newcommand{\R}{\mathbb{R}}
\newcommand{\C}{\mathbb{C}}
\newcommand{\Z}{\mathbb{Z}}
\newcommand{\N}{\mathbb{N}}
\newcommand{\Q}{\mathbb{Q}}
\newcommand{\I}{\mathcal{I}}

\newcommand{\matrixgroup}[2]{%
  \ifthenelse{\equal{#2}{}}
  {\ensuremath{\mathbf{#1}}}
  {\ensuremath{\mathbf{#1}(#2)}}}
\newcommand{\GL}[1]{\matrixgroup{GL}{#1}}
\newcommand{\SO}[1]{\matrixgroup{SO}{#1}}

\newcommand{\SU}[1]{\matrixgroup{SU}{#1}}

\newcommand{\SL}[1]{\matrixgroup{SL}{#1}}

\newcommand{\tr}{\text{tr}}

\renewcommand{\l}{\ell}

\newcommand{\im}{f}
\newcommand{\simtwo}{\matrixgroup{SIM}{2}}
\usepackage[abbreviations,toc,nogroupskip,nonumberlist]{glossaries-extra}
\setabbreviationstyle{long-short}

\newabbreviation{FFT}{FFT}{Fast Fourier Transform}
\newabbreviation{CNN}{CNN}{convolutional neural network}
\newabbreviation{G-CNN}{G-CNN}{group equivariant convolutional neural network}
\newcommand{\cnn}{\gls{CNN}}
\newcommand{\cnns}{\glspl{CNN}}

\newcommand{\gcnn}{\gls{G-CNN}}
\newcommand{\gcnns}{\glspl{G-CNN}}
\newcommand{\Gcnns}{\Glspl{G-CNN}}
\newcommand{\fft}{\gls{FFT}}

\newabbreviation{PTN}{PTN}{polar transformer network}
\newcommand{\ptn}{\gls{PTN}}
\newcommand{\ptns}{\glspl{PTN}}
\newcommand{\Ptns}{\Glspl{PTN}}

\newabbreviation{STN}{STN}{spatial transformer network}
\newcommand{\stn}{\gls{STN}}
\newcommand{\stns}{\glspl{STN}}

\newabbreviation{EMVN}{EMVN}{equivariant multi-view network}

\newcommand{\emvns}{\glspl{EMVN}}

\newabbreviation{SCHN}{SCHN}{spherical convolutional hourglass network}
\newcommand{\schn}{\gls{SCHN}}
\newcommand{\Schn}{\Gls{SCHN}}

\newabbreviation{SVD}{SVD}{singular-value decomposition}
\newcommand{\svd}{\gls{SVD}}

\newabbreviation{pca}{PCA}{principal component analysis}
\newcommand{\pca}{\gls{pca}}

\newabbreviation{sft}{SFT}{spherical Fourier transform}
\newcommand{\sft}{\gls{sft}}

\newabbreviation{isft}{ISFT}{inverse spherical Fourier transform}
\newcommand{\isft}{\gls{isft}}

\newabbreviation{soft}{SOFT}{rotation group Fourier transform}
\newcommand{\soft}{\gls{soft}}

\newabbreviation{swsft}{SWSFT}{spin-weighted spherical Fourier transform}
\newcommand{\swsft}{\gls{swsft}}

\newabbreviation{ReLU}{ReLU}{rectified linear unit}
\newcommand{\relu}{\gls{ReLU}}
\newcommand{\relus}{\glspl{ReLU}}

\newabbreviation{sgd}{SGD}{stochastic gradient descent}
\newcommand{\sgd}{\gls{sgd}}

\newabbreviation{map}{mAP}{mean average precision}
\newcommand{\map}{\gls{map}}

\newabbreviation{miou}{mIoU}{mean intersection over union}
\newcommand{\miou}{\gls{miou}}

\newabbreviation{RGB}{RGB}{red, green, blue}
\newcommand{\rgb}{\acrshort{RGB}}

\newabbreviation{SVHN}{SVHN}{street view house numbers}
\newcommand{\svhn}{\gls{SVHN}}

\newabbreviation{MV}{MV}{multi-view}
\newcommand{\mv}{\gls{MV}}
\newcommand{\Mv}{\Gls{MV}}

\newabbreviation{mvcnn}{MVCNN}{multi-view convolutional neural network}
\newcommand{\mvcnn}{\gls{mvcnn}}

\newabbreviation{hspace}{H-space}{homogeneous space}
\newcommand{\hspc}{\gls{hspace}}

\newabbreviation{gconv}{G-conv}{group convolution}
\newcommand{\gconv}{\gls{gconv}}
\newcommand{\gconvs}{\glspl{gconv}}

\newabbreviation{hconv}{H-conv}{homogeneous space convolution}
\newcommand{\hconv}{\gls{hconv}}

\newabbreviation{hcorr}{H-corr}{homogeneous space cross-correlation}
\newcommand{\hcorr}{\gls{hcorr}}

\newabbreviation{wap}{WAP}{weighted average pooling}
\newcommand{\wap}{\gls{wap}}
\newcommand{\Wap}{\Gls{wap}}

\newabbreviation{wgap}{WGAP}{weighted global average pooling}
\newcommand{\wgap}{\gls{wgap}}
\newcommand{\Wgap}{\Gls{wgap}}

\newabbreviation{spp}{SP}{spectral pooling}
\newcommand{\spp}{\gls{spp}}

\newcommand{\magl}{MAG-L}
\newcommand{\healpix}{HEALPix}

\newcommand{\simtwomnist}{SIM$2$MNIST}
\newcommand{\Simtwomnist}{SIM$2$MNIST}
\newcommand{\mnist}{MNIST}
\newcommand{\mnistr}{MNIST-R}
\newcommand{\mnistrts}{MNIST-RTS}
\newcommand{\ccnn}{CCNN}
\newcommand{\pcnn}{PCNN}
\newcommand{\sift}{SIFT}
\newcommand{\rotsvhn}{ROTSVHN}
\newcommand{\resnet}[1]{ResNet$#1$}
\newcommand{\ptnresnet}[1]{PTN-\resnet{#1}}
\newcommand{\mvcnnm}[1]{%
  \ifthenelse{\equal{#1}{}}
  {MVCNN-M}
  {MVCNN-M-$#1$}}

\newcommand{\zz}{\ensuremath{z/z}}
\newcommand{\sotsot}{\ensuremath{\mathbf{SO}(3)/\mathbf{SO}(3)}}
\newcommand{\zsot}{\ensuremath{z/\mathbf{SO}(3)}}

\DeclareMathOperator*{\argmax}{arg\,max}
\DeclareMathOperator*{\argmin}{arg\,min}

\newabbreviation{pnp}{PnP}{Perspective-n-Point}
\newcommand{\pnp}{\gls{pnp}}

\newabbreviation{mse}{MSE}{mean squared error}
\newcommand{\mse}{\gls{mse}}

\newabbreviation{swsh}{SWSH}{spin-weighted spherical harmonic}
\newcommand{\swsh}{\gls{swsh}}

\newcommand{\swshs}{\glspl{swsh}}

\newabbreviation{swsf}{SWSF}{spin-weighted spherical function}
\newcommand{\swsf}{\gls{swsf}}

\newcommand{\swsfs}{\glspl{swsf}}

\newabbreviation{swscnn}{SWSCNN}{spin-weighted spherical CNN}
\newcommand{\swscnn}{\gls{swscnn}}

\newcommand{\swscnns}{\glspl{swscnn}}

\newabbreviation{hks}{HKS}{heat kernel signature}
\newcommand{\hks}{\gls{hks}}

\newabbreviation{svfmnist}{SVFMNIST}{spherical vector field MNIST}
\newcommand{\svfmnist}{\gls{svfmnist}}

\makenoidxglossaries

\begin{document}
\frenchspacing
\raggedbottom
\selectlanguage{american} %
\pagenumbering{roman}
\pagestyle{plain}

\def\mytitle{LEARNING EQUIVARIANT REPRESENTATIONS} %
\def\myauthor{Carlos Henrique Machado Silva Esteves}
\def\myauthorfull{Carlos Henrique Machado Silva Esteves}
\def\mysupervisorname{Kostas Daniilidis}
\def\mysupervisortitle{Professor of Computer and Information Science}
\newlength{\superlen}   %
\settowidth{\superlen}{\mysupervisorname, \mysupervisortitle} %
\def\gradchairname{Mayur Naik}
\def\gradchairtitle{Professor of Computer and Information Science}
\newlength{\chairlen}   %
\settowidth{\chairlen}{\gradchairname, \gradchairtitle} %
\newlength{\maxlen}
\setlength{\maxlen}{\maxof{\superlen}{\chairlen}}
\def\mydepartment{Computer and Information Science}
\def\myyear{2020}
\def\signatures{46 pt} %

\pagenumbering{roman}
\deftripstyle{pgnumbottomcenter}{}{}{}{}{\pagemark{}}{}
\pagestyle{pgnumbottomcenter}

\begin{titlepage}
\thispagestyle{empty} %
\begin{center}

\doublespacing

{\mytitle}
\\[1ex]
\myauthor
\\[1ex]
{A DISSERTATION}

in 

\mydepartment

Presented to the Faculties of the University of Pennsylvania

in 

Partial Fulfillment of the Requirements for the

Degree of Doctor of Philosophy

\myyear

\end{center}

\vfill %

\begin{flushleft}

Supervisor of Dissertation\\[\signatures] %

\renewcommand{\tabcolsep}{0 pt}
\begin{table}[h]
\begin{tabularx}{\maxlen}{l}
\toprule
\mysupervisorname, \mysupervisortitle\\ %
\end{tabularx}
\end{table}

Graduate Group Chairperson\\[\signatures] %

\begin{table}[h]
\begin{tabularx}{\maxlen}{l}
\toprule
\gradchairname, \gradchairtitle\\ %
\end{tabularx}
\end{table}

\doublespacing

Dissertation Committee %

Jean Gallier, Professor of Computer and Information Science

Jianbo Shi, Professor of Computer and Information Science

Alejandro Ribeiro, Professor of Electrical and Systems Engineering

Ameesh Makadia, Staff Research Scientist, Google Research

\end{flushleft}

\end{titlepage}

\doublespacing

\thispagestyle{empty} %

\vspace*{\fill}

\begin{flushleft}
\mytitle

 \copyright \space COPYRIGHT
 
\myyear

\myauthorfull\\[24 pt] %

This work is licensed under the \\
Creative Commons Attribution \\
NonCommercial-ShareAlike 3.0 \\
License

To view a copy of this license, visit

\url{http://creativecommons.org/licenses/by-nc-sa/3.0/}

\end{flushleft}
\pagebreak 

\setcounter{page}{3}  %

\begin{center}
\textit{To Sam.}
\end{center}

\clearpage
\begin{center}
{\Large\spacedallcaps{ACKNOWLEDGEMENT}}
\end{center}
\vspace{1cm}
First and foremost I must thank my advisor, Kostas Daniilidis,
for accepting me into his wonderful group and for the guidance and support
throughout these five years.
Coming to Penn was a turning point in my life;
it has given me enormous personal growth and prepared me
for a future career doing what I love.
None of this would have happened without Kostas.
I will be forever grateful.

Jean Gallier was always a source of inspiration
and his lectures and writings influenced me tremendously.
I am also grateful for having watched
the writing of \emph{Aspects of Harmonic Analysis and Representation Theory} from the beginning,
for partaking in the ``underground'' Tuesday meetings,
for his comments about my work and for our impromptu conversations.

I thank Ameesh Makadia for being a fantastic mentor during the past three years,
for all his prior work that we built upon,
and also for the great and productive time during my internship at Google Research in New York.
I am excited about our future endeavours!

I am grateful for the interactions with Jianbo Shi and Alejandro Ribeiro
both in classes and as part of my thesis committee,
and for the inspiring work that their groups produce.

I was fortunate to have spent a summer at Facebook AI Research in California.
It was nice to work on different projects
and learn about related topics that I had not approached before.
I thank Georgia Gkioxari and Justin Johnson for that.

Finally, I thank my wife Cristiane for staying by my side the whole time and
for being so understanding and supportive,
our son Samuel for the joy he gives us every day,
and my parents for prioritizing my education and for all the support in my early years.

\clearpage
\begin{center}
{\Large\spacedallcaps{ABSTRACT}}
\vspace{2cm}

\mytitle

\myauthor

\mysupervisorname
\end{center}
\vspace{1cm}
State-of-the-art deep learning systems often require large amounts of data and computation.
For this reason, leveraging known or unknown structure of the data is paramount.
Convolutional neural networks (CNNs) are successful examples of this principle,
their defining characteristic being the shift-equivariance.
By sliding a filter over the input, when the input shifts, the response shifts by the same amount,
exploiting the structure of natural images where semantic content is independent of absolute pixel positions.
This property is essential to the success of CNNs in audio, image and video recognition tasks.
In this thesis, we extend equivariance to other kinds of transformations, such as rotation and scaling.
We propose equivariant models for different transformations defined by groups of symmetries.
The main contributions are
(i) polar transformer networks, achieving equivariance to the group of similarities on the plane,
(ii) equivariant multi-view networks, achieving equivariance to the group of symmetries of the icosahedron,
(iii) spherical CNNs, achieving equivariance to the continuous $3$D rotation group,
(iv) cross-domain image embeddings, achieving equivariance to $3$D rotations for $2$D inputs, and
(v) spin-weighted spherical CNNs, generalizing the spherical CNNs and
achieving equivariance to $3$D rotations for spherical vector fields.
Applications include image classification, $3$D shape classification and retrieval, panoramic image classification and segmentation, shape alignment and pose estimation.
What these models have in common is that they leverage symmetries in the data to reduce sample and model complexity and improve generalization performance.
The advantages are more significant on (but not limited to) challenging tasks where
data is limited or input perturbations such as arbitrary rotations are present.

\onehalfspacing
\clearpage
\tableofcontents

\doublespacing
\cleardoublepage
\pagestyle{scrheadings}
\pagenumbering{arabic}
\cleardoublepage

\chapter{Introduction}
\label{sec:introduction}

Learning representations from data enabled enormous progress in a wide variety of applications
in domains such as audio~\cite{oord2016wavenet}, image~\cite{he2016deep}, and natural language~\cite{devlin2018bert}.
Most state-of-the-art approaches consist of deep learning systems that require large amounts of data and computation.
For this reason, leveraging the known or unknown structure of the data is paramount,
and leads to reduced amount of required training data, fewer model parameters and faster training times.

\emph{Convolution} is a way to leverage the structure of the data.
Recall the familiar convolution of functions $f$ and $k$ on the real line
\begin{align*}
  (f * k)(x) = \int\limits_{t\in \R} f(t)k(x-t)\, dt.
\end{align*}
We define the shift operator $(\lambda_yf)(x) = f(x-y)$.
One important convolution property is the \emph{shift-equivariance}:
$(\lambda_y f) * k = \lambda_y (f * k)$.
Intuitively, if the filter $k$ is designed to respond to some pattern in $f$,
this property tells us that the response will be the same (just shifted) no matter where the
pattern appears.
This is fundamental to the success of \cnns\ introduced by \textcite{fukushima1980neocognitron}.
For example, the application of \cnns\ to image analysis exploits the structure of natural images where
the semantic content is independent of absolute pixel positions.

The filter $k$ is \emph{learned} and compactly supported,
and convolution allows weight-sharing, in contrast with fully connected networks.
The combination of \cnns\ and the backpropagation algorithm (\textcite{lecun1989backpropagation}) is an essential part of the recent deep learning revolution.

In this thesis, we generalize shift equivariance
and present models equivariant to transformations defined by different groups of symmetries.
This property is called \emph{group equivariance}~\cite{cohen2016group}.
Let $\lambda_g,\, \lambda'_g$ denote left group actions on $X,\, Y$ for some $g \in G$.
We say that a map \fun{\Phi}{X}{Y} is equivariant to actions of $G$ when
\begin{align}
  \Phi(\lambda_g (f)) = \lambda'_g (\Phi (f)),
  \label{eq:equivariance}
\end{align}
equivalently represented by the commutative diagram
\[
  \begin{tikzcd}
    X \arrow{r}{\Phi} \arrow[swap]{d}{\lambda_g} & Y \arrow{d}{\lambda'_g} \\
    X \arrow{r}{\Phi} & Y.
  \end{tikzcd}
\]
For each model, we design and parametrize $\Phi$ such that it is equivariant and its parameters are optimizable.
The actions $\lambda$ and $\lambda'$ are not necessarily the same since
$\Phi$ may map between different spaces.
When $\lambda'$ is the identity, we say that $\Phi$ is \emph{invariant} to $G$.
Some authors reserve the term \emph{equivariant} for when $\lambda=\lambda'$ and use
\emph{covariant} otherwise, but we will not make this distinction.

\Gcnns\ are \cnns\ that exhibit group equivariance.
They excel in scenarios with limited data and where inputs are subjected to a
large class of transformations (e.g., rotations).
There are successful applications in $3$D shape analysis~\cite{sphcnnijcv,deng2018ppf,ZhangHRY19},
spherical data analysis~\cite{sphhg,CohenWKW19,esteves20_spin_weigh_spher_cnns},
medical imaging~\cite{cohencube,bekkers2018roto,graham20_dense_steer_filter_cnns_exploit},
satellite/aerial imaging \cite{cyclicsym,henriques2017warped},
cosmology~\cite{cyclicsym,perraudin2019deepsphere}, and
physics/chemistry~\cite{s.2018spherical,kondor2018clebsch,AndersonHK19}.

Performance improvements were also achieved in popular upright
natural image datasets such as CIFAR\num{10/100}~\cite{CohenW17,weiler2019general},
showing that equivariance is not only beneficial when dealing with global input transformations.
This is because local patches can still be seen as transformations of some canonical patch.
For example, low level features such as corners may appear in any orientation even if
inputs are globally aligned, so an equivariant corner detector may reduce the burden of
learning different corner detectors for different orientations.

The second major theme of this thesis is \cnns\ on non-Euclidean spaces.
Most \cnns\ employ convolution on Euclidean spaces;
for example, $\R$ for audio, $\R^2$ for images and $\R^3$ for volumetric occupancy grids.
When exploring group-equivariance, it makes sense to consider features on spaces where
the group acts transitively; these spaces are not necessarily Euclidean.
The quintessential example is the group of rotations \SO{3}%
\footnote{$\SO{3}$ is the group of special orthogonal $3\times 3$ matrices,
 which is identified with $3$D rotations.}
acting on its homogeneous space,
the sphere $S^2$, which we discuss in depth in \cref{sph:sec:sphcnn}.

\section{Contributions and organization}
The following list shows the organization of this thesis,
summarizing the contributions presented in each chapter.
\begin{itemize}
\item In \cref{sec:prelim}, we introduce the theoretical background that enables our contributions.
  In particular, we cover the machinery necessary to define and evaluate integrals and convolutions on groups, which includes group representation theory, the Haar integral, and harmonic analysis.
  It originally appeared as part of \textcite{abs-2004-05154}.
\item In \cref{ptn:sec:ptn}, we discuss the \ptns, which achieve invariance to translation and
  equivariance to continuous rotations and scale,
  by doing a polar transform on the input image with a learned center.
  It was originally published in~\textcite{esteves2018polar}
  and resulted in state-of-the-art performance on the rotated \mnist\ and
  \simtwomnist\ image classification benchmarks.
\item In \cref{emvn:sec:emvn}, we discuss the \emvns, which assemble deep descriptors from multiple
  views of an object or scene as a function on the icosahedral group and
  achieves equivariance to this group through discrete group convolutions.
  It was originally published in~\textcite{Esteves_2019_ICCV}
  and resulted in state-of-the-art performance in multiple $3$D shape retrieval
  and classification benchmarks.
\item In \cref{sph:sec:sphcnn}, we discuss the spherical \cnns, which achieve equivariance to
  continuous $3$D rotations through spherical convolutions computed in the spectral domain.
  Its was originally published in~\textcite{esteves18eccv}
  with an extended version in~\textcite{sphcnnijcv},
  and resulted in performance comparable to the state of the art in $3$D shape
  classification and retrieval, but with orders of magnitude fewer model parameters.
  We also present an extension that was the first equivariant model for panoramic image segmentation,
  and appeared originally in~\textcite{sphhg}.
\item In \cref{cross:sec:cross}, we discuss cross-domain equivariant embeddings,
  in which we learn a mapping
  from $2$D views of a $3$D object to the spherical CNN features of the object.
  The encoded $3$D properties and inherited $3$D equivariance enable
  (i) computation of the $3$D relative pose between two views using spherical correlation, and
  (ii) synthesis of novel views with an inverter network by rotating the embeddings.
  It was originally published in~\textcite{esteves-icml19}.
\item In \cref{spin:sec:spin}, we discuss the \swscnns, which are a generalization of the
  spherical \cnns\ from \cref{sph:sec:sphcnn}.
  By considering the class of \swsfs, we are able use anisotropic filters in a memory
  and computation efficient manner, while also extending \SO{3}-equivariance to
  vector fields on the sphere for the first time.
  The approach yields state-of-the-art performance on spherical image classification
  and semantic segmentation.
  It was originally published in~\textcite{esteves20_spin_weigh_spher_cnns}.
  \item
  In \cref{sec:conclusion}, we summarize the contributions
  and discuss directions for future work.
  The first direction involves using mean curvature flows to map $3$D meshes to the sphere.
  This results in invertible maps that can be represented as spherical vector fields,
  and allows the application of spherical \cnns\ and spin-weighted spherical \cnns\ to
  new problems such as $3$D object part segmentation and mesh prediction.
  The second direction is to apply equivariant representations to large scale computer vision
  problems.
  The third direction is about unsupervised learning of symmetries,
  where the goal is to detect and exploit symmetries present in the data
  without assuming what they are.
\end{itemize}

\section{Related work}
This section contains a broad discussion of related work involving symmetries, invariances
and equivariances in signal processing, computer vision, and machine learning.
Chapter-specific related work is discussed in each chapter.

The concept of equivariance as described in \cref{eq:equivariance} is well established in mathematics,
but its use in computer vision and pattern recognition is more recent.
We are interested in equivariance to transformations other than translation,
since standard \cnns\ are already translation-equivariant.
The most often encountered of such transformations are rotations.

One of the earliest studies of rotation invariance in pattern recognition
was by \textcite{danielsson1980rotation}, while
\textcite{nordberg1996equivariance} introduced one of the first rotation equivariant
features in computer vision.

\textcite{segman1992canonical} introduced the canonical coordinates method that,
for some groups, gives a change of coordinates that transform the group action in a translation.

A closely related topic is steerability, introduced by \textcite{freeman1991design},
which is a way of using linear combinations of basis filters to synthesize new filters
transformed by some group action.
The convolution with a filter bank constructed
as the orbit of a canonical filter by some group
is equivariant to the group.

In classic computer vision, \textcite{harris1988combined} already
sought rotation-invariance for their early image corner detectors.
Similarly, \textcite{lowe1999object,lowe2004distinctive} designed
rotation-invariant local image feature descriptors.

In $3$D object recognition, simple rotation-invariant moment-based global descriptors
appeared as early as \textcite{lo19893},
being further developed by \textcite{burel1995three,kazhdan2003rotation}
using the spherical Fourier transform invariance properties.

In another direction, \textcite{kondor2008group} introduced several group theoretical methods for
machine learning problems, including translation and rotation invariant
image features obtained from group spectral coefficients.

With the massive popularization of deep learning and \cnns,
researchers started to seek invariant and equivariant deep-learned representations.
\textcite{kivinen2011transformation} developed translation and rotation-equivariant
restricted Boltzmann machines.
\textcite{bruna2013invariant} introduced one of the first rotation and scale
invariant convolutional networks, however the wavelet-based filters were not learned.
\textcite{gens2014deep} presented a \cnn\ model that can be made approximately invariant
to arbitrary groups.

\textcite{cohen2016group} formalized \gcnns\ as a generalization of \cnns\ using group convolutions.
Its applications were to small discrete groups of planar rotations and reflections.
\textcite{worrall2017harmonic} achieved equivariance to the continuous group of $2$D rotations,
while \textcite{esteves2018polar} introduced equivariance to the group of planar similarities.

The equivariant \cnns\ mentioned so far have scalar fields as feature maps
(meaning each channel transforms independently).
\textcite{CohenW17} introduced more general features that are
vectors in a group representation vector space.

When $3$D inputs are considered, \SO{3}-equivariance become desirable.
\textcite{s.2018spherical,esteves18eccv} achieved it by considering
spherical inputs and computing convolutions in the spectral domain.
\textcite{weiler3dsteerable} obtained \SO{3}-equivariance for volumetric inputs and
\textcite{tensor} for point clouds,
both following the framework of the steerable \cnns~\cite{CohenW17}.

More recently, \textcite{CohenWKW19} removed the usual constraint that features must live
on homogeneous spaces by introducing gauge-equivariant \cnns,
which work on general manifolds.
\textcite{Bekkers20} removed the usual constraint of considering only discrete
or compact groups by introducing a method to design \cnns\ equivariant to any Lie group.

While most works target practical applications of equivariant representations,
there were theoretical developments seeking to characterize and generalize these models.
\textcite{kondor18icml} proved that equivariance to the action of a compact group requires
a group-convolutional structure, while
\textcite{cohen2019general} generalized this result from scalar fields to general fields,
and introduced a taxonomy to categorize dozens of prior works.

\section{Results from cognitive science}
Cognitive science is often a source of inspiration for artificial intelligence research.
In particular, the study of biological vision has lead to
advancements in computer vision.
In the context of this thesis, it makes sense to review what is known about the
invariances and equivariances in biological visual systems.

The seminal work of \textcite{hubel1959receptive,hubel1968receptive} discovered
cells in the visual system with localized receptive fields, and that cells at higher levels
can receive inputs from multiple cells at lower level, exhibiting a larger receptive field.
The replication of units composed of multiple cells over the whole visual field results in
a translation-equivariant representation.
This inspired the introduction of \cnns\ by \textcite{fukushima1980neocognitron}.

\textcite{hubel1959receptive,hubel1968receptive} also found neurons
that are sensitive to edges in specific orientations.
\textcite{bosking1997orientation} showed two different arrangements of such neurons containing
a cycle of possible orientations.
Both arrangements can be interpreted as equivariant representations.
The first is referred to as pinwheel,
where an input edge rotation also results in a rotation of the activations;
\textcite{petitot2003neurogeometry} interpreted it as a circle bundle over the retina.
The second arrangement is linear, such that an edge rotation corresponds to a circular shift.

There seems to exist little evidence of
viewpoint-invariant neurons~\cite{kandel2013principles}.
\textcite{schwartz1983shape} found neurons invariant to scale,
position within their receptive field, contrast, color and texture,
while \textcite{logothetis1995shape} found viewpoint-sensitive neurons
in the inferior temporal cortex of monkeys trained to recognize complex $3$D shapes.

A classic experiment by \textcite{shepard1971mental} asked human participants to
tell if two images from different viewpoints correspond to the same object.
They showed that the time to solve the task is proportional to the
rotation angle between both images.
The evidence is that humans solve this task by creating a mental model of the object
and executing a mental rotation to align both views,
which implies that there is no direct rotation invariant representation
as the ones obtained with current equivariant \cnns\
(e.g., with our methods in~\cref{sph:sec:sphcnn,cross:sec:cross}).
Humans exercise a form of high-level reasoning to solve this task that is not
yet possible with current artificial neural networks.

Recently, further connections between deep learning and neuroscience have been explored,
in the direction of modeling biological neural responses with artificial neural networks.
\textcite{yamins2014performance} trained a number of biologically-plausible neural networks
on image classification tasks
and discovered that models that match human performance have activations correlated with activations
on the inferior temporal cortex and V4 cortex.
\textcite{klindt2017neural} leveraged the translation equivariance of \cnns\ to model
neural responses in the V1 cortex. This exploits the fact that there are multiple neurons
computing the approximately the same functions, replicated along the visual field.
\textcite{ecker2018a} extended these results by observing that cycles of orientation-sensitive neurons
are also replicated along the visual field, so
a translation and rotation-equivariant \cnn\ is more suitable for the task.

Refer to \textcite{kandel2013principles,bermudez2014cognitive} for more details
about cognitive science and biological visual systems.
\glsresetall
\chapter{Theoretical background}
\label{sec:prelim}
\section{Introduction}
This thesis has two major themes, (i) neural networks that are group equivariant
and (ii) neural networks on non-Euclidean spaces.
In this chapter we present the theoretical background that enables our contributions.
Fortunately, the non-Euclidean spaces we consider are homogeneous spaces of the groups,
so the theory is interconnected.

In this chapter, we present the theory behind \gcnns, in particular of group convolutions,
which is not usually covered in recent papers due to space constraints.
We discuss group representation theory (\cref{h:sec:group}),
integration and harmonic analysis on non-Euclidean spaces (\cref{h:sec:haar,h:sec:harm}).
\Cref{h:sec:appl} shows how this theory is applied to \gcnns.

Most of the material in \cref{h:sec:group,h:sec:haar,h:sec:harm}
is presented in a more rigorous and complete way in \textcite{gallierncharm}.
We omit deep proofs related to the Haar measure and the Peter-Weyl theorem,
and often tailor the material to just the parts required to understand the current
\gcnns.
We do, nevertheless, derive the irreducible representations of
$\SL{2, \C}$, $\SU{2}$ and $\SO{3}$ and show how special functions,
including the spherical harmonics, arise in the process.
Furthermore, we define and prove the formulas for \SO{3} and spherical convolutions
and cross-correlations that are used in recent works.
While \textcite{gallierncharm} is the main reference utilized,
we sometimes follow \textcite{nakahara2003geometry,serre1977linear,dieudonn1980special,folland2016course,hall2015lie,vilenkin1978special,rudin2006real} when more appropriate.

The content of this chapter appeared originally
as part of a literature review (\textcite{abs-2004-05154}).

\section{Group representation theory }
\label{h:sec:group}
Group representation theory is the study of groups by the way they act on vector spaces,
which is done by \emph{representing} elements of the group as linear maps between vector spaces.
\subsection{Groups and homogeneous spaces}
We begin with basic definitions about groups.
\begin{definition}[group]
\label{h:def:group}
A \emph{group} $(G, \cdot)$ is a set $G$ equipped with an associative binary operation $\fun{\cdot}{G \times G}{G}$,
an identity element, and where every element has an inverse also in the set.
When $\cdot$ is commutative, we call the group \emph{abelian} or \emph{commutative}.
When the set is equipped with a topology where $\cdot$ and the inverse map are continuous,
we call it a \emph{topological group}.
When such topology is compact, we call the group a \emph{compact group}.
When $G$ is a \emph{smooth manifold} and $\cdot$ and the inverse map are smooth, it is a \emph{Lie group}.
A \emph{subgroup} $(H, \cdot)$ of a group $(G, \cdot)$ is a group such that $H \subseteq G$.
\end{definition}
\begin{examples}
  \begin{itemize}
  \item[]
  \item The integers under addition $(\Z, +)$ form an abelian,
    non-compact group.
  \item The group of all permutations of a set of $n$ symbols,
called the symmetric group $S_n$ is a finite, non-commutative group of $n!$ elements.
  \item The group of rotations in $3$D, \SO{3}, is a compact, non-commutative Lie group.
  \end{itemize}
\end{examples}
\noindent For a negative example,
consider the sphere $S^2$ and its north pole $\nu=(0, 0, 1)$.
We can identify any point on the sphere by angles $(\alpha, \beta)$,
which represent a rotation of the north pole $R_y(\beta)$ (around $y$)
followed by $R_z(\alpha)$ (around $z$); we write
${x = R_z(\alpha)R_y(\beta)\nu}$.
Now define the operation $\cdot$ such that
$x_1 \cdot x_2 =  R_z(\alpha_1)R_y(\beta_1)R_z(\alpha_2)R_y(\beta_2)\nu$.
Any rotation in $\SO{3}$ can be represented as $R_z(\alpha_1)R_y(\beta_1)R_z(\alpha_2)R_y(\beta_2)$,
and not only the ones of the form $R_z(\alpha_1)R_y(\beta_1)$;
therefore the operation $\cdot$ as defined is not closed in $S^2$,
and $(S^2, \cdot)$ is not a group.

While $S^2$ is not a group, we will show that it is a homogeneous space
of $\SO{3}$
Intuitively, homogeneous spaces are spaces where the group acts ``nicely''.
For this reason, they are useful as the feature domain in \gcnns.
Homogeneous spaces are closely related to coset spaces;
we now define both structures and show how they relate.
\begin{definition}[homogeneous space]
  \label{h:def:hspace}
  The action of a group $G$ is \emph{transitive} on a space $X$
  if for any pair of elements $x$ and $y$ in $X$,
  there exists an element $g$ in $G$ such that $y = gx$.
  A \emph{homogeneous space} $X$ of a group $G$ is a space where the group acts transitively.
\end{definition}

\begin{definition}[coset space]
  Given a subgroup $H$ and an element $g$ of a group $G$,
  we define the \emph{left coset} $gH$ as ${gH = \{gh \mid h \in H\}}$.
  The set of left cosets partition $G$ and is called the \emph{left coset space} $G/H$.
  We define the right cosets $Hg$ and their coset space $H\backslash G$ analogously.
\end{definition}
\noindent Let $o \in X$ be an arbitrarily chosen origin of $X$ and $H_o$ its stabilizer.
Then, there is a bijection%
\footnote{The bijection will be a homeomorphism is all cases we consider, but not in general.}
between $X$ and $G/H_o$.

We will often refer to elements of a homogeneous space $X \cong G/H_o$ by the coset $gH_o$,
and the map $g \mapsto gH_o$ is a projection from the group $G$ to the homogeneous space $X$.
Since we are interested in maps that are equivariant to actions of some group $G$,
we will frequently consider maps between homogeneous spaces of $G$.
\begin{example}
  Let us return to the sphere $S^2$ and its north pole $\nu=(0, 0, 1)$.
  The sphere is a homogeneous space since $\SO{3}$ acts transitively on it.
  The set of rotations that do not move $\nu$ is the stabilizer
  $H_\nu = \{R_z(\delta) \mid \delta \in [0, 2\pi) \}$.
  Any rotation in $R\in \SO{3}$ can be written as $R = R_z(\alpha)R_y(\beta)R_z(\gamma)$,
  and generate left cosets of the form
  \[RH_\nu = \{R_z(\alpha)R_y(\beta)R_z(\gamma+\delta) \mid \delta \in [0, 2\pi)\}.\]
  The pair $(\alpha,\, \beta)$ uniquely identify each coset, which gives an isomorphism
  between points on the sphere and the set of all cosets $\SO{3}/H_\nu$.
  Since $H_\nu$ is isomorphic to group of planar rotations $\SO{2}$,
  we write $S^2 \cong \SO{3} / \SO{2}$.
\end{example}

\subsection{Group representations}
Group representations have numerous applications.
Most important to our purposes are
(i) they represent actions on vector spaces
(for example, $\lambda_g$ in \cref{eq:equivariance} could be a linear representation),
and (ii) they form bases for spaces of functions on groups, as will be detailed in \cref{h:sec:harm}.

\begin{definition}[representation]
  A \emph{group homomorphism} between groups $G$ and $H$ is a map \fun{f}{G}{H} such that $f(g_1g_2) = f(g_1)f(g_2)$.
  Let $G$ be a group and $V$ a vector space over some field.
  A \emph{linear representation} is a group homomorphism \fun{\rho}{G}{\GL{V}},
  where $\GL{V}$ is the general linear group.%
  \footnote{When $V$ is finite-dimensional and $n = \dim V$, $\GL{V}$ is identifiable with the group of $n \times n$ invertible matrices.}
  If $V$ is an inner product space and $\rho$ is continuous and preserves the inner product,
  it is called a \emph{unitary representation}.
  The \emph{character} of a representation $\rho$ is the map \fun{\chi_\rho}{G}{\C}
  such that $\chi_\rho(g) = \tr(\rho(g))$.
\end{definition}
\begin{example}
  Consider the multiplicative group $G$ of complex numbers of the form $g_\theta = e^{i\theta}$.
  The map
  \[\rho(e^{i\theta}) = \twobytwo{\cos \theta}{ \sin\theta}{-\sin\theta}{\cos\theta}\]
  is a representation of $G$ on $\R^2$.
  We can check that $g_\theta g_\phi = g_{\theta+\phi}$ and
  $\rho(e^{i(\theta + \phi)}) = \rho(e^{i\theta})\rho(e^{i\phi})$.
\end{example}
\begin{example}
  Let $L^2(G)$ be the Hilbert space of square integrable functions on $G$, and
  let \fun{\rho}{G}{\GL{L^2(G})} act on \fun{f}{G}{\C} as
  $(\rho(g)(f))(x) = f(g^{-1}x)$.
  $\rho$ defined this way is a representation of $G$; specifically,
  it is a \emph{left regular representation} of $G$.
\end{example}
\begin{definition}[irreducible representation]
  Let \fun{\rho}{G}{\GL{V}} be a representation of $G$ on a vector space $V$,
  and $W$ be a vector subspace of $V$.
  When $W$ is invariant under the action of $G$, i.e.,
  for all $g \in G$ and $w\in W$ we have $\rho(g)(w) \in W$,
  the restriction of $\rho$ to $W$ is a representation of $G$ on $W$, called a \emph{subrepresentation}.
  When the only subrepresentations of $\rho$ are $V$ and the zero vector space,
  we call $\rho$ an \emph{irreducible representation} or \emph{irrep}.
\end{definition}
\begin{example}
  Consider the group $\SO{3}$ and the vector space $V$ of $3\times 3$ real matrices ($V \cong \R^9$).
  We define a representation \fun{\rho}{\SO{3}}{\GL{V}} such that ${\rho(g)(A) = g^\top A g}$.
  Now consider the subspace $W$ of $V$ comprised of antisymmetric matrices ($B=-B^\top$).
  It turns out $W$ is invariant to $\rho$,
  \begin{align}
    (\rho(g)(B))^\top = (g^\top B g)^\top = -g^\top B g = -\rho(g)(B)
  \end{align}
  so $\rho(g)(B) \in W$ for all $g \in \SO{3}$ and $B \in W$.
  Therefore $\rho$ is a reducible representation.
  It is, however, irreducible as a representation on $W$.
\end{example}
\begin{remark}
Every representation of a finite group is a direct sum of irreps (Maschke's theorem).
\end{remark}
\begin{remark}
Every finite-dimensional unitary representation of a compact group is a direct sum of unitary irreducible representations (unirreps).
\end{remark}

We often want to determine all irreducible representations of a group,
or decompose a representation in its irreducible parts.
The characters ${\fun{\chi_\rho}{G}{\C}}$ play an important role in this task.
First, we define the inner product of characters $\inner{\chi_i, \chi_j} = \int_G \chi_i(g)\overline{\chi_j(g)}\, dg$.%
\footnote{This involves integration on the group, which we will define in \cref{h:sec:haar}.}
The following properties hold:
\begin{itemize}
\item Isomorphic representations have the same character.
  The converse is only true for semisimple representations,
  which include unitary representations and all representations of finite or compact groups.
\item Distinct characters of irreducible representations of compact groups are orthogonal,
  $\inner{\chi_i, \chi_j} = 0$ when $i \neq j$.
\item A representation of a compact group is irreducible if and only if its character $\chi$ satisfies $\inner{\chi, \chi} = 1$.
\item The character of a direct sum of representations is the sum of the individual characters.
\end{itemize}

Now let \fun{\rho_1}{G}{\GL{V_1}} and \fun{\rho_2}{G}{\GL{V_2}} be finite-dimensional representations.
The map \fun{\rho_{12}}{G}{V_1 \otimes V_2}
obtained via tensor product $\rho_{12}(g) = \rho_1(g) \otimes \rho_2(g)$
is a representation of $V_1 \otimes V_2$.
This representation is not irreducible in general,
and the Clebsch-Gordan theory studies how it decomposes into irreps.
\begin{definition}[G-map]
  Given two representations \fun{\rho_1}{G}{\GL{V_1}} and \fun{\rho_2}{G}{\GL{V_2}},
  a \emph{G-map} is a linear map \fun{f}{V_1}{V_2} such that $${f(\rho_1(g)(v_1)) = \rho_2(g)(f(v_1))}$$
  for every $g \in G$ and $v_1\in V_1$.
  If $f$ is invertible, we say that $\rho_1$ and $\rho_2$ are \emph{equivalent},
  and we can define equivalence classes of representations.
  A G-map is sometimes called a \emph{G-linear}, \emph{G-equivariant}, or \emph{intertwining} map.
\end{definition}
\begin{remark}
  In the context of neural networks, we usually have alternating linear maps and nonlinearities.
  In equivariant neural networks, we want the linear maps to be G-maps.
  The representations will often be the natural action $(\rho(g)f)(x) = f(g^{-1}x)$.  
\end{remark}

The following is an important result characterizing G-maps between irreps.
\begin{theorem}[Schur's Lemma]
\label{h:thm:schur}
Let \fun{\rho_1}{G}{\GL{V_1}} and \fun{\rho_2}{G}{\GL{V_2}} be irreducible representations of $G$,
and \fun{f}{V_1}{V_2} a \emph{G-map} between them.
Then $f$ is either zero or an isomorphism.
If $V_1=V_2$ and $\rho_1=\rho_2$ are complex representations, then $f$ is a multiple of the identity map, $f = \lambda \text{id}$.
\end{theorem}
Henceforth, we assume representations are complex (representation vector space is over $\C$)
except when stated otherwise.

This concludes our introduction to group representation theory.
For more details we recommend \textcite{gallierncharm,serre1977linear,hall2015lie}.
\section{Integration}
\label{h:sec:haar}
In order to compute Fourier transforms and convolutions on groups,
we need to integrate functions on groups.
The key ingredient is the Haar measure.
We begin with the familiar Riemann integral,
discuss its limitations and introduce Lebesgue integration as the remedy.
The Lebesgue integral allows integration over arbitrary sets
given an appropriate measure.
Finally, we define the Haar measure,
which is the appropriate measure used for integration on
locally compact groups.
\subsection{The Riemann integral}
Intuitively, the Riemann integral is the familiar ``area under the curve''
of a continuous function on an interval of the real line \fun{f}{[a,b]}{\R}.
The idea is to partition the integration interval and
define the integral as the sum of areas of the rectangles defined by one value
of $f$ on each subinterval and the subinterval width,
on the limit where such widths tend to zero.

\begin{definition}[Riemann integral]
  For an interval $[a,b] \subset \R$ and a subdivision $T=\{t_i\}$ with $0 \le i \le n$,
  $t_0=a$, $t_{n}=b$, and $t_k < t_{k+1}$ for all $k < n$,
  the \emph{Cauchy-Riemann sum} $s_T(f)$ of a continuous function \fun{f}{[a,b]}{\R} is
  \begin{align*}
    s_T(f) = \sum_{k=0}^{n-1}(t_{k+1}-t_{k})f(t_k).
  \end{align*}
  The \emph{diameter} of the subdivision $T$ is $\delta(T) = \max_k\{t_{k+1}-t_k\}$.
  Now consider any sequence of subdivisions $T_j$ such that $\lim_{j \to \infty}\delta(T_j) = 0$
  (consequently, $n \to \infty$).
  We define the \emph{Riemann integral} as $\int_a^b f(t)\,dt = \lim_{j \to \infty} s_{T_j}(f)$.
\end{definition}
It can be shown that $s_{T_j}$ always converge to the same limit
for any sequence of subdivisions $T_j$.
Importantly, the map $f \mapsto \int_a^bf(t) \, dt$ is linear.
The Riemann integral can be extended to functions on products of closed intervals on $R^n$
and to vector valued functions.
However, it cannot be defined on more general domains;
the Lebesgue integration was introduced to overcome this limitation.
\subsection{Lebesgue integration}
Lebesgue integration can be defined on arbitrary sets,
and allows taking limits of sequences of functions under integration,
which is necessary in Fourier analysis, for example.

In this section, we follow \textcite{rudin2006real} for the most part.
Refer to \textcite{gallierncharm} for a more general approach
which allow functions taking value on arbitrary (possibly infinite-dimensional)
vector spaces.

We begin by defining the crucial concept of \emph{measure}.
\begin{definition}[measure]
  A collection $\Sigma$ of subsets of a set $X$ is a \emph{$\sigma$-algebra}
  if it contains $X$ and is closed under complementation and countable unions.
  We call the tuple $(X, \Sigma)$ a \emph{measurable space},
  and the subsets in $\Sigma$ are \emph{measurable sets}.
  A function \fun{f}{X}{Y} is \emph{measurable} if
  the preimage of every measurable set in $Y$ is in $\Sigma$.
  A \emph{measure}  is a function \fun{\mu}{\Sigma}{[0, \infty]}
  which is \emph{countably additive},
  \begin{align}
    \mu\left( \bigcup_{i=0}^\infty A_i \right) = \sum_{i=0}^{\infty} \mu(A_i) \label{h:eq:additivity}
  \end{align}
  for a disjoint collection of $A_i \in \Sigma$.
  The tuple $(X, \Sigma, \mu)$ is called a \emph{measure space}.
\end{definition}
\begin{example}
  On the real line $\R$, we define $\mathcal{B}(\R)$ as the smallest $\sigma$-algebra containing every open interval.
  This is known as the $\sigma$-algebra of Borel sets, or the Borel algebra.
  Then \fun{\mu}{\mathcal{B}(\R)}{[0, \infty]} defined such that
  $\mu((a,b]) = b-a$ is a measure in $(\R, \mathcal{B}(\R))$;
  it is usually called the Borel measure.
\end{example}

Carathéodory's theorem allows the construction of measures and measure spaces from an outer measure.
\begin{theorem}[Carathéodory]
  \label{h:thm:caratheodoty}
  An \emph{outer measure} $\mu^*$ on a set $X$ is a function \fun{\mu^*}{2^X}{[0, \infty]}
  such that (i) $\mu^*(\emptyset)=0$, (ii) if $A \subseteq B$, $\mu^*(A) \le \mu^*(B)$ and (iii)
  \begin{align}
    \mu^*\left( \bigcup_{i=0}^\infty A_i \right) \le \sum_{i=0}^{\infty} \mu^*(A_i).
    \label{h:eq:subadditivity}
  \end{align}
  Note that \cref{h:eq:subadditivity} is a relaxation of \cref{h:eq:additivity}, called \emph{subadditivity.}
  We can construct an outer measure on $X$ as
  \begin{align}
    \mu^*(A) = \inf \left\{\sum_{n=0}^\infty \lambda(I_n) \mid A \subseteq \bigcup_{n=0}^\infty I_n \right\}
    \label{h:eq:outer}
  \end{align}
  where $\lambda$ is any positive function with $\lambda(\emptyset)=0$
  and there is a family $\{I_n\}$ of subsets of $X$ that contains the empty set and
  covers any subset $A \subseteq X$.
  Now consider the family of subsets
  \begin{align*}
    \Sigma = \{A \in 2^X \mid \mu^*(A) = \mu^*(E \cap A) + \mu^*(E \cap (X-A)),\, \forall E \in 2^X\}.
  \end{align*}
  Then $\Sigma$ is a $\sigma$-algebra and the restriction $\mu$ of $\mu^*$ to $\Sigma$ is a measure,
  so $(X, \Sigma, \mu)$ is a measure space.
\end{theorem}

\begin{example}
  Let $\mu^*$ be an outer measure constructed as in \cref{h:eq:outer}
  where $\{I_n\}$ is the set of all open intervals in $\R$
  and $\lambda((a,b)) = b-a$.
  By applying \cref{h:thm:caratheodoty} to $\mu^*$ we obtain
  the Lebesgue measure $\mu_L$,
  and the $\sigma$-algebra of Lebesgue measurable sets $\mathcal{L}(\R)$.
  It can be shown that $\mathcal{B}(\R) \subset \mathcal{L}(\R)$;
  this extends to $\R^n$.
\end{example}

Equipped with the notion of measures and measurable functions,
we can define the Lebesgue integral.
\begin{definition}[Lebesgue integral]
  Let $(X, \Sigma, \mu)$ be a measure space.
  We define the \emph{characteristic function} $\chi_A$ of a measurable set $A$ as the
  indicator function \fun{\chi_A}{X}{\{0,1\}} that is 1 when $x \in A$ and 0 otherwise.
  A \emph{simple function} is a function $s$ on $X$ whose range consist only of finitely many distinct values;
  formally, $s(x) = \sum_{i=0}^n \alpha_i\chi_{A_i}(x)$ where $\{\alpha_i\}$
  is the set of distinct values.
  We define the integral of a measurable simple function over a set $E \in \Sigma$ as
  \begin{align*}
    \int\limits_E s\,d\mu = \sum_{i=0}^n\alpha_i \mu(A_i \cap E).
  \end{align*}
  We call a function $f$ \emph{positive} when $f(x) \ge 0$ for all $x$,
  and say that $f \le k$ when $k-f$ is positive.
  For a measurable positive function \fun{f}{X}{[0, \infty]} we define the \emph{Lebesgue integral} as
  \begin{align*}
    \int\limits_E f\,d\mu = \sup \int\limits_E s\,d\mu,
  \end{align*}
  where the supremum is over all simple functions $s$ such that $0 \le s \le f$.
\end{definition}
The Lebesgue integral is easily extended to complex valued functions \fun{f}{X}{\C}
by noting that we can write $f = u^{+} - u^{-} + i(v^{+} - v^{-})$
for positive functions $u^{+}$, $u^{-}$, $v^{+}$, $v^{-}$;
the integral is then obtained by linearity.

Intuitively, while the Riemann integral partitions the domain of $f$ to compute
the integral, the Lebesgue integral partitions its range.
This is the key to enable integration over more general domains.
\begin{example}
  Consider again the measure space $(\R, \mathcal{B}(\R), \mu)$,
  and the indicator function for the rational numbers \fun{f}{\R}{\{0, 1\}},
  $f(x)=1$ if $x \in \Q$ and $f(x)=0$ otherwise.
  The function is not Riemann-integrable since here is no interval where it is continuous.
  However it is a simple function that takes the value 1 on a set of measure zero
  (since $\Q$ is countable), and 0 elsewhere.
  Hence, $f$ is Lebesgue integrable and its integral is zero on any interval.
\end{example}

\subsection{The Haar measure}
\label{h:haar}
The Lebesgue integral allows integration on arbitrary sets,
when they are given the structure of a measure space.
The Haar measure gives such structure to locally compact groups.

\begin{theorem}[Haar measure]
  Consider a locally compact Hausdorff topological group $G$,
  and the Borel $\sigma$-algebra $\mathcal{B}(G)$ generated by its open subsets.
  There exists a unique measure $\mu$, up to a multiplicative constant, such that
  $\mu$ is left-invariant, i.e., $\mu(gE) = \mu(E)$ for all $E \in \mathcal{B}(G)$ and $g \in G$.
  Furthermore, $\mu$ is $\sigma$-regular,
  \begin{align*}
    \mu(E) &= \inf \{\mu(U) \mid E \subseteq U,\, U \text{ open}  \}, \\
    \mu(E) &= \inf \{\mu(K) \mid K \subseteq E,\, K \text{ compact} \}.
  \end{align*}
  The measure $\mu$ defined as such is called the \emph{left Haar measure}.
  We define the \emph{right Haar measure} analogously;
  both measures are not necessarily equal.
\end{theorem}
It can be shown that $\mu(U) > 0$ for any non-empty open $U \in \mathcal{B}(G)$ and
$\mu(K) < \infty$ for any compact $K \in \mathcal{B}(G)$.

The construction idea is to define the measure of a subset $K \in \mathcal{B}(G)$
as the number of left-translations of a small $U \in  \mathcal{B}(G)$ necessary to cover $K$.
It is made precise by taking appropriate limits and enforcing measure properties.

Now define the left action operator $\lambda_u(f)(g) = f(u^{-1}g)$.
Given a left Haar measure $\mu$, its left invariance implies
\begin{align}
  \int \lambda_s(f)\,d\mu = \int f \, d\lambda_s(\mu) = \int f d\mu
\end{align}
for any \fun{f}{G}{\C} and $s \in G$.
We write $d\mu(g) = dg$ to simplify the notation;
then the relation $\int_G f(s^{-1}g) \,dg = \int_G f(g) \,dg$
gives a variable substitution formula that appears in many proofs.
For functions on the line, this translates to the usual
$\int_{-\infty}^{\infty} f(x-y)\, dx = \int_{-\infty}^{\infty} f(x)\, dx$,
where the Lebesgue measure is also a Haar measure.
\begin{example}
  Consider again the group $G$ of unitary complex numbers of the form $g_\theta = e^{i\theta}$,
  for $-\pi \le \theta < \pi$, and the function \fun{\lambda}{G}{\R} such that $\lambda(e^{i\theta}) = \theta$.
  We obtain the Haar measure from the Lebesgue measure $\mu_L$ on $\R$ as
  $\mu(A) = \mu_L(\lambda(A))$; it can be shown to be left-invariant.
\end{example}
\begin{example}
  For the group $\GL{n, \R}$, the Haar measure is given by $dA / |det(A)|^n$,
  where $dA$ is the Lebesgue measure on $R^{n^2}$.
\end{example}
\noindent On a Lie group of dimension $n$, we can construct an alternating $n$-form on the tangent space at the
identity and transport it to the tangent space at any point using left group actions.
The result is a left-invariant differential $n$-form (volume form) on the group that induces the left Haar measure \cite{hall2015lie}.

Next, we introduce modular functions, which are useful to determine some group properties.
\begin{definition}[modular function]
Let $\mu$ be a left Haar measure on a group $G$,
and define the right action operator $\rho_s(f)(g) = f(gs)$.
It follows that $\rho_s(\mu)$ is also a left Haar measure and
since the left Haar measure is unique up to scalar multiplication,
we have $\rho_s(\mu) = \Delta(s)\mu$ for \fun{\Delta}{G}{(0, \infty]}.
We call the function $\Delta$ a \emph{modular function}.
If $\Delta(s) = 1$ for all $s \in G$, we call $G$ \emph{unimodular}.
\end{definition}
\noindent In particular, a left Haar measure is also a right Haar measure
if and only if the group is unimodular.
Abelian groups are unimodular, and so are compact groups.

Next, we want to obtain measures on homogeneous spaces from measures on groups.
Let $G$ be a locally compact group with a subgroup $H$.
Now consider the homogeneous space $G/H$
where we suppose there is a measure $\gamma$.
We call $\gamma$ \emph{$G$-invariant} if ${\lambda_u(\gamma) = \gamma}$, for all $u \in G$,
where $\lambda_u$ is the left action operator ${\lambda_u(f)(g) = f(u^{-1}g)}$.
The following theorem gives the conditions for the existence of a $G$-invariant measure.
\begin{theorem}
  \label{h:thm:hmeasure}
  Let $G$ be a locally compact group with a subgroup $H$,
  $\mu$ a left Haar measure on $G$ and $\xi$ a left Haar measure in $H$.
  There is a unique G-invariant measure $\gamma$ on G/H (up to scalar multiplication)
  if and only if the modular function $\Delta_H$ equals the restriction of $\Delta_G$ to $H$.
  We can then write
  \begin{align*}
    \int\limits_G f(u)\,d\mu(u) = \int\limits_{G/H}\int\limits_{H} f(uh)\,d\xi(h)d\gamma(uH),
  \end{align*}
  for any function $f$ of compact support on $G$.
\end{theorem}

\begin{remark}
  When the $\Delta_H$ is not equal to the restriction of $\Delta_G$ to $H$,
  there is a weaker form of invariance in measures, called \emph{quasi-invariance}.
  Quasi-invariant measures on $G/H$ always exist. Refer to \textcite{folland2016course,gallierncharm}
  for details.
\end{remark}

\section{Harmonic analysis}
\label{h:sec:harm}
Recall the Fourier series expansion of a periodic function $f$
\begin{align*}
  f(\theta) &= \sum_{m \in \Z} \hat{f}(m) e^{i  m \theta}, \\
  (\mathcal{F}f)(m) &= \hat{f}(m) = \frac{1}{2\pi} \int\limits_{-\pi}^{\pi} f(\theta)e^{-i m \theta}\, d\theta.
\end{align*}
A periodic scalar-valued function $f$ can be seen as a function on the circle \fun{f}{S^1}{\R}.
The expansion in Fourier series is a decomposition in the basis $\{e^{im\theta}\}$ for $m \in \Z$
of the space of square-integrable functions on the circle, $L^2(S^1)$.
Fourier analysis has numerous applications in signal processing, differential equations and number theory.
Most important for our purposes is the convolution theorem,
\begin{align}
  \mathcal{F}(f * k)(m) = (\mathcal{F}f)(m)(\mathcal{F}k)(m) = \hat{f}(m)\hat{k}(m),
\end{align}
which states that convolution in the spatial domain corresponds to multiplication in the spectral domain.
This has immense practical implications for efficient computation of convolutions,
thanks to the \fft\ algorithm.

In this section, we generalize these concepts to functions on compact groups.
We consider a compact group $G$,
and the Hilbert space $L^2(G)$ of square integrable functions on $G$.
Integrals on compact groups are well defined as discussed in \cref{h:sec:haar}.
We state the Peter-Weyl theorem, which gives an
orthonormal basis for $L^2(G)$ constructed from irreducible representations of $G$.
This paves the way to harmonic analysis on compact groups,
which we demonstrate by generalizing the Fourier transform and convolution theorem.
Again these have important practical applications
and are used to compute group convolutions in recent equivariant neural networks.
Finally, we show how the theory applies to homogeneous spaces of compact groups.
\subsection{The Peter-Weyl Theorem}
The Peter-Weyl theorem gives an explicit
orthonormal basis for $L^2(G)$, constructed from irreducible representations of a group $G$.
The basis is formed by matrix elements, which we define first.
\begin{definition}[matrix elements]
  \label{h:def:me}
  Let $\rho$ be a unitary representation of a compact group $G$.
  We denote $\phi_{x,y}(g) = \inner{\rho(g)(x), y}$ the \emph{matrix elements} of $\rho$.
  Note that $\phi_{e_i,e_j}(g)$ is one entry of the matrix $\rho(g)$ when $e_i,\, e_j$ are basis vectors,
  so we define $\rho_{ij}(g) = \phi_{e_i,e_j}(g)$.%
\end{definition}
\begin{theorem}[Peter-Weyl]
  \label{h:thm:pw}
  Let $G$ be a compact group. We present the theorem in three parts.
  The first relates matrix elements and spaces of functions on $G$.
  The second decomposes representations of $G$,
  and the third gives a basis for $L^2(G)$.
  \paragraph{Part I} The linear span of the set of matrix elements of unirreps
  of $G$ is dense in the space of continuous complex valued functions on $G$, under the uniform norm.
  This implies it is also dense in $L^2(G)$.
  \paragraph{Part II} Let $\hat{G}$ be the set of equivalence classes of unirreps of $G$.
  For a unirrep $\rho$ of $G$,
  we denote its representation space by $H_\rho$ where $\dim H_\rho = d_\rho$,
  and its equivalence class by $[\rho] \in \hat{G}$.
  If $\pi$ is a (reducible) unitary representation of $G$,
  it splits in the orthogonal direct sum $H_\pi = \bigoplus_{[\rho] \in \hat{G}} M_\rho $,
  where $M_\rho$ is the largest subspace where $\pi$ is equivalent to $\rho$.
  Each $M_\rho$ splits in equivalent irreducible subspaces $M_\rho = \bigoplus_{i=1}^{n} H_\rho$,
  where $n$ is the \emph{multiplicity} of $[\rho]$ in $\pi$.
  \paragraph{Part III}
  Let $\varepsilon_\rho$ be the linear span of the matrix elements of $\rho$ for $[\rho] \in \hat{G}$.
  $L^2(G)$ can be decomposed as $L^2(G) = \bigoplus_{[\rho] \in \hat{G}} \varepsilon_\rho$.
  If $\pi$ is a regular representation on $L^2(G)$, the multiplicity of $[\rho] \in \hat{G}$ in $\pi$ is $d_\rho$.
  An orthonormal basis of $L^2(G)$ is
  \begin{equation*}
    \{ \sqrt{d_\rho}\rho_{ij} \mid 1 \le i,\, j \le d_\rho,\, [\rho] \in \hat{G} \}
  \end{equation*}
  where $\rho_{ij}$ is as in \cref{h:def:me}.
  Constructing the basis involves choosing a representative per equivalence class.
\end{theorem}
\begin{example}
  The $\SO{3}$ irreducible representations $\rho^\ell$ can be written as $2\ell+1 \times 2\ell+1$ matrices for $\ell \in \N$, with entries \fun{\rho_{ij}^\ell}{\SO{3}}{\C},
  \begin{align*}
    \rho^{0} = (\rho_{0,0}^{0}), &&
                                   \rho^{1} = \threebythree
                                   {\rho_{-1,-1}^{1}}{\rho_{-1,0}^{1}}{\rho_{-1,1}^{1}}
                                   {\rho_{0,-1}^{1}}{\rho_{0,0}^{1}}{\rho_{0,1}^{1}}
                                   {\rho_{1,-1}^{1}}{\rho_{1,0}^{1}}{\rho_{1,1}^{1}}, &&
                                                                                        \rho^{2} = \cdots\, ,
  \end{align*}
  and the matrix elements $\rho_{i,j}^{\ell}$ form a basis for $L^2(\SO{3})$.
  We will derive these elements in \cref{h:sec:su2}.
\end{example}
For simplicity, we avoided introducing Hilbert algebras, ideals,
and the interesting connection between representations of groups and of algebras.
We refer the reader to \textcite{gallierncharm} for a complete description
of the Peter-Weyl theorem, with proofs.
\subsection{Fourier analysis on compact groups}
Part III of \cref{h:thm:pw} gives an orthonormal basis for $L^2(G)$,
so for any $f \in L^2(G)$ we can write,
\begin{align}
  f(g) &= \sum_{[\rho] \in \hat{G}}\sum_{i,j=1}^{d_\rho} c_{ij}^\rho \rho_{ij}(g), \label{h:eq:fwcoord} \\
  c_{ij}^\rho &= d_\rho \int\limits_{g \in G} f(g)\overline{\rho_{ij}(g)}\, dg, \label{h:eq:decomp}
\end{align}
where \cref{h:eq:decomp} is the inner product in $L^2(G)$,
the matrix elements $\rho_{ij}$ are as in \cref{h:def:me},
and the coefficients $c_{ij}^\rho$ absorb an extra $\sqrt{d_\rho}$ for simplification.

Now we define the Fourier transform of $f \in L^2(G)$ as a function on $\hat{G}$
whose values are on $\GL{H_\rho}$ for an input $[\rho]$.
\begin{equation}
  \hat{f}(\rho) = \mathcal{F}(f)([\rho]) = \int\limits_{g \in G} f(g) \rho(g)^{*}\, dg, \label{h:eq:ft}
\end{equation}
where $\rho$ is the representative of $[\rho]$,
$^{*}$ indicates the conjugate transpose,
and we introduce $\hat{f}(\rho)$ to shorten notation.
It is easy to see that the element $i$, $j$ of $\hat{f}(\rho)$
corresponds to $\frac{c_{ji}^\rho}{d_\rho}$ as defined in \cref{h:eq:decomp}, and
\begin{equation*}
  \sum_{i,j=1}^{d_\rho} c_{ij}^\rho \rho_{ij}(g) =
  \sum_{i,j=1}^{d_\rho} d_\rho \hat{f}(\rho)_{ji}\rho_{ij}(g) =
  d_\rho \tr(\hat{f}(\rho) \rho(g)).
\end{equation*}
Applying this result to \cref{h:eq:fwcoord} yields the Fourier inversion formula,
\begin{equation}
  f(g) = \sum_{[\rho] \in \hat{G}} d_\rho \tr(\hat{f}(\rho) \rho(g)). \label{h:eq:fi}
\end{equation}
\begin{remark}
\Cref{h:eq:ft,h:eq:fi} give the Fourier transform and inverse
independently of the choice of a basis, in contrast with \cref{h:eq:fwcoord,h:eq:decomp}.
\end{remark}
\begin{remark}
We are not discussing convergence here;
refer to \textcite{folland2016course,gallierncharm} for details.
\end{remark}
\begin{example}
  Consider the multiplicative group of complex numbers of the form $e^{i\theta}$,
  identified with the planar rotation group $\SO{2}$.
  The unirreps of this group on $\C$ are given by $\rho_n(e^{i\theta}) = e^{in\theta}$,
  for $n \in \Z$.
  Since they are one dimensional, they are also the matrix elements and hence form an orthonormal basis
  for $L^2(\SO{2})$.
  We can index $\rho_n$ by $n$ and $\SO{2}$ by $\theta$,
  and write the Fourier transform and inverse on $L^2(\SO{2})$ as
  \begin{align}
    \hat{f}(n) &= \int\limits_{0}^{2\pi} f(\theta) e^{-in\theta} \, \frac{d\theta}{2\pi}, \\
    f(\theta) &= \sum_{n=-\infty}^{\infty} \hat{f}(n) e^{in\theta},
  \end{align}
  which are the familiar formulas for the Fourier series of periodic functions.

  This simple example shows how the Fourier analysis of periodic functions on the line fit in the theory described.
  See \cref{h:sec:su2} for a more complete example with a non-commutative group.
\end{example}
\subsection{Convolution theorem on compact groups}
\label{h:sec:gconv}
Given the existence of the left Haar measure as discussed in \cref{h:sec:haar},
we define the convolution between functions \fun{f,k}{G}{\C} on a group $G$ as
\begin{equation}
  (f * k)(g) = \int\limits_{u \in G} f(u)k(u^{-1}g)\, du = \int\limits_{u \in G} f(gu)k(u^{-1})\, du.
  \label{h:eq:gconv}
\end{equation}

A simple change of variables leveraging the measure left-invariance
shows that group convolution is equivariant,
\begin{align*}
  (\lambda_u f * k)(g)
  &= \int\limits_{v \in G} f(u^{-1}v) k(v^{-1}g) \, dv\\
  &= \int\limits_{w \in G} f(w) k((uw)^{-1}g)  \, dw && (v \mapsto uw) \\
  &= \int\limits_{w \in G} f(w) k(w^{-1}u^{-1}g)  \, dw\\
  &= (f * k)(u^{-1}g) \\
  &= (\lambda_u (f * k))(g).
\end{align*}

\begin{theorem}[Convolution theorem]
  \label{h:thm:conv}
  Let $f$ and $k$ be square integrable functions on a compact group $G$ ($f,\,g \in L^2(G)$).
  The Fourier transform of the convolution $f*k$ is $(\widehat{f * k})(\rho) = \hat{k}(\rho) \hat{f}(\rho)$.
\end{theorem}
\begin{proof}
Let us compute the Fourier transform of $f * k$ using \cref{h:eq:ft},
\begin{align*}
  (\widehat{f * k})(\rho) &= \int\limits_{g \in G} \left(\int\limits_{u \in G} f(u)k(u^{-1}g)\, du\right) \rho(g)^{*}\, dg \\
    &= \int\limits_{u \in G} \int\limits_{g \in G} f(u)k(u^{-1}g) \rho(g)^{*} \,dg\, du \\
    &= \int\limits_{u \in G} \int\limits_{v \in G} f(u)k(v) \rho(uv)^{*} \,dv\, du && (v = u^{-1}g) \\
    &= \int\limits_{v \in G} k(v) \rho(v)^{*} \int\limits_{u \in G} f(u)\rho(u)^{*} \,dv\, du
    && \text{(homomorphism, reorder)} \\
    &= \hat{k}(\rho) \hat{f}(\rho).
\end{align*}
\end{proof}
\begin{remark}
  There is an analogous cross-correlation theorem that we prove in the same way.
  We define the group cross-correlation as
  \[(f \star k)(g) = \int\limits_{u \in G} f(u)  k(g^{-1} u) \, du \]
  and follow the same steps as before, obtaining
\begin{align*}
  (\widehat{f \star k})(\rho)
  &= \int\limits_{v \in G} k(v) \rho(v^{-1})^{*} \int\limits_{u \in G} f(u)\rho(u)^{*} \,dv\, du.
\end{align*}
Note that the only difference is the term $v^{-1}$.
Since $\hat{k}(\rho) = \int_{v\in G}k(v)\rho(v)^*$,
assuming real-valued $k$ we have
$\hat{k}(\rho)^* = \int_{v\in G}k(v)\rho(v^{-1})^*$ and
\begin{align}
  (\widehat{f \star k})(\rho) = \hat{k}(\rho)^*\hat{f}(\rho).
  \label{h:eq:groupcorrspectral}
\end{align}
\end{remark}
\noindent This shows that the Fourier transform of the compact group convolution
is the matrix product of the Fourier transforms of each input.
It generalizes the convolution theorem on the circle,
which says that the Fourier transform of the convolution is the scalar multiplication
of the inputs Fourier transforms.

The convolution theorem is fundamental for the efficient computation of convolutions,
since the \fft\ can be generalized to compact groups \cite{driscoll1994computing,kostelec2008ffts}.
Furthermore, the spectral computation avoids interpolation errors and
extra computational cost caused by the lack of regular grids for arbitrary groups.
\subsection{Examples: SL$(2)$, SU$(2)$ and SO$(3)$}
\label{h:sec:su2}
Now we find expressions for the matrix elements of
representations of $\SL{2, \C}$, $\SU{2}$ and $\SO{3}$,
which allow computing the Fourier transforms and convolutions on these groups.
We follow one of the approaches by \textcite{vilenkin1978special},
also used by \textcite{dieudonn1980special,gurarie2007symmetries}.

The strategy is to first find the matrix elements for irreps of $\SL{2, \C}$,
then restrict them to $\SU{2}$ and $\SO{3}$.
\subsubsection{Representations of SL($2$, $\C$)}
The special linear group $\SL{2, \C}$ consists of $2\times 2$ complex matrices with determinant 1,
\begin{align}
  g = \twobytwo{a}{c}{b}{d} \label{h:eq:gsl2}
\end{align}
where $ad - bc = 1$.

Now consider the space $V_{\ell}$ of homogeneous polynomials of degree $2\ell$ in two complex variables,
where $\ell$ is integer or half-integer,
\begin{align*}
  x = \twobyone{x_1}{x_2}, && P_\ell(x) = P_\ell(x_1, x_2) = \sum_{i=-\l}^\l\alpha_ix_1^{\ell-i}x_2^{\ell+i}.
\end{align*}
We define \fun{\pi_\ell}{\SL{2,\C}}{\GL{V_{\ell}}} as
\begin{align}
  (\pi_\ell(g)P_\ell)(x) = P_\ell(g^{-1}x), \label{h:eq:reprsl}
\end{align}
which is linear and a group homomorphism.
Hence, $\pi_\ell$ is a representation of $\SL{2, \C}$ on $V_{\ell}$
(of dimension $2\ell+1$).
Furthermore, it can be shown that these are irreducible,
and in fact these are the only irreps of $\SL{2, \C}$ and $\SU{2}$, up to equivalence.

Now let us derive expressions for the matrix elements.
Consider the polynomial in one variable
$Q_\ell(x) = P_\ell(x, 1) = \sum_{i=-\l}^\l\alpha_ix^{\ell-i}$, of degree $2\ell$.
Writing $P_\ell$ in terms of $Q_\ell$ yields
\begin{align}
  P_\ell(x_1, x_2) = x_2^{2\ell}Q_\ell(x_1/x_2). \label{h:eq:pq}
\end{align}
We denote $H_\ell$ the space of all polynomials $Q_\ell$ (of degree $2\ell$) for $\ell \ge 0$.
We rewrite \cref{h:eq:reprsl} for $g$ as in \cref{h:eq:gsl2} where $g^{-1}=\twobytwo{d}{-c}{-b}{a}$,
\begin{align}
  (\pi_\ell(g)P_\ell)(x_1, x_2) = P_\ell(dx_1-cx_2, -bx_1 + ax_2), \label{h:eq:reprp}
\end{align}
and define $\rho_\ell$ as the application of $\pi_\ell$ to $Q_\ell\in H_\ell$ using \cref{h:eq:reprp,h:eq:pq}
\begin{align}
  (\rho_\ell(g)Q_\ell)(x) = (-bx + a)^{2\ell}Q_\ell\left(\frac{dx-c}{-bx + a}\right). \label{h:eq:repq}
\end{align}

The monomials $x^{\ell-m}$ for $-\ell \le m \le \ell$ are a basis of $H_\ell$.
Now consider the inner product on $H_\ell$ defined by
\begin{align}
  \inner{x^{\ell-m}, x^{\ell-n}} &= 0, && m \neq n \label{h:eq:innerq1} \\
  \inner{x^{\ell-m}, x^{\ell-m}} &= (\ell-m)!(\ell+m)!, \label{h:eq:innerq2}
\end{align}
which is adapted from an inner product (sometimes called the Bombieri scalar product) on $V_\ell$
\[ \inner{x^{\ell+m}y^{\ell-m},x^{\ell+m}y^{\ell-m}} = (\ell+m)!(\ell-m)!/(2\ell)!. \]

It turns out the representation $\rho_\ell$ defined as in \cref{h:eq:repq}
is unitary under the inner product defined by \cref{h:eq:innerq1,h:eq:innerq2}.
The following is an orthonormal basis $\{\psi_m\}$ for $H_\ell$ with this inner product
\begin{align*}
  \psi_m(x) = \frac{x^{\ell-m}}{\sqrt{(\ell-m)!(\ell+m)!}}.
\end{align*}
The element at position $(m, n)$%
\footnote{Not the conventional way of indexing since $-\ell \le m, n \le \ell$,
but convenient in our notation.}
of the matrix for $\rho_\ell(g)$ under this basis is
\begin{align}
\rho_{\ell}^{mn}(g) = \inner{\rho_{\ell}(g)(\psi_n), \psi_m}. \label{h:eq:rhoij}
\end{align}
According to \cref{h:eq:repq}, $\rho_\ell$ acts on $Q(x) = x^{\ell-n}$ as
\begin{align*}
\rho_\ell(g)x^{\ell-n} = (-bx + a)^{\ell+n}(dx-c)^{\ell-n}
\end{align*}
for $g$ as in \cref{h:eq:gsl2}.
We substitute it in \cref{h:eq:rhoij} to obtain
\begin{align}
  \rho_{\ell}^{mn}(g) = \frac{\inner{(-bx + a)^{\ell+n}(dx-c)^{\ell-n}, x^{\ell-m}}}
  {\sqrt{(\ell-n)!(\ell+n)!(\ell-m)!(\ell+m)!}}. \label{h:eq:rhoijlong}
\end{align}
Observe that $\inner{Q(x),x^{\ell-m}}$ for some polynomial $Q(x)$
is the coefficient of $x^{\ell-m}$ in $Q(x)$ multiplied by $(\ell-m)!(\ell+m)!$,
according to \cref{h:eq:innerq2}.
Recall that the Taylor formula for a function $Q(x)$ around $x=0$ is
$Q(x) = \sum_{n=0}^\infty \frac{d^n}{dx^n}\frac{x^n}{n!}$.
We apply it to obtain the coefficient of $x^{\ell-m}$ in \cref{h:eq:rhoijlong},
\begin{align}
  \rho_{\ell}^{mn}(g) = \sqrt{\frac{(\ell+m)!}{(\ell-n)!(\ell+n)!(\ell-m)!}}
  \frac{d^{\ell-m}}{dx^{\ell-m}}\left((-bx + a)^{\ell+n}(dx-c)^{\ell-n}\right)\Big|_{x=0},
  \label{h:eq:mesl2}
\end{align}
with $g=\twobytwo{a}{c}{b}{d}$ as usual.
Substituting $z=b(dx-c)$ and using that $ad-bc=1$ yields
\begin{align}
  \rho_{\ell}^{mn}(g) = \sqrt{\frac{(\ell+m)!}{(\ell-n)!(\ell+n)!(\ell-m)!}}
  \frac{b^{n-m}}{d^{n+m}}
  \frac{d^{\ell-m}}{dz^{\ell-m}}\left((1-z)^{\ell+n}z^{\ell-n}\right)\Big|_{z=-bc}.
  \label{h:eq:mesl2z}
\end{align}
This is a general formula for matrix elements of the unirreps
of $\SL{2, \C}$, which generate an orthonormal basis of $L^2(\SL{2, \C})$
as stated by the Peter-Weyl theorem.
\subsubsection{Representations of SU($2$)}
We now restrict the $\SL{2, \C}$ representations to $g \in \SU{2}$,
the group of $2\times 2$ unitary matrices with determinant \SI{1}{}.
So for $g \in \SU{2}$ we have $g^{*}g = gg^* = I$, which implies
\begin{align}
  g = \twobytwo{a}{-\overline{b}}{b}{\overline{a}} \label{h:eq:gsu2}
\end{align}
where $a\overline{a} + b\overline{b} = 1$, and the bar denotes the complex conjugate.
It follows that $\SU{2} < \SL{2, \C}$.
We can factor $g \in \SU{2}$ as
\begin{align}
  g_{\alpha\beta\gamma} =
  \twobytwo{e^{-i\alpha/2}}{0}{0}{e^{i\alpha/2}}
  \twobytwo{\cos(\beta/2)}{-\sin(\beta/2)}{\sin(\beta/2)}{\cos(\beta/2)}
  \twobytwo{e^{-i\gamma/2}}{0}{0}{e^{i\gamma/2}},
  \label{h:eq:factorsu2}
\end{align}
where $\alpha$, $\beta$ and $\gamma$ are ZYZ Euler angles,
$0 \le \alpha < 2\pi$, $0 \le \beta < \pi$ and $-2\pi \le \gamma < 2\pi$.
Now consider representations \fun{\rho_\ell}{\SU{2}}{\GL{H_\ell}},
which are a special case of the representations of $\SL{2, \C}$,
and hence inherit their properties.
Since $\rho_\ell$ is a group homomorphism,
\begin{align}
  \rho_\ell(g_{\alpha \beta \gamma}) =
  \rho_\ell(g_{\alpha 0 0})
  \rho_\ell(g_{0 \beta 0})
  \rho_\ell(g_{0 0 \gamma}).
  \label{h:eq:factorrho}
\end{align}
Since $\rho(g_{\alpha 0 0})$ corresponds to $a=e^{-i\alpha/2}$, $d=\overline{a}$, and $b=c=0$
in \cref{h:eq:repq}, we find that
$\rho_\ell(g_{\alpha 0 0})(\psi_m)= e^{-i\alpha m} \psi_m$,
which implies that only the diagonal elements of $\rho_\ell(g_{\alpha 0 0})$
are nonzero; they are
\begin{align}
\rho_\ell^{mm}(g_{\alpha 0 0}) = e^{-i\alpha m}. \label{h:eq:rholmm}
\end{align}
The expression for $\rho_\ell(g_{0 0 \gamma})$ is analogous.
The middle factor in \cref{h:eq:factorrho} is multiplied by diagonal matrices on both sides,
so we write the matrix elements
\begin{align*}
\rho_\ell^{mn}(g_{\alpha \beta \gamma}) = e^{-i(m\alpha + n\gamma)}\rho_\ell^{mn}(g_{0 \beta 0}).
\end{align*}
To compute $\rho_\ell^{mn}(g_{0 \beta 0})$,
we apply
\begin{align*}
a=d=\cos(\beta/2) \text{, } b=\sin(\beta/2) \text{ and } c=-\sin(\beta/2)
\end{align*}
to \cref{h:eq:mesl2z},
and note that the derivative is evaluated at $z=-bc=\sin^2(\beta/2)$.
We define $P_{mn}^\ell(\cos\beta) = \rho_\ell^{mn}(g_{0 \beta 0})$,
and make the substitution
\[z \mapsto \frac{1-x}{2},\]
where the derivative is now evaluated at $x=-2\sin^2(\beta/2)+1 = \cos\beta$.
Then $b=\sqrt{(1-x)/2}$ and $d=\sqrt{(1+x)/2}$. We have,
\begin{align}
  P_{mn}^\ell(x)
  &=  c_{mn}^\ell \frac{\left(\frac{1-x}{2}\right)^{\frac{n-m}{2}}}
    {\left(\frac{1+x}{2}\right)^{\frac{n+m}{2}}}
    \frac{(-1)^{\ell-m}}{2^{m-\ell}}
  \frac{d^{\ell-m}}{dx^{\ell-m}}
    \left(\left(\frac{1+x}{2}\right)^{\ell+n}\left(\frac{1-x}{2}\right)^{\ell-n}\right) \nonumber \\
  &=  c_{mn}^\ell \frac{(-1)^{\ell-m}}{2^{\ell}}
    \frac{(1-x)^{\frac{n-m}{2}}}{(1+x)^{\frac{n+m}{2}}}
    \frac{d^{\ell-m}}{dx^{\ell-m}}
    \left((1+x)^{\ell+n}(1-x)^{\ell-n}\right). \label{h:eq:mesu2}
\end{align}
where
\begin{align*}
    c_{mn}^\ell = \sqrt{\frac{(\ell+m)!}{(\ell-n)!(\ell+n)!(\ell-m)!}},
\end{align*}
and
\begin{align}
  \rho_\ell^{mn}(g_{\alpha \beta \gamma}) = e^{-i(m\alpha + n\gamma)}P_{mn}^\ell(\cos\beta),
  \label{h:eq:mesu2general}
\end{align}
which is a general formula for matrix elements of $\SU{2}$ unirreps.
The matrices formed with the $\rho_\ell^{mn}$ and $P_{mn}^\ell$ are
also known as a Wigner-D and Wigner-d matrices, respectively.
\subsubsection{Representations of SO($3$)}
$\SU{2}$ is isomorphic to the group of unit quaternions, hence a double cover of $\SO{3}$,
which is easily verifiable by noting that every rotation in $\SO{3}$ can be written as two
different quaternions $q$ and $-q$.
We have ${\SO{3} \cong \SU{2}/\{I, -I\}}$.
The representations of $\SO{3}$ are then those representations of $\SU{2}$ where
$\rho_\ell(I) = \rho_\ell(-I)$.
By substituting $b=c=0$ in \cref{h:eq:mesl2},
we see that the only nonzero terms outside the square root
occur when $n=m$, yielding diagonal matrices with entries
proportional to $a^{\ell + m}d^{\ell - m}$,
\begin{align}
  \rho_{\ell}^{mm}\left(\twobytwo{a}{0}{0}{d}\right)
  &= \frac{1}{(\ell-m)!}a^{\ell + m}d^{\ell - m}.
\end{align}
Recall that for $\SU{2}$ representations, $\ell$ can be integer or half integer.
For $a=d=1$ the expression reduces to $1/(\ell-m)!$
while for $a=d=-1$ it reduces to $(-1)^{2\ell}/(\ell-m)!$,
from where we conclude that $\rho_\ell(I) = \rho_\ell(-I)$ only when $\ell$ is integer.
Therefore, the representations of $\SO{3}$ are also given by \cref{h:eq:mesu2general},
but with $\ell$ taking only integer values.

\subsubsection{Relation with special functions}
The Jacobi polynomials generalize the Gegenbauer, Legendre, and Chebyshev polynomials,
and thus give origin to several special functions.
One way to represent the Jacobi polynomials is via the Rodrigues' formula%
\footnote{Not to be confused with the Rodrigues' rotation formula.}
\begin{align*}
  P_n^{(\alpha,\beta)}(z) = \frac{(-1)^n}{2^n n!}
  (1-z)^{-\alpha} (1+z)^{-\beta}
  \frac{d^n}{dz^n}
  \left( (1-z)^{\alpha+n} (1+z)^{\beta+n} \right).
\end{align*}
Note how it is tightly related to our expression for the matrix elements in \cref{h:eq:mesu2},
showing how the special functions arise in the study of group representations.

By setting $m=n=0$ and $\ell$ integer in \cref{h:eq:mesu2}, we get
\begin{align}
  P_{00}^\ell(x) = \frac{(-1)^{\ell}}{2^{\ell}\ell !}
  \frac{d^{\ell}}{dz^{\ell}}
  (1-x^2)^{\ell},
\end{align}
the Legendre polynomials, which describe the zonal spherical harmonics.

The associated Legendre polynomials can be written as
\begin{align*}
  P_m^\ell(x) = \frac{(-1)^{\ell + m}}{2^\ell \ell!} (1-x^2)^{m/2} \frac{d^{\ell+m}}{dx^{\ell+m}} (1-x^2)^\ell.
\end{align*}
By setting $\ell$ integer and $n=0$ in \cref{h:eq:mesu2},
we can relate $P_{m,n}^\ell$ with the associated Legendre polynomials,
\begin{align*}
  P_{-m,0}^\ell(x)
  &= \frac{(-1)^{\ell+m}}{2^\ell \ell!} \sqrt{\frac{(\ell-m)!}{(\ell+m)!}}
  (1-x^2)^{m/2}
  \frac{d^{\ell+m}}{dz^{\ell+m}}
  \left((1-x^2)^{\ell}\right) \\
  &= \sqrt{\frac{(\ell-m)!}{(\ell+m)!}} P_m^\ell(x).
\end{align*}
Noting that $P_{m,n}^\ell=P_{-m,-n}^\ell$ we write
\begin{align}
  \label{h:eq:assocleg}
  P_m^\ell(x) = \sqrt{\frac{(\ell+m)!}{(\ell-m)!}} P_{m0}^\ell(x).
\end{align}

The spherical harmonics are usually defined in terms of the associated Legendre polynomials
\begin{align}
  Y_m^\ell(\theta, \phi) = \sqrt{\frac{(2\ell + 1)}{4\pi} \frac{(\ell-m)!}{(\ell+m)!}}
  P_{m}^\ell(\cos\theta)e^{im\phi}.
  \label{h:eq:sphharmdef}
\end{align}
Using \cref{h:eq:mesu2general,h:eq:assocleg}, we obtain a relation
between the spherical harmonics and the representations $\rho_\ell^{mn}$ ,
\begin{align}
  \rho_\ell^{m0}(g_{\alpha \beta \gamma})
  &= P_{m0}^\ell(\cos\beta) e^{-im\alpha} \nonumber \\
  \label{h:eq:wig2sph}
  &= \sqrt{\frac{(\ell-m)!}{(\ell+m)!}} P_{m}^\ell(\cos\beta) e^{-im\alpha} \nonumber \\
  &= \sqrt{\frac{4\pi}{(2\ell+1)}} \overline{Y_m^\ell(\beta, \alpha)}.
\end{align}
With this relation, we find an expression for the rotation of spherical harmonics.
Let $g\nu$ be the point obtained by rotating the north pole by $g$.
Since $\rho_{\ell}(g_1g_2) = \rho_{\ell}(g_1) \rho_{\ell}(g_2)$,
\begin{align*}
  \rho_{\ell}^{m0}(g_1g_2) &= \sum_{n=-\ell}^{\ell} \rho_{\ell}^{mn}(g_1) \rho_{\ell}^{n0}(g_2), \\
  \overline{Y_m^\ell(g_1 g_2 \nu)} &= \sum_{n=-\ell}^{\ell} \rho_{\ell}^{mn}(g_1)
                                     \overline{Y_n^\ell(g_2\nu)}.
\end{align*}
Taking conjugates on both sides we arrive at the spherical harmonics rotation formula,
which will be useful in following proofs.
For $x\in S^2$ and $g \in \SO{3}$,
\begin{align}
  \label{h:eq:sphharmrot}
  Y_m^\ell(g x) &=
\sum_{n=-\ell}^{\ell} \overline{\rho_{\ell}^{mn}(g)} Y_n^\ell(x),
\end{align}
which we write in vector notation as
$Y^\ell(g x) = \overline{\rho_{\ell}(g)} Y^\ell(x).$

\subsection{Fourier analysis on homogeneous spaces}
\label{h:sec:fourierh}
We now consider functions on the homogeneous space $G/H$ of a compact group $G$ with subgroup $H$;
specifically, consider square integrable functions in $L^2(G/H)$.
Recall that $G/H$ is the set of left cosets and that ${gHh = gH}$ for all $gH \in G/H$ and $h \in H$.
Hence, we can regard functions in $L^2(G/H)$ as the functions in $L^2(G)$ such that $f(gh) = f(g)$
for all $g \in G$ and $h \in H$
(functions that are constant on each coset $gH$ for all $g \in G$).
Using \cref{h:eq:fwcoord}, we write
$f(g) = \sum_{[\rho] \in \hat{G}}\sum_{i,j=1}^{d_\rho} c_{ij}^\rho \rho_{ij}(g)$,
and expand $f(gh)$ as
\begin{align*}
  f(gh) &= \sum_{[\rho] \in \hat{G}}\sum_{i,j=1}^{d_\rho} c_{ij}^\rho \rho_{ij}(gh) \\
        &= \sum_{[\rho] \in \hat{G}}\sum_{i,j=1}^{d_\rho} c_{ij}^\rho \sum_{k=1}^{d_\rho}\rho_{ik}(g)\rho_{kj}(h) \\
        &= \sum_{[\rho] \in \hat{G}}\sum_{i=1}^{d_\rho}\sum_{k=1}^{d_\rho}\left( \sum_{j=1}^{d_\rho}  c_{ij}^\rho \rho_{kj}(h)\right) \rho_{ik}(g).
\end{align*}
We want $f(g)=f(gh)$, so we compare this expression with the expansion of $f(g)$.
Since the $\rho_{ij}$ are linearly independent, we have
\begin{equation}
  \sum_{j=1}^{d_\rho}  c_{ij}^\rho \rho_{kj}(h) = c_{ik}^{\rho} \label{h:eq:cij}
\end{equation}
for all $\rho$, $i$, $k$, and $h$.
Now suppose the trivial representation of $H$ has multiplicity $n_\rho \ge 1$ in
the restriction of $\rho$ to $H$.
We can reorder the basis such that the trivial representations appear first.
This implies $\rho_{kj}(h) = \delta_{kj}$ for $j \le n_\rho$ which agrees with $\cref{h:eq:cij}$.
After reordering, $\rho_{kj}$ integrates to zero for $k > n_\rho$ or $j > n_\rho$,
(only trivial matrix elements integrals are nonzero).
Applying this to both sides of \cref{h:eq:cij} yields $c_{ik}^\rho = 0$ for $k > n_\rho$,
which implies that any $f \in L^2(G/H)$ can be expanded as
\begin{align}
  f(gH) &= \sum_{[\rho] \in \hat{G}}\sum_{i=1}^{d_\rho}\sum_{j=1}^{n_\rho} c_{ij}^\rho \rho_{ij}(g),
  \label{h:eq:expandh}
\end{align}
where $n_\rho$ is the multiplicity of the trivial representation of $H$ in $\rho$ (which may be zero).
Only the first $n_\rho$ columns of each $\rho$ are necessary for the Fourier analysis on homogeneous spaces.
In the special case that $n_\rho=1$, only the matrix elements $\rho_{i1}$ will appear.
These are called the \emph{associated spherical functions} \cite{vilenkin1978special}.

When considering functions on the homogeneous space of right cosets, $L^2(H\backslash G)$,
we arrive at similar results where only the first $n_\rho$ rows will appear in the expansion.
When considering functions on the double coset space $L^2(H\backslash G/H)$,
only the first $n_\rho$ rows and columns will appear.
In this last case, when $n_{\rho}=1$, only the matrix elements $\rho_{11}$ appear.
These are called \emph{zonal spherical functions}.
When $n_{\rho} \leq 1$ for every $\rho$, the algebra (with the convolution product)
$L^2(H\backslash G/H)$ is commutative.
\begin{remark}
  The functions just defined are called \emph{spherical} because of the
  special case $G=\SO{3}$ and $H=\SO{2}$ (recall that $S^2 \cong \SO{3}/\SO{2}$).
  These terms apply, however, to any compact group and its homogeneous spaces.
\end{remark}
\begin{remark}
  This discussion generalizes to a locally compact group $G$ (not necessarily compact),
  and compact subgroup $K$, under certain conditions where $(G, K)$ is called a Gelfand pair.
  Refer to \textcite{gallierncharm} for details.
\end{remark}

\subsection{Example: Fourier analysis on $S^2$}
\label{h:sec:fouriers2}
We apply the results of \cref{h:sec:fourierh} to the group $G=\SO{3}$ and subgroup $H=\SO{2}$,
where the homogeneous space is isomorphic to the sphere $S^2 \cong \SO{3}/\SO{2}$.
Elements of $\SO{3}$ decompose in Euler angles components similarly to \cref{h:eq:factorsu2}
and by setting $\alpha=\beta=0$ we obtain a subgroup isomorphic
to $\SO{2}$ consisting of rotations around the axis through the poles.
We obtain the restriction of $\SO{3}$ irreps to this subgroup
by setting $\alpha=\beta=0$ for integer $\ell$ in \cref{h:eq:factorrho},
resulting in $\rho_\ell(g_{0 0 \gamma})$ which is diagonal and defined by
$\rho_\ell^{nn}(g_{0 0 \gamma}) = e^{-i\gamma n}$ (\cref{h:eq:rholmm}).
Therefore the trivial representation of $\SO{2}$ appears only when $n=0$ and
its multiplicity is 1 for all $\ell$,
and using that $\overline{Y_m^\ell} = (-1)^mY_{-m}^\ell)$,
the expansion in \cref{h:eq:expandh} reduces to
\begin{align*}
  f(g_{\alpha\beta\gamma})
  &= \sum_{\ell \in \N}\sum_{m=-\ell}^{\ell}b_{m}^\ell \rho_\ell^{m0}(g_{\alpha\beta\gamma}) \\
  &= \sum_{\ell \in \N}\sum_{m=-\ell}^{\ell}b_{m}^\ell i^{m}\sqrt{\frac{4\pi}{(2\ell+1)}} \overline{Y_m^\ell(\beta, \alpha)} && (\cref{h:eq:wig2sph}) \\
  &= \sum_{\ell \in \N}\sum_{m=-\ell}^{\ell}b_{-m}^\ell (-i)^{-m}\sqrt{\frac{4\pi}{(2\ell+1)}} Y_{m}^\ell(\beta, \alpha) \\
  &= \sum_{\ell \in \N}\sum_{m=-\ell}^{\ell} c_m^\ell Y_{m}^\ell(\beta, \alpha).
\end{align*}
We rewrite the expansion as
\begin{align}
  \label{h:eq:sphharm}
  f(\theta, \phi) &= \sum_{\ell \in \N}\sum_{m=-\ell}^{\ell}\hat{f}_{m}^\ell Y_m^\ell(\theta, \phi),
\end{align}
which shows that the spherical harmonics $Y_m^\ell$ form indeed
an orthonormal basis for $L^2(S^2)$.
The decomposition into the basis is then given by
\begin{align}
  \hat{f}_{m}^\ell &= \int\limits_{x\in S^2} f(x) \overline{Y_m^\ell(x)} \, dx,
\end{align}
where $x \in S^2$ can be parameterized angles $\theta$ and $\phi$.

This concludes our introduction to harmonic analysis.
For more details we recommend \textcite{gallierncharm,dieudonn1980special,folland2016course}.
\textcite{chirikjian2000engineering} present an applied take on the subject.
\section{Group convolutional neural networks}
\label{h:sec:appl}
We can see a typical deep neural network as a chain of affine operations
$W_i$ whose parameters are optimized, interspersed with nonlinearities $\sigma_i$,
\begin{align}
  f_{\text{out}} = W_n(\cdots \sigma_2(W_2(\sigma_1(W_1f_{\text{in}}))) \cdots).
\end{align}
In \cnns, these operations are convolutions with an added bias.
The most common nonlinearities are pointwise; one popular example is the \relu,
$\sigma(x_i) = \max(x_i, 0)$.

In \gcnns, the operations are group or homogeneous space convolutions.
There are different classes of networks that vary with respect to the group considered,
whether the equivariance is to global transformations or local (patch-wise),
and whether the feature maps are scalar or more general fields.
In this section, we discuss the classes of networks we cover in this thesis,
in light of the theory presented so far.
\subsection{Finite group CNNs}
On a finite group,
the counting measure can be used and convolution reduces to summing over each element of the group,
\begin{align}
(f * k)(g) = \frac{1}{|G|} \sum\limits_{x \in G} f(x) k(x^{-1} g).
\end{align}
This simple operation has been successfully applied for rotation equivariance on discrete subgroups of $\SO{3}$;
\textcite{cohencube,worrall2018cubenet} consider the octahedral group of \SI{24} elements.
We show an application that uses the icosahedral group of \SI{60} elements in \cref{emvn:sec:emvn}.

\subsection{Spherical CNNs}
\label{h:sec:sphcnns}

The sphere is not a group, but there are two group convolutional operations that we can define
for spherical functions: the spherical cross-correlation and spherical convolution.
Both are equivariant to \SO{3} and can be used as \gcnns\ operations.

\paragraph{Spherical Cross-Correlation}
The spherical cross-correlation between a function $f$ and a filter $k$ lifts
the results to a function on $\SO{3}$,
\begin{align}
\label{h:eq:sphcorreq}
(f \star k)(g) = \int\limits_{x \in S^2} k(g^{-1}x)f(x) \, dx.
\end{align}
This operation has a pattern matching interpretation.
Suppose $k$ is a rotated version of $f$; then the correlation achieves
its maximum value when $g$ is the rotation that aligns $k$ and $f$.
Note that $f$ and $k$ are functions on $S^2$,
while $f\star k$ is a function on $\SO{3}$.
\begin{proposition}[spherical cross-correlation]
  The spherical cross-correlation between $f,\, k \in L^2(S^2)$ as defined in \cref{h:eq:sphcorreq}
  can be computed in the spectral domain via outer products of vectors of spherical harmonics coefficients,
  \begin{align*}
    \widehat{(f \star k)}^{\ell} = \overline{\hat{k}^{\ell}} (\hat{f}^\ell)^\top.
  \end{align*}
\end{proposition}
\begin{proof}
We evaluate \cref{h:eq:sphcorreq} by expanding $f$ and $k$ as in \cref{h:eq:sphharm},
where $Y^\ell(x) \in \C^{2\ell + 1}$ contains the spherical harmonics of degree $\ell$ evaluated at $x$,
and $\hat{f}^{\ell} \in \C^{2\ell + 1}$ contains the respective coefficients.
We assume real-valued functions (hence the complex conjugation on the first line),
and use the spherical harmonics rotation formula from \cref{h:eq:sphharmrot}.
\begin{align*}
  (f \star k)(g)
  &= \int\limits_{x \in S^2}
  \sum_{\ell'} (\overline{\hat{k}^{\ell'}})^{\top} \overline{Y^{\ell'}(g^{-1}x)}
    \sum_\ell Y^\ell(x)^{\top} \hat{f}^\ell  \, dx \\
  &= \int\limits_{x \in S^2}
  \sum_{\ell'} (\overline{\hat{k}^{\ell'}})^{\top} \rho_{\ell}(g^{-1}) \overline{Y^{\ell'}(x)}
    \sum_\ell Y^\ell(x)^{\top} \hat{f}^\ell  \, dx \\
  &= \sum_{\ell,\ell'}
    (\overline{\hat{k}^{\ell'}})^{\top} \rho_{\ell}(g)^\top
    \int\limits_{x \in S^2} \overline{Y^{\ell'}(x)}
    Y^\ell(x)^{\top} \hat{f}^\ell  \, dx.
\end{align*}
By orthonormality of the spherical harmonics, $\int_{x \in S^2} \overline{Y^{\ell'}(x)}Y^{\ell}(x)^{\top}$
is the identity $I_{2\ell + 1}$ when $\ell=\ell'$ and zero otherwise. Then,
\begin{align*}
  (f \star k)(g)
  &= \sum_{\ell}
    (\overline{\hat{k}^{\ell}})^{\top} \rho_{\ell}(g)^\top
    \hat{f}^\ell \\
  &= \sum_{\ell} \tr(\hat{f}^\ell(\overline{\hat{k}^{\ell}})^{\top} \rho_{\ell}(g)^\top)
  && (x^\top A y = \tr(yx^\top A)) \\
  &= \sum_{\ell} \tr(\overline{\hat{k}^{\ell}} (\hat{f}^\ell)^\top \rho_{\ell}(g) ).
\end{align*}
where we used the cyclic and transpose properties of the trace in the last part.
The last line is a Fourier expansion of a function on \SO{3} (\cref{h:eq:fi})
with coefficients given by the outer product of the input coefficients.
This can be restated as
\begin{align}
  \widehat{(f \star k)}^{\ell} = \overline{\hat{k}^{\ell}} (\hat{f}^\ell)^\top,
\end{align}
or in terms of matrix elements,
$\widehat{(f \star k)}_{mn}^{\ell} = \overline{\hat{k}_m^{\ell}} \hat{f}_n^\ell$.
\end{proof}
\begin{remark}
  \textcite{makadia2006} show an alternative proof of this result.
  The efficient computation for sampled functions
  relies on the sampling theorem described by \textcite{kostelec2008ffts}.
\end{remark}

This operation is used in the first layer of \textcite{s.2018spherical},
where the following layers have inputs and outputs on \SO{3} and thus use
pure group cross-correlation as shown in \cref{h:eq:groupcorrspectral}.

The spherical cross-correlation has further applications
in pose estimation \cite{makadia2006,makadia2007correspondence}
and $3$D shape retrieval \cite{makadia2010spherical}.
We show applications for pose estimation in a deep learning setting in \cref{cross:sec:cross,sph:sec:align}.

\paragraph{Spherical Convolution}
The spherical convolution has inputs and outputs on the sphere,
\begin{align}
(f * k)(x) = \int\limits_{g \in \SO{3}} f(g \nu) k(g^{-1} x) \, dg, \label{h:eq:sphconveq}
\end{align}
where $\nu$ is a fixed point on the sphere (the north pole).
To interpret this operation, we split the integral as in \cref{h:thm:hmeasure},
which holds since $\SO{3}$ and $\SO{2}$ are unimodular,
\begin{align*}
  (f * k)(x)
  &= \int\limits_{g_{\alpha\beta} \in \SO{3}/\SO{2}} \int\limits_{g_\gamma \in \SO{2}}
    f(g_{\alpha\beta}g_\gamma \nu) k((g_{\alpha\beta}g_\gamma)^{-1} x) \, dg_{\alpha\beta}dg_{\gamma}, \\
  &= \int\limits_{g_{\alpha\beta} \in \SO{3}/\SO{2}}
    f(g_{\alpha\beta} \nu) \left( \int\limits_{g_\gamma \in \SO{2}} k(g_\gamma^{-1}g_{\alpha\beta}^{-1} x) \, dg_{\gamma} \right) dg_{\alpha\beta}. \\
\end{align*}
The inner integral averages $k$ over rotations around the $z$ axis,
resulting in a zonal function (constant on latitudes);
note that this  limits the expressivity of the filters.
The outer integral is then a spherical inner product
where $x$ determines the filter orientation.
\textcite{driscoll1994computing} shows how to compute the convolution
efficiently in the spectral domain. The following lemma will be necessary.
\begin{lemma}
  \label{h:lemma:dh}
  For $f \in L^2(S^2)$, let $\rho_\ell^{mn}$ be the matrix elements of the unirreps of $\SO{3}$,
  $\nu$ the north pole, and $\hat{f}_n^\ell$ the spherical harmonic coefficient of $f$
  corresponding to $Y_n^\ell$. The following holds
  \begin{align*}
    \int\limits_{u \in \SO{3}} f(u \nu) \overline{\rho_\ell^{mn}(u^{-1})}\, du =
    2\pi\sqrt{\frac{4\pi}{2\ell+1}} \hat{f}_n^\ell
  \end{align*}
  for $m=0$. The integral is 0 otherwise.
\end{lemma}
\begin{proof}
  We apply the change of variables ${u \mapsto ug_{\alpha 0 0}}$ (a rotation around $z$)
  to the following expression,
\begin{align*}
  \int\limits_{u \in \SO{3}} f(u \nu) \overline{\rho_\ell(u^{-1})}\, du
  &= \int\limits_{u \in \SO{3}} f(ug_{\alpha 0 0} \nu) \overline{\rho_\ell(g_{-\alpha 0 0}u^{-1})}\, du \\
  &= \int\limits_{u \in \SO{3}} f(u \nu) \rho_\ell(g_{\alpha 0 0})\overline{\rho_\ell(u^{-1})}\, du,
\end{align*}
where we used that a rotation around $z$ does not move the north pole, $g_{\alpha 0 0}\nu = \nu$.
The left and right hand sides must be equal for all $\alpha$,
and $\rho_\ell^{mm}(g_{\alpha 0 0}) = e^{-i\alpha m}$ (\cref{h:eq:rholmm}),
so the rows of  $\int_{u \in \SO{3}} f(u \nu) \rho_\ell(u^{-1})\, du$
must be zero for all $m \neq 0$.
Only the matrix values $\overline{\rho_\ell^{0n}(u^{-1})}$ influence the nonzero row,
and $\overline{\rho_\ell^{0n}(u^{-1})} = \rho_\ell^{n0}(u)$ holds.
Using \cref{h:eq:wig2sph}, we obtain
\begin{align*}
  \int\limits_{u \in \SO{3}} f(u \nu) \overline{\rho_\ell^{0n}(u^{-1})}\, du
  &= \int\limits_{u \in \SO{3}} f(u \nu) \rho_\ell^{n0}(u)\, du \\
  &= \sqrt{\frac{4\pi}{2\ell+1}} \int\limits_{u \in \SO{3}} f(u \nu) \overline{Y_n^\ell(u\nu)}\, du \\
  &= \sqrt{\frac{4\pi}{2\ell+1}} \int\limits_{h \in \SO{2}} \int\limits_{x \in S^2} f(x) \overline{Y_n^\ell(x)}\, dx\, dh \\
  &= 2\pi\sqrt{\frac{4\pi}{2\ell+1}} \hat{f}_n^\ell,
\end{align*}
where \cref{h:thm:hmeasure} was used in the last passage.
\end{proof}

The spherical convolution is efficiently computed in the spectral domain.
\begin{proposition}[spherical convolution]
  The spherical convolution between $f,\, k \in L^2(S^2)$ as defined in \cref{h:eq:sphconveq}
  can be computed in the spectral domain via pointwise multiplication of spherical harmonics coefficients,
  \begin{align*}
    \widehat{f * k}_m^\ell = 2\pi\sqrt{\frac{4\pi}{2\ell+1}} \hat{f}_m^\ell\hat{k}_0^\ell.
  \end{align*}
\end{proposition}
\begin{proof}
  Now we replace $k$ in \cref{h:eq:sphconveq} by its spherical harmonics expansion
  \begin{align*}
    (f \star k)(x)
    &= \int\limits_{g \in \SO{3}} f(g \nu) k(g^{-1} x) \, dg \\
    &= \int\limits_{g \in \SO{3}} f(g \nu) \sum_\ell (\hat{k}^\ell)^\top Y^\ell(g^{-1} x) \, dg \\
    &= \int\limits_{g \in \SO{3}}  f(g \nu) \sum_\ell(\hat{k}^\ell)^\top \overline{\rho_\ell(g^{-1})}Y^\ell(x) \, dg \\
    &= \sum_\ell (\hat{k}^\ell)^\top \left(\int\limits_{g \in \SO{3}}  f(g \nu) \overline{\rho_\ell(g^{-1})} \, dg \right) Y^\ell(x)
  \end{align*}
  Applying \cref{h:lemma:dh} to the integral within parenthesis, we obtain a
  matrix which has a single nonzero row corresponding to $m=0$,
  so only the $\hat{k}^\ell$ element corresponding to $m=0$ will influence the result.
  We write,
  \begin{align*}
    (f \star k)(x)
    &= \sum_\ell 2\pi\sqrt{\frac{4\pi}{2\ell+1}} \hat{k}_0^\ell (\hat{f}^\ell)^\top Y^\ell(x),
  \end{align*}
  which is the expansion in spherical harmonics of $(f \star k)(x)$.
  The relation
  \[\widehat{f * k}_m^\ell = 2\pi\sqrt{\frac{4\pi}{2\ell+1}} \hat{f}_m^\ell\hat{k}_0^\ell\]
 follows immediately.
\end{proof}
\begin{remark}
  Observe that only the coefficients $\hat{k}_0^\ell$ appear in the expression,
  which corresponds the coefficients of a zonal spherical function.
  This implies that for any $k$, there is always a zonal function $k'$ such that
  $f * k = f * k'$.
\end{remark}
\begin{remark}
The spherical convolution as described here is equivalent to the Funk-Hecke formula,
which can be extended to $S^n$; refer to \textcite{gallierdiffgeom} for details.
\end{remark}
In \cref{sph:sec:sphcnn} we present a spherical CNN where
the spherical convolution is the main operation and all inputs,
features and outputs are functions on the sphere.
The efficient computation for sampled functions
on the sphere relies on the sampling theorem as shown by \textcite{driscoll1994computing}.

\glsresetall

\def \i {\item}
\def \bi {\begin{itemize}\item}
\def \ei {\end{itemize}}
\def \be {\begin{equation*}}
\def \ee {\end{equation*}}

\newcommand{\bfmath}[1]{\mbox{\boldmath $#1$}}
\newcommand{\skewsym}[1]{[#1]_{\times}}

\newcommand{\quat}{\bfmath{{q}}}
\newcommand{\cquat}{\bfmath{\bar{q}}}
\newcommand{\xvec}{\bfmath{\vec{x}}}
\newcommand{\qvec}{\bfmath{\vec{q}}}
\newcommand{\avec}{\bfmath{\vec{a}}}
\newcommand{\bvec}{\bfmath{\vec{b}}}
\newcommand{\pvec}{\bfmath{\vec{p}}}
\newcommand{\qsc}{{q}_{0}}

\newcommand{\ldirvec}{\bfmath{\vec{l}}}
\newcommand{\lmomvec}{\bfmath{\vec{m}}}

\newcommand{\Rot}{\bfmath{R}}
\newcommand{\tv}{\bfmath{\vec{t}}}
\newcommand{\pv}{\bfmath{\vec{p}}}
\newcommand{\pquat}{\bfmath{{p}}}
\newcommand{\xquat}{\bfmath{{x}}}
\newcommand{\aquat}{\bfmath{{a}}}
\newcommand{\bquat}{\bfmath{{b}}}

\newcommand{\ldir}{{{\bfmath{l}}}}
\newcommand{\lmom}{{\bfmath{m}}}
\newcommand{\tq}{\bfmath{t}}

\newcommand{\e}{\epsilon}
\newcommand{\duq}{\bfmath{\check{q}}}
\newcommand{\cduq}{\bfmath{{\bar{\check{q}}}}}
\newcommand{\dua}{\bfmath{\check{a}}}
\newcommand{\cdua}{\bfmath{{\bar{\check{a}}}}}
\newcommand{\dub}{\bfmath{\check{b}}}
\newcommand{\cdub}{\bfmath{{\bar{\check{b}}}}}

\newcommand{\unk}{\bfmath{{q}}_v}

\newcommand{\quq}{\bfmath{{q}}}
\newcommand{\qua}{\bfmath{{a}}}
\newcommand{\qub}{\bfmath{{b}}}

\newcommand{\cquq}{\bfmath{\bar{q}}}
\newcommand{\cqua}{\bfmath{\bar{a}}}
\newcommand{\cqub}{\bfmath{\bar{b}}}

\newcommand{\inn}[2]{#1^{T}#2} 

\newcommand{\davec}{\bfmath{\check{\vec{a}}}}
\newcommand{\dbvec}{\bfmath{\check{\vec{b}}}}

\newcommand{\qmat}{\bfmath{Q}}
\newcommand{\wmat}{\bfmath{W}}

\newcommand{\Smat}{\bfmath{S}}
\newcommand{\Tmat}{\bfmath{T}}
\newcommand{\Umat}{\bfmath{U}}
\newcommand{\Vmat}{\bfmath{V}}
\newcommand{\Sigmat}{\bfmath{\Sigma}}

\newcommand{\veight}{\bfmath{\vec{v}_{8}}}
\newcommand{\vseven}{\bfmath{\vec{v}_{7}}}
\newcommand{\vvec}{\bfmath{\vec{v}}}
\newcommand{\uvec}{\bfmath{\vec{u}}}

\newcommand{\Ifour}{\bfmath{I}_{4\times4}}
\newcommand{\Othreethree}{\bfmath{0}_{3\times3}}
\newcommand{\othreeone}{\bfmath{0}_{3\times1}}
\newcommand{\ofourone}{\bfmath{0}_{4\times1}}

\newcommand{\Amat}{\bfmath{A}}
\newcommand{\Bmat}{\bfmath{B}}
\newcommand{\Cmat}{\bfmath{C}}
\newcommand{\Imat}{\bfmath{I}}
\newcommand{\Mmat}{\bfmath{M}}
\newcommand{\Nmat}{\bfmath{N}}
\newcommand{\Rmat}{\bfmath{R}}
\newcommand{\Xmat}{\bfmath{X}}
\newcommand{\Ymat}{\bfmath{Y}}
\newcommand{\nv}{\bfmath{\vec{n}}}
\newcommand{\xv}{\bfmath{\vec{x}}}

\newcommand{\pitwo}{\mathbb{P}^2}
\newcommand{\pithree}{\mathbb{P}^3}
\newcommand{\rone}{\mathbb{R}}
\newcommand{\rotdil}{SO(2)\times\rone^+}
\newcommand{\dilrot}{\ensuremath{\SO{2}\times\rone^+}}

\newcommand{\rtwo}{\mathbb{R}^2}
\newcommand{\two}{\mathbb{Z}^2}
\newcommand{\rthree}{\mathbb{R}^3}
\newcommand{\rthreeo}{\mathbb{R}^3 \setminus \{ (0,0,0) \}}
\newcommand{\rfouro}{\mathbb{R}^4 \setminus \{ (0,0,0,0) \}}
\newcommand{\uvw}{\left(\begin{array}{c} u \\ v \\ w\end{array}\right)}
\newcommand{\uvwp}{\left(\begin{array}{c} u' \\ v' \\ w'\end{array}\right) }

\newcommand{\starg}{\star_G}
\newcommand{\redc}[1]{{\color{red} #1}}

\chapter{Equivariance to planar similarities}
\label{ptn:sec:ptn}
\chaptersubtitle{The Polar Transformer Networks}
\section{Introduction}
Whether at the global pattern or local feature level \cite{granlund1978search},
the quest for invariant and equivariant representations
is as old as the fields of computer vision and pattern recognition.

The state of the art in ``hand-crafted'' approaches is typified by \sift~\cite{lowe2004distinctive}.
These detector/descriptors identify the intrinsic scale and rotation of a region \cite{lindeberg1994scale,chomat2000local} and produce an equivariant descriptor,
then normalized for scale and rotation invariance.
More recently, \textcite{sifre2013rotation} proposed the scattering transform which offers
representations invariant to translation, scaling, and rotations.

The current consensus is representations should be learned not designed.
Equivariance to translations by convolution and invariance to local deformations by pooling are now
textbook material \cite[335]{goodfellow2016deep} but approaches to equivariance to more general
deformations are still maturing.

Most recent approaches with learned filters are equivariant only to a small subgroup of planar rotations~\cite{dieleman2015rotation,cohen2016group,marcos16_rotat_equiv_vector_field_networ,Zhou_2017_CVPR}.
\textcite{worrall2017harmonic} introduce a notable exception that is equivariant to continuous rotations,
but using constrained filters with limited expressivity.

In this chapter, we describe the polar transformer networks,
which are equivariant to continuous rotations and dilations
and have unconstrained learned filters.
We combine ideas of \stns\ and canonical coordinate representations~\cite{segman1992canonical}
to achieve invariance to translations and equivariance to rotations and dilations.
The three stage network learns to identify the object center
then transforms the input into log-polar coordinates (see \cref{ptn:fig:logpolar}).
In this coordinate system, planar convolutions correspond to
group convolutions in rotation and scale.
\Ptns\ produce an equivariant representation
without the challenging parameter regression of \stns.
We enlarge the notion of equivariance in \cnns\ beyond
harmonic networks \cite{worrall2017harmonic} and group convolutions \cite{cohen2016group}
by capturing both rotations and dilations of arbitrary precision.
The \ptns\ handle only global deformations, as do \stns.

\begin{figure}[htbp]
  \begin{center}
    \includegraphics[width=.5\linewidth]{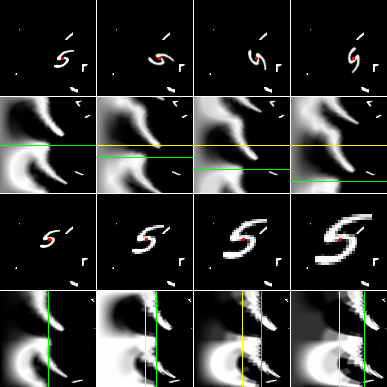}
  \end{center}
  \caption{In the log-polar representation, rotations around the origin become vertical shifts, and dilations around the origin become horizontal shifts. The distance between the yellow and green lines is proportional to the rotation angle/scale factor. Top rows: sequence of rotations, and the corresponding polar images. Bottom rows: sequence of dilations, and the corresponding polar images. }
  \label{ptn:fig:logpolar}
\end{figure}

We present state-of-the-art performance on rotated \mnist\ and \simtwomnist, which we introduce.
To summarize the contributions of this chapter:
\begin{itemize}
\item
  We develop a \cnn\ architecture capable of learning an image representation invariant to
  translation and equivariant to rotation and dilation.
\item
  We propose the polar transformer module, which performs a differentiable log-polar transform,
  amenable to backpropagation training. The transform origin is a latent variable.
\item
  We show how the polar transform origin can be learned effectively as the centroid of a single channel heatmap predicted by a fully convolutional network.
\end{itemize}

Most of the content in this chapter appeared originally in~\textcite{esteves2018polar}.
Source code is available at \url{https://github.com/daniilidis-group/polar-transformer-networks}.

\section{Related Work}

\textcite{nordberg1996equivariance} proposed one of the earliest equivariant feature extraction schemes,
suggesting the discrete sampling of $2$D rotations of a complex angle modulated filter.
About the same time, the image and optical processing community discovered the
Mellin transform as a modification of the Fourier transform \cite{zwicke1983new,casasent1976scale}.
The Fourier-Mellin transform is equivariant to rotation and scale while its modulus is invariant.

During the 80s and 90s, invariances of integral transforms were developed
through methods based in the Lie generators of the respective transforms
starting from one-parameter transforms~\cite{ferraro1988relationship}.
\textcite{segman1992canonical} generalized this to Abelian subgroups of the affine group
and proved that for certain classes of transformations there exist \emph{canonical coordinates}
where deformation of the input presents as translation of the output.

Closely related to equivariance is steerability, the interpolation of responses to any group action
using the response of a finite filter basis.
\textcite{freeman1991design} introduced an exact steerability framework,
where rotational steerability for Gaussian derivatives was explicitly computed;
\textcite{simoncelli1992shiftable} extended it to the shiftable pyramid,
which handle rotation and scale.
\textcite{perona1995deformable} proposed a method of approximating steerability by learning a lower dimensional representation of the image deformation from the transformation orbit and the \svd.

A unification of Lie generator and steerability approaches was introduced by \textcite{teo1998design} who used \svd\ to reduce the number of basis functions for a given transformation group.
Teo and Hel-Or developed the most extensive framework for steerability \cite{teo1998design,hel1998canonical},
and proposed the first approach for non-Abelian groups starting with exact steerability for
the largest Abelian subgroup and incrementally steering for the remaining subgroups.
\textcite{CohenW17,jacobsen17_dynam_steer_block_deep_resid_networ} recently combined steerability and learnable filters.

A recent ``hand-crafted'' approach to equivariant representations is the scattering transform \cite{sifre2013rotation} which composes rotated and dilated wavelets.
In a sense similar to \sift~\cite{lowe2004distinctive}, this approach relies on the equivariance of
anchor points (e.g. the maxima of filtered responses in space).
Translation invariance is through the modulus operation, computed after each convolution.
The final scattering coefficient is invariant to translations and equivariant to local rotations and scalings.

Within the context of \cnns, different methods of enforcing equivariance were attempted.
\paragraph{Constraining filters}
Equivariance can be obtained by constraining filter structure similarly to Lie generator
based approaches \cite{segman1992canonical,hel1998canonical}.
Harmonic Networks \cite{worrall2017harmonic} use filters derived from the complex harmonics achieving
both rotational and translational equivariance.

\paragraph{Input orbit}
\textcite{Laptev_2016_CVPR} achieve transformation invariance by pooling feature maps computed over
the input orbit, which scales poorly as it requires forward and backward passes for each orbit element.

\paragraph{Filter orbit}
A filter orbit which is itself equivariant can be used to obtain group equivariance.
\textcite{cohen2016group} convolve with the orbit of a learned filter and prove the equivariance of
group convolutions and preservation of rotational equivariance in the presence of rectification and pooling.
\textcite{dieleman2015rotation} process elements of the image orbit individually and use
the set of outputs for classification.
\textcite{gens2014deep} produce maps of finite-multiparameter groups,
\textcite{Zhou_2017_CVPR,marcos16_rotat_equiv_vector_field_networ} use a rotational filter
orbit to produce oriented feature maps and rotationally invariant features, and
\textcite{lenc2015understanding} propose a transformation layer which acts as a group convolution
by first permuting then transforming by a linear filter.

\paragraph{} \noindent
We achieve global rotational equivariance and expand the notion of \cnn\ equivariance to include scaling.
Our \ptns\ employ log-polar coordinates (canonical coordinates in \textcite{segman1992canonical}) to achieve
rotation-dilation group convolution through translational convolution subject to the assumption of an
image center estimated similarly to the \stns.
Most related to our method is \textcite{henriques2017warped}, which achieves equivariance by warping the inputs to a fixed grid, with no learned parameters.

When learning features from $3$D objects, invariance to transformations is usually achieved through
augmenting the training data with transformed versions of the inputs~\cite{wu20153d},
or pooling over transformed versions during training and/or test \cite{maturana2015voxnet,vam}.
\textcite{sedaghat16_orien_boost_voxel_nets_objec_recog} show that a multi-task approach, i.e.
prediction of both the orientation and class, improves classification performance.
In our extension to $3$D object classification, we explicitly learn representations equivariant to
rotations around a family of parallel axes by transforming the input
to cylindrical coordinates about a predicted axis.

\section{Theoretical Background}
This section is divided into two parts, the first is a review of equivariance
and group convolutions.
The second is an explicit example of the equivariance of group convolutions through the
$2$D similarity transformations group, \simtwo, comprised of translations, dilations and rotations.
Reparameterization of \simtwo\ to canonical coordinates allows for the application of
the \simtwo\ group convolution using translational convolution.

\subsection{Group Equivariance}
Equivariant representations are useful as they encode both semantic and deformation
information in a predictable way.
Let $G$ be a transformation group and $\lambda_g\im$ be the group action applied to
an input \fun{\im}{\Z^2}{\R^n}.
Recall that a mapping $\Phi:E\rightarrow F$ is equivariant to the group G when
for all $g\in G$,
\begin{equation}
\Phi(\lambda_g\im) = \lambda'_g(\Phi(\im))
\end{equation}
where $\lambda_g$ and $\lambda'_g$ correspond to application of $g$ to $E$ and $F$ respectively.
Invariance is the special case of equivariance where $\lambda'_g$ is the identity. In the context of image classification and \cnns,
the group actions can be thought of as image deformations
and $\Phi$ is a map from input image to a feature map or between feature maps.

The inherent translational equivariance of \cnns\ is independent of the convolutional kernel and
evident in the corresponding translation of the output in response to translation of the input. 
Let $f$ and $g$ be real-valued functions on $G$,
the group convolution is defined as in \cref{h:eq:gconv}  %
\begin{align*}
(f * k)(g) = \int\limits_{h \in G} f(h) k(h^{-1}g) \, dh.
\end{align*}

Group convolution requires integrability over a group and identification of the
Haar measure $dg$ as shown in \cref{h:haar},
and is equivariant as shown in \cref{h:sec:gconv}.

\subsection{Equivariance in SIM($2$)}
A similarity transformation, $\rho\in\simtwo$, acts on a point in $x\in\rtwo$ by
\begin{equation}
\rho x \mapsto s\,R\,x + t\quad s\in\rone^+,\, R\in \SO{2},\, t\in\mathbb{R}^2,
\end{equation}
where \SO{2} is the rotation group.
To take advantage of the standard planar convolution in classical \cnns\ we decompose
a $\rho\in\simtwo$ into a translation, $t$ in $\rtwo$ and
a dilated-rotation $r$ in \dilrot.

Equivariance to \simtwo\, is achieved by learning
the center of the dilated rotation, 
shifting the original image accordingly then transforming the image to canonical coordinates.
In this reparameterization, the standard translational convolution is equivalent to the dilated-rotation group convolution.

The origin predictor is an application of \stn~\cite{jaderberg15nips} to global translation prediction;
the centroid of the output is taken as the origin of the input.

The transformation of the image $\lambda_t \im = \im(t-t_o)$ (canonization in \textcite{soatto2013actionable}) reduces the $\simtwo$ deformation to a dilated-rotation
when $t_o$ is the true translation.
After centering, we wish to perform \dilrot\ convolutions on the new image $\im_o=\im(x-t_o)$.
As usual in \gcnns~\cite{cohen2016group}, the first layer lifts the input image to a feature map on the group,
\begin{equation}
f_j(r) = \int\limits_{x \in \rtwo} \im_o(x) k_j(r^{-1}x) \,\,dx
\end{equation}
and subsequent layers have inputs and outputs on the group
\begin{equation}
f_{j+1}(r) = \int\limits_{s \in \rotdil} f_j(s) k_{j+1}(s^{-1}r) \,\,ds
\end{equation}
where $r,s\in$ \dilrot.

We compute this convolution through the use of canonical coordinates for Abelian Lie groups \cite{segman1992canonical}.
The centered image $\im_o$ is transformed to log-polar coordinates,
\[f_p(\xi,\,\theta) = f_o( e^{\xi} \cos(\theta),\,e^{\xi} \sin(\theta)),\]
with $(\xi,\theta)\in$ \dilrot.
In canonical coordinates, for $s=(\xi,\, \theta)$ and $r=(\xi_r,\, \theta_r)$, we have
$s^{-1}r = (\xi_r -\xi,\,\theta_r-\theta)$ and the \dilrot\ group convolution
can be expressed and efficiently implemented as a planar convolution
\begin{equation}
\int\limits_{s } f(s) k(s^{-1}r) \,\,ds
= \int\limits_{s } f_p(\xi,\theta) k(\xi_r -\xi,\theta_r-\theta)
\,\, d\xi d\theta.
\end{equation}

To summarize, we
(1) construct a network of translational convolutions,
(2) take the centroid of the last layer and shift the original image to it,
(3) convert to log-polar coordinates, and
(4) apply a second network\footnote{the network employs rectifier and pooling which preserve equivariance \cite{cohen2016group}.} of translational convolutions.
The result is a feature map invariant to translation and
equivariant to dilated-rotations around the centroid.

\section{Architecture}
Our model consists of two main components connected by the polar transformer module.
The first part is the polar origin predictor and the second is the classifier (a conventional fully convolutional network).
The building block of the network is a $3 \times 3\times K$ convolutional layer followed by batch normalization, a \relu\ and occasional subsampling through strided convolution.
We will refer to this building block simply as \textit{block}.
\Cref{ptn:fig:network} shows the architecture.

\begin{figure}[htbp]
  \begin{center}
    \includegraphics[width=\linewidth]{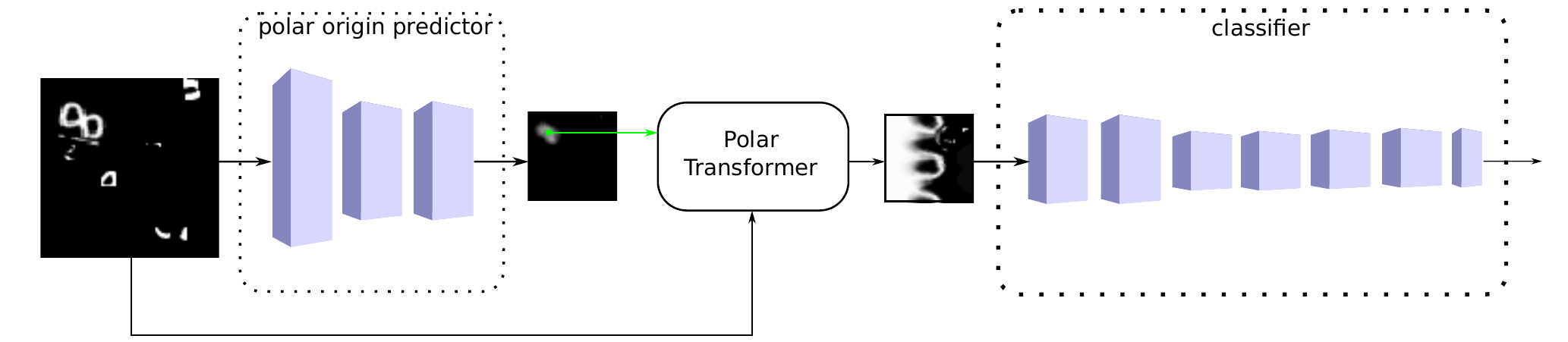}
  \end{center}
  \caption{Network architecture.
    The input image passes through a fully convolutional network, the polar origin predictor, which outputs a heatmap.
    The heatmap's centroid (two coordinates) and the input image go into the polar transformer module,
    which performs a polar transform with origin at the input coordinates.
    The output polar representation is invariant with respect to the original object location;
    rotations and dilations are now shifts,
    which are processed equivariantly by a conventional classifier CNN.}
\label{ptn:fig:network}
\end{figure}

\subsection{Polar Origin Predictor}
\label{ptn:sec:origin}
The polar origin predictor operates on the original image and comprises a sequence of blocks followed by a $1 \times 1$ convolution. The output is a single channel heatmap,
and we make its centroid the polar transform's origin.

There are some difficulties in training a neural network to predict coordinates in images.
Some approaches attempt to use fully connected layers to directly regress the coordinates with limited success~\cite{Toshev_2014_CVPR}.
A better option is to predict heatmaps \cite{NIPS2014_5573, newell2016stacked}, and take their argmax.
However, this is problematic since backpropagation gradients are zero in all but one point, impeding learning.

The usual approach to heatmap prediction is evaluation of a loss against some ground truth.
In this approach the argmax gradient problem is circumvented by supervision.
In our model, the gradient of the output coordinates must be taken with respect to the heatmap
since the polar origin is unknown and must be learned.
We avoid the argmax by taking the centroid of the heatmap as the polar origin.
The gradient of the centroid with respect to the heatmap is constant and nonzero for all points, making learning possible.

\subsection{Polar transformer module}
The polar transformer module takes the origin prediction and image as inputs and outputs the log-polar representation of the input.
The module uses the same differentiable image sampling technique
as the \stns~\cite{jaderberg15nips}, expressing output coordinates
in terms of input and source sample point coordinates $(x_i^s, y_i^s)$.
The log-polar transform in terms of the source sample points and target regular grid $(x_i^t, y_i^t)$ is:

\begin{align}
  x_i^s &=  x_0 + r^{{x_i^t}/{W}} \cos{\frac{2\pi y_i^t}{H}} \\
  y_i^s &= y_0 + r^{{x_i^t}/{W}} \sin{\frac{2\pi y_i^t}{H}}
\end{align}
where $(x_0,\, y_0)$ is the origin, $W,\,H$ are the output width and height, and $r$ is the maximum distance from the origin, set to $0.5\sqrt{H^2 + W^2}$ in our experiments.

\subsection{Wrap-around padding}
\label{ptn:sec:wrap}
To maintain feature map resolution, most CNN implementations use zero-padding.
This is not ideal for the polar representation, as it
is periodic about the angular axis.
A rotation of the input result in a vertical shift of the output, wrapping at the boundary;
hence, identification of the top and bottom most rows is most appropriate.
We achieve this with wrap-around padding on the vertical dimension,
where the top rows of the feature map are padded using the bottom rows and vice versa. %
The horizontal dimension is zero-padded as usual.
\Cref{ptn:tab:ablation} shows a performance evaluation with and without the proposed padding.

\subsection{Polar origin augmentation}
To improve robustness of our method, we augment the polar origin during training time by adding a random shift to the regressed polar origin coordinates.
Note that this comes for little computational cost compared to conventional augmentation methods such as rotating the input image.
\Cref{ptn:tab:ablation} quantifies the performance gains of this kind of augmentation.

\subsection{Relation to human vision}
The approach presented is loosely related to human vision.
The fovea is the central part of the retina and where
the photoreceptor cells are most densely packed,
resulting in the most accurate visual perception.
The peripheral region is coarsely populated, which explains the
lower accuracy of peripheral vision.
The same properties are observed in the log-polar grid that we adopt;
the central pixels are more densely packed than the peripheral ones.

Furthermore, humans exhibit a fixational eye movement,
where the eyes fixate on regions of interest to
leverage the accuracy of the foveal vision.
This behavior is analogous to the origin prediction described in \cref{ptn:sec:origin},
where we first detect a point of interest (the origin)
and then compute the log-polar transform around it,
which maximizes the resolution around the origin.
We refer to \textcite{kandel2013principles} for an introduction to the human visual system.

\section{Experiments}
We consider the image classification task on different datasets,
and compare with different models.
We first describe the models in \cref{ptn:sec:arch},
then the datasets in \cref{ptn:sec:dsets},
and the results follow.

\subsection{Architecture details}
\label{ptn:sec:arch}
We implement the following architectures for comparison,
\begin{itemize}
\item[] \textbf{Conventional CNN (\ccnn)}\quad a fully convolutional network, composed of a sequence of convolutional layers and some rounds of subsampling.
\item[] \textbf{Polar CNN (\pcnn)}\quad same architecture as \ccnn, operating on polar images.
  The log-polar transform is pre-computed at the image center before training, as in \textcite{henriques2017warped}.
  The fundamental difference between our method and this is that we learn the polar origin implicitly, instead of fixing it.
\item[] \textbf{Spatial Transformer Network (STN)}\quad our implementation of \textcite{jaderberg15nips}, replacing the localization network by four blocks of $20$ filters and stride $2$,
  followed by a $20$ unit fully connected layer, which we found to perform better.
  The transformation regressed is in $\simtwo$, and a \ccnn\ comes after the transform.
\item[] \textbf{Polar Transformer Network (PTN)}\quad our proposed method.
  The polar origin predictor comprises three blocks of $20$ filters each,
  with stride $2$ on the first block (or the first two blocks, when input is $96 \times 96$).
  The classification network is the \ccnn.
  \item[] \textbf{PTN-CNN}\quad we classify based on the sum of the per class scores of instances of PTN and \ccnn\ trained independently.
\end{itemize}

The following suffixes qualify the architectures described above:
\begin{itemize}
\item[] \textbf{S}\quad ``small'' network, with seven blocks of $20$ filters and one round of subsampling (equivalent to the Z$2$CNN in \textcite{cohen2016group}).
\item[] \textbf{B}\quad ``big'' network, with $8$ blocks with the following number of filters:
  $16$, $16$, $32$, $32$, $32$, $64$, $64$, $64$.
  We apply subsampling by strided convolution whenever the number of filters increase.
  We add up to two extra blocks of $16$ filters with stride $2$ at the beginning to handle larger
  input resolutions (one for $42 \times 42$ and two for $96 \times 96$).
\item[] \textbf{+}\quad training time rotation augmentation by continuous angles.
\item[] \textbf{++}\quad training and test time rotation augmentation.
  We input $8$ rotated versions the query image and classify using the sum of the per class scores.
\end{itemize}

We perform rotation augmentation for polar-based methods.
In theory, the effect of input rotation is just a shift in the corresponding polar image, which should not affect the classifier \cnn.
In practice, interpolation and angle discretization effects result in slightly different polar images for rotated inputs, so even the polar-based methods benefit from this kind of augmentation.

\subsection{Dataset details}
\label{ptn:sec:dsets}
\begin{itemize}
\item[] \textbf{Rotated \mnist}\quad The rotated \mnist\ dataset~\cite{larochelle2007empirical} contains  $360^\circ$-rotated $28 \times 28$  images of handwritten digits.
  The training, validation and test sets are of sizes \num{10}{k}, \num{2}{k}, and \num{50}{k}, respectively.
\item[] \textbf{\mnistr}\quad we replicate it from \textcite{jaderberg15nips}.
  It has \num{60}{k} training and \num{10}{k} testing samples, where the digits of the original \mnist\ are rotated between $[-90^\circ,\, 90^\circ]$.
  It is also know as half-rotated \mnist~\cite{Laptev_2016_CVPR}.
\item[] \textbf{\mnistrts}\quad we replicate it from \textcite{jaderberg15nips}.
  It has \num{60}{k} training and \num{10}{k} testing samples, where the digits of the original
  \mnist\ are rotated between $[-45^\circ,\, 45^\circ]$,
  scaled between $0.7$ and $1.2$, and shifted within a $42 \times 42$ black canvas.
\item[] \textbf{\simtwomnist}\quad we introduce a more challenging dataset, based on \mnist\,
  perturbed by random transformations from $\simtwo$.
The images are $96 \times 96$, with $360^\circ$ rotations; the scale factors range from $1$ to $2.4$, and the digits can appear anywhere in the image.
The training, validation and test set have size \num{10}{k}, \num{5}{k}, and \num{50}{k}, respectively.
\Cref{ptn:fig:sim2mnist} shows samples from the dataset.
\end{itemize}

\begin{figure}[htbp]
  \begin{center}
    \includegraphics[width=\linewidth]{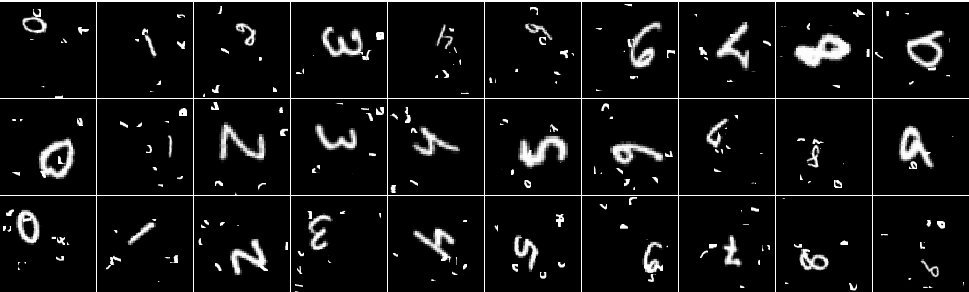}
  \end{center}
  \caption{\Simtwomnist\ samples.
    Large variations in digit scale, rotation and position,
    along with a small training set make this a challenging dataset.}
    \label{ptn:fig:sim2mnist}
\end{figure}

\subsection{Rotated \mnist~\cite{larochelle2007empirical}}
\Cref{ptn:tab:rot} shows the results.
We divide the analysis in two parts; (i) on top we show approaches with smaller networks and no
rotation augmentation, and (ii) on bottom there are no restrictions.

Between the restricted approaches, the harmonic networks~\cite{worrall2017harmonic} outperform
our model by a small margin,
but with almost four times longer training time because the convolutions on complex variables are more costly.
Also worth mentioning is the poor performance of the \stns~\cite{jaderberg15nips} with no augmentation,
which shows that learning the transformation parameters is much harder than learning the polar origin coordinates.

Between the unrestricted approaches, most variants of our model outperform the current
state of the art, with significant improvements when combined with \ccnn\ and/or test time augmentation.

Finally, we note that the \pcnn\ achieves a relatively high accuracy in this dataset because
the digits are mostly centered, so using the polar transform origin as the image center is reasonable.
Our method, however, outperforms it by a high margin, showing that even in this case,
it is possible to find an origin away from the image center that results in a more distinctive representation.

\begin{table}[htbp]
  \caption{Performance on rotated \mnist.
    Errors are averages of several runs, with standard deviations within parenthesis.
    Times are average training time per epoch.}
  \label{ptn:tab:rot}
  \centering
  \begin{tabular}{@{}lS[table-format=1.2(2)]S[table-format=4e3]S[table-format=2.2(2)]@{}}
    \toprule
    {Model}                                                  & {error [\%]}      & {params} & {time [s]}             \\
    \midrule
    PTN-S                                                    & \bgl 1.83 +- 0.04 & 27e3     & 3.64 +- 0.04           \\
    PCNN-S                                                   & 2.6 +- 0.08       & 22e3     & 2.61 +- 0.04           \\
    CCNN-S                                                   & 5.76 +- 0.35      & 22e3     & 2.43 +- 0.02           \\
    STN-S                                                    & 7.87 +- 0.18      & 43e3     & 3.90 +- 0.05           \\
    HNet \cite{worrall2017harmonic}                          & \bgd 1.69         & 33e3     & 13.29 +- 0.19          \\
    P4CNN \cite{cohen2016group}                              & 2.28              & 22e3     & {-}                    \\
    \midrule
    PTN-B+                                                   & 1.14 +- 0.08      & 129e3    & 4.38 +- 0.02           \\
    PTN-B++                                                  & \bgl 0.95 +- 0.09 & 129e3    & 4.38 \footref{ptn:fn:testtime} \\
    PTN-CNN-B+                                               & 1.01 +- 0.06      & 254e3    & 7.36                   \\
    PTN-CNN-B++                                              & \bgd 0.89 +- 0.06 & 254e3    & 7.36 \footref{ptn:fn:testtime} \\
    PCNN-B+                                                  & 1.37 +- 0.01      & 124e3    & 3.30 +- 0.04           \\
    CCNN-B+                                                  & 1.53 +- 0.07      & 124e3    & 2.98 +- 0.02           \\
    STN-B+                                                   & 1.31 +- 0.05      & 146e3    & 4.57 +- 0.04           \\
    OR-TIPooling \cite{Zhou_2017_CVPR}                       & 1.54              & 1000e3   & {-}                    \\
    TI-Pooling \cite{Laptev_2016_CVPR}                       & 1.2               & 1000e3   & 42.90                  \\
    RotEqNet \cite{marcos16_rotat_equiv_vector_field_networ} & 1.01              & 100e3    & {-}                    \\
    \bottomrule
  \end{tabular}
\end{table}

\subsection{Other \mnist\ variants}
We also perform experiments with other \mnist\ variants.
\mnistr\ and \mnistrts\ are replicated from \cite{jaderberg15nips}.
We introduce \simtwomnist, with a more challenging set of transformations from \simtwo.
See \cref{ptn:sec:dsets} for more details about the datasets.

\cref{ptn:tab:sim2} shows the results. We can see that the our model's performance matches
the \stns\ on both \mnistr\ and \mnistrts.
The deformations on these datasets are mild and training data is plenty,
so the performance may be saturated.

\stepcounter{footnote}
\footnotetext{Test time performance is 8x slower when using test time augmentation. \label{ptn:fn:testtime}}

On \simtwomnist, though, the deformations are more challenging and the training set five times smaller.
The \pcnn\ performance is significantly lower, which reiterates the importance of predicting the best polar origin.
\textcite{worrall2017harmonic} outperform the other methods (except our \ptn),
thanks to its translation and rotation equivariance properties.
Our method is more efficient both in number of parameters and training time, and is also equivariant to dilations, achieving the best performance by a large margin.

\begin{table}[htbp]
  \caption{Performance on \mnist\ variants.}
  \label{ptn:tab:sim2}
  \centering
      \scriptsize
      \begin{tabular}{@{} l
        S[table-format=1.2(2)]S[table-format=4]S[table-format=2.1, table-figures-decimal=1,table-auto-round] c
        S[table-format=1.2(2)]S[table-format=3]S[table-format=2.1, table-figures-decimal=1,table-auto-round] c
        S[table-format=2.2(3)]S[table-format=3]S[table-format=2.1, table-figures-decimal=1,table-auto-round] @{}}
        \toprule
                                                               & \multicolumn{3}{c}{\mnistr} &                 & \multicolumn{3}{c}{\mnistrts} &     & \multicolumn{3}{c}{\simtwomnist\footref{ptn:fn:noaug}}                                                                             \\
        \cmidrule{2-4} \cmidrule{6-8} \cmidrule{10-12}
                                                               & {error}                     & {par.}          & {\multirow{2}{*}{time}}         &     & {error}           & {par.}          & {\multirow{2}{*}{time}} &     & {error}           & {pars}          & {\multirow{2}{*}{time}} \\
        \addlinespace[-5pt]
                                                               & {[\%]}                      & {$\times 10^3$} &                               &     & {[\%]}            & {$\times 10^3$} &                       &     & {[\%]}            & {$\times 10^3$} &                       \\
        \midrule
        PTN-S+                                                 & 0.88 +- 0.04                & 29              & 19.72                         &     & 0.78 +- 0.05      & 32              & 24.48                 &     & \bgl 5.44 +- 0.03 & 35              & 11.92                 \\
        PTN-B+                                                 & \bgl 0.62 +- 0.04           & 129             & 20.37                         &     & 0.57 +- 0.03      & 134             & 28.74                 &     & \bgd 5.03 +- 0.11 & 134             & 12.02                 \\
        PCNN-B+                                                & 0.81 +- 0.04                & 124             & 13.97                         &     & 0.70 +- 0.01      & 129             & 17.19                 &     & 15.46 +- 0.22     & 129             & 5.33                  \\
        CCNN-B+                                                & 0.74 +- 0.01                & 124             & 12.79                         &     & 0.62 +- 0.07      & 129             & 15.97                 &     & 11.73 +- 0.57     & 129             & 5.28                  \\
        STN-B+                                                 & \bgd 0.61 +- 0.02           & 146             & 23.12                         &     & \bgl 0.54 +- 0.02 & 150             & 27.90                 &     & 12.35 +- 1.61     & 150             & 10.41                 \\
        STN \cite{jaderberg15nips}                             & 0.7                         & 400             & {-}                           &     & \bgd 0.5          & 400             & {-}                   &     & {-}               & {-}             & {-}                   \\
        HNet\footref{ptn:fn:modver} \cite{worrall2017harmonic} & {-}                         & {-}             &                               & {-} & {-}               & {-}             &                       & {-} & 9.28 +- 0.05      & 44              & 31.42                 \\
        TI-Pooling \cite{Laptev_2016_CVPR}                     & 0.8                         & 1000            & {-}                           &     & {-}               & {-}             & {-}                   &     & {-}               & {-}             & {-}                   \\
        \bottomrule
      \end{tabular}
\end{table}
\stepcounter{footnote}
\footnotetext{No augmentation is used with \simtwomnist, despite the + suffixes. \label{ptn:fn:noaug}}
\stepcounter{footnote}
\footnotetext{Our modified version, with two extra layers with subsampling to account for larger input. \label{ptn:fn:modver}}

\subsection{Ablation Study}
We quantify the performance boost obtained with wrap around padding, polar origin augmentation,
and training time rotation augmentation.
Results are with our PTN-B variant trained on Rotated \mnist.
We remove one operation at a time and verify that the performance consistently drops, which indicates that all operations are indeed helpful.
\Cref{ptn:tab:ablation} shows the results.

\begin{table}[htbp]
  \caption{Ablation study.
    Rotation and polar origin augmentation during training time, and wrap around padding all contribute to reduce the error.
    Results are from PTN-B on the rotated \mnist. }
  \label{ptn:tab:ablation}  
  \centering
  \begin{tabular}{cccS[table-format=1.2(2)]}
    \toprule
    Origin aug. & Rotation aug. & Wrap padding & {Error [\%]} \\
    \midrule
    Yes         & Yes           & Yes          & 1.12 +- 0.03 \\
    No          & Yes           & Yes          & 1.33 +- 0.12 \\
    Yes         & No            & Yes          & 1.46 +- 0.11 \\
    Yes         & Yes           & No           & 1.31 +- 0.06 \\
    \bottomrule
  \end{tabular}
\end{table}

\subsection{Visualization}
\begin{figure}[htbp]
  \begin{center}
    \includegraphics[width=\linewidth]{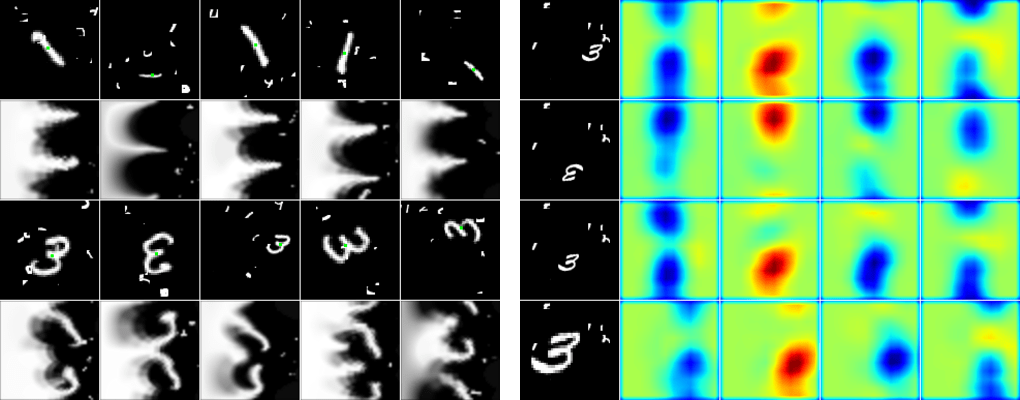}
  \end{center}
  \caption{
    \textbf{Left:} The rows alternate between samples from simtwomnist\, where the predicted origin is shown in green, and their learned polar representation.
    Note how rotations and dilations of the object become shifts.
    \textbf{Right:} Each row shows a different input and correspondent feature maps on the last convolutional layer.
    The first and second rows show that the $180^\circ$ rotation results in a half-height vertical shift of the feature maps.
    The third and fourth rows show that the $2.4 \times$ dilation results in a shift right of the feature maps.
    The first and third rows show invariance to translation.}
\label{ptn:fig:vis}
\end{figure}

We visualize network activations to confirm our claims about invariance to translation and equivariance to rotations and dilations.

\Cref{ptn:fig:vis} (left) shows some of the predicted polar origins and the results of the polar transform.
We can see that the network learns to reject clutter and find a suitable origin for the polar transform, and that the representation after the polar transformer module does present the properties claimed.

We proceed to visualize how the properties are preserved in deeper layers.
\Cref{ptn:fig:vis} (right) shows the activations of selected channels from the last convolutional layer, for different rotations, dilations, and translations of the input.
The reader can verify that the equivariance to rotations and dilations, and the invariance to translations are indeed preserved during the sequence of convolutional layers.

\subsection{Street-view house numbers (SVHN)}
In order to demonstrate the efficacy of our method on real-world \rgb\ images,
we run experiments on the \svhn\ dataset~\cite{netzer2011reading},
and a rotated version that we introduce (\rotsvhn).
The dataset contains cropped images of single digits,
as well as the slightly larger images from where the digits are cropped.
Using the latter, we can extract $360^\circ$ rotated digits without introducing artifacts.
\Cref{ptn:fig:rotsvhn} shows some examples from \rotsvhn.

We use a $32$ layer residual network~\cite{he2016deep} as a baseline (\resnet{32}).
The \ptnresnet{32} has $8$ residual convolutional layers as the origin predictor,
followed by a \resnet{32}.

In contrast with handwritten digits,
the sixes and nines in house numbers are usually indistinguishable.
To remove this effect from our analysis, we also run experiments removing those classes
from the datasets (which is indicated by appending a minus to the dataset name).
\Cref{ptn:tab:svhn} shows the results.

Note that rotations cause a significant performance loss with the conventional \resnet;
the error increases from \num{2.09}{\%} to \num{5.39}{\%},
even when removing sixes and nines from the dataset.
With our model, on the other hand, the error goes from \num{2.85}{\%} to \num{3.96}{\%},
which shows more robustness to the perturbations,
although the performance on the unperturbed datasets is slightly lower.
We expect the \ptns\ to be even more advantageous when large scale variations are also present.
\begin{figure}[htbp]
  \begin{center}
    \includegraphics[width=\linewidth]{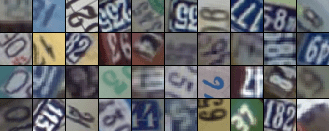}
  \end{center}
  \caption{\rotsvhn\ samples.
    Since the digits are cropped from larger images, no artifacts are introduced when rotating.
    The sixes and nines are indistinguishable when rotated.
    Note that there are usually visible digits on the sides, which pose a challenge for classification and PTN origin prediction.}
\label{ptn:fig:rotsvhn}
\end{figure}

\begin{table}[htbp]
  \caption{\svhn classification performance (error in \%).
    The minus suffix indicate removal of sixes and nines.
    PTN shows slightly worse performance on the unperturbed dataset, but is clearly superior when rotations are present.}
  \label{ptn:tab:svhn}
  \centering
  \begin{tabular}{@{} lSSSS @{}}
    \toprule
    & {\svhn} & {\rotsvhn} & {\svhn-} & {\rotsvhn-}\\
    \midrule
    \ptnresnet{32} (Ours) & 2.82 +- 0.07 &  7.90 +- 0.14 & 2.85 +- 0.07 &  3.96 +- 0.04 \\
    \resnet{32} &  2.25 +- 0.15 & 9.83 +- 0.29 &  2.09 +- 0.06 & 5.39 +- 0.09\\
    \bottomrule
  \end{tabular}
\end{table}

\subsection{Extension to 3D object classification}
We extend our model to perform $3$D object classification from voxel occupancy grids.
We assume inputs perturbed by random rotations around an axis from a family of parallel axes.
In this case, a rotation around that axis corresponds to a translation in cylindrical coordinates.

In order to achieve equivariance to rotations, we predict an axis and use it as the origin to transform to cylindrical coordinates.
If the axis is parallel to one of the input grid axes, the cylindrical transform amounts to channel-wise polar transforms, where the origin is the same for all channels and each channel is a $2$D slice of the $3$D voxel grid.
In this setting, we can just apply the polar transformer layer to each slice.

We use a technique similar to the anisotropic probing of \textcite{vam} to predict the axis.
Let $z$ denote the input grid axis parallel to the rotation axis.
We treat the dimension indexed by $z$ as channels, and run regular $2$D convolutional layers, reducing the number of channels on each layer, eventually collapsing to a single $2$D heatmap.
The heatmap centroid gives one point of the axis, and the direction is parallel to $z$.
In other words, the centroid is the origin of all channel-wise polar transforms.
We then proceed with a regular $3$D \cnn\ classifier, acting on the cylindrical representation.
The $3$D convolutions are equivariant to translations; since they act on cylindrical coordinates, the learned representation is equivariant to input rotations around axes parallel to $z$.

The axis prediction part of the cylindrical transformer network contains four $2$D blocks,
with $5\times 5$ kernels and $32$, $16$, $8$, and $4$ channels, no subsampling.
The classifier comprises eight $3$D convolutional blocks,
with $3\times 3 \times 3$ kernels,
the following number of filters: $32$, $32$, $32$, $64$, $64$, $64$, $128$, $128$,
and subsampling whenever the number of filters increase.
Total number of parameters is approximately \num{1}{M}.

We run experiments on ModelNet40~\cite{wu20153d}, which contains objects rotated around the gravity direction ($z$).
\Cref{ptn:fig:chair_bench_cyl} shows examples of input voxel grids and their cylindrical coordinates representation, while table~\ref{ptn:tab:res3d} shows the classification performance.
To the best of our knowledge, our method outperformed all published voxel-based methods,
even with no test time augmentation, at the time of the original submission~\cite{esteves2018polar}.
However, the multi-view based methods generally outperform the voxel-based~\cite{vam}.

Note that we could also achieve equivariance to scale by using log-cylindrical or log-spherical coordinates, but none of these change of coordinates would result in equivariance to arbitrary $3$D rotations.

\begin{figure}[htbp]
  \begin{center}
    \includegraphics[width=\linewidth]{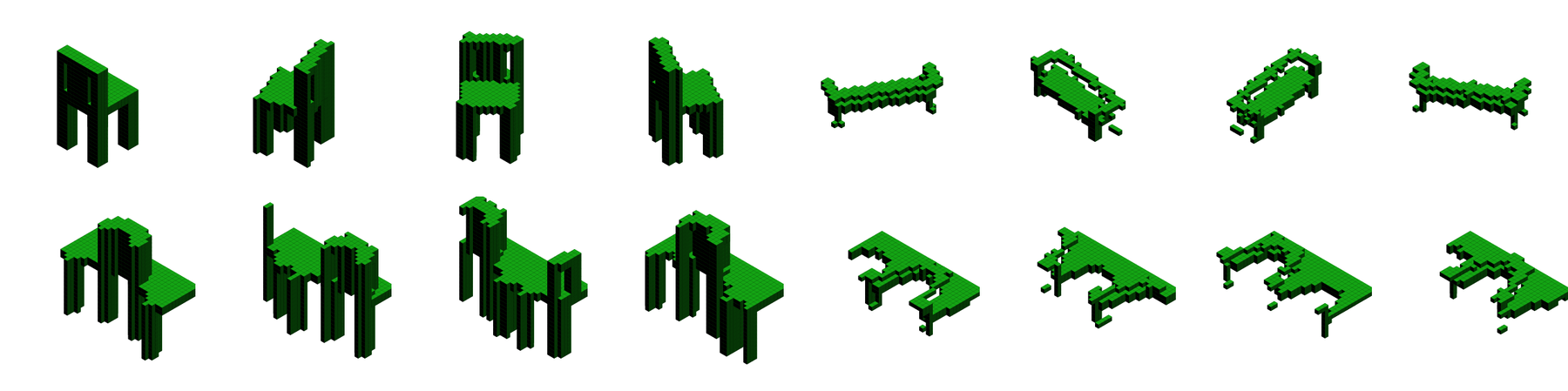}
  \end{center}
  \caption{Top: rotated voxel occupancy grids.
    Bottom: corresponding cylindrical representations.
    Note how rotations around a vertical axis correspond to translations over a horizontal axis.}
\label{ptn:fig:chair_bench_cyl}
\end{figure}

\begin{table}[htbp]
  \caption{ModelNet40 classification performance. We compare only with voxel-based methods.}
  \label{ptn:tab:res3d}
  \centering
  \begin{tabular}{@{} lSS @{}}
    \toprule
    Model                             & {Avg. class accuracy [\%]} & {Avg. instance accuracy [\%]} \\
    \midrule
    Cylindrical transformer (Ours)    & \bgd 86.5                  & 89.9                          \\
    $3$D ShapeNets \cite{wu20153d}   & 77.3                       & {-}                           \\
    VoxNet \cite{maturana2015voxnet} & 83                         & {-}                           \\
    MO-SubvolumeSup \cite{vam}       & \bgl 86.0                  & 89.2                          \\
    MO-Aniprobing \cite{vam}         & 85.6                       & 89.9                          \\
    \bottomrule
  \end{tabular}

\end{table}

\section{Conclusion}
In this chapter, we presented a novel network whose output is invariant to translations and
equivariant to the group of dilated rotations.
Similarly to the spatial transformers~\cite{jaderberg15nips}, we directly predict the translation,
though we also provide equivariance for scaling and rotation through a change of coordinates.
Our model avoids the commonly used fully connected layers for pose regression
by taking the centroid of a heatmap as the predicted transformation.
We formulate equivariance with to dilated rotations as a group convolution,
which we compute by transforming the inputs to canonical coordinates.
Our results improve the state-of-the-art performance on Rotated \mnist\ by a large margin,
and outperform all other considered methods on a new dataset we call \simtwomnist.
We expect our approach to be applicable to other problems, where the presence of different orientations and scales hinder the performance of conventional \cnns.

\glsresetall
\chapter{Equivariance to icosahedral symmetries}
\label{emvn:sec:emvn}
\chaptersubtitle{The Equivariant Multi-View Networks}

\section{Introduction}
The proliferation of large scale $3$D datasets of objects \cite{wu20153d,shapenet2015} and whole
scenes \cite{chang17_matter,dai17_scann} enables training of deep learning models that  produce global
descriptors suitable for classification and retrieval tasks.

The first challenge that arises is how to represent the inputs.
Despite numerous attempts with volumetric \cite{wu20153d,maturana2015voxnet}, point-cloud
\cite{qi2017pointnet,simonovsky2017dynamic} and mesh-based \cite{masci2015geodesic,monti2017geometric}
representations, using multiple views of the $3$D input allows switching to the $2$D domain where all
the recent image based deep learning breakthroughs (e.g., \textcite{he2016deep}) can be directly applied,
resulting in state-of-the-art performance~\cite{su2015multi,kanezaki16_rotat}.

\Mv\ based methods require some form of view-pooling, which can be
(i) pixel-wise pooling over some intermediate convolutional layer~\cite{su2015multi},
(ii) pooling over the final $1$D view descriptors~\cite{su18_deeper_look_at_shape_class}, or
(iii) combining the final logits~\cite{kanezaki16_rotat}, which can be seen as independent voting.
These operations are usually invariant to view permutations.

Our key observation in this chapter is that the conventional view pooling
occurs before any joint processing of the set of views and
will inevitably discard useful features, leading to subpar descriptors.
We solve the problem by first realizing that each view can be associated with an element of the
rotation group \SO{3}, so the natural way to combine multiple views is as a function on the group.
A traditional \cnn\ produces view descriptors that compose this function.
We design a group-convolutional network (\acrshort{G-CNN}, inspired by \textcite{cohen2016group})
to learn representations that are equivariant to transformations from the group.
This differs from the invariant representations obtained through usual view-pooling that discards  information.
We obtain invariant descriptors useful for classification and retrieval by pooling over the last \gcnn\ layer.
Our \gcnn\ has filters with localized support on the group and learns hierarchically more complex
representations as we stack more layers and increase the receptive field.

We take advantage of the finite nature of multiple views and
consider finite rotation groups like the icosahedral,
in contrast with~\textcite{s.2018spherical,esteves18eccv}
(described in \cref{sph:sec:sphcnn})
which operate on the continuous group.
To reduce the computational cost of processing one view per group element,
we greatly reduce the number of required views by transforming views to canonical coordinates with
respect to the group of in-plane dilated rotations (log-polar coordinates).
This yields an initial representation on a \hspc\ of the group,
which is lifted to a function on the group via cross-correlation, while maintaining equivariance.

We focus on $3$D shapes but our model is applicable to any task where multiple views can represent the input, as demonstrated by an experiment on panoramic scenes.

\begin{figure}[htbp]
  \centering
  \includegraphics[width=\textwidth]{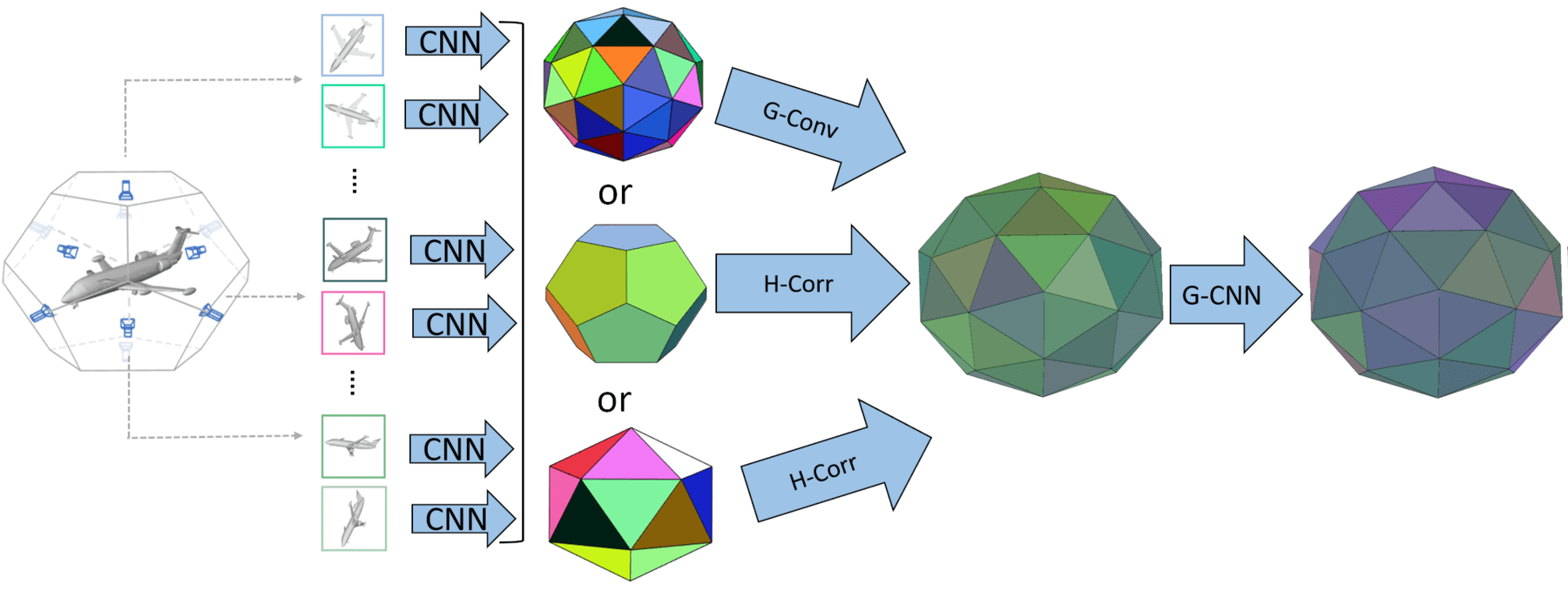}
  \caption{Our \acrfullpl{EMVN} aggregate multiple views as functions on rotation groups,
    then processed with group convolutions.
    This ensures equivariance to $3$D rotations and
    jointly reasoning over all views, leading to superior shape descriptors.
    We show functions on the icosahedral group and \acrfull{hspace}
    on appropriate solids.
    Each view is first processed by a \cnn\ and resulting descriptors are associated with a group (or \hspc) element.
    When views are identified with an \hspc, the first operation is a cross-correlation that lifts  features to the group.
    Once we have an initial representation on the group, a group \cnn\ is applied.}
    \label{emvn:structure}
\end{figure}

\Cref{emvn:structure} illustrates our model.
The contributions of this chapter are:
\begin{itemize}
\item We introduce a novel method of aggregating multiple views whether ``outside-in'' for $3$D shapes or ``inside-out'' for panoramic views.
  Our model exploits the underlying group structure, resulting in equivariant features that are functions on the rotation group.
\item We introduce a way to reduce the number of views while maintaining equivariance, via a transformation to canonical coordinates of in-plane rotations followed by homogeneous space cross-correlation.
\item We explore the finite rotation groups and homogeneous spaces and present a discrete \gcnn\ model on the largest group to date, the icosahedral group.
  We further explore the concept of filter localization for this group.
\item We achieve state of the art performance on multiple shape retrieval benchmarks, both in canonical poses and perturbed with rotations, and show applications to panoramic scene classification.
\end{itemize}

Most of the content in this chapter appeared originally in~\textcite{Esteves_2019_ICCV}.
Source code is available at \url{https://github.com/daniilidis-group/emvn}.
\section{Related work}

\paragraph{$3$D shape analysis}
Performance of $3$D shape analysis is heavily dependent on the input representation.
The main representations are volumetric, point cloud and multi-view.

Early examples of volumetric approaches are in~\textcite{shapenet2015}, who introduced the ModelNet
dataset and trained a $3$D shape classifier using a deep belief network on voxel representations,
and \textcite{maturana2015voxnet}, who present a standard architecture with $3$D convolutional layers followed by fully connected layers.

\textcite{su2015multi} realized that by rendering multiple views of the $3$D input one can transfer the power of image-based \cnns\ to $3$D tasks.
They show that a conventional \cnn\ can outperform the volumetric methods even using only a single view of the input, while an \acrfull{MV} model further improves the classification accuracy.

\textcite{vam} study volumetric and multi-view methods and propose improvements to both;
\textcite{kanezaki16_rotat} introduce an \mv\ approach that achieves state-of-the-art classification
performance by jointly predicting class and pose, though without explicit pose supervision.

\textcite{gvcnn} learns how to combine different view descriptors to obtain a
view-group-shape representation; they refer to arbitrary combinations of features as ``groups''.
This differs from our usage of the term  ``group'' which is the algebraic definition.

Point-cloud based methods~\cite{qi2017pointnet} achieve intermediate performance between volumetric
and multi-view, but are much more efficient computationally.
While meshes are arguably the most natural representation and widely used in computer graphics, only
limited success has been achieved with learning models operating directly on them
\cite{masci2015geodesic,monti2017geometric}.

In order to better compare $3$D shape descriptors, we will focus on the retrieval performance.
Recent approaches show significant improvements on retrieval: \textcite{you2018pvnet} combine point cloud and \mv\ representations,
\textcite{yavartanoo18_spnet} introduce multi-view stereographic projection, and
\textcite{han2019seqviews2seqlabels} implement a recurrent \mv\ approach.

We also consider more challenging tasks on rotated ModelNet and the SHREC'17~\cite{savva2017shrec}
large scale retrieval challenge, which contains rotated shapes.
The presence of arbitrary rotations motivates the use of equivariant representations.
\paragraph{Equivariant representations}
A number of workarounds have been introduced to deal with $3$D shapes in arbitrary orientations.
Typical examples are training time rotation augmentation and/or test time voting~\cite{vam} and
learning an initial rotation to a canonical pose~\cite{qi2017pointnet}.
The view-pooling in \textcite{su2015multi} is invariant to permutations of the set of input views.

A principled way to handle rotations is to use representations that are equivariant by design.
There are mainly three ways to embed equivariance into \cnns.
The first way is to constrain the filter structure, which is similar to Lie generator based approach ~\cite{segman1992canonical,hel1998canonical}.
\textcite{worrall2017harmonic} take advantage of circular harmonics to have both translational and
$2$D rotational equivariance in \cnns.
\textcite{tensor} extends this idea introducing a tensor field to keep translational and rotational
equivariance for $3$D point clouds,
while \textcite{weiler3dsteerable} does the same for voxel grids.

The second way is through a change of coordinates;
\textcite{henriques2017warped,esteves2018polar}%
\footnote{\textcite{esteves2018polar} is also described in \cref{ptn:sec:ptn}}
take the log-polar transform of the input and transfer rotational and scaling equivariance about a
single point to translational equivariance.

The third way is to make use of an equivariant filter orbit.
\textcite{cohen2016group} proposed the \gcnns\ with the square cyclic rotation group,
later extended to the hexagon~\cite{hex}.
\textcite{worrall2018cubenet} proposed CubeNet using Klein's four-group on $3$D voxelized data.
\textcite{cohencube} implement $3$D group convolution on the octahedral symmetry group
for volumetric CT images.
\textcite{CohenWKW19} recently considered functions on the icosahedron,
however their convolutions are on the cyclic group and not on the icosahedral as ours.
\textcite{s.2018spherical,esteves18eccv} focus on the continuous group \SO{3},
and use the spherical harmonic transform
for exact implementation of spherical convolution or correlation.
The main issue with both approaches is that the input spherical representation
is lossy and does not capture the complexity of an object's shape;
they are also less efficient and face bandwidth challenges.

\section{Preliminaries}
Our goal is to leverage symmetries in data.
A {\em symmetry} is an operation that preserves some structure of an object.
If the object is a discrete set with no additional structure, each operation can be seen as a permutation of its elements.

The term \emph{group} is used in its classic algebraic definition of a set with an operation satisfying the closure, associativity, identity, and inversion properties (\cref{h:def:group}).
A transformation group like a permutation is the ``missing link between abstract group and the notion of symmetry'' ~\cite{miller72}.

We refer to \emph{view} as an image taken from an oriented camera.
This differs from \emph{viewpoint} that refers to the optical axis direction, either \emph{outside-in} for a moving camera pointing at a fixed object, or \emph{inside-out} for a fixed camera pointing at different directions.
Multiple \emph{views} can be taken from the same \emph{viewpoint}; they are related by in-plane rotations.

\paragraph{Equivariance}
Representations that are equivariant by design are an effective way to exploit symmetries.
Recall the definition of equivariance for a map \fun{\Phi}{X}{Y}, group $G$,
and left group actions $\tau_g$ and $\tau'_g$ on the sets $X$ and $Y$, respectively.
We say that $\Phi$ is equivariant to $G$ if for any $g \in G$ and $f \in X$,
\[\Phi(\lambda_g(f))=\lambda'_{g}(\Phi(f)). \]
In the context of \cnns, $X$ and $Y$ are sets of input and feature representations, respectively.
This definition encompasses the case when $\lambda'$ is the identity,
making $\Phi$ invariant to $G$ and discarding information about $g$.
In this chapter, we are interested in non-degenerate cases that preserve information.

\paragraph{Convolution on groups}
We represent multiple views as a functions on a group and seek equivariance to the group,
so \gconv\ is the natural operation for our method.
Recall the planar convolution between \fun{f,\,k}{\R^2}{\R}, which is the main operation of \cnns:
\begin{equation*}
  (f * k)(y) = \int\limits_{x\in\R^2} f(x)k(y-x)\,dx.
\end{equation*}
We can interpret this convolution as an operation over the group of translations on the plane,
where the group action is addition of coordinate values;
the convolution is equivariant to translation.

Convolution generalizes to any group $G$ under mild conditions related to the
Haar measure as described in \cref{h:haar}.
We define the group convolution between two functions on the group \fun{f,\,k}{G}{\R} as
\begin{equation}
  (f * k)(y) = \int\limits_{g\in G} f(g)k(g^{-1}y)\,dg, \label{emvn:eq:gconv}
\end{equation}
which is equivariant to group actions from $G$ as shown in \cref{h:sec:gconv}.

\paragraph{Convolution on homogeneous spaces}
\label{emvn:sec:homogeneous}
For efficiency, we may relax the requirement of one view per group element and consider only one view per element of a homogeneous space of lower cardinality.
For example, we can represent the input on the 12 vertices of the icosahedron (a \hspc),
instead of on the 60 rotations of the icosahedral group.

A homogeneous space $X$ of a group $G$ is defined as a space where $G$ acts transitively: for any $x_1,x_2 \in X$, there exists $g\in G$ such that $x_2=gx_1$ (\cref{h:def:hspace}).

Two convolution-like operations can be defined between functions on homogeneous spaces \fun{f,\,h}{X}{\R}:
\begin{align}
  (f * h)(y) &= \int\limits_{g\in G} f(g\nu)h(g^{-1}y)\,dg, \label{emvn:eq:hconv} \\
  (f \star h)(g) &= \int\limits_{x\in X} f(gx)h(x)\,dx, \label{emvn:eq:hcorr}
\end{align}
where $\nu \in X$ is an arbitrary canonical element.
We call \cref{emvn:eq:hconv} \acrfull{hconv},
and \cref{emvn:eq:hcorr} \acrfull{hcorr}.
The integrals in the continuous case depend on the Haar measure and
its induced measure on homogeneous spaces, as shown in \cref{h:haar}.
Note that convolution produces a function on the homogeneous space $X$
while correlation lifts the output to the group $G$.
Both operations are equivariant.
For \acrshort{hconv} (\cref{emvn:eq:hconv}), where \fun{f,\,h}{X}{\R}, we have:
\begin{align*}
  (\lambda_{k}f* h)(y)
    &= \int\limits_{u\in U}f(k^{-1}u\nu)h(u^{-1}y)\,du\\
    &= \int\limits_{w\in G}f(w\nu)h((kw)^{-1}y)\,dw && (u \mapsto kw)\\
    &= \int\limits_{w\in G}f(w\nu)h(w^{-1}k^{-1}y)\,dw\\
    &=(f*h)(k^{-1}y)\\
    &=\lambda_{k}(f*h)(y).
\end{align*}              %
For \acrshort{hcorr} (\cref{emvn:eq:hcorr}), where \fun{f,\,h}{X}{\R}, we have:
\begin{align*}
(\lambda_{k}f \star h)(g)
&=\int\limits_{x\in X}f(k^{-1}gx)h(x)\,dx\\
&=(f \star h)(k^{-1}g)\\
&=\lambda'_{k}(f \star h)(y).
\end{align*}
In this case, $\lambda'_{k}$ is not necessarily equal $\lambda_{k}$ because inputs and outputs may be in different spaces.

We refer to~\textcite{kondor18icml,cohen2019general}
for more detailed expositions on group and homogeneous space convolution
in the context of neural networks.

\paragraph{Finite rotation groups}
\label{emvn:sec:rotgroups}
Our representation is a finite set of views identified with a group of rotations,
so we consider finite subgroups of the rotation group \SO{3}.
A finite subgroup of \SO{3} can be the cyclic group $\mathcal{C}_k$ of multiples of $2\pi/k$, the dihedral group $\mathcal{D}_k$ of symmetries of a regular $k$-gon, the tetrahedral, octahedral, or  icosahedral group~\cite{artin}.

Our main results are on the icosahedral group $\I$, the 60-element non-abelian group of symmetries of the icosahedron (illustrated in \cref{emvn:rotation60,emvn:Caylay2}).
The symmetries can be divided in sets of rotations around a few axes.
For example, there are five rotations around each axis passing through vertices of the icosahedron or three rotations around each axis passing through its faces centers.

\Cref{emvn:rotation60} illustrates all elements of the group by their actions on one edge of the icosahedron,
while \cref{emvn:Caylay2} shows the Cayley table; the color assigned for each group element matches the color in \cref{emvn:fig:equi60}.
\begin{figure}[htbp]
\centering
  \includegraphics[width=\textwidth]{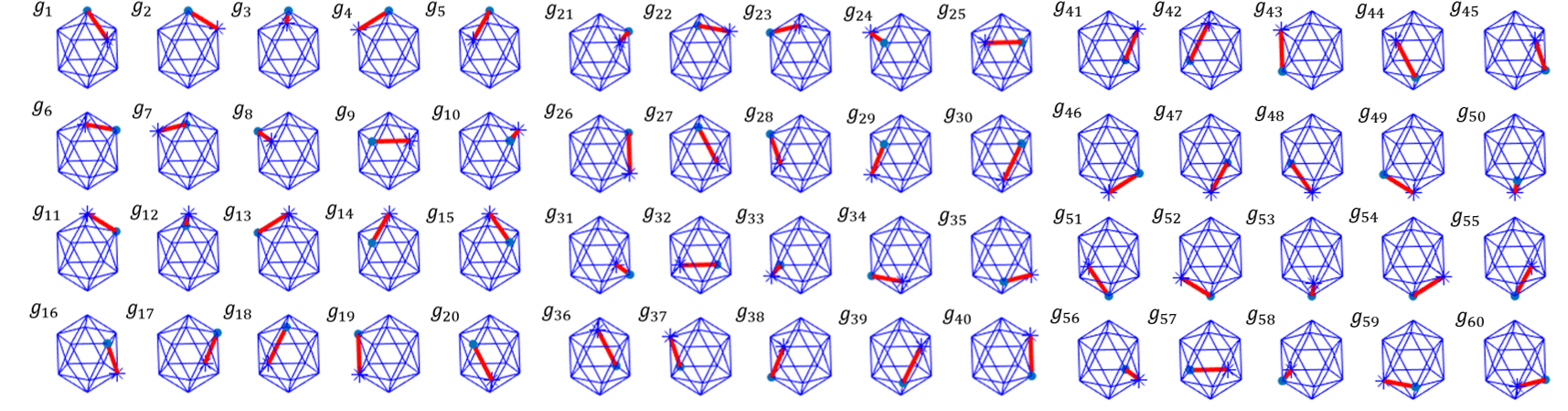}
  \caption{The 60 rotations of the icosahedral group $\I$.
    We consider $g_1$ the identity, highlight one edge, and show how each $g_i \in \I$ transforms the highlighted edge.}
  \label{emvn:rotation60}
\end{figure}
\begin{figure}[htbp]
  \centering
  \includegraphics[width=0.5\linewidth]{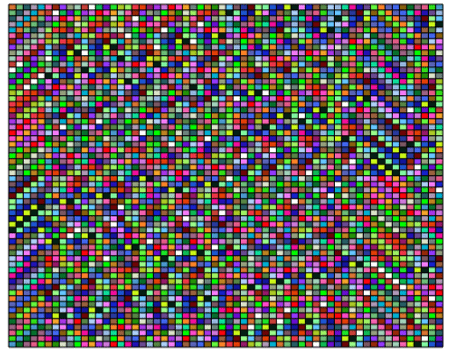}
  \caption{Cayley table for the icosahedral group $\I$.
    We can see that $\I$ is non-abelian, since the table is not symmetric.}
  \label{emvn:Caylay2}
\end{figure}

\paragraph{Equivariance via canonical coordinates}
\label{emvn:sec:canoncoords}
Some configurations (set of views identified with a discrete subgroup of \SO{3})
produce views that are related by in-plane rotations.
We leverage this to reduce the number of required views by obtaining rotation invariant view descriptors through a change to canonical coordinates followed by a \cnn.

\textcite{segman1992canonical} show that changing to a canonical coordinate system allows certain transformations of the input to appear as translations of the output.
For the group of dilated rotations on the plane (isomorphic to $\SO{2} \times \R^+$),
canonical coordinates are given by the log-polar transform.

Since planar convolutions are equivariant to translation,
converting an image to log-polar and applying a \cnn\ results
in features equivariant to dilated rotation,
which can be pooled to invariant descriptors on the last layer.
This is similar to \textcite{henriques2017warped} and
a simplified version of the model introduced in \cref{ptn:sec:ptn};
here we do not learn the transformation center.

\section{Method}
Our first step is to obtain $\abs{G}$ views of the input where each view $x_i$ is associated with a group element $g_i \in G$%
\footnote{Alternatively, we can use $\abs{X}$ views for a homogeneous space $X$ as shown in \cref{emvn:sec:fewerviews}.}.
Each view is input to a \cnn\ $\Phi_1$,
and we combine the $1$D descriptors extracted from the last layer
(before projection into the number of classes)
to form a function on the group \fun{y}{G}{\R^n}, where $y(g_i) = \Phi_1(x_i)$.
A \gcnn\ $\Phi_2$ operating on $G$ is then used to process $y$, and global average pooling on the last
layer yields an invariant descriptor useful for classification or retrieval.
Training is end-to-end.
\Cref{emvn:structure} shows the model.

\subsection{View configurations}
\label{emvn:sec:viewconf}
There are several possible view configurations of icosahedral symmetry,
consisting of vertices or faces of solids with this symmetry.
Two examples are associating viewpoints with faces/vertices of the icosahedron,
which are equivalent to the vertices/faces of its dual, the dodecahedron.
These configurations are based on platonic solids, which guarantee a uniform distribution of viewpoints.
By selecting viewpoints from the icosahedron faces, we obtain 20 sets of 3 views that differ only by \ang{120} in plane rotations; we refer to this configuration as $20\times 3$.
Similarly, using the dodecahedron faces we obtain the $12\times 5$ configuration.

In the context of $3$D shape analysis, multiple viewpoints are useful to handle self-occlusions and ambiguities.
Views that are related by in-plane rotations are redundant in this sense,
but necessary to keep the group structure.

To minimize redundancy, we propose to associate viewpoints with
the 60 vertices of the truncated icosahedron (which has icosahedral symmetry).
There is a single view per viewpoint in this configuration.
This is not a uniformly spaced distribution of viewpoints, but the variety is beneficial.
\Cref{emvn:fig:camera} shows some view configurations we considered.

\Cref{emvn:fig:equi60} shows that the map from $3$D object to
list of views determined by the icosahedral group is equivariant;
a rotation of the object incurs in a permutation of the list of views.

Note that our configurations differ from both the 80-views from \textcite{su2015multi} and
20-views from \textcite{kanezaki16_rotat} which are not isomorphic to any rotation group.
Their 12-views configuration is isomorphic to the more limited cyclic group.

\begin{figure}[htbp]
 \centering
 \includegraphics[width=\linewidth]{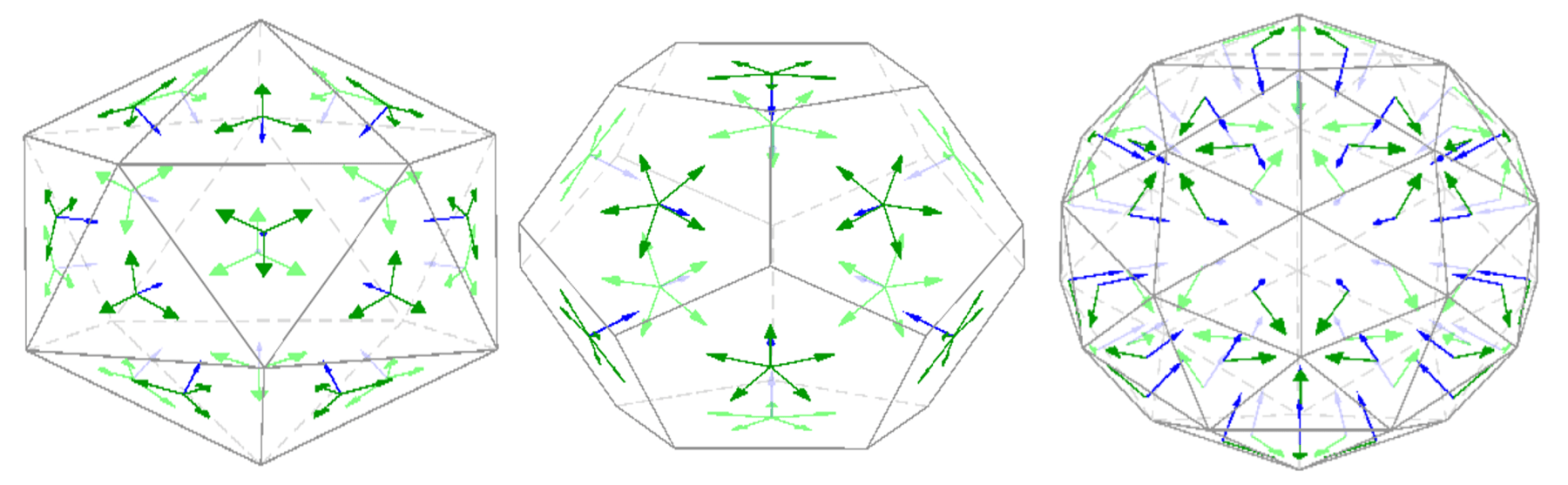}
\caption{Outside-in camera configurations considered.
  Left to right: $20\times 3$, $12\times 5$, and $60\times 1$.
  Blue arrows indicate the optical axis and green, the camera up direction.
  Object is placed at the intersection of all optical axes.
  Only the $60\times 1$ configuration avoids views related by in-plane rotations.}
\label{emvn:fig:camera}
\end{figure}

\begin{figure}[htbp]
  \centering
  \includegraphics[width=\linewidth]{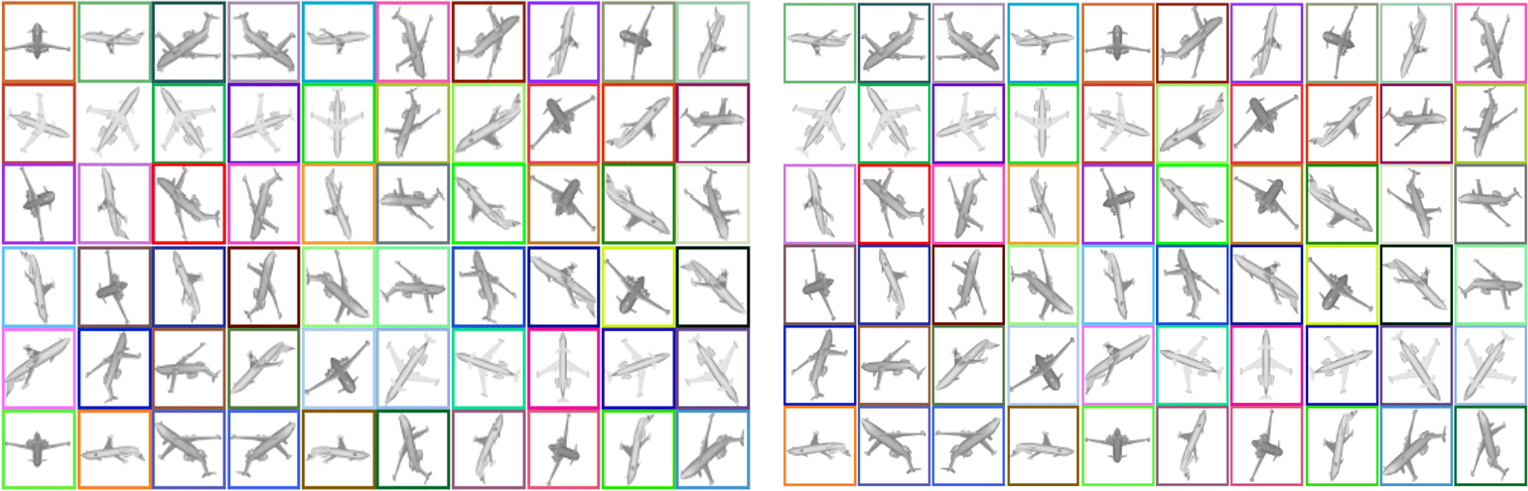}
  \caption{Equivariance of view configurations to $\I$.
    The views on the left and right are from $3$D shapes separated by a \ang{72} rotation in the discrete group.
    We mark corresponding views before and after rotation with same border color.
    Notice the five first views in the second row --
    the axis of rotation is aligned with their optical axis;
    the rotation effect is a shift right of one position for these views.
    It is clear that when $g \in \I$ is applied to the object,
    the views permute in the order given by the Cayley table,
    showing that the mapping from $3$D shape to view set is equivariant.}
    \label{emvn:fig:equi60}
\end{figure}

\subsection{Group convolutional network}
The core of the group convolutional part of our method is the discrete version of \cref{emvn:eq:gconv}.
A group convolutional layer with $c_i$ input and $c_j$ output channels,
and nonlinearity $\sigma$ is then given by
\begin{align}
  f_j^{\ell+1}(y) = \sigma\left( \sum_{i=1}^{c_i} \frac{1}{\abs{G}}\sum_{g\in G}
  f_i^\ell(g)h_{ij}(g^{-1}y) \right),
  \label{emvn:eq:gconvd}
\end{align}
where $f_i^\ell$ is the channel $i$ at layer $\ell$ and $h_{ij}$ is the filter between channels $i$ and $j$, where $1 \leq j \leq c_j$.
This layer is equivariant to actions of $G$.

Our most important results are on the icosahedral group $\I$ which has 60 elements and is the largest discrete subgroup of the rotation group \SO{3}.
To the best of our knowledge, this is the largest group ever considered in the context of discrete \gcnns.
Since $\I$ only coarsely samples \SO{3}, equivariance to arbitrary rotations is only approximate.
Our results show, nevertheless,
that the combination of invariance to local deformations provided by \cnns\
and exact equivariance by \gcnns\ is powerful enough to achieve state of the art performance in different tasks.

When considering the group $\I$, inputs to $\Phi_2$ are $60 \times n$ where $n$ is the number of channels in the last layer of $\Phi_1$ (e.g., $n=512$ for \resnet{18}).
There are $c_i \times c_j$ filters per layer,
each has at most as many parameters as the cardinality of the group.

\paragraph{\acrshort{mvcnn} as a special case}
The \mvcnn\ with late-pooling from \textcite{kanezaki16_rotat},
which outperforms the original by \textcite{su2015multi},
is a special case of our method where $\Phi_2$ just copies the inputs over
and the descriptor is $y$ averaged over $G$.
Suppose we fix the filters $h_{ij}(g)$ as follows, where $i$ and $j$ denote the output and
input channel, and $g$ denotes the element in group
\begin{align*}
h_{ij}(g)=\left\{
\begin{array}{rcl}
\abs{G} & & {i=j}\quad \text{and} \quad {g=e}\\
0 & & \text{otherwise.}
\end{array} \right.
\end{align*}

Applying group correlation with these filters, we get
\begin{align*}
(f\star h)_{i}(k) &= \sum_{j=1}^{c_1} \frac{1}{\abs{G}}\sum_{g \in G}f_j(kg)h_{ij}(g)\\
&=\left\{
\begin{aligned}
&f_i(k) && 1\leq i\leq c_{i}\\
&\quad 0  && i>c_{i},
\end{aligned} \right.
\end{align*}
where $c_{i}$ is the number of input channels.
In this way, the input is ``copied'' into the output and the our model produces the exact same descriptor as an \mvcnn\ with late pooling after the last layer.
The same result could also be achieved using group convolution.

\paragraph{Feature visualization} Our features are functions on a subgroup of the rotation group \SO{3}.
Since \SO{3} is a 3-manifold (which can be embedded in $\R^5$), visualization is challenging.
As we operate on the discrete subgroup of 60 rotations, we choose a solid with icosahedral symmetry and 60 faces as a proxy for visualization -- the pentakis dodecahedron, which is the dual of the truncated icosahedron (the ``soccer ball'' with 60 vertices).

We associate the color of each face with the feature vector at that element of the group.
Since the vector is high-dimensional (usually $256$ or $512$D),
we use \pca\ over all feature vectors in a layer (or groups of channels in a layer) and
project it into the 3 principal components that can be associated with an \rgb\ value.
The same idea is applied to visualize functions on the homogeneous spaces,
where the dodecahedron and icosahedron serve as proxies.
\Cref{emvn:fig:fmaps} shows some equivariant feature maps learned by our method.
\begin{figure}[htbp]
 \centering
 \includegraphics[width=\linewidth]{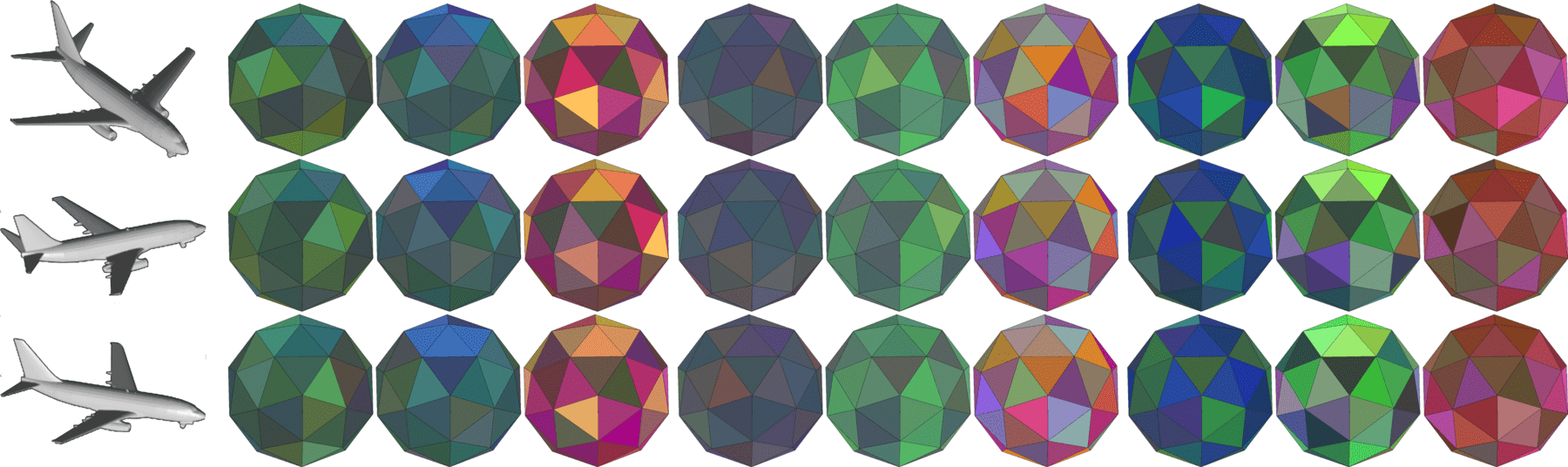}
 \caption{Features learned by our method are visualized on the pentakis dodecahedron,
   which has icosahedral symmetry so its 60 faces are identified
   with elements of the discrete rotation group $\I$.
   Columns show learned features from different channels/layers.
   The first two rows are related by a rotation of \ang{72} in $\I$.
   Equivariance is exact in this case,
   as can be verified by the feature maps rotating around the polar axis
   (notice how the top 5 cells shift one position).
   The first and third row are related by a rotation of \ang{36} around the same axis,
   which is in the midpoint between two group elements.
   Equivariance is approximate in this case, and features are a mixture of the two above.}
\label{emvn:fig:fmaps}
\end{figure}

\subsection{Equivariance with fewer views}
\label{emvn:sec:fewerviews}
As illustrated in \cref{emvn:fig:camera}, the icosahedral symmetries can be divided in sets of rotations around a few axes.
If we arrange the cameras such that they lie on these axes, images produced by each camera are related by in-plane rotations.

As shown in \cref{emvn:sec:canoncoords}, converting one image to canonical coordinates can transform in-plane rotations in translations.
We refer to converted images as ``polar images''.
Since fully convolutional networks can produce translation-invariant descriptors,
by applying them to polar images we effectively achieve invariance to in-plane
rotations~\cite{esteves2018polar,henriques2017warped},
which makes only one view per viewpoint necessary.
These networks require circular padding in the angular dimension (as described in \cref{ptn:sec:wrap}).

When associating only a single view per viewpoint,
the input is on a space of points instead of a group of rotations%
\footnote{They are isomorphic for the $60\times 1$ configuration.}.
In fact, the input is a function on a homogeneous space of the group; concretely,
for the view configurations we consider, it is on the
vertices of the icosahedron or dodecahedron.

We can apply discrete versions of convolution and correlation on homogeneous spaces as defined in \cref{emvn:sec:homogeneous}:
\begin{align}
  ^*f_j^{\ell+1}(y) &= \sigma\left( \sum_{i=1}^{c_i}\sum_{g\in G} f_i^\ell(g\nu)h_{ij}(g^{-1}y) \right),\\
  {^\star}f_j^{\ell+1}(g) &= \sigma\left( \sum_{i=1}^{c_i}\sum_{x\in X} f_i^\ell(gx)h_{ij}(x) \right).
\end{align}

The benefit of this approach is that since it uses five (resp. three) times fewer views
when starting from the $12\times 5$ (resp. $20\times 3$) configuration,
it is roughly five (resp. three) times faster as most of the computation occurs before the \gcnn.
The disadvantage is that learning from polar images can be challenging.
\Cref{emvn:fig:polar} shows one example of polar images produced from views.

When inputs are aligned (in canonical pose), an equivariant intermediate representation is not
necessary; in this setting, we can use the same method to reduce the number of required views,
but without the polar transform.
\begin{figure}[htbp]
  \centering
  \includegraphics[width=\linewidth]{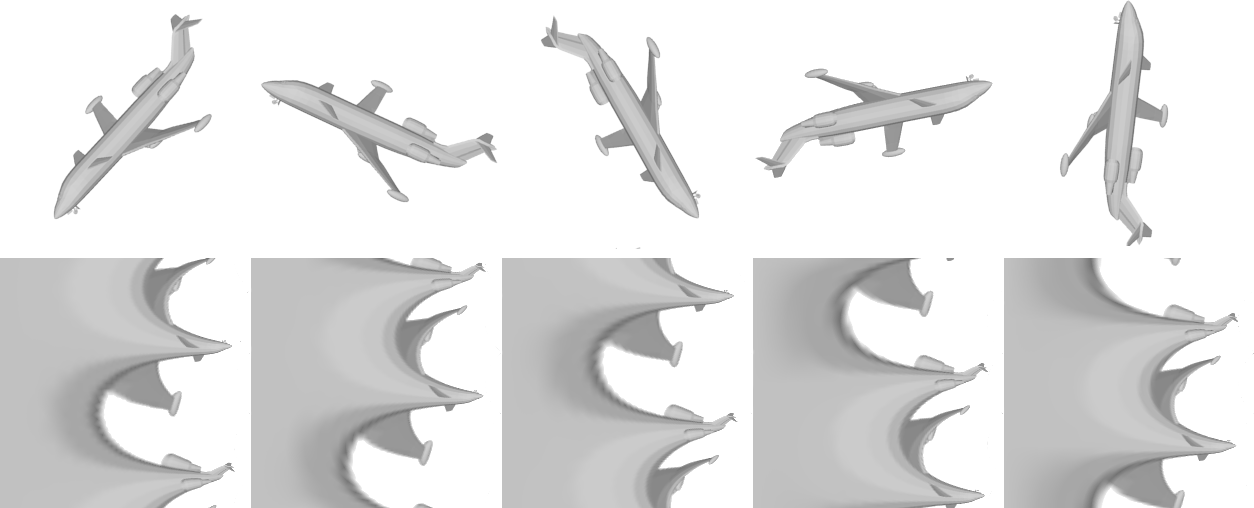}
  \caption{One subset of in-plane related views from the $12\times 5$ configuration and correspondent polar images.
    Note how the polar images are related by circular vertical shifts so their CNN descriptors are approximately invariant to the in-plane rotation.
    There are 12 such subsets for the $12\times 5$ configuration and 20 for the $20\times 3$; this allows us to maintain equivariance with 12 or 20 views instead of 60.}
 \label{emvn:fig:polar}
\end{figure}

\subsection{Filter localization}
\Gcnns\ filters are functions on a group $G$, which can have up to $\abs{G}$ entries.
Recent results obtained with deep \cnns\ show the benefit of using limited support filters
(the use of $3\times 3$ kernels throughout is common).
The advantages are two-fold: (i) convolution with limited support is computationally more efficient,
and (ii) it allows learning of hierarchically more complex features as layers are stacked.
Inspired by this idea, we introduce localized filters for discrete \gcnns%
\footnote{Localization for the continuous case was introduced in \textcite{esteves18eccv};
  we discuss it in \cref{sph:sec:sphcnn}.}.
For a filter \fun{h}{G}{\R}, we simply choose a subset $S$ of $G$ to have nonzero filter values
while $h(G - S)$ is set to zero.
Since $S$ is a fixed hyperparameter, we can compute \cref{emvn:eq:gconvd} more efficiently:
\begin{align}
  f_j^{\ell+1}(y) = \sigma\left( \sum_{i=1}^{c_i}\sum_{g\in S} f_i^\ell(yg^{-1})h_{ij}(g) \right).
\end{align}

To ensure filter locality, it is desirable that elements of $S$ are close to each other in the manifold of rotations.
The 12 smallest rotations in $\I$ are of \ang{72}.
We therefore choose $S$ to contain the identity and a number of \ang{72} rotations.

One caveat of this approach is that we need to make sure $S$ spans $G$,
otherwise the receptive field will not cover the whole input no matter how many layers are stacked,
which can happen if $S$ is in a proper subgroup of $G$ (see \cref{emvn:fig:localized}).
In practice this is not a challenging condition to satisfy; for our heuristic of choosing only \ang{72} rotations we only need to guarantee that at least two are around different axes.

\begin{figure}[htbp]
  \centering
  \includegraphics[width=\linewidth]{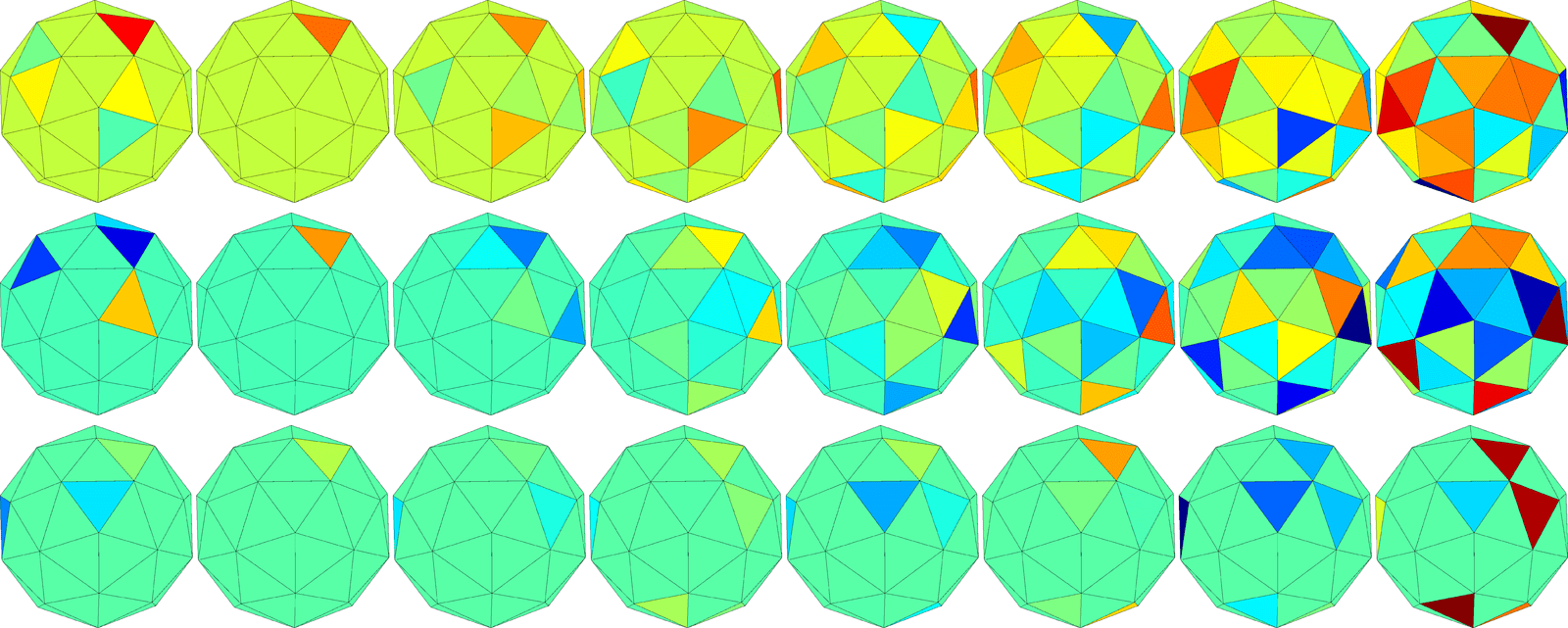}
  \caption{Localized filters and their receptive fields as we stack more layers.
    First column shows the filter, second the input, and others are results of stacking group convolutions with the same filter.
    Top row filter has 12 nonzero elements; middle and bottom have 5.
    The support for the bottom row contains elements of a 12 element subgroup, so its receptive field cannot cover the full input space.}
\label{emvn:fig:localized}
\end{figure}

\section{Experiments}
We evaluate our approach on $3$D shape classification, retrieval and scene classification.
First, we discuss the architectures, training procedures, and datasets.
\paragraph{Architectures}
We use a \resnet{18}~\cite{he2016deep} as the view processing network $\Phi_1$, with weights initialized from ImageNet~\cite{imagenet_cvpr09} pre-training.
The \gcnn\ part contains three layers with 256 channels and nine elements on its support
(note that the number of parameters is the same as one conventional $3\times 3$ layer).
We project from 512 to 256 channels so the number of parameters stay close to the baseline.
When the method in \cref{emvn:sec:fewerviews} is used to reduce the number of views,
the first \gconv\ layer is replaced by a \hcorr.

We denote variations of our method Ours-X and Ours-R-X.
The R suffix indicate retrieval specific features, that consist of (i) a triplet loss
and (ii) reordering the retrieval list so that
objects classified as the query's predicted class come first.
Before reordering, the list is sorted by cosine distance between descriptors.
For SHREC'17, choosing the number N of retrieved objects is part of the task -- in this case we simply return all objects classified as the query's class.

For fair assessment of our contributions, we implement a variation of \mvcnn,
denoted \mvcnnm{X} for $X$ input views, where the best-performing $X$ is shown.
\mvcnnm{X} has the same view-processing network, training procedure and dataset as ours;
the only difference is that it performs pooling over view descriptors instead of using a \gcnn.

\paragraph{Triplet loss}
We implement a simple triplet loss.
During training, we keep a set containing the descriptors for the last seen instance of each class,
$Z=\{z_i\}$, where $i$ is the class label.
For each entry in the mini-batch, let $c$ be the class and $z$ its descriptor.
We take the descriptor in $Z$ of the same class as a positive example ($z_c$), and chose the hardest between all the others in the set as the negative: $z_n = \text{argmin}_{z_i\in Z,\,i\neq c}(d(z_i,z))$, where $d$ is a distance function.
The contribution of this entry to the loss is then,
\begin{equation}
  \mathcal{L} = \max(d(z,z_c) - d(z,z_n) + \alpha, 0),
\end{equation}
where $\alpha$ is a margin.
We use $\alpha=0.2$ and $d$ is the cosine distance.
Note that this method is only used in the ``Ours-R'' variations of our method.

\paragraph{Training} We train using \sgd\ with Nesterov momentum as the optimizer.
The number of epochs is 15 for ModelNet and 10 for SHREC'17.
Following \textcite{he18_bag_trick_image_class_with},
the learning rate linearly increases from 0 to $lr$ in the first epoch,
then decays to zero following a cosine quarter-cycle.
When training with 60 views, we set the batch size to six, and $lr$ to $0.0015$.
This requires around \num{11}{Gb} of memory.
When training with 12 or 20 views, we linearly increase both the batch size and $lr$.

Training our 20-view model on ModelNet40 for one epoch takes approximately \num{353}{s}
on an NVIDIA 1080 Ti, while the corresponding \mvcnnm{} takes \num{308}{s}.
Training RotationNet \cite{kanezaki16_rotat} for one epoch under same conditions
takes approximately \num{1063}{s}.

\paragraph{Datasets}
We render $12\times 5$, $20\times 3$ and $60\times 1$ camera configurations (\cref{emvn:sec:viewconf})
for ModelNet and the ShapeNet SHREC'17 subset, for both rotated and aligned versions.
For the aligned datasets, where equivariance to rotations is not necessary,
we fix the camera up-vectors to be in the plane defined by the object center, camera and north pole.
This reduces the number of views from $12\times 5$ to $12$ and from $20\times 3$ to $20$.
For the rotated datasets, all renderings have 60 views and follow the group structure.
Note that the rotated datasets are not limited to the discrete group and
contain continuous rotations from $\SO{3}$.
We observe that the $60\times 1$ configuration performs best so
those are the numbers shown for ``Ours-60''.
For the experiment with fewer views, we chose $12$ from $12\times 5$ and $20$ from $20\times 3$
that are converted to log-polar coordinates (\cref{emvn:sec:fewerviews}).
For the scene classification experiment, we sample $12$ overlapping views from panoramas.
No data augmentation is performed in any experiment.

\subsection{SHREC'17 retrieval challenge}
The SHREC'17 large scale $3$D shape retrieval challenge~\cite{savva2017shrec} utilizes
the ShapeNet Core55~\cite{shapenet2015} dataset and has two modes:
``normal'' and ``perturbed'' which correspond to ``aligned'' and ``rotated''
as we defined in \cref{emvn:sec:modelnet}.
The challenge happened in 2017 but there has been recent interest on it,
especially on the ``rotated'' mode~\cite{s.2018spherical,esteves18eccv,kondor2018clebsch}.

\Cref{emvn:tab:shrec} shows the results.
N is the number of retrieved elements,
chosen to be the objects classified as the same class as the query.
The Normalized Discounted Cumulative Gain (NDGC) score
uses ShapeNet subclasses to measure relevance between retrieved models.
Methods are ranked by the mean of micro (instance-based) and macro (class-based) \map.
Only the best performing methods are shown;
refer to~\textcite{savva2017shrec} for more results.

Our model surpass the state of the art for both ``rotated'' and ``aligned''
modes even without the triplet loss, which, when included, increase the margins.
This is the most important result in this chapter, since it is on the largest available
$3$D shape retrieval benchmark and there are numerous published results on it.

\begin{table}[htbp]
  \centering
  \caption{SHREC'17 retrieval results.
    Top block: aligned dataset; bottom: rotated.
    Methods are ranked by the average between micro and macro \map\
    (the ``score'' in the second column).
    We also show precision (P), recall (R), F-score (F1), \map,
    and normalized discounted cumulative gain (G),
    where N is the number of retrieved elements.
    \mvcnn\ models without the ``M'' suffix are from \textcite{su2015multi}}
  \label{emvn:tab:shrec}
  \scriptsize
  \begin{tabular}{lS r
    S[table-format=2.1]S[table-format=2.1]S[table-format=2.1]S[table-format=2.1]S[table-format=2.1] r
    S[table-format=2.1]S[table-format=2.1]S[table-format=2.1]S[table-format=2.1]S[table-format=2.1]}
      \toprule
                                                   &           & \phantom{} & \multicolumn{5}{c}{micro} & \phantom{} & \multicolumn{5}{c}{macro}                                                                        \\
      \cmidrule{2-2} \cmidrule{4-8} \cmidrule{10-14}
                  {Method}                         & {score}   &            & {P@N}                     & {R@N}      & {F1@N}    & {mAP}     & {G@N}     &  & {P@N}     & {R@N}     & {F1@N}    & {mAP}     & {G@N}     \\
      \midrule
                  RotatNet~\cite{kanezaki16_rotat} & 67.8      &            & 81.0                      & 80.1       & 79.8      & 77.2      & 86.5      &  & 60.2      & 63.9      & 59.0      & 58.3      & 65.6      \\
                  ReVGG~\cite{savva2017shrec}      & 61.8      &            & 76.5                      & 80.3       & 77.2      & 74.0      & 82.8      &  & 51.8      & 60.1      & 51.9      & 49.6      & 55.9      \\
                  DLAN~\cite{furuya2016deep}       & 57.0      &            & 81.8                      & 68.9       & 71.2      & 66.3      & 76.2      &  & 61.8      & 53.3      & 50.5      & 47.7      & 56.3      \\
                  MVCNN-$12$    & 65.1      &            & 77.0                      & 77.0       & 76.4      & 73.5      & 81.5      &  & 57.1      & 62.5      & 57.5      & 56.6      & 64.0      \\
                  \mvcnnm{12}                      & 69.1      &            & 83.1                      & 77.9       & 79.4      & 74.9      & 83.8      &  & \bgl 66.8 & 68.4      & 65.2      & 63.2      & 70.3      \\
                  Ours-12                          & 70.7      &            & 83.1                      & 80.5       & 81.1      & 77.7      & 86.3      &  & 65.3      & 68.7      & 64.8      & 63.6      & 70.8      \\
                  Ours-20                          & 71.4      &            & \bgl 83.6                 & \bgl 80.8  & \bgl 81.5 & \bgl 77.9 & \bgl 86.8 &  & 66.4      & \bgl 70.1 & 65.9      & 64.9      & 71.9      \\
                  Ours-60                          & \bgl 71.7 &            & \bgd 84.0                 & 80.5       & 81.4      & 77.8      & 86.4      &  & \bgd 67.1 & \bgd 70.7 & \bgd 66.6 & \bgd 65.6 & \bgd 72.3 \\
                  Ours-R-20                        & \bgd 72.2 &            & \bgl 83.6                 & \bgd 81.7  & \bgd 82.0 & \bgd 79.1 & \bgd 87.5 &  & \bgl 66.8 & 69.9      & \bgl 66.1 & \bgl 65.4 & \bgd 72.3 \\
      \midrule
                  DLAN~\cite{furuya2016deep}       & 56.6      &            & \bgd 81.4                 & 68.3       & 70.6      & 65.6      & 75.4      &  & \bgd 60.7 & 53.9      & 50.3      & 47.6      & 56.0      \\
                  ReVGG~\cite{savva2017shrec}      & 55.7      &            & 70.5                      & \bgd 76.9  & 71.9      & \bgl 69.6 & 78.3      &  & 42.4      & 56.3      & 43.4      & 41.8      & 47.9      \\
                  RotatNet~\cite{kanezaki16_rotat} & 46.6      &            & 65.5                      & 65.2       & 63.6      & 60.6      & 70.2      &  & 37.2      & 39.3      & 33.3      & 32.7      & 40.7      \\
                  MVCNN-80      & 45.1      &            & 63.2                      & 61.3       & 61.2      & 53.5      & 65.3      &  & 40.5      & 48.4      & 41.6      & 36.7      & 45.9      \\
                 \mvcnnm{60}                    & 57.5      &            & 77.7                      & 67.6       & 71.1      & 64.1      & 75.9      &  & 55.7      & 56.9      & 53.5      & 50.9      & 59.7      \\
                  Ours-12                      & 58.1      &            & 76.1                      & 70.0       & 72.0      & 66.4      & 76.7      &  & 54.6      & 55.7      & 52.6      & 49.8      & 58.6      \\
                  Ours-20                      & 59.3      &            & 76.4                      & 70.5       & 72.4      & 66.9      & 77.0      &  & 54.6      & 58.0      & 53.7      & 51.7      & 60.2      \\
                  Ours-60                      & \bgl 62.1 &            & \bgl 78.7                 & 72.9       & \bgl 74.7 & \bgl 69.6 & \bgl 79.6 &  & 57.6      & \bgl 60.1 & \bgl 56.3 & \bgl 54.6 & \bgl 63.0 \\
                  Ours-R-60                    & \bgd 63.5 &            & \bgl 78.7                 & \bgl 75.0  & \bgd 75.9 & \bgd 71.8 & \bgd 81.1 &  & \bgl 58.3 & \bgd 60.6 & \bgd 56.9 & \bgd 55.1 & \bgd 63.3 \\
      \bottomrule
    \end{tabular}
\end{table}

\subsection{ModelNet classification and retrieval}
\label{emvn:sec:modelnet}
We evaluate $3$D shape classification and retrieval on variations of ModelNet~\cite{wu20153d}.
In order to compare with most publicly available results, we evaluate on ``aligned'' ModelNet,
and use all available instances with the original train/test split (9843 for training, 2468 for test).
We also evaluate on the more challenging ``rotated'' ModelNet40,
where each instance appears with a random rotation from \SO{3}.

\Cref{emvn:tab:modelnet,emvn:tab:rotmodelnet} show the results.
We show only the best performing methods and refer to the ModelNet website\footnote{\url{http://modelnet.cs.princeton.edu}} for the complete leaderboard.
Classification performance is given by accuracy (acc)
and retrieval by the \map.
Averages are over instances.
We include class-based averages in \cref{emvn:sec:extramnet}.

We outperform the retrieval state of the art for both ModelNet10 and ModelNet40,
even without retrieval-specific features.
When including such features (triplet loss and reordering by class label),
the margin increases significantly.

We focus on retrieval and do not claim state of the art on classification, which is held by RotationNet~\cite{kanezaki16_rotat}.
While ModelNet retrieval was not attempted by \textcite{kanezaki16_rotat},
the SHREC'17 retrieval was, and we show superior performance on it (\cref{emvn:tab:shrec}).
We show more comparisons with RotationNet~\cite{kanezaki16_rotat} in \cref{emvn:sec:rotnet}.

\begin{table}[htbp]
    \caption{Aligned ModelNet classification and retrieval.
      We only compare with published retrieval results.
      We achieve state of the art retrieval performance even without retrieval-specific model features.
      This shows that our view aggregation is useful even when global equivariance is not necessary.}
  \label{emvn:tab:modelnet}
  \centering
  \begin{tabular}{l
    S[table-format=2.2]S[table-format=2.2] l
    S[table-format=2.2]S[table-format=2.2]}
      \toprule
                                             & \multicolumn{2}{c}{M40 (aligned)} & \phantom{abc} & \multicolumn{2}{c}{M10 (aligned)} \\
      \cmidrule{2-3} \cmidrule{5-6}
                                             & {acc}                               & {mAP}           &  & {acc}        & {mAP}               \\
      \midrule
      MVCNN-12~\cite{su2015multi}                  & 90.1                              & 79.5          &  & {-}          & {-}                 \\
      SPNet~\cite{yavartanoo18_spnet}        & 92.63                             & 85.21         &  & \bgl 97.25 & 94.20             \\
      PVNet~\cite{you2018pvnet}              & 93.2                              & 89.5          &  & {-}          & {-}                 \\
      SV2SL~\cite{han2019seqviews2seqlabels} & 93.40                             & 89.09         &  & 94.82      & 91.43             \\
      PANO-ENN \cite{sfikas2018ensemble}     & \bgd 95.56                        & 86.34         &  &	96.85      & 93.2              \\
      \mvcnnm{12}                             & 94.47                             & 89.13         &  & 96.33      & 93.54             \\
      \midrule
      Ours-12                                & 94.51                             & \bgl 91.82    &  & 96.33      & 95.30             \\
      Ours-20                                & \bgl 94.69                        & 91.42         &  & \bgd 97.46 & \bgl 95.74        \\
      Ours-60                                & 94.36                             & 91.04         &  & 96.80      & 95.25             \\
      Ours-R-12                              & 94.67                             & \bgd 93.56    &  & 96.78      & \bgd 96.18        \\
      \bottomrule%
    \end{tabular}
\end{table}

\begin{table}[htbp]
  \caption{Rotated ModelNet40 classification and retrieval.
      Note that the gap between ``Ours'' and ``MVCNN-M'' is much larger than in the aligned dataset,
      which demonstrates the advantage of our equivariant representation.}
  \label{emvn:tab:rotmodelnet}
  \centering
    \begin{tabular}{lSS}
      \toprule
                               & \multicolumn{2}{c}{M40 (rotated)} \\
      \cmidrule{2-3}
                               & {acc}        & {mAP}                  \\
      \midrule
      MVCNN-80~\cite{su2015multi}    & 86.0       & {-}                    \\
      RotationNet~\cite{kanezaki16_rotat} & 80.0       & 74.20                \\
      Spherical CNN~\cite{s.2018spherical} & 86.9       & {-}                    \\
      \mvcnnm{60}               & 90.68      & 78.18                \\
      Ours-12                  & 88.50      & 79.58                \\
      Ours-20                  & 89.98      & 80.73                \\
      Ours-60                  & \bgl 91.00 & \bgl 82.61           \\
      Ours-R-60                & \bgd 91.08 & \bgd 88.57           \\
      \bottomrule
    \end{tabular}
\end{table}

\subsection{Comparison with RotationNet}
\label{emvn:sec:rotnet}
We provide further comparison against RotationNet \cite{kanezaki16_rotat}.
While RotationNet remains the state of the art on aligned ModelNet classification,
our method is superior on all retrieval benchmarks.
We also outperform RotationNet on more challenging classification taks: rotated and aligned ShapeNet, and rotated ModelNet.
\Cref{emvn:tab:rotat} shows the results.

\begin{table}[htbp]
  \caption{Classification accuracy (acc) and retrieval (\acrshort{map}) comparison
    against RotationNet~\cite{kanezaki16_rotat}.
    Results for ModelNet40 (MNet40) aligned (al) and rotated (rot) datasets, and for the SHREC'17 split of ShapeNet.
    The score for SHREC'17 is the average between micro and macro \map.
  }
  \label{emvn:tab:rotat}
  \small
  \centering
  \begin{tabular}{l
    S[table-format=2.2]S[table-format=2.2] l
    S[table-format=2.2]S[table-format=2.2] l
    S[table-format=2.2]S[table-format=2.2] l
    S[table-format=2.2]S[table-format=2.2]}
    \toprule
    & \multicolumn{2}{c}{MNet40 (al)} & \phantom{}  & \multicolumn{2}{c}{MNet40 (rot)} & \phantom{} & \multicolumn{2}{c}{SHREC'17 (al)} & \phantom{} & \multicolumn{2}{c}{SHREC'17 (rot)}  \\
    \cmidrule{2-3}  \cmidrule{5-6} \cmidrule{8-9} \cmidrule{11-12}
    & {acc}            & {mAP}                        &  & {acc}            & {mAP}            &  & {acc}   & {score}           &  & {acc}   & {score}           \\
    \midrule
    RotationNet~\cite{kanezaki16_rotat}        &  97.37 & 93.00          &  & 80.0           & 74.20              &  & 85.39    & 67.8          &  & 77.37     & 46.6          \\
    Ours   & 94.67 &  93.56  &  &  91.08 &  88.57 &  &  89.15 &  72.2 &  &  85.93 &  63.5 \\
    \bottomrule
  \end{tabular}
\end{table}

\subsection{Ablation}
We run an experiment to compare effects of (i) filter support size, (ii) number of \gconv\ layers, and (iii) missing views. %
We evaluate on rotated ModelNet40 with ``Ours-60'' model as baseline.
The base model has a filter support of nine elements,
three \gconv\ layers and uses all 60 views.

When considering less than 60 views, we introduce view dropout during training
where a random number (between 1 and 30) of views is selected for every mini-batch.
This improves robustness to missing views.
During test, we use a fixed number of views.
\Cref{emvn:tab:ablation} shows the results.
As expected, we can see some decline in performance with fewer layers and smaller support, which reduces the receptive field at the last layer.
Our method is robust to missing up to \num{50}{\%} of the views, with noticeable drop in performance when missing \num{80}{\%} or more.

\begin{table}[htbp]
  \centering
  \caption{Ablation study on rotated ModelNet40.
  Our best performing model is on the top row.}
  \label{emvn:tab:ablation}
  \begin{tabular}{SSScSS}
    \toprule
    {support} & {layers} & {views} & pretrained & {acc}   & {mAP}   \\
    \midrule
    9       & 3         & 60       &   yes         & 91.00 & 82.61 \\
    6       & 3         & 60       &   yes         & 90.63 & 81.90 \\
    3       & 3         & 60       &   yes         & 89.74 & 80.49 \\
    9       & 2         & 60       &   yes         & 91.00 & 81.47 \\
    9       & 1         & 60       &   yes         & 90.88 & 79.59 \\
    9       & 3         & 30       &   yes         & 89.50 & 79.20 \\
    9       & 3         & 10       &   yes         & 88.32 & 74.65 \\
    9       & 3         & 5        &   yes         & 82.77 & 64.88 \\
    9       & 3         & 60       &   no          & 87.40 & 70.44 \\
    \bottomrule
  \end{tabular}
\end{table}

\subsection{Scene classification}
So far we have shown experiments for object-centric configurations (outside-in),
but our method is also applicable to camera-centric configurations (inside-out),
which we demonstrate on the Matterport3D~\cite{chang17_matter} scene classification from panoramas task.
We sample multiple overlapping azimuthal views from the panorama
as shown in \cref{emvn:fig:matterport},
and apply our model over the cyclic group of 12 rotations, with six elements in the filter support.
\Cref{emvn:tab:sceneclass} shows the results.

\begin{figure}[htbp]
  \centering
  \includegraphics[width=\linewidth]{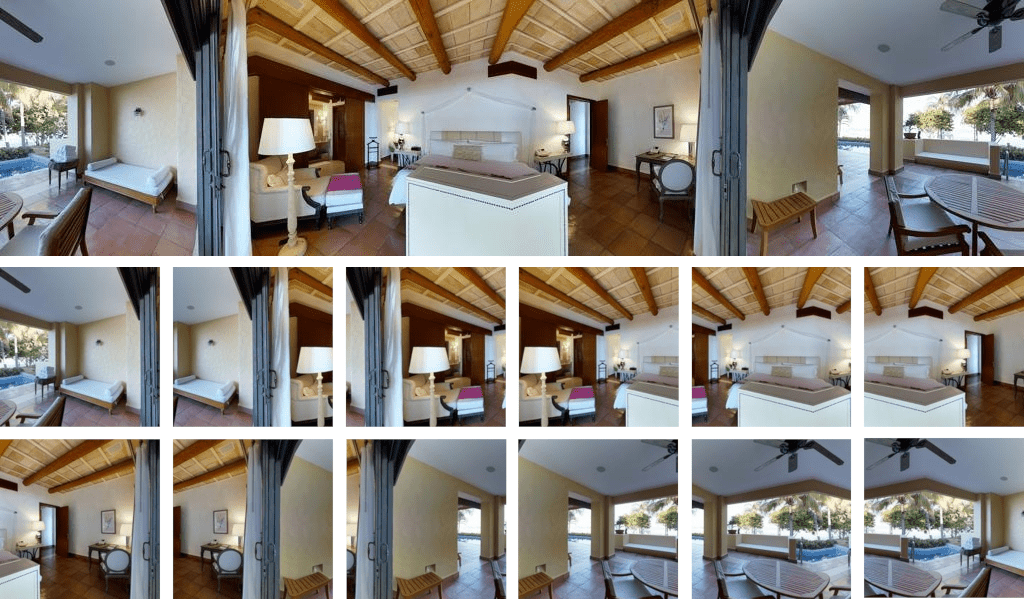}
  \caption{Top: original input from MatterPort3D~\cite{chang17_matter} scene classification task.
    Bottom: our set of 12 overlapping views.}
\label{emvn:fig:matterport}
\end{figure}

\begin{table}[htbp]
  \caption{Matterport3D panoramic scene classification results.
  We show accuracy in \% per category.}
  \label{emvn:tab:sceneclass}
  \centering
    {\fontsize{7}{9}\selectfont
    \begin{tabular}{lS[table-format=2.1]
      S[table-format=2.1]S[table-format=2.1]S[table-format=2.1]S[table-format=2.1]
      S[table-format=2.1]S[table-format=2.1]S[table-format=2.1]S[table-format=2.1]
      S[table-format=2.1]S[table-format=2.1]S[table-format=2.1]S[table-format=2.1]}
      \toprule
      & {avg}    & {office}  & {lounge}  & {family}  & {entry-} & {dining}  & {living}  & {stairs}  & {kitchen} & {porch}   & {bath-}   & {bed-}    & {hall-} \\
      \addlinespace[-5pt]
                                   &           &           &           & {room}    &        {way}    & {room}    & {room}    &           &           &           & {room}    & {room}   & {way}          \\
      \midrule
      sing.~\cite{chang17_matter} & 33.3      & 20.3      & \bgd 21.7 & 16.7      & 1.8        & 20.4      & 27.6      & 49.5      & 52.1      & 57.4      & 44.0      & 43.7      & 44.7      \\
      pano~\cite{chang17_matter}   & 41.0      & \bgl 26.5 & 15.4      & 11.4      & 3.1        & 27.7      & 34.0      & 60.6      & 55.6      & 62.7      & 65.4      & 62.9      & 66.6      \\
      MV-M-12                   & \bgl 51.9 & 18.0      & \bgl 16.4 & \bgl 23.8 & \bgl 8.6   & \bgl 46.7 & \bgl 37.1 & \bgd 84.1 & \bgl 73.3 & \bgd 81.0 & \bgd 78.2 & \bgl 81.7 & \bgd 73.8 \\
      Ours-12                      & \bgd 53.8 & \bgd 27.9 & \bgl 16.4 & \bgd 33.3 & \bgd 11.4  & \bgd 51.1 & \bgd 41.3 & \bgl 80.4 & \bgd 75.8 & \bgl 79.0 & \bgl 72.5 & \bgd 82.9 & \bgl 73.5 \\
      \bottomrule
    \end{tabular}
    }
\end{table}

The \mv\ approach is superior to operating directly on panoramas because
(i) it allows higher overall resolution while sharing weights across views, and
(ii) views match the scale of natural images so pre-training is better exploited.
Our \mvcnnm{} outperforms both baselines, and our proposed model outperforms it, which shows that the group based view aggregation is also useful in this setting.
Our representation is equivariant to azimuthal rotations here;
a \cnn\ operating directly on the panorama is also equivariant,
but without properties (i) and (ii) aforementioned.%

\subsection{Discussion}
Our model shows state of the art performance on multiple $3$D shape retrieval benchmarks.
We argue that the retrieval problem is more appropriate than classification
to evaluate shape descriptors because it requires a complete rank of similarity
between models instead of only a class label.

Our results for aligned datasets show that the full set of 60 views is not necessary and
may be even detrimental in this case;
but even when equivariance is not required, the principled view aggregation with \gconvs\
is beneficial, as direct comparison between \mvcnnm{} and our method show.
For rotated datasets, results show that performance increases with the number of views,
and that the aggregation with \gconvs\ brings major improvements.

Interestingly, our \mvcnnm{} baseline outperforms many competing approaches.
The differences with respect to the original \mvcnn~\cite{su2015multi} are
(i) late view-pooling,
(ii) use of \resnet,
(iii) improved rendering, and
(iv) improved learning rate schedule.
These significant performance gains were also observed in~\textcite{su18_deeper_look_at_shape_class}, and attest to the potential of multi-view representations.

One limitation is that our feature maps are equivariant only to discrete rotations,
and while classification and retrieval performance under continuous rotations is good,
for tasks such as continuous pose estimation it may not be.
Another limitation is that we assume views to follow the group structure, which may be difficult to achieve for real images.
This is not a problem for $3$D shape analysis, though, because we can render any arbitrary view.

\section{Conclusion}
In this chapter we presented an approach that leverages the representational power of
conventional deep \cnns\ and exploits the finite nature of the multiple views
to design a group convolutional network that performs an exact equivariance in discrete groups,
most importantly the icosahedral group.
We also introduced localized filters and convolutions on homogeneous spaces in this context.
Our method enables joint reasoning over all views as opposed to traditional view-pooling,
and surpass the state of the art by large margins on several $3$D shape retrieval benchmarks.
\section{Extra results and visualization}
\subsection{ModelNet}
\label{emvn:sec:extramnet}
Since some methods show ModelNet40 results as averages per class instead of the more common average per instance, we include extended tables with these metrics.
We also present results on rotated ModelNet10.
\Cref{emvn:tab:mupext} shows the results.

\begin{table}[htbp]
  \caption{ModelNet results.
    We include classification accuracy and retrieval mAP per class (cls) and per instance (ins).}
  \label{emvn:tab:mupext}
  \centering
  \small
    \begin{tabular}{lSSSSrSSSS}
      \toprule
                & \multicolumn{4}{c}{M40 (aligned)} & \phantom{} & \multicolumn{4}{c}{M10 (aligned)}                              \\
      \cmidrule{2-5} \cmidrule{7-10}
                & {acc ins}                          & {acc cls}      & {mAP ins} & {mAP cls} &  & {acc ins} & {acc cls} & {mAP ins} & {mAP cls} \\
      Ours-12   & 94.51                             & 92.49         & 91.82   & 88.28   &  & 96.33    & 96.00    & 95.30   & 95.00   \\
      Ours-20   & 94.69                             & 92.56         & 91.42   & 87.71   &  & 97.46    & 97.34    & 95.74   & 95.58   \\
      Ours-60   & 94.36                             & 92.40         & 91.04   & 87.30   &  & 96.80    & 96.58    & 95.25   & 95.01   \\
      Ours-R-20 & 94.44                             & 92.49         & 93.19   & 89.65   &  & 97.02    & 96.97    & 96.59   & 96.46   \\
      \midrule
                & \multicolumn{4}{c}{M40 (rotated)} & \phantom{} & \multicolumn{4}{c}{M10 (rotated)}                                   \\
      \cmidrule{2-5} \cmidrule{7-10}
      Ours-12   & 88.50                             & 85.77         & 79.58        & 74.64   &  & 91.89    & 91.54   & 86.93    & 86.08   \\
      Ours-20   & 89.98                             & 87.65         & 80.73        & 75.65   &  & 92.60    & 92.35   & 87.27    & 86.65   \\
      Ours-60   & 91.00                             & 89.24         & 82.61        & 78.02   &  & 92.83    & 92.80   & 88.47    & 88.02   \\
      Ours-R-20 & 91.08                             & 88.94         & 88.57        & 84.37   &  & 93.05    & 93.08   & 92.07    & 91.99   \\
      \bottomrule
    \end{tabular}
\end{table}

\subsection{Feature maps}
We visualize more examples of our equivariant feature maps in \cref{emvn:fig:fmapsdodec,emvn:fig:fmapsico,emvn:fig:fmapspentakis}.
Each figure shows 8 different input rotations, the first 5 are from a subgroup of rotations around one axis with \ang{72} spacing, the other 3 are from other subgroup with \ang{120} spacing.
We show the axis of rotation in red.
The first column is a view of the input, the second is the initial representation on the group or \hspc\ the other three are features on each \gcnn\ layer.

Our method is equivariant to the 60-element discrete rotation group even with only 12 or 20 input views.
In \cref{emvn:fig:fmapsdodec} we take only 12 input views, giving initial features on the \hspc\ represented by faces of the dodecahedron.
Note that the five first rotations in this case are in-plane for the views corresponding to the axis of rotation.
Due to our procedure described in \cref{emvn:sec:fewerviews},
this gives an invariant descriptor which can be visualized as the face with constant color.
Similarly, in \cref{emvn:fig:fmapsico},
we take 20 views and the invariant descriptor appears in the last three rotations.

Equivariance is easily visualized on faces neighboring the axis of rotation.
For the dodecahedron, we can see cycles of five when the axis is on one face and
cycles of three when the axis is on one vertex.
For the icosahedron, we can see cycles of three when the axis is on one face and
cycles of five when the axis is on one vertex.
For the pentakis dodecahedron (\cref{emvn:fig:fmapspentakis}), we can see
groups of five cells that shift one position when rotation is of \ang{72} and
groups of six cells that shift two positions when rotation is of \ang{120}.

\begin{figure}[htbp]
  \centering
  \includegraphics[width=0.75\linewidth]{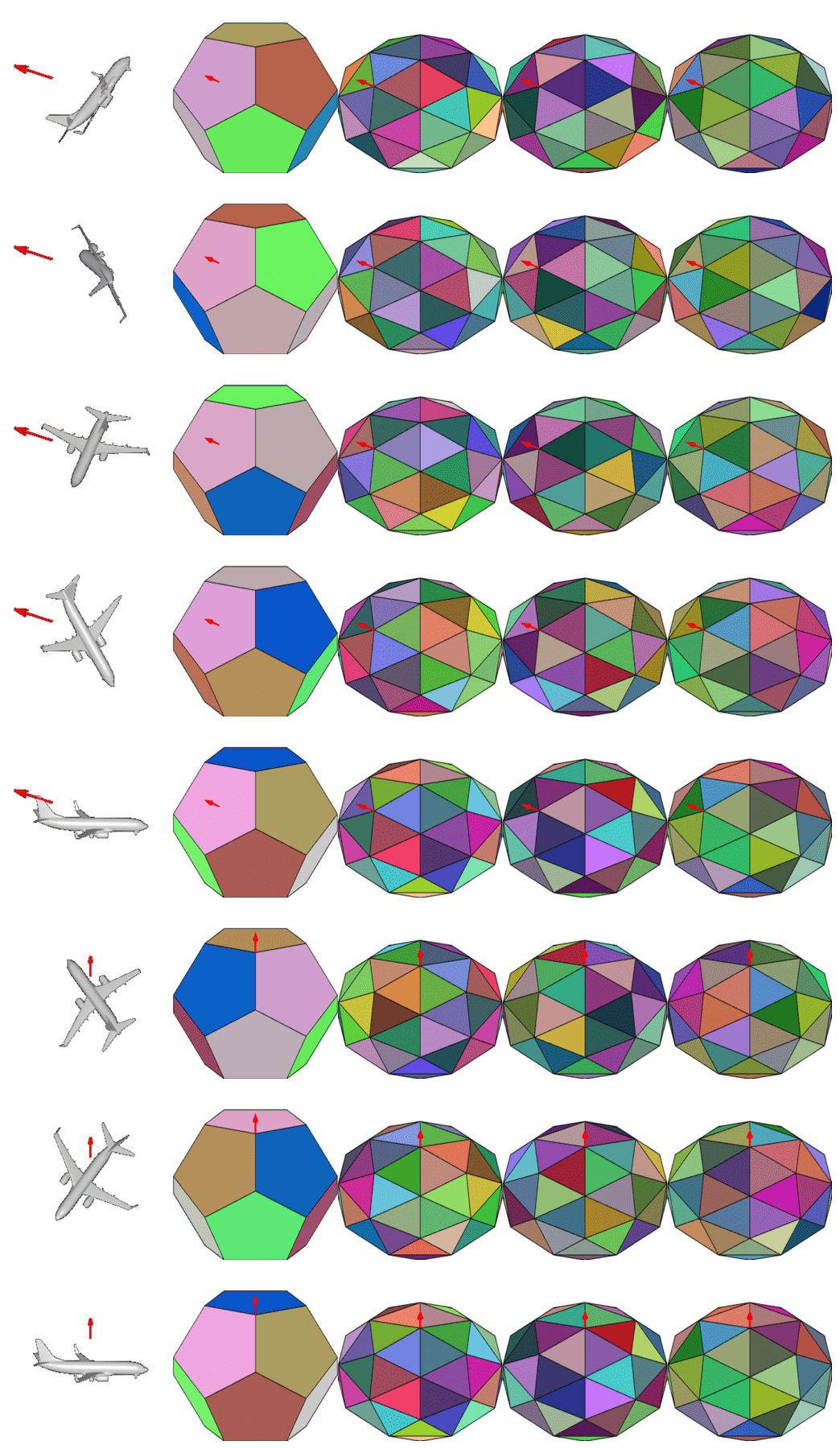}
  \caption{Feature maps with 12 input views.
  }%
  \label{emvn:fig:fmapsdodec}
\end{figure}

\begin{figure}[htbp]
  \centering
  \includegraphics[width=0.75\linewidth]{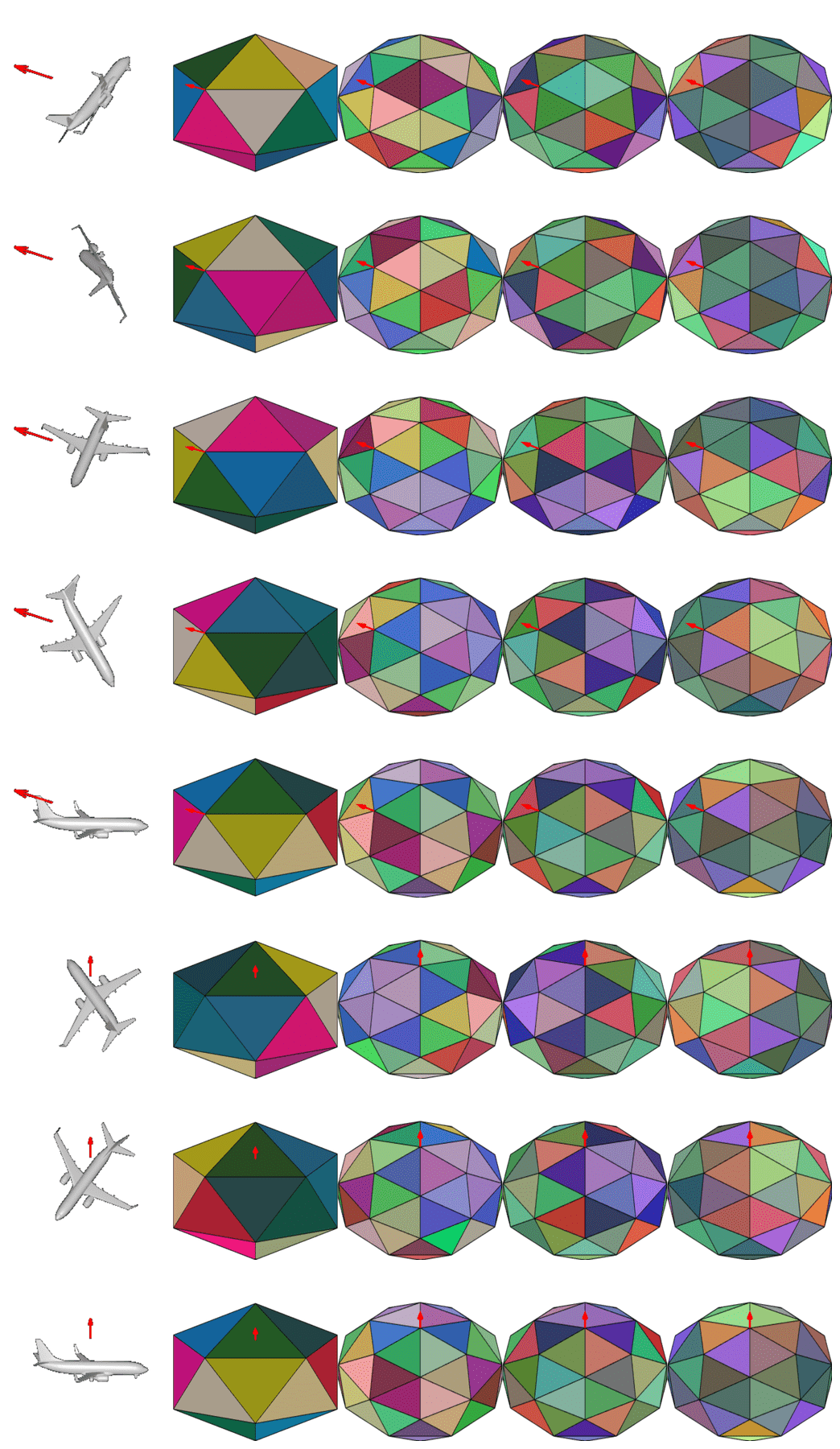}
  \caption{Feature maps with 20 input views.
    }%
  \label{emvn:fig:fmapsico}
\end{figure}

\begin{figure}[htbp]
  \centering
  \includegraphics[width=0.75\linewidth]{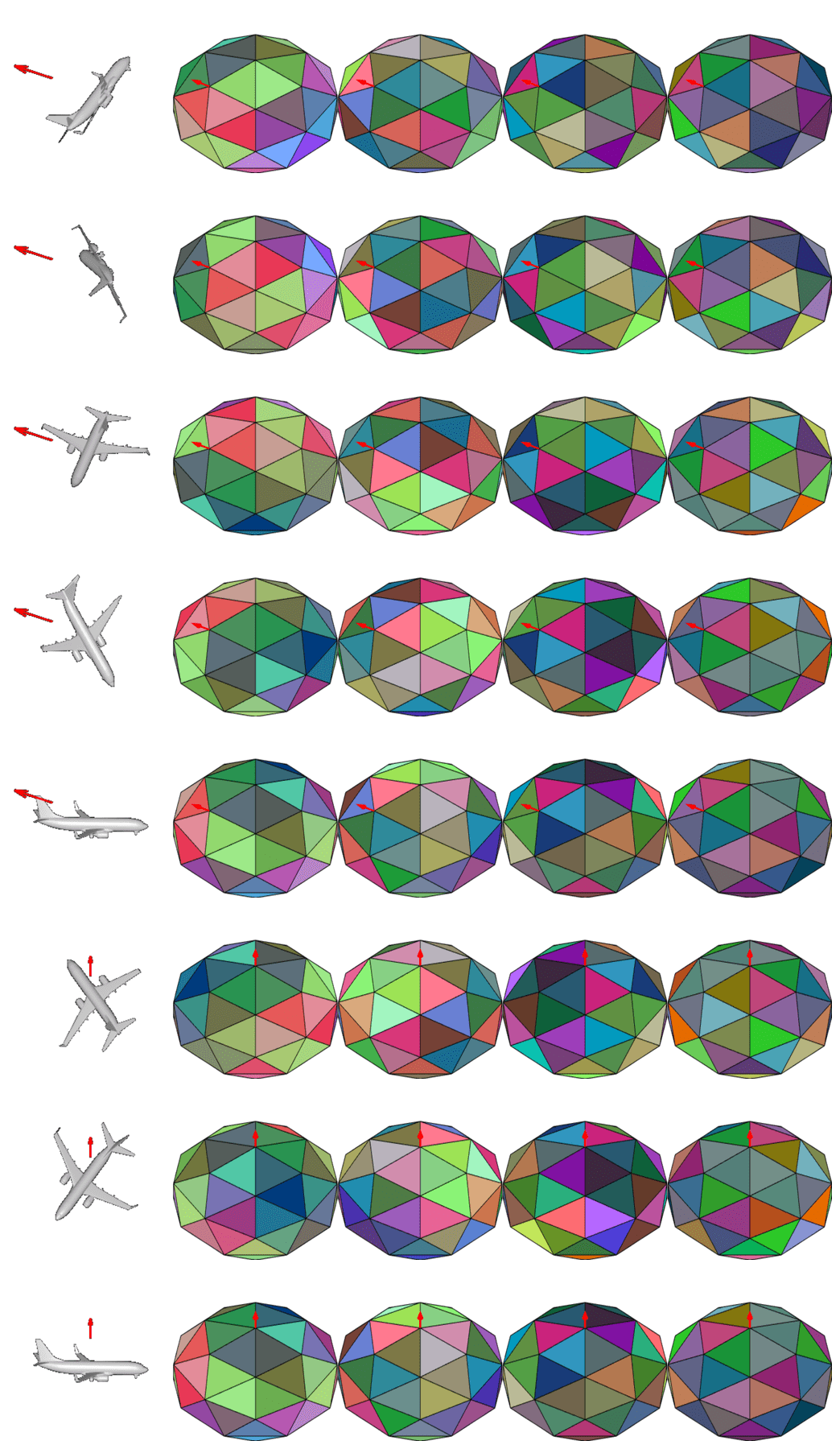}
  \caption{Feature maps with 60 input views.
  }%
  \label{emvn:fig:fmapspentakis}
\end{figure}

\glsresetall
\chapter{Equivariance to continuous 3D rotations}
\label{sph:sec:sphcnn}
\chaptersubtitle{The Spherical CNNs}

\section{Introduction}
One of the reasons for the tremendous success of \cnns\ is their equivariance to translations in
Euclidean spaces and the resulting invariance to local deformations.
The traditional way to address invariance with respect to other nuisances is with data augmentation,
while non-Euclidean inputs like point-clouds are often approximated by
euclidean representations like voxel spaces.
Only recently, equivariance with respect to other groups was considered~\cite{cohen2016group,worrall2017harmonic}
and \cnns\ for manifolds and graphs were proposed~\cite{bruna2013learning,bronstein2017geometric,s.2018spherical}.

Equivariant networks retain information about group actions on the input and
on the feature maps throughout the layers of a network.
Because of their special structure, feature transformations are directly related to
spatial transformations of the input.
Such equivariant structures yield a lower model complexity in terms of number of parameters
than alternatives like the \stns~\cite{jaderberg15nips},
where a learned canonical transformation is applied to the original input.

In this chapter, we are primarily interested in analyzing $3$D shapes
for alignment, retrieval and classification.
Translation and scale invariance are easily achieved in volumetric and point-cloud based approaches
by setting the object's origin to its center and constraining its extent to a fixed constant.
However, $3$D rotations remain a challenge.
\Cref{sph:fig:methods-zz-so3so3-zso3} illustrates how classification performance for conventional
methods suffers when arbitrary rotations are introduced.

\begin{figure}[htbp]
    \centering
    \includegraphics[width=0.7\textwidth]{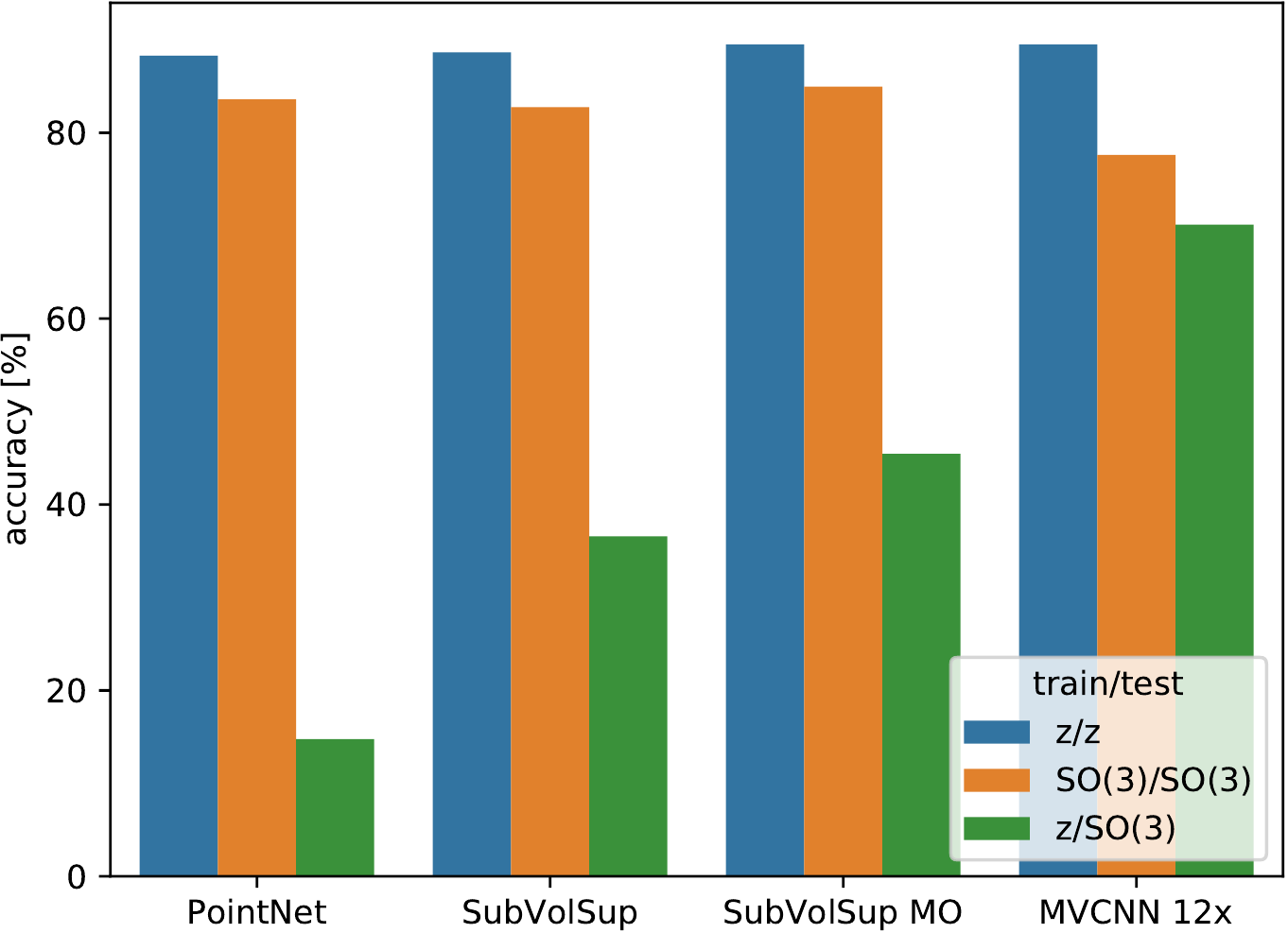}
    \caption{\small ModelNet40 classification for 
      point cloud~\cite{qi2017pointnet}, volumetric~\cite{vam}, and multi-view~\cite{su2015multi} methods. The significant drop in accuracy illustrates that conventional methods do not generalize to arbitrary ($\mathbf{SO}(3)$/$\mathbf{SO}(3)$) and unseen orientations (z/$\mathbf{SO}(3)$).}
  \label{sph:fig:methods-zz-so3so3-zso3}
\end{figure}

\begin{figure}[htbp]
  \centering
    \includegraphics[width=0.6\textwidth]{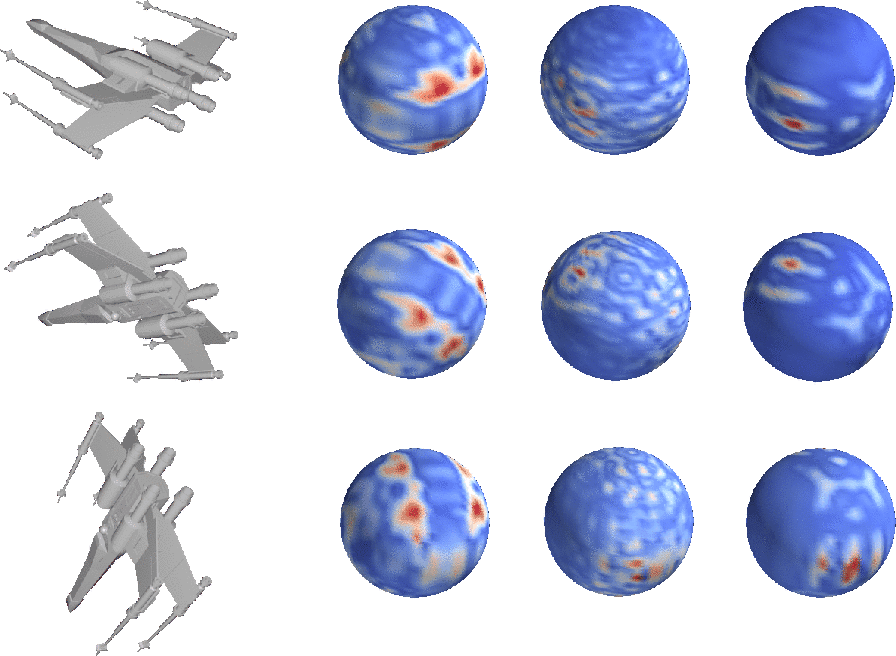}
    \caption{Each row shows the rotated input mesh and a few corresponding spherical feature maps learned by our network.
      Note the activations on the aircraft engines on the second column; they clearly illustrate rotation equivariance.}
    \label{sph:fig:intro}
  \end{figure}

  We model $3$D shapes with vector-valued spherical functions
  and introduce a novel equivariant convolutional neural network with spherical inputs
  (\cref{sph:fig:intro} illustrates the equivariance).
  The main operation is the spherical convolution, which has spherical outputs and
  is different from the cross-correlation that has outputs in the rotation group \SO{3}.

  We employ exact convolutions that yield zonal filters,
  i.e., filters with constant values along the same latitude.
  Convolutions cannot be evaluated efficiently on the spatial domain as is usual on Euclidean spaces,
  but can be exactly computed as pointwise multiplication in the spectral domain
  through decomposition in the spherical harmonics basis.

  It is natural then to apply pooling in the spectral domain.
  Spectral pooling has the advantage that it retains equivariance
  while spatial pooling on the sphere is only approximately equivariant.
  We also propose a weighted averaging pooling where the weights are proportional to the cell area.
  The only reason to return to the spatial domain is the rectifying nonlinearity,
  which is a pointwise operator.

  To obtain localized filters, we enforce a smooth spectrum by
  learning weights only on few anchor frequencies and interpolating between them,
  yielding, as additional advantage, a number of weights independent of the spatial resolution.

  We perform $3$D retrieval, classification, and alignment experiments,
  and also present an extension to semantic segmentation of spherical panoramas.
  Our aim is to show that we can achieve near state-of-the-art performance
  with a much lower network capacity,
  which we achieve for the ModelNet40~\cite{wu20153d} dataset
  and the SHREC'17 large scale $3$D shape retrieval challenge~\cite{savva2017shrec}.

The following summarizes the main contributions in this chapter.
\begin{itemize}
\setlength\itemsep{0em}
\item We propose the first neural network based on spherical convolutions.
\item We introduce pooling and parameterization of filters in the spectral domain,
  with enforced spatial localization and capacity independent of the resolution.
\item In addition to the conventional equiangular grid,
  we explore a uniform spherical grid that and show its benefits.
\item Our model has much lower capacity than non-spherical counterparts applied to $3$D data
  without sacrificing performance.
\item We present an extension of our model that
  is the first equivariant model for panoramic image segmentation.
\end{itemize}

Most of the content in this chapter appeared originally in~\textcite{esteves18eccv,sphhg,sphcnnijcv},
Source code is available at \url{https://github.com/daniilidis-group/spherical-cnn}.

\section{Related work}
\label{sph:sec:related_work}
We will start describing related work on group equivariance,
in particular equivariance on the sphere, then delve into \cnn\ representations for $3$D data.

There are different methods for enabling equivariance in \cnns.
Equivariance can be obtained by constraining filter structure similarly to Lie generator
based approaches~\cite{segman1992canonical,hel1998canonical}.
\textcite{worrall2017harmonic} is a representative of these methods in a \cnn\ setting,
using filters derived from the complex harmonics achieving
both rotational and translational equivariance.
Another way is to use a filter orbit which is itself equivariant to obtain group equivariance.
\textcite{cohen2016group} formalized these methods in the context of \cnns.

Recently, a body of work on graph convolutional networks (GCN) has emerged.
There are two threads within this space, spectral \cite{bruna2013spectral,defferrard2016convolutional,kipf2016semi,yi2016syncspeccnn}
and spatial \cite{boscaini2016learning,masci2015geodesic,monti2017geometric,dgcnn}.
These approaches learn filters on irregular but structured graph representations.
These methods differ from ours in that we are looking to explicitly learn equivariant
and invariant representations for $3$D-data modeled as spherical functions under rotation.
While such properties are difficult to construct for general manifolds, we leverage the
group action of rotations on the sphere.

\textcite{s.2018spherical} is the closest to our approach and developed in parallel.
It uses spherical correlation to map spherical inputs to features on \SO{3},
then processed with a series of cross-correlations on \SO{3}.
The main difference is that we use spherical convolutions,
which are potentially one order of magnitude faster,
with smaller (one fewer dimension) filters and feature maps.
In addition, we enforce smoothness in the spectral domain that results
in better localization of the receptive fields on the sphere and
we perform pooling in two different ways, either as a low-pass filter in the spectral domain
or as a weighted averaging in the spatial domain.
Moreover, our method outperforms \textcite{s.2018spherical} on the
SHREC'17 benchmark and on the spherical \mnist\ dataset.

Spherical representations for $3$D data have been used for retrieval tasks before
the deep learning era~\cite{frome2004recognizing,kazhdan2002harmonic}
because of their invariance properties and the efficient implementation of spherical correlation~\cite{makadia2010spherical}.

A variety of $3$D shape representations besides the spherical
have been explored in the context of deep learning.
The most natural adaptation of $2$D methods is to use a voxel-grid representation of
the $3$D object and amend the $2$D \cnn\ framework to use $3$D filters
for cascaded processing in the place of conventional $2$D filters.
Such approaches require a tremendous amount of computation even for small voxel resolutions.
The first attempts in this line were by \textcite{wu20153d} and \textcite{maturana2015voxnet},
which propose a volumetric network with $3$D convolutional layers followed by fully-connected layers.
\textcite{vam} observe significant overfitting when attempting to train such models
end-to-end and amend the technique by using subvolume classification as an auxiliary task.
They also propose an alternative model that learns to project the volumetric representation
to a $2$D representation that is then processed using a conventional $2$D \cnn.
Even with these adaptations, \textcite{vam} are challenged by overfitting and
suggest augmentation in the form of orientation pooling as a remedy.

\textcite{qi2017pointnet} present a neural network that operates directly on point clouds,
which was followed by several others~\cite{klokov17_escap_from_cells,li18_so_net}.
While much more efficient than the volumetric approaches, the generalization performance
of these models is lower (as exemplified in \cref{sph:fig:methods-zz-so3so3-zso3}),
because they operate directly on coordinate values.
Later iterations such as \textcite{qi2017pointnet++} present improvements
by learning features hierarchically, but they come with increased a computational cost.

Currently, the most successful approaches for $3$D shape analysis are view-based,
operating on rendered views of the $3$D object.
\textcite{su2015multi} introduced the idea, which gave rise to numerous follow-ups~\cite{vam,kanezaki16_rotat,bai2016gift}.
The high performance of these methods is in part due to the use of large pre-trained $2$D CNNs
(on ImageNet~\cite{Russakovsky2015imagenet}, for instance).

Volumetric and point cloud methods are not generally equivariant to $3$D rotations.
The multi-view methods are usually invariant to the discrete set of views considered,
and a large number of views would be required to approximate equivariance to continuous rotations.
These approaches all struggle with shape understanding in arbitrary orientations, even with significant training data augmentation.
The main objective of this chapter is to overcome this limitation.

\section{Preliminaries}
\label{sph:sec:math}

\subsection{Group and homogeneous space convolution}
Recall that a map \fun{\Phi}{E}{F} is equivariant
to a group $G$ when for any $g\in G$
\[\Phi(\lambda_g (f)) = \lambda'_g (\Phi (f)),\]
where $\lambda_g$ and $\lambda'_g$ are the group actions
on $E$ and $F$, respectively.

A straightforward example of equivariant representation is an orbit.
For an object $x$, its orbit $O(x)$ with respect to the group $G$ is
\begin{equation}
  O(x) = \{ \lambda_g x\; |\;  g\in G\}.
\end{equation}
When seeing $O(x)$ as a set (unordered), it is invariant to the action of $g$.
When seeing it as a list (ordered), it is equivariant, since $O(x)$ and $O(\lambda_u x)$
are related by a permutation.

Through this example it is possible to develop an intuition into the equivariance of the \gconv;
it can be viewed as averaging the inner-products of some function $f$
with all elements of the orbit of a ``flipped'' filter $k$.
Formally, we define the group convolution between \fun{f,\,k}{G}{\R} as
\[(f * k)(g) = \int\limits_{u \in G} f(u)(\lambda_uk)(g)\,dg  =
  \int\limits_{u \in G} f(u)k(u^{-1}g)\,dg \]

As shown in \cref{h:sec:gconv}, group convolution is equivariant to actions of the group.
In this chapter, we are interested in learning equivariant representations of spherical functions.
Since the sphere is not a group but a homogeneous space of \SO{3}, we specialize the
\hconv\ and \hcorr\ defined in~\cref{emvn:sec:homogeneous}.
For \fun{f,\,k}{S^2}{\R}, we have, where $\nu$ is north pole on the sphere,
\begin{align}
  (f * k)(y) &= \int\limits_{g\in \SO{3}} f(g\nu)k(g^{-1}y)\,dg, \label{sph:eq:sphconv}\\
  (f \star k)(g) &= \int\limits_{x\in S^2} f(gx)k(x)\,dx. \label{sph:eq:sphcorr}
\end{align}
Note that $f * k$ is on $S^2$ while $f \star k$ is on $\SO{3}$.
Since the space considered in this chapter is always the sphere,
we refer to \cref{sph:eq:sphconv} as spherical convolution
and to \cref{sph:eq:sphcorr} as spherical cross-correlation.

We also evaluated group convolutions on \cref{ptn:sec:ptn,emvn:sec:emvn},
but those are simpler cases.
In the \ptns\ of \cref{ptn:sec:ptn}, the dilated rotation group is abelian so the
simple change to canonical coordinates transformed the group convolution in a planar convolution.
In the \emvns\ of \cref{emvn:sec:emvn}, only discrete groups and homogeneous spaces were considered,
so the evaluation could be simplified by enumerating all elements.
For evaluation on continuous spaces these techniques do not work;
the solution is computation in the spectral domain, which we discuss next.

\subsection{Spherical harmonics}
\label{sph:sec:spher-conv}
To implement \cref{sph:eq:sphconv,sph:eq:sphcorr},
it is desirable to sample the sphere with well-distributed and compact cells with transitivity
(rotations exist which bring cells into coincidence).
Unfortunately, such a discretization does not exist~\cite{thurston97geotop}.
Neither the familiar sampling by latitude and longitude
nor the uniformly distributed sampling according to Platonic solids satisfies all constraints.
These issues are compounded with the eventual goal of performing cascaded convolutions on the sphere.

To circumvent these issues, we evaluate the spherical convolution in the spectral domain.
This is possible since the machinery of Fourier analysis extendeds
the well-known convolution theorem to functions on the sphere:
the spherical Fourier transforms of the spherical convolution and cross-correlation are products
of spherical Fourier transforms coefficients, as proved in \cref{h:sec:sphcnns}.
Recall the spherical Fourier transform and its inverse for \fun{f}{S^2}{\R} as
discussed in \cref{h:sec:fouriers2}.
For a function \fun{f}{S^2}{\R},
the spherical harmonics \fun{Y_m^\ell}{S^2}{\C} of degree $\ell$ and order $m$,
and coefficients $\hat{f}_m^{\ell} \in \C$, we have
\begin{align}
  f(x) &= \sum_{0 \le \ell \le b}\sum_{|m| \le \ell}\hat{f}_m^{\ell}Y_m^{\ell}(x) \label{sph:eq:isft}, \\
  \hat{f}_m^{\ell} &= \int\limits_{S^2} f(x) \overline{Y_m^{\ell}}(x)\, dx \label{sph:eq:sft},
\end{align}
where $b$ is the bandwidth of $f$.
We refer to \cref{sph:eq:sft} as the \acrfull{sft},
and to \cref{sph:eq:isft} as its inverse (\acrshort{isft}).
Revisiting \cref{sph:eq:sphconv}, we compute the spherical convolution
in the spectral domain  as
\begin{align}
  \widehat{f * k}_m^{\ell} = 2\pi \sqrt{\frac{4\pi}{2\ell+1}} \hat{f}_m^{\ell} \hat{k}_0^{\ell}. \label{sph:eq:sphconvspec}
\end{align}
To compute the convolution of a signal $f$ with a filter $k$,
we (i) expand $f$ and $k$ into their spherical harmonic basis (\cref{sph:eq:sft}),
(ii) compute the pointwise product (\cref{sph:eq:sphconvspec}),
and (iii) invert the spherical harmonic expansion (\cref{sph:eq:isft}).

This definition of spherical convolution differs
from spherical correlation which produces an output response on \SO{3}.
Convolution here can be seen as marginalizing the angle responsible
for rotating the filter about its north pole,
or, equivalently, considering zonal filters on the sphere.

\subsection{Practical considerations and optimizations}
To evaluate the \sft\ on a discretized setting,
we use equiangular samples on the sphere according
to the sampling theorem of~\textcite{driscoll1994computing}
\begin{align}
  \hat{f}_m^{\ell} &= \frac{\sqrt{2\pi}}{2b}\sum_{j=0}^{2b-1}\sum_{k=0}^{2b-1} a_j^{(b)} f(\theta_j, \phi_k)\overline{Y_m^{\ell}}(\theta_j, \phi_k), \label{sph:eq:dsft} \end{align}
where $\theta_j=\pi j/2b$ and $\phi_k=\pi k/b$ form the sampling grid,
and $a_j^{(b)}$ are the sample weights.
All required operations are matrix pointwise multiplications and sums,
which are differentiable and readily available in most automatic differentiation frameworks.
In our direct implementation, we precompute all needed $Y_m^{\ell}$,
and store them as constants in the computational graph.

\paragraph{Separation of variables}
We also implement a potentially faster \sft\ based on separation of variables as shown in \textcite{driscoll1994computing}.
Expanding $Y_m^{\ell}$ in \cref{sph:eq:dsft}, we obtain
\begin{equation}
  \begin{aligned}
  \hat{f}_m^{\ell} &= \sum_{j=0}^{2b-1}\sum_{k=0}^{2b-1} a_j^{(b)} f(\theta_j, \phi_k) q_m^{\ell} P_m^{\ell}(\cos{\theta_j})e^{-im\phi_k}   \\
  &= q_m^{\ell} \sum_{j=0}^{2b-1}a_j^{(b)} P_m^{\ell}(\cos{\theta_j}) \sum_{k=0}^{2b-1}f(\theta_j, \phi_k) e^{-im\phi_k}, \label{sph:eq:sep}
  \end{aligned}
\end{equation}
where $P_m^{\ell}$ is the associated Legendre polynomial,
and $q_m^{\ell}$ a normalization factor.
We compute the inner sum with a row-wise \fft\
and what remains is an associated Legendre transform, computed directly.
The same idea applies for the \isft.
We found that convolution computed using this method is roughly
as efficient as the naive approach when $b = 32$, but \num{2.4} times faster for $b = 64$.
There are faster \sft\ algorithms~\cite{driscoll1994computing,healy2003ffts}, which we did not attempt.

\paragraph{Leveraging symmetry} For real-valued inputs, $\hat{f}_{-m}^{\ell} = (-1)^{m}\overline{\hat{f}_{m}^{\ell}}$ (this follows from $\overline{Y_{-m}^{\ell}} = (-1)^m Y_m^{\ell}$). We thus only need compute half of the coefficients ($m > 0$).
Furthermore, we can rewrite the \sft\ and \isft\ to avoid computationally expensive complex number multiplications:
\begin{equation}
f = \sum_{0 \le \ell \le b} \left(\hat{f}_0^{\ell}Y_0^{\ell} + 2\sum_{m=1}^{\ell}  \,\Re(\hat{f}_m^{\ell})\Re(Y_m^{\ell}) - \,\Im(\hat{f}_m^{\ell})\Im(Y_m^{\ell})\right),
\end{equation}
where $\Re(x)$ indicate the real part of $x$ and $\Im(x)$ the imaginary.

\section{Method}
\label{sph:sec:method}
\Cref{sph:fig:overview} shows an overview of our method.
We define a block as one spherical convolutional layer, followed by optional pooling, and  nonlinearity.
A weighted global average pooling is applied at the last layer to obtain an invariant descriptor.
This section details the architectural design choices.

\begin{figure}[htbp]
\centering
\includegraphics[width=\linewidth]{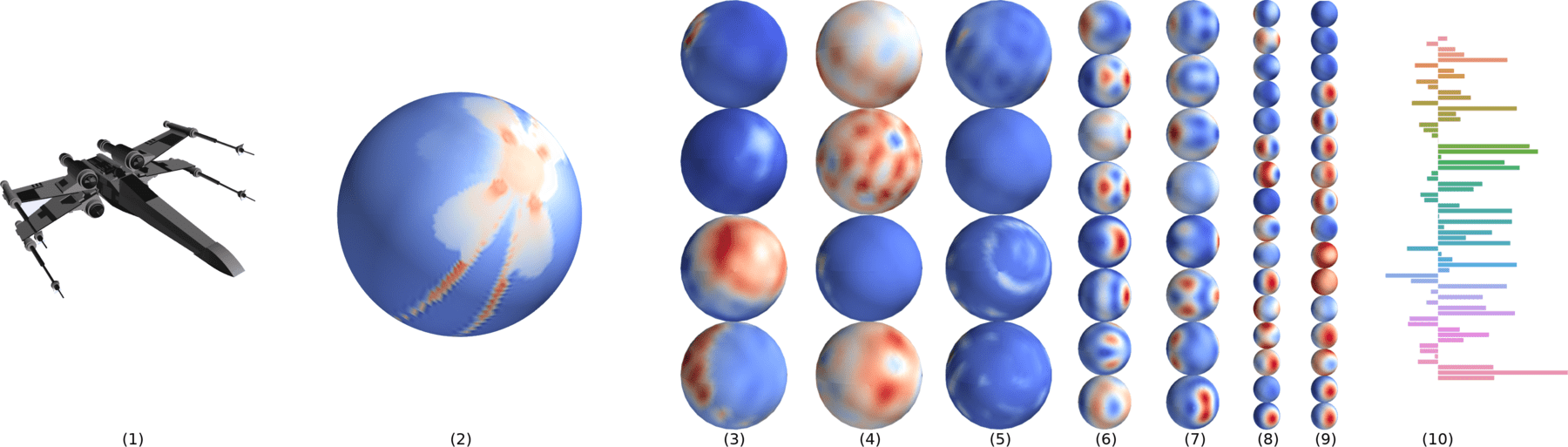}
\caption{
  Overview of our method.
  From left to right: a $3$D model (1) is mapped to a spherical function (2),
  which passes through a sequence of spherical convolutions, nonlinearities and pooling,
  resulting in equivariant feature maps (3--9).
  We show only a few channels per layer.
  A global weighted average pooling of the last feature map results in a descriptor invariant to rotation (10),
  which can be used for classification or retrieval.
  The input spherical function (2) may have multiple channels, in this picture we show the distance to intersection representation.
  }
\label{sph:fig:overview}
\end{figure}

\subsection{Spectral filtering}
In this section, we define the filter parameterization.
One possible approach would be to define a compact support around one of the poles
and learn the values for each discrete location, setting the rest to zero.
The downside of this approach is that there are no guarantees that the filter will be bandlimited.
If it is not, the \sft\ will be implicitly bandlimiting the signal,
which causes a discrepancy between the parameters and the actual realization of the filters
in the form of ringing effects.

To avoid this problem, we parameterize the filters in the spectral domain.
In order to compute the convolution of a function $f$ and a filter $k$, only the \sft\ coefficients of order $m=0$ of $k$ are necessary.
In the spatial domain, this implies that for any $k$,
there is always a zonal filter (constant value per latitude) $k_z$
such that $f * k = f * k_z$ for all $f$.
Thus, it makes sense to constrain the learned filters to be zonal.

The spectral parameterization is also faster because
it eliminates the need to compute the filter \sft,
since the filters are already in the spectral domain
as required by the convolution computation.

\paragraph{Non-localized filters}
A first approach is to parameterize the filters by all \sft\ coefficients of order $m=0$,
which are real-valued when the filter is real-valued.
For example, given $32 \times 32$ inputs, the maximum bandwidth is $b=16$,
so there are $16$ parameters to be learned: $\hat{h}_0^0, \ldots \hat{h}_0^{15} $.
A downside is that the filters may not be local; however, locality may be learned.

\paragraph{Localized filters}
\label{sph:sec:locfilters}
From Parseval's theorem and the derivative rule from Fourier analysis we can show that spectral smoothness corresponds to spatial decay.
This idea is used in the construction of graph-based neural networks \cite{bruna2013spectral},
and also applies to the filters spanned by the family of spherical harmonics of order zero ($m=0$).

Consider a normalized, zero-mean zonal filter $k'(\theta,\phi) = k(\cos\theta)$
and the functional $\Lambda_k$, which measures how spread out $k'$ is
with respect to the north pole ($\theta=0$):
\begin{align}
  \Lambda_k = \int\limits_{-1}^1 (x-1)^2k(x)^2dx.
  \label{sph:eq:loc}
\end{align}
Let us write $(x-1)k(x)$ in terms of $\hat{k}^\ell$, the Legendre coefficients of $k(x)$.
We'll need the following recursive relation between the Legendre polynomials~\cite{lebedev1972special}
\begin{align}
xP^{\ell}(x) = \frac{(\ell+1)P^{\ell+1}(x) + \ell P^{\ell-1}(x)}{2\ell + 1}.
\end{align}
We write
\begin{align*}
  k(x) &= \sum_{\ell=1}^\infty \hat{k}^\ell P^\ell(x),\\
  xk(x) &= \sum_{\ell=1}^\infty \hat{k}^\ell
          \frac{(\ell+1)P^{\ell+1}(x) + \ell P^{\ell-1}(x)}{2\ell + 1}, \\
  (x-1)k(x) &= \sum_{\ell=1}^\infty \left( \frac{\ell\hat{k}^{\ell-1}}{2\ell-1} +
              \frac{(\ell+1)\hat{k}^{\ell+1}}{2\ell+3} -
              \hat{k}^\ell \right)
              P^\ell(x) \\
       &\approx \sum_{\ell=1}^\infty \left( \frac{\hat{k}^{\ell-1}}{2} +
              \frac{\hat{k}^{\ell+1}}{2} -
              \hat{k}^\ell \right)
                P^\ell(x) \\
        &= \sum_{\ell=1}^\infty \frac{\Delta_{2} \hat{k}^\ell}{2} P^\ell(x),
\end{align*}
where $\Delta_{2} \hat{k}^\ell$ is the second order finite difference
of the coefficients $\hat{k}$ around $\ell$, a metric of smoothness.
We finally return to \cref{sph:eq:loc} and write
\begin{align}
  \Lambda_k = \int\limits_{-1}^1 (x-1)^2k(x)^2dx \approx \sum_{\ell=1}^\infty (\Delta_{2} \hat{k}^\ell)^2n_{\ell}^2,
\end{align}
where $n_\ell$ are constants.
This shows that minimizing second order finite differences of Legendre coefficients results in localized filters.
In particular, $\Delta_{2} \hat{h}^\ell$ is zero when $\hat{h}^{\ell-1}$, $\hat{h}^{\ell}$,
and $\hat{h}^{\ell+1}$ are collinear, which is what we encourage.

We fix $n$ uniformly spaced degrees $\ell_i$ (denoted anchor points) and learn the correspondent coefficients $\hat{h}^{\ell_i}$.
The coefficients for missing degrees are then obtained by linear interpolation.
Given consecutive anchor points at $\ell_i$ and $\ell_j$,
we have $\Delta_{2} \hat{h}^\ell=0$ for all $\ell_i < \ell < \ell_j$,
encouraging filter localization.

A second advantage of this procedure is that
the number of parameters per filter is independent of the input resolution.
\Cref{sph:fig:sphconv} illustrates the complete spherical convolution computation with localized filters, and
\cref{sph:fig:conv0} shows some filters learned by our model; the right side filters
are with the localization procedure.

\begin{figure}[htbp]
\centering
\includegraphics[width=\linewidth]{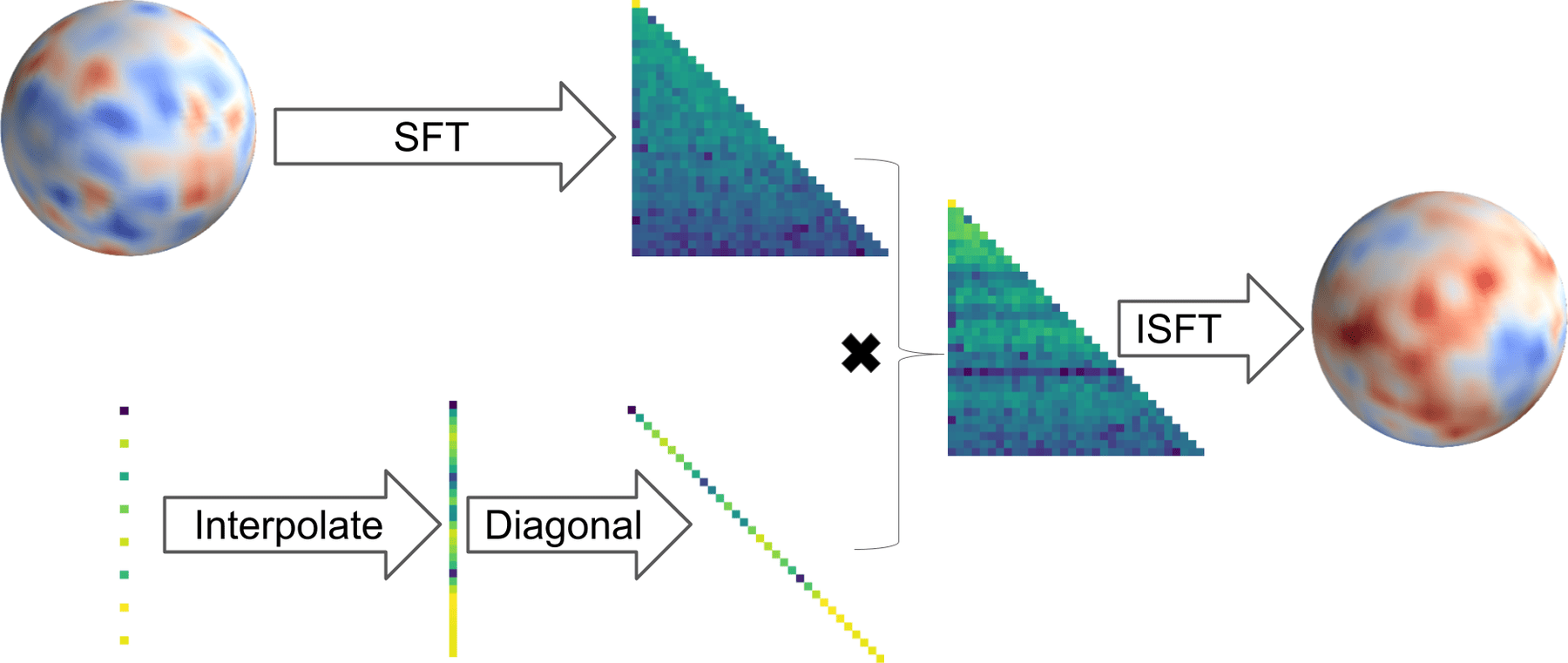}
\caption{Spherical convolution with localized filters.
  We arrange the input \sft\ in a lower-triangular matrix where the $i$-th row
  contains coefficients of order $\ell=i$.
  The anchor points shown in the bottom-left are learned (eight parameters, in this example);
  the rest of the filter spectrum is linearly interpolated.
  Then, evaluation of \cref{sph:eq:sphconvspec} for all degrees is a simple multiplication with a
  diagonal matrix constructed from the zonal filter coefficients.
  Finally, we apply the \isft\ to the resulting spectrum to recover the output spherical function.
  }
  \label{sph:fig:sphconv}
\end{figure}

\begin{figure}[htbp]
\centering
\includegraphics[width=\linewidth]{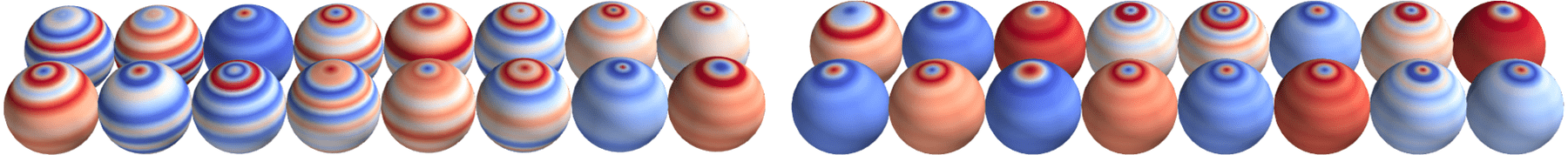}
\caption{
  Filters learned in the first layer.
  The filters are zonal.
  \emph{Left:} 16 nonlocalized filters. \emph{Right:} 16 localized filters.
  Nonlocalized filters are parameterized by all spectral coefficients (16, in the example).
  Even though locality is not enforced, some filters learn to respond locally.
  Localized filters are parameterized by a few points of the spectrum (4, in the example), the rest of the spectrum is obtained by interpolation; notice how the energy is more concentrated around the pole.
  }
  \label{sph:fig:conv0}
\end{figure}
\subsection{Pooling}
The conventional spatial max pooling used in \cnns\ has two drawbacks in spherical \cnns:
(i) it requires an expensive \isft\ to convert back to spatial domain, and
(ii) equivariance is not fully preserved,
especially because of unequal cell areas from equiangular sampling.
\Wap\ takes into account the cell areas to mitigate the latter,
but is still affected by the former.

We introduce the \spp\ for spherical \cnns. %
If the input has bandwidth $b$, we remove all coefficients with degree larger or equal than $b/2$ (effectively, a lowpass box filter).
Such operation causes ringing artifacts,
which can be mitigated by previous smoothing,
although we did not find any performance advantage in doing so.
Note that spectral pooling was proposed before for conventional \cnns~\cite{rippel15_spect_repres_convol_neural_networ},
where the high-frequency $2$D Fourier transform coefficients are dropped.

We found that spectral pooling is significantly faster%
\footnote{For the experiments in \cref{sph:tab:ablation},
  one epoch for the \wap\ model in the first row takes \num{234}{s},
  versus \num{132}{s} for the \spp\ model in the third row, both on a Nvidia 1080 Ti.},
reduces the equivariance error, but also reduces classification accuracy.
The choice between \spp\ and \wap\ is application-dependent.
For example, we found that \spp\ more suitable for applications that directly require low equivariance error, such as shape alignment.
\Cref{sph:tab:equivariance} shows the equivariance errors, while
\Cref{sph:tab:ablation} shows the classification performance for each method.

\subsection{Global pooling}
In fully convolutional networks,
it is usual to apply a global average pooling at the last layer to obtain a descriptor vector
where each entry is the average of one feature channel.
We use the same idea; however, the equiangular spherical sampling results in cells of different areas,
so we compute a weighted average instead, where a cell's weight is the sine of its latitude.
We denote it \wgap.
Note that the \wgap\ is invariant to rotation, therefore the descriptor is also invariant. \Cref{sph:fig:invariance} shows examples of such descriptors.

An alternative is to use the magnitude per degree of the \sft\ coefficients;
formally, if the last layer has bandwidth $b$ and
$\hat{f^{\ell}} = [\hat{f}_{-\ell}^{\ell},\hat{f}_{-\ell+1}^{\ell}, \ldots, \hat{f}_{\ell}^{\ell}]$,
then $d = \left[\norm{\hat{f}^0}, \norm{\hat{f}^1}, \ldots \norm{\hat{f}^{b-1}}\right]$
is an invariant descriptor \cite{arfken1966mathematical}.
We denote this approach \magl\ (magnitude per degree $\ell$).
We found no difference in classification performance when using it (see \cref{sph:tab:ablation}).

\begin{figure}[htbp]
\centering
\includegraphics[width=\linewidth]{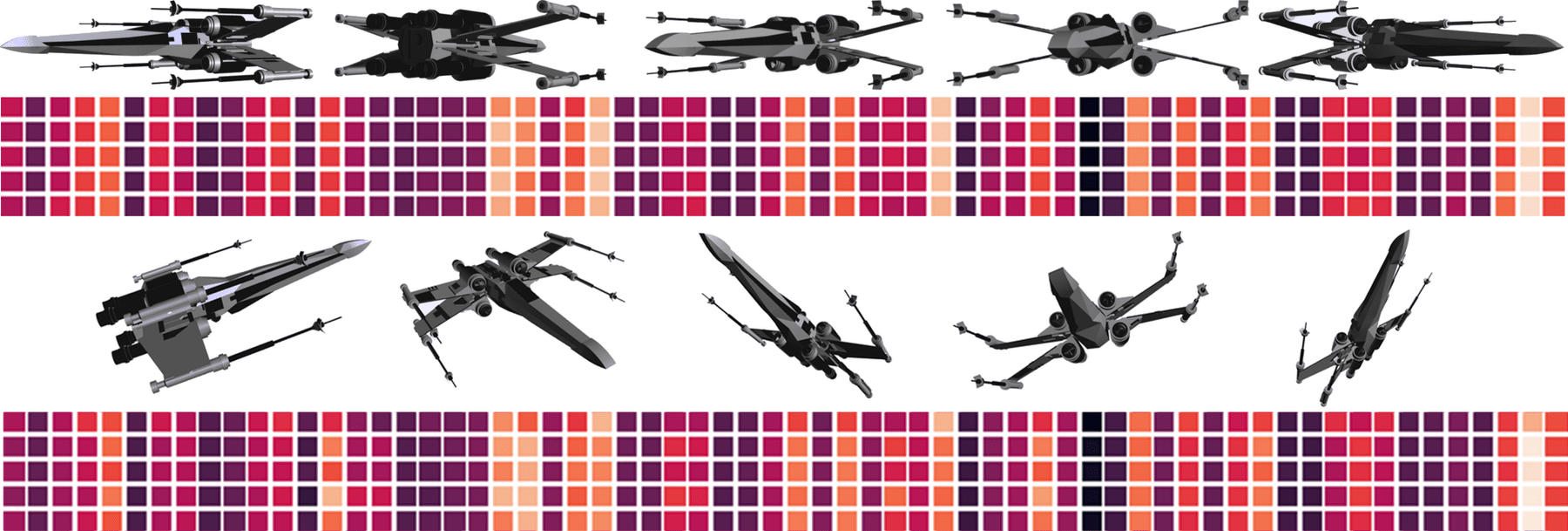}
\caption{
  Our model learns descriptors that are nearly invariant to input rotations.
  From top to bottom: five azimutal rotations and correspondent descriptors (one per row),
  five arbitrary rotations and correspondent descriptors.
  The invariance error is negligible for azimuthal rotations;
  since we use equiangular sampling, the cell area varies with the latitude,
  and rotations around $z$ preserve latitude.
  Arbitrary rotations brings a small invariance error, for reasons detailed in \cref{sph:sec:equivariance}.
  }
\label{sph:fig:invariance}
\end{figure}

\subsection{Spherical sampling}
The most common way to sample a function on the sphere is with an equiangular grid.
For instance, we use the grid from \textcite{driscoll1994computing} in most experiments, defined for $n \times n$ resolution as $\theta_i=\pi i/n$, $\phi_j=2\pi j/n$ with $0 < i,\,j < n-1$.

A major problem with equiangular grids is that the sampling near the poles is
much finer than near the equator.
This would not be an issue if we always had bandlimited input signals,
but there is no such guarantee when inputs are constructed from arbitrary meshes, as is our case.
This manifests as equivariance errors,
because some high frequency details may only come to light under certain orientations.

A potential improvement is the \healpix\ spherical grid~\cite{healpix},
which is widely used in the astrophysics community and has several appealing properties:
\begin{itemize}
\item \textbf{Hierarchical}
  The grid consists of a quadrilateral mesh on the sphere.
  At the coarsest resolution it has 12 cells; to increase the resolution, each cell is divided in 4.
  This is convenient when performing pooling,
  as it is trivial to obtain any cell's parent at a lower resolution.
  When using the \healpix\ grid we do not apply spectral or weighted average pooling;
  the average or max over sibling cells is the proper aggregation operation.
\item \textbf{Equal area}
  The area of all quadrilateral cells at some resolution is the same,
  which results in an uniform sampling of the sphere.
\item \textbf{Iso-latitude}
  The \healpix\ pixels cannot be arranged in a $2$D matrix as the equiangular grids.
  However, they are arranged in a number of parallel latitude circles.
  This allows some memory savings by using a method similar to the separation of variables in
  \cref{sph:eq:sep}, where each latitude circle is processed separately.
\end{itemize}
\Cref{sph:fig:dh-hp} shows the grid points and examples of a mesh converted to spherical function using different grids.

\begin{figure}[htbp]
\centering
\includegraphics[width=\linewidth]{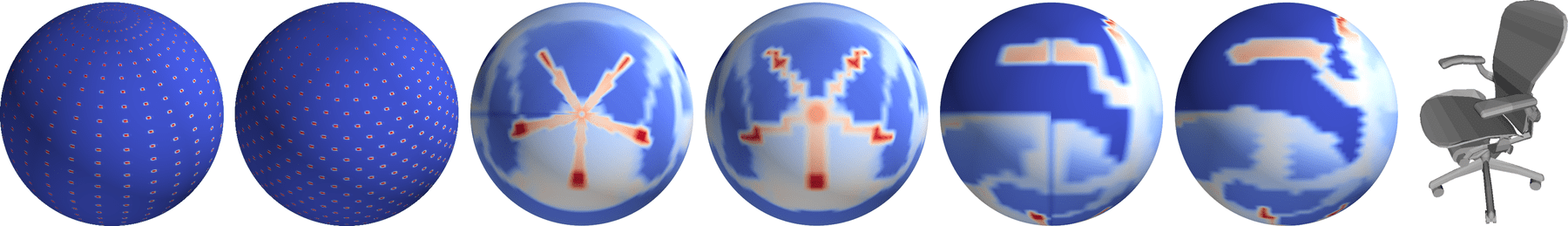}
\caption{
  We show the grids and views of a mesh of a chair converted to spherical functions
  $f_{E}$ and $f_H$ for the equiangular and \healpix\ grids, respectively.
  From left to right:
  (1) equiangular grid,
  (2) \healpix\ grid,
  (3) south pole of $f_E$,
  (4) south pole of $f_H$,
  (5) equator of $f_E$,
  (6) equator of $f_H$,
  (7) original mesh.
  Note how the equiangular grid results in higher resolution at the poles and lower at the equator,
  while \healpix\ is approximately uniform everywhere.
  In particular, the \healpix\ capture better the arms of the chair,
  even with fewer points in total (3072 vs 4096).
  }
\label{sph:fig:dh-hp}
\end{figure}

However, one disadvantage is that while there are sampling theorems that guarantee exact spherical harmonics decomposition and reconstruction of bandlimited functions for equiangular grids \cite{driscoll1994computing,healy2003ffts}, no such theorems exist for arbitrary grids.
This means that applying an \sft\ followed by an \isft\ to a function $f$
sampled on a \healpix\ grid does not result in $f$, even when $f$ is bandlimited.
Our experiments show that the advantages of a uniform grid are worth anyway (see \cref{sph:tab:m40}).

Most of our equiangular grid implementation is also applicable for the \healpix\ grid,
because after conversion to the spectral domain, the input grid does not matter anymore.
The number and arrangement of spectral coefficients is the same for any input grid.
Since our filters are also defined in the spectral domain,
the same model can be used with any input grid.
\subsection{Architecture}
\label{sph:sec:arch}
Our main architecture has two branches, one for distances and one for surface normals.
This performs better than having two input channels
and slightly better than having two separate voting networks for distance and normals.
Each branch has eight spherical convolutional layers,
and $16,\,16,\,32,\,32,\,64,\,64,\,128,\,128$ channels per layer.
We perform pooling and feature concatenation of one branch into the other
when the number of channels increase, and eight anchor points per filter are used.
\Wgap\ is applied after the last layer, which is then projected into the number of classes.

\section{Experiments}
\label{sph:sec:experiments}

The major advantage of our model is inherent equivariance to \SO{3} and
our experiments are on tasks that benefit from it;
namely, shape classification and retrieval in arbitrary orientations, shape alignment,
and panoramic image segmentation.
The focus is on problems related to $3$D shapes due to
the availability of large datasets and published results on them.
\subsection{Preliminaries}
\paragraph{Shape to sphere projection}
\label{sph:sec:spherical-3d-object}
$3$D shapes are usually represented by mesh or voxel grid,
which need to be converted to spherical functions.
The conversion function itself must be equivariant to rotations;
our learned representation will not be equivariant
if the input is pre-processed by a non-equivariant function.

Given a mesh or voxel grid, we first find the bounding sphere and its center.
For the equiangular grid%
\footnote{An analogous procedure applies when using the \healpix\ grid.}
given a desired resolution $n$,
we cast $n \times n$ equiangular rays from the center,
and obtain the intersections between each ray and the mesh/voxel grid.

Let $d_{jk}$ be the distance from the center to the farthest point of intersection,
for a ray at direction $(\theta_j, \phi_k)$.
We define the function on the sphere as $f(\theta_j, \phi_k) = d_{jk}$, $1 \le j,k \le n$.

For mesh inputs, we also compute the angle $\alpha$ between the ray and the surface normal at the
intersecting face, yieding two channels $f(\theta_j, \phi_k) =  [d_{jk}, \sin\alpha]$.

Technically, this representation is suitable for star-shaped objects,
defined as objects that contain an interior point from where the whole boundary is visible.
Moreover, the center of the bounding sphere must be one of such points.
In practice, we do not check if these conditions hold --
results show that even if the representation is ambiguous or non-invertible, it is still useful.

We do not present results for point clouds,
but projection to the sphere is also possible for this kind of input.
We can assign each point to the closest ray and
average or max pooling with respect to the distance to center
can then be used to obtain a single channel on the sphere.

\paragraph{Training}
Except when stated otherwise, we train using Adam~\cite{KingmaB14},
for $48$ epochs, initial learning rate of \num{1e-3}, divided by $5$ on epochs $32$ and $40$.
We use data augmentation for training, performing rotations,
anisotropic scaling and mirroring on input meshes,
and adding jitter to the bounding sphere center when constructing the spherical function.
Even though our learned representation is equivariant to rotations,
augmenting the inputs with rotations is still beneficial due to interpolation and sampling effects.

\subsection{Rotated handwritten digit classification}
Our initial experiment is on the spherical \mnist\ dataset introduced by \textcite{s.2018spherical}.
The dataset consists of handwritten digits from \mnist\
projected into a hemisphere and optionally rotated.
On the rotated versions, each of the \num{50}{k} \mnist\ test entries
is assigned one of 100 possible rotations,
and the \num{10}{k} test entries are assigned to 20 different possible rotations.
This requires generalization to unseen rotations to achieve good performance.
We do not perform rotation augmentation in this experiment to keep the comparison fair.

We utilize a network with six spherical convolutional layers,
and $16,\,16,\,32,\,32,\,58,\,58$ channels per layer,
with a total of \num{57}{k} parameters to match \textcite{s.2018spherical}.
Pooling is performed when the number of channels increase.
We train for $12$ epochs, initial learning rate of \num{1e-3},
divided by $5$ on epochs $6$ and $10$.

\Cref{sph:tab:sphmnist} shows the results.
We outperform the baseline in all modes, which evidences that
the limitation of our zonal filters is overcome by having deeper and wider networks,
which is possible because the spherical convolutions we use
are much more efficient than the \SO{3} cross-correlations of \textcite{s.2018spherical}.
We manage to keep the number of parameters low even with
deeper networks by parameterizing the spectra as described in~\cref{sph:sec:locfilters}.

\begin{table}[htbp]
  \caption{Spherical \mnist\ classification accuracy.
    \emph{c} means canonical orientation (no rotation).
    \emph{x/y} indicates training on \emph{x} and testing on \emph{y}.
    Comparison is against \textcite{s.2018spherical}.
   \label{sph:tab:sphmnist}}
  \centering
    {
      \begin{tabular}{@{} l
        S[table-format=1.3]
        S[table-format=1.3]
        S[table-format=1.3]
        S[table-format=2e1] @{}}
      \toprule
      Method         & {c/c}      & {\sotsot} & {c/\SO{3}} & {\# params} \\
      \midrule
      planar \cite{s.2018spherical}   & \bgl 0.98  & 0.23            & 0.11       & 58e3         \\
      $S^2$CNN \cite{s.2018spherical} & 0.96       & \bgl 0.95       & \bgl 0.94  & 58e3         \\
      Ours     & \bgd 0.987 & \bgd 0.985      & \bgd 0.981 & 57e3         \\
      \bottomrule
    \end{tabular}
    }
\end{table}

\subsection{3D object classification}

\begin{table}[htbp]
  \caption{ModelNet40 classification accuracy per instance.
    Spherical CNNs are robust to arbitrary rotations, even when not seen during training,
    while also having one order of magnitude fewer parameters and faster training. \label{sph:tab:m40}
  }
  \centering
  \begin{tabular}{l
    S[table-format=2.1]
    S[table-format=2.1]
    S[table-format=2.1]
    S[table-format=2.1e1]
    r}
      \toprule
      Method                                  & {\zz}     & {\sotsot} & {\zsot} & {params}   & {inp. size}            \\
      \midrule
      PointNet \cite{qi2017pointnet}          & 89.2      & 83.6            & 14.7       & 3.5e6      & \bgl {$2048 \times 3$}      \\
      PointNet++ \cite{qi2017pointnet++}      & 89.3      & 85.0            & 28.6       & 1.7e6      & \bgd {$1024 \times 3$}      \\
      VoxNet \cite{maturana2015voxnet}        & 83.0      & 73.0            & {-}        & 0.9e6      & {$30^3$}               \\
      SubVolSup \cite{vam}                    & 88.5      & 82.7            & 36.6       & 17e6       & {$30^3$}               \\
      SubVolSup MO \cite{vam}                 & 89.5      & 85.0            & 45.5       & 17e6       & {$20 \times 30^3$}     \\
      MVCNN 12x \cite{su2015multi}            & 89.5      & 77.6            & 70.1       & 99e6       & {$12 \times 224^2$}    \\
      MVCNN 80x \cite{su2015multi}            & \bgl 90.2 & 86.0            & \bgl 81.5  & 99e6       & {$80 \times 224^2$}    \\
      RotationNet 20x \cite{kanezaki16_rotat} & \bgd 92.4 & 80.0            & 20.2       & 58.9e6     & {$20 \times 224^2$}    \\
      Ours (equiangular)                      & 88.9      & \bgl 86.9       & 78.6       & \bgd 0.5e6 & {$2 \times 64^2$}      \\
      Ours (\healpix)                          & 88.3      & \bgd 87.4       & \bgd 82.6  & \bgd 0.5e6 & \bgl {$2 \times 3072$} \\
      \bottomrule
    \end{tabular}
\end{table}

This section shows classification performance on ModelNet40~\cite{wu20153d}.
We consider the following three modes.
\begin{description}
\item[\zz] trained and tested with azimuthal rotations,
\item[\sotsot] trained and tested with arbitrary rotations, and
\item[\zsot] trained with azimuthal and tested with arbitrary rotations.
\end{description}
\Cref{sph:tab:m40} shows the results.
All competing methods suffer a sharp drop in performance when arbitrary rotations are present,
even when they are seen during training.
Our model is more robust, but there is a noticeable drop for mode \zsot.
In the equiangular sampling case, the cell area varies with latitude.
Rotations around $z$ preserve latitude,
so regions at same height are sampled at same resolution during training,
but not during test in mode \zsot.
We show that this is improved by using the \healpix\ spherical sampling.
Even at a lower resolution (3072 vs 4096 pixels), the \healpix\ grid achieves superior performance in when the full rotation group is considered.

We evaluate competing methods using default settings of their published code.
The volumetric~\cite{vam} and point cloud based~\cite{qi2017pointnet,qi2017pointnet++}
methods cannot generalize to unseen orientations (\zsot).
The multi-view~\cite{su2015multi,kanezaki16_rotat} methods can be seen
as a brute force approach to equivariance;
\acrshort{mvcnn}~\cite{su2015multi} generalizes to unseen orientations up to a point.
Yet, our spherical \cnn\ outperforms it,
even with orders of magnitude fewer parameters and faster training.
Interestingly, RotationNet~\cite{kanezaki16_rotat}, which is the
state of the art on ModelNet40 classification,
fails to generalize to unseen rotations, despite being multi-view based.
This was also observed in one of their supplementary experiments,
and communication with the authors confirmed our evaluation results.

Equivariance to \SO{3} is not needed when only azimuthal rotations are present (\zz);
the full potential of our model is not exercised in this case.
The multi-view based models outperform ours with limited rotations due to ImageNet~\cite{Russakovsky2015imagenet} pre-training and their extra capacity,
which allows discriminating small details between shapes.

\subsection{3D object retrieval}
We run retrieval experiments on ShapeNet Core55~\cite{shapenet2015},
following the SHREC'17 $3$D shape retrieval rules~\cite{savva2017shrec},
which include random \SO{3} perturbations.

We train the network for classification on the 55 core classes (we do not use the subclasses),
with an extra in-batch triplet loss  to encourage descriptors to be close
for matching categories and far for non-matching.
The triplet loss follows \textcite{schroff2015facenet}.
Let $f$ produce descriptors for a given input,
($p_i$, $p_j$) be pairs with the same label in the same batch and $\alpha$ be a margin,
\begin{equation}
  \mathcal{L} = \sum_{(p_i,p_j)}\norm{f(p_i)-f(p_j)} - \norm{f(p_i)-f(n_{i,j})} + \alpha,
\end{equation}
where we obtain the $n_{i,j}$
using semi-hard negative mining over the in-batch elements $n$ that have different label than $p_i$:
\begin{align}
  & n_{i,j} =  \argmin_n \norm{f(p_i) - f(n)}  \nonumber \\
  &\text{such that } \norm{f(p_i)-f(n)} > \norm{f(p_i)-f(p_j)}.
\end{align}

The invariant descriptor is used with a cosine distance for retrieval.
We first compute a threshold per class that maximizes the training set F-score.
For test set retrieval, we return elements whose distances are below their class threshold and include all elements classified as the same class as the query.
\Cref{sph:tab:shrec} shows the results.
Our model matches the state-of-the-art performance at the time
(from \textcite{furuya2016deep}),
with significantly fewer parameters, smaller input size, and no pre-training.

\begin{table}[htbp]
  \caption{SHREC'17 perturbed dataset results.
    We show precision (P), recall (P) and mean average precision (mAP).
    \emph{micro} average is adjusted by category size, \emph{macro} is not.
    The sum of \emph{micro} and \emph{macro} mAP is the score used for ranking.
    We match the state of the art even with significantly fewer parameters,
    smaller input resolution, and no pre-training.
    Top results background is dark, runner-ups light.
  }%
  \centering
  \scriptsize
  \begin{tabular}{l
    S[table-format=1.2,table-auto-round]S[table-format=1.2,table-auto-round]S[table-format=1.2,table-auto-round] c
    S[table-format=1.2,table-auto-round]S[table-format=1.2,table-auto-round]S[table-format=1.2,table-auto-round] c
    S[table-format=1.4]
    rS[table-format=2.1]}
    \toprule
                                           & \multicolumn{3}{c}{micro} &            & \multicolumn{3}{c}{macro} &  &  &  & {\multirow{2}{*}{params}}                                            \\
    \cmidrule{2-4} \cmidrule{6-8}
                                           & {P@N}                       & {R@N}        & {mAP}                       &  & {P@N}                      & {R@N}                           & {mAP}        &  &     {score}       &     {input size}                  & {$\times 10^6$} \\
    \midrule
    Furuya et al. \cite{furuya2016deep}    & \bgd 0.814                & 0.683      & 0.656                     &  & \bgd 0.607               & 0.539                         & \bgd 0.476 &  & \bgl 1.132 & {$126\times 10^3$}    & 8.4             \\
    Ours (equiangular)                     & \bgl 0.717                & 0.737      & 0.685                     &  & \bgl 0.450               & 0.550                         & \bgl 0.444 &  & 1.129      & \bgl {$2\times 64^2$} & \bgd 0.5        \\
    Ours (\healpix)                        & 0.695                     & \bgd 0.774 & \bgl 0.692                &  & 0.416                    & \bgd 0.606                    & 0.442      &  & \bgd 1.134 & \bgd {$2\times 3072$} & \bgd 0.5        \\
    Tatsuma et al. \cite{tatsuma2009multi} & 0.705                     & \bgl 0.769 & \bgd 0.696                &  & 0.424                    & \bgl 0.563                    & 0.418      &  & 1.11       & {$38\times224^2$}     & 3               \\
    \textcite{s.2018spherical}             & 0.701                     & 0.711      & 0.676                     &  & {-}                      & {-}                           & {-}        &  & {-}        & {$6\times 128^2$}     & 1.4             \\
    \textcite{bai2016gift}                 & 0.660                     & 0.650      & 0.567                     &  & 0.443                    & 0.508                         & 0.406      &  & 0.97       & $50\times224^2$       & 36              \\
    \bottomrule
  \end{tabular}
  \label{sph:tab:shrec}
\end{table}

\subsection{Shape alignment}
\label{sph:sec:align}
Our learned equivariant feature maps are applicable to shape alignment using spherical correlation.
Given two shapes from the same category (not necessarily the same instance), under arbitrary orientations, we input them to the network and collect the feature maps at some layer.
We compute the correlation between each pair of corresponding feature maps, and add the results.
The result is a real-valued function on \SO{3}. The input that maximizes this function
corresponds to the rotation that aligns both shapes~\cite{makadia2010spherical}.

Features from deeper layers are richer and carry more semantic value, but are at lower resolution.
We run an experiment to determine the performance of the shape alignment per layer,
while also comparing with the spherical correlation done at the network inputs (not learned).

\begin{table}[htbp]
  \caption{Shape alignment median angular error in degrees.
    The intermediate learned features are best suitable for this task.}
  \label{sph:tab:alignment}
  \centering%
    \begin{tabular}{@ {} lSSSS @{}}
      \toprule
      & {bed}        & {chair}      & {sofa}       & {toilet}     \\
      \midrule
      input & 91.63      & 111.47     & 12.15      & 21.65      \\
      conv2 & 85.64      & 21.10      & 14.47      & 14.95      \\
      conv4 & \bgd 12.73 & \bgd 14.63 & \bgd 10.03 & \bgd 11.03 \\
      conv6 & 16.70      & 18.92      & 15.83      & 17.62      \\
      \bottomrule
    \end{tabular}
\end{table}

We select categories from ModelNet10 that do not have rotational symmetry
so that the ground truth rotation is unique and the angular error is measurable.
These categories are: \emph{bed, sofa, toilet, chair}.
Only entries from the test set are used.
Results are in \cref{sph:tab:alignment},
while \cref{sph:fig:alignment} shows some examples.
The learned features are superior to the spherical shape representation (the inputs to our network)
for this task, and best performance is achieved when aligning intermediate layers.
The resolution at conv4 is $32 \times 32$, which corresponds to cell dimensions up to \ang{11.25},
so we cannot expect errors much lower than this.

\begin{figure}[htbp]
\centering
\includegraphics[width=\linewidth]{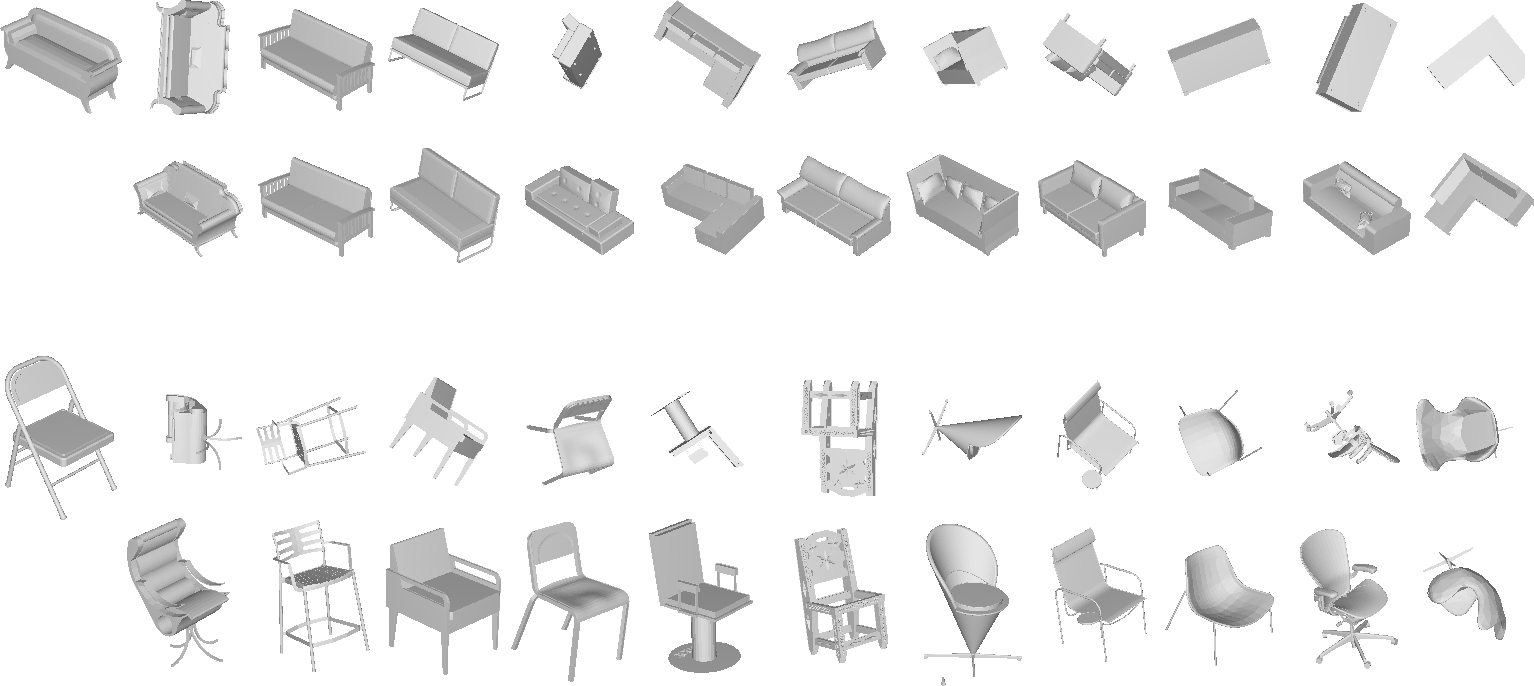}
\caption{
  Shape alignment for two categories.
  We align shapes by running spherical correlation on their feature maps.
  The semantic features learned can be used to align shapes from the same class even with large appearance variation.
  \emph{1st and 3rd rows:} reference shape, followed by queries from the same category.
  \emph{2nd and 4th rows:} Corresponding aligned shapes.
  Last column shows failure cases.
}
\label{sph:fig:alignment}
\end{figure}

\subsection{Equivariance error analysis}
\label{sph:sec:equivariance}
Conventional planar \cnns\ are often said to be translation equivariant,
but in reality they exhibit a degree of translational equivariance error
introduced by max pooling and discretization~\cite{zhang2019shiftinvar}.
Analogous effects happen in our model.

Even though spherical convolutions are equivariant to \SO{3} for bandlimited inputs,
and spectral pooling preserves bandlimit,
there are other factors that may introduce equivariance errors.
We quantify these effects in this section.

We create a new test set by randomly rotating each input in the original test set,
and collect feature maps produced by our model from both sets.
Since each relative rotation is known,
we apply the rotation to the feature maps and measure the average relative error.
\Cref{sph:tab:equivariance} shows the results, which elicit a number of conclusions.
The pointwise nonlinearity does not preserve bandlimit, and cause equivariance errors (rows 1, 4).
The mesh to sphere map is only approximately equivariant,
which can be mitigated with larger input dimensions (\emph{input} column for rows 1, 5).
Error is smaller when the input is bandlimited (rows 1, 7).
Spectral pooling is exactly equivariant,
while max-pooling introduces higher frequencies and has larger error than \wap\ (rows 1, 2, 3).
Error for an untrained model demonstrates that the equivariance is by design and not learned (row 6); the error is actually smaller because the learned filters are usually high-pass,
which increase the pointwise relative error.
A linear model with bandlimited inputs has negligible equivariance error, as expected (row 8).

\begin{table}[htbp]
  \caption{Equivariance error.
    Error is negligible for bandlimited inputs and linear layers.
    Pointwise nonlinearities increase equivariance errors.
    In practice, the error reduces with spectral pooling and larger input/feature resolutions.
  }
  \label{sph:tab:equivariance}
  \centering
  \scriptsize
    \begin{tabular}{@{} l
      ccccc c
      S[table-format=1.2]S[table-format=1.2]S[table-format=1.2]S[table-format=1.2]S[table-format=1.2]S[table-format=1.2]S[table-format=1.2]
      @{}}
      \toprule
                     & \multicolumn{5}{c}{configuration} &       & \multicolumn{7}{c}{error per layer}                                                              \\
      \cmidrule{2-6}       \cmidrule{8-14}
                     & res.                              & blim. & pool & lin. & train &  & {input} & {conv1} & {conv2} & {conv3} & {conv4} & {conv5} & {conv6} \\
      \midrule
      1. baseline    & $64^2$                            & no    & \wap  & no     & yes     &  & 0.05    & 0.11    & 0.12    & 0.14    & 0.16    & 0.17    & 0.15    \\
      2. maxpool     & $64^2$                            & no    & max  & no     & yes     &  & 0.05    & 0.11    & 0.12    & 0.14    & 0.18    & 0.19    & 0.15    \\
      3. specpool    & $64^2$                            & no    & \spp   & no     & yes     &  & 0.05    & 0.11    & 0.12    & 0.10    & 0.10    & 0.09    & 0.08    \\
      4. linear      & $64^2$                            & no    & \wap  & yes    & yes     &  & 0.05    & 0.12    & 0.13    & 0.15    & 0.14    & 0.12    & 0.04    \\
      5. lowres      & $32^2$                            & no    & \wap  & no     & yes     &  & 0.09    & 0.15    & 0.18    & 0.21    & 0.21    & 0.21    & 0.20    \\
      6. untrained   & $64^2$                            & no    & \wap  & no     & no      &  & 0.05    & 0.09    & 0.07    & 0.07    & 0.11    & 0.07    & 0.04    \\
      7. blim        & $64^2$                            & yes   & \wap  & no     & yes     &  & 0.00    & 0.10    & 0.11    & 0.11    & 0.15    & 0.14    & 0.04    \\
      8. blim/lin/sp & $64^2$                            & yes   & \spp   & yes    & yes     &  & 0.00    & 0.01    & 0.01    & 0.00    & 0.00    & 0.00    & 0.00    \\
      \bottomrule
\end{tabular}
\end{table}

\subsection{Ablation study}
In this section we evaluate variations of our method to determine the sensitivity to design choices.
We assess the effects from our contributions \spp, \wap, \wgap, and localized filters,
and evaluate how the network size affects performance.
Results in \cref{sph:tab:ablation} show that the use of \wap, \wgap,
and localized filters significantly improve performance,
and also that larger networks lead to further performance improvements.
In summary, factors that increase bandwidth (e.g. max-pooling) also increase equivariance error and
may reduce accuracy.
Global operations in early layers (e.g., non-local filters) prevent hierarchical feature learning
and also reduce accuracy.

\begin{table}[htbp]
  \caption{Ablation study.
    Spherical \cnn\ classification accuracy on rotated ModelNet40. %
    We compare combinations of input resolution,
    local and global pooling,
    filter localization and number of network parameters.
  }
  \label{sph:tab:ablation}
  \centering
    \begin{tabular}{@{} cccccc S[table-format=2.1] @{}}
      \toprule
      res.   & pool & global pool & loc. & params & details  & {acc. [\%]}     \\
      \midrule
      $3072$ & avg  & avg         & yes  & $0.49$M  & \healpix  & 87.4          \\
      $64^2$ & \wap  & \wgap        & yes  & $0.49$M  & default  & 86.9          \\
      $64^2$ & \wap  & \magl       & yes  & $0.54$M  &          & 86.9          \\
      $64^2$ & \spp   & \wgap        & yes  & $0.49$M  &          & 85.8          \\
      $64^2$ & max  & \wgap        & yes  & $0.49$M  &          & 86.7          \\
      $64^2$ & avg  & \wgap        & yes  & $0.49$M  &          & 86.7          \\
      $64^2$ & \wap  & avg         & yes  & $0.49$M  &          & 86.4          \\
      $64^2$ & \wap  & \wgap        & no   & $0.49$M  &          & 85.9          \\
      \midrule
      $32^2$ & \wap  & \wgap        & yes  & $0.39$M  &          & 85.0          \\
      $32^2$ & \wap  & \wgap        & yes  & $0.69$M  & deeper   & 85.6          \\
      $32^2$ & \wap  & \wgap        & yes  & $1.06$M  & wider    & 85.5          \\
      $32^2$ & \wap  & \wgap        & yes  & $0.12$M  & narrower & 83.8          \\
      \bottomrule
    \end{tabular}
  \end{table}

\section{Extension to panorama segmentation}
\label{sph:sec:pano}
Panoramic sensors are common for tasks that benefit from $360^\circ$ field of views.
For example, omnidirectional sensing for robotic navigation was explored as early as \textcite{yagi90iros},
and panoramic images have provided the building blocks for early VR environments~\cite{chen95siggraph}.
While the hardware profile of the early imaging devices limited their broad adoption
(e.g. mirror-lens catadioptric sensors~\cite{nayar97cvpr}),
recent hardware and algorithmic advances have created a proliferation of consumer-grade $360^\circ$ cameras.
With the resulting surge in panoramic image datasets, it is natural to investigate machine learning solutions for visual perception tasks on the sphere.
Recent efforts include PanoContext~\cite{zhang14eccv}, Im2Pano3D~\cite{song2016im2pano3d},
and \textcite{deng17iccar}.

In this section, we introduce a \schn\ for dense labeling on the sphere,
which is equivariant to camera orientation,
lifting the usual requirement for ``upright'' panoramic images,
and scalable for larger practical datasets.
The \schn\ leverages spherical residual bottleneck blocks arranged in an encoder-decoder style hourglass
architecture~\cite{newell2016stacked} to produce dense labels in an \SO{3}-equivariant fashion.

The approach presented on this section was, to the best of our knowledge,
the first to bring \SO{3}-equivariance to the task of spherical panorama segmentation,
and one of the first to naturally handle the spherical geometry.
After publication, interest in this task has increased and impressive
results were obtained by both equivariant~\cite{CohenWKW19}
and non-equivariant methods~\cite{JiangHKPMN19,zhang2019orientation}.

\subsection{Architecture}
\begin{figure}[htbp]
  \includegraphics[width=\textwidth]{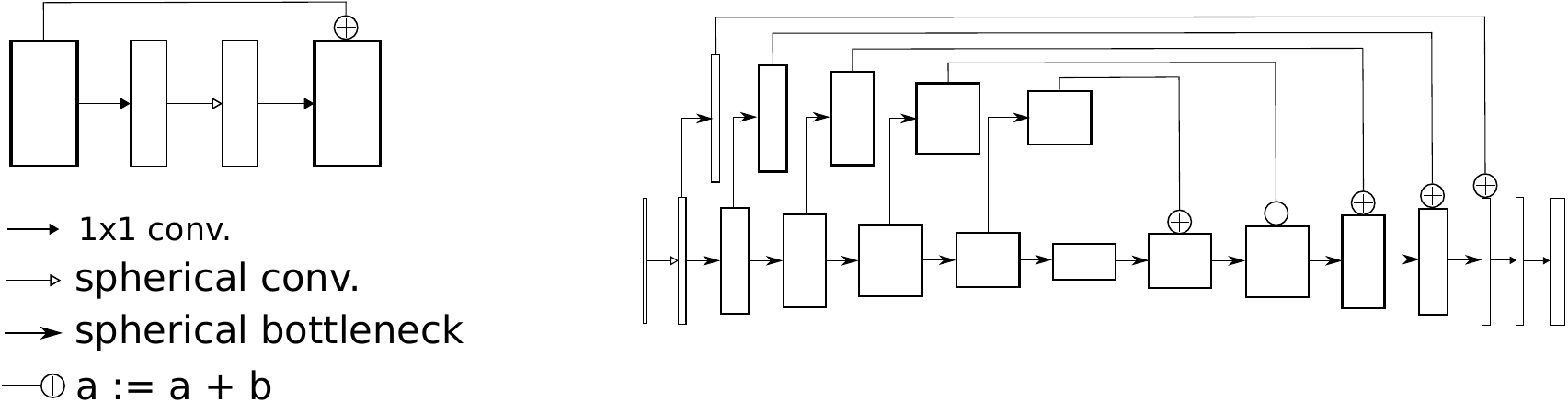}
  \caption{\Schn\ architecture.
    Blocks represent feature maps and arrows, operations. 
    The height of a block represent the spatial resolution and the width, the number of channels.
    Left: spherical residual bottleneck block.
    Right: spherical hourglass network.}
  \label{sphcnn:fig:sphhg}
\end{figure}

For dense labeling, we need deeper and more sophisticated architectures
than the ones presented so far in this chapter.
The outline of our architecture resembles an hourglass, with a series of downsampling blocks
followed by upsampling blocks, enabling high resolution outputs.
One key observation of \textcite{he2016deep} is that a residual block with two $3\times 3$ convolutional
layers can be replaced by a bottleneck block with $1\times 1$, $3\times 3$, and $1\times 1$ layers,
saving compute and increasing performance.
Since $1\times 1$ convolutions are pointwise operations, hence \SO{3}-equivariant,
we can apply the same idea to spherical convolutional layers,
yielding the spherical residual bottleneck blocks (\cref{sphcnn:fig:sphhg}).

Both dilated~\cite{yu2015multi} and deformable~\cite{dai2017deformable} convolutions have proven useful  for semantic segmentation and the spherical filters we use share some of their properties.
As explained in \cref{sph:sec:locfilters}, the number of anchor points in the spectrum loosely
determines their receptive field, as in a dilated convolution.
While the number of anchor points is fixed and small,
the weights learned at these anchors also change the support to some amount.

\subsection{Experiments}
For segmentation experiments we use all the synthetic labeled panoramas
from~\textcite{song2016im2pano3d} (as rendered by~\textcite{song2016ssc}).
We map each sky-box image onto the sphere, with a train-test split of \num{75}{k}-\num{10}{k}.

The input size is $256\times 256$, we use between $32$ and $256$ channels per layer
and $16$ anchor points for filter localization.
We create a $2$D baseline (2DHG) that has the exact same architecture of \schn,
with the spherical convolutions replaced by $2$D convolutional layers with $3\times 3$ kernels.
\Cref{sphcnn:tab:sphhg} shows the results after training for 10 epochs.
\schn\ outperforms the baseline under arbitrary orientations, localized filters outperform global, and using larger models can improve the \schn\ performance.
\Cref{sphcnn:fig:sphhg1} shows some sample outputs of our model and the
2DHG non-equivariant baseline.

\begin{table}[htbp]
  \caption{Spherical panorama semantic segmentation results.
    We show the intersection-over-union (IoU) for different combinations of
    canonical orientation (c), and $\mathbf{SO}(3)$ uniformly sampled perturbations
    on train and test sets.}
  \label{sphcnn:tab:sphhg}
  \centering
  \begin{tabular}{l
    S[table-format=0.3,table-auto-round]
    S[table-format=0.3,table-auto-round]
    S[table-format=0.3,table-auto-round]}
    \toprule
                                       & \multicolumn{3}{c}{train/test orientation}         \\
                                       & {c/c}              & {\sotsot} & {c/\SO{3}}  \\
    \midrule
    2DHG                               & \bgd 0.6393        & 0.5292          & 0.2237      \\
    SCHN/global                        & 0.5343             & \bgl 0.5376     & \bgl 0.4758      \\
    \schn\ (ours)                        & \bgl 0.5683        & \bgd 0.5582     & \bgd 0.5024 \\
    \midrule
    \schn/large                         & 0.5983             & 0.5873          & {-}         \\
    Im2Pano3D \cite{song2016im2pano3d} & 0.330\footnotemark & {-}             & {-}         \\
    \bottomrule
  \end{tabular}
  \footnotetext{Results for a harder extrapolation problem. Included here for reference only.}
\end{table}

\begin{figure}[htbp]
  \centering
  \includegraphics[width=0.8\textwidth]{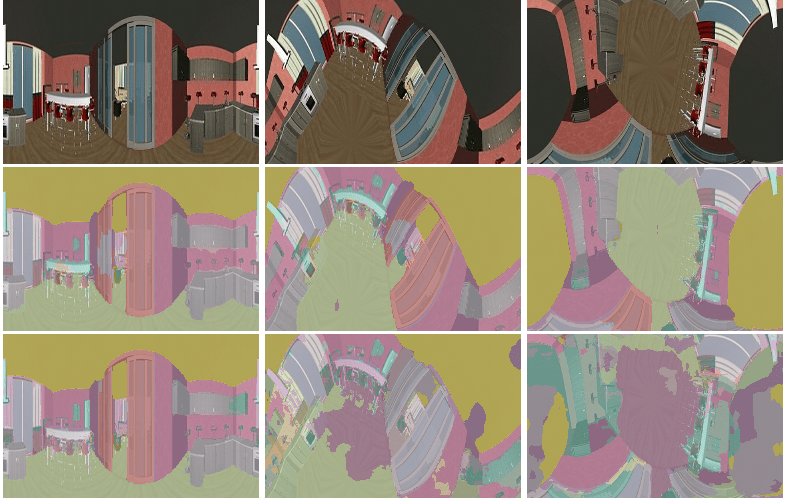}
  \caption{Panorama semantic segmentation results.
    Top: input spherical panoramas.
    Middle: segmentation masks produced by our network.
    Bottom: 2DHG baseline (non-equivariant) results.
    The baseline model can handle only azimuthal rotations (leftmost frame)}
  \label{sphcnn:fig:sphhg1}
\end{figure}

\section{Conclusion}
In this chapter we presented the spherical \cnns,
which leverage spherical convolutions to achieve equivariance to continuous \SO{3} perturbations.
We show applications to $3$D object classification, retrieval, and alignment,
as well as an extension to semantic segmentation of spherical panoramas.
The method is applicable to any data that can be represented as a spherical function;
for example, meteorological and cosmological data are good candidates.
We show that our model can naturally handle arbitrary input orientations,
requiring fewer parameters and smaller input sizes than the alternatives.

\glsresetall
\chapter{Equivariance across domains}
\chaptersubtitle{The Cross-Domain $3$D Equivariant Image Embeddings}

\label{cross:sec:cross}

\section{Introduction}
The success of \cnns\ in computer vision has shown that large training datasets
and task-specific supervision are sufficient to learn rich feature representations
for a variety of tasks such as image classification and object detection~\cite{he16_resnetv2}.
However, numerous challenges remain, such as motion estimation and view synthesis,
which require complex geometric reasoning and for which labeled data is not available at scale.
For such problems there is a trend towards developing models with geometry-aware latent representations
that can learn the structure of the world without requiring full geometric supervision~\cite{kulkarni15nips,rhodin18eccv,yang2015weaklysupervised,yan16nips,mahjourian18cvpr,Eslami1204}.

A desirable property for an image embedding is
robustness to $3$D geometric transformations of the scene.
Rotations are challenging to computer vision algorithms because $3$D rotations of objects in the world
can induce large transformations in image space.
In recent years, there has been much attention given to the study of equivariant neural networks~\cite{cohen2016group,worrall2017harmonic,WeilerHS18},
as equivariant maps provide a natural formulation to address group transformations on images.
Despite these advances, designing a $3$D rotation equivariant map of $2$D images is an open challenge.
This is because the rotation of a $3$D object does not act directly on the pixels
of the resulting image due to the intervening camera projection.
Thus, a map that is equivariant by design cannot be constructed
and instead an (approximate) equivariant map must be learned.
This is the central objective of this chapter:
how to learn an embedding for images of $3$D objects
that is equivariant to $3$D rotations of the objects?

Our solution borrows from recent works on $3$D rotation equivariant \cnns\ for $3$D shape
representations~\cite{s.2018spherical,esteves18eccv}, as presented in \cref{sph:sec:sphcnn}.
These show that spherical convolutional networks can achieve state of the art performance
on $3$D shape classification and pose estimation tasks,
and the equivariance property allows handling of $3$D shapes in arbitrary orientations
with minimal impact on performance.

In this chapter, we propose to learn equivariant embeddings of images by mapping them
into the equivariant feature space of a spherical \cnn\ trained on $3$D shape datasets.
This approach is unique in that we directly supervise the desired target embeddings with
pretrained $3$D shape features, without any other task-specific training losses.
By bootstrapping with features of $3$D shapes, our model
(i) encodes images with the shape properties of the observed object and
(ii) has an underlying spherical structure that is equivariant to $3$D rotations of the object.

The cross-domain embeddings serve different applications, either directly or indirectly,
without requiring additional task-specific supervised training.
We illustrate this point by showing results on two disparate challenges which we now describe.
\paragraph{Relative orientation estimation}
Our model maps images to rotation equivariant embeddings defined on the sphere  (\cref{cross:fig:pose_synth}-left).
The relative orientation between two images of a $3$D object
is the rotation that brings their embeddings into alignment.
We compute it with a simple spherical cross-correlation,
avoiding the usual formulation of pose estimation as
classification~\cite{vpsKpsTulsianiM15} or regression~\cite{mahendran17cvprw} tasks.
This method approaches state-of-the-art performance even though it uses no task-specific training.
The same procedure is useful to align $2$D images with $3$D shapes.

\paragraph{Novel view synthesis}
The learned embeddings also encode enough shape properties to synthesize new views.
By training a decoder from the spherical embedding space with a photometric loss,
we have a model for novel view synthesis.
To generate new views, we simply rotate the latent embedding (\cref{cross:fig:pose_synth}-right)
before feeding it to the decoder.
No task specific supervision in the form of an image and its rotated counterpart is necessary.

\begin{figure}[htbp]
\centering
 \includegraphics[width=\linewidth]{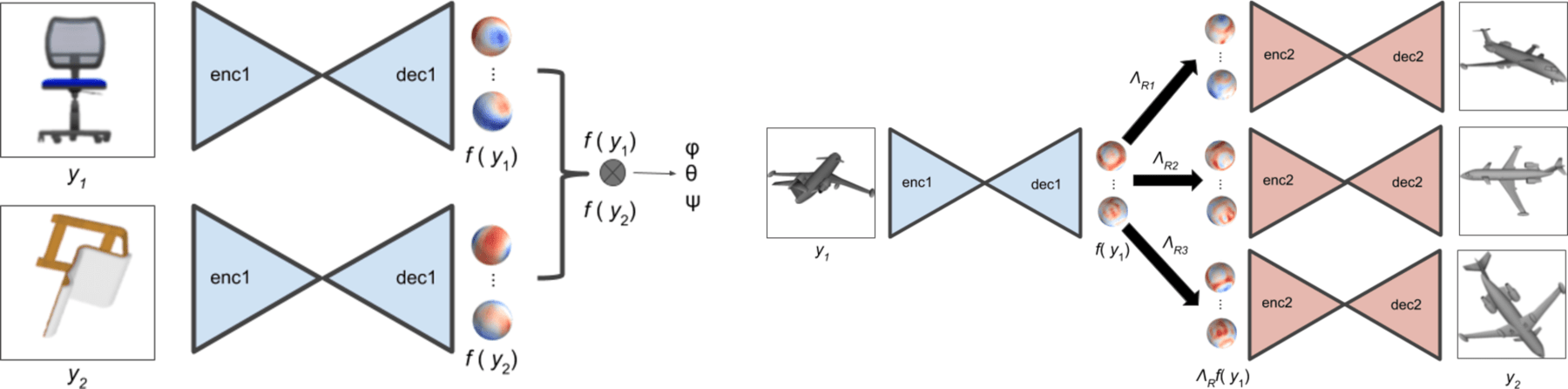}%
\caption{
  \textbf{Overview.} We learn category based spherical $3$D equivariant embeddings that can be correlated for relative pose estimation, and rotated for novel view synthesis.
  \textit{Left: relative pose estimation}.
  Given 2 images of objects from same class, we obtain the respective spherical embeddings.
  The relative pose is computed from the spherical correlation between the spherical embeddings.
  \textit{Right: novel view synthesis.}
  We first embed the input view into the spherical representation,
  then we apply the target rotation to the spherical feature maps, and feed them to the synthesizer to generate novel views.
}
\label{cross:fig:pose_synth}
\end{figure}

To reiterate, our main contribution is a novel cross-domain neural model
that can map $2$D images into a $3$D rotation equivariant feature space.
Generating spherical feature maps from $2$D images is a complex high-dimensional regression task,
mapping between topologies, which requires a novel encoder-decoder architecture.
We consider the relative pose and view synthesis tasks as proxies
for analyzing the representation power of our learned embeddings.
Nonetheless, our promising experimental results indicate
these cross-domain embeddings may be useful for a variety of tasks.

Most of the content in this chapter appeared originally in~\textcite{esteves-icml19}.

\section{Related Work}
\label{cross:sec:related_work}
A number of recent works have introduced geometric structure to
the feature representations of deep neural networks.
The most common setting is to learn intermediate features that can be directly manipulated
or transformed for a particular task.
For example, in \textcite{rhodin18eccv,worrall2017iccv,cohen15iclr,hinton11icann,yang2015weaklysupervised,kulkarni15nips},
geometric transformations can be directly applied to image features
(in some cases disentangled pose features), in order to synthesize new views.
In a related approach, \textcite{tatarchenko16eccv} use an encoder-decoder architecture
that augments pose information to the latent image embedding.

One drawback of these methods is that they typically require full supervision,
and both the geometric transformation parameters and the corresponding target image
must be available during training.
Furthermore, training with source-target pairs requires covering a large sample space --
synthesizing views from arbitrary relative $3$D orientations requires sampling pairs
of poses (sampling space is $\SO{3} \times \SO{3}$).
In contrast, our model is more sample efficient and trains with a single image per example
(sampling space is \SO{3}).

Different to the methods aforementioned, the Homeomorphic VAEs \cite{falorsi18icmlw} provide
an unsupervised way to learn an $\SO{3}$-latent-embedding for images.
However, it is unclear if it scales to practical scenarios,
as it requires a dense sampling of views to learn a continuous embedding
while dealing with intra-class variations.

There is a variety of other approaches to view synthesis.
Most relevant to our setting are the self-supervised methods that
learn geometrically meaningful embeddings using differentiable rendering to match semantic
maps~\cite{yao20183d}, shading information~\cite{henderson2018learning},
fusing latent embeddings from multiple views
and improving synthesis using multiple rendering steps~\cite{Eslami1204}.

\paragraph{Pose Estimation}
The task of object pose estimation has been a long standing problem with numerous applications
in computer vision and robotics.
Most approaches can be categorized as keypoint-based or
direct pose estimation as regression or classification.
Keypoint-based methods for object pose estimation include \textcite{pavlakos17object3d}
and \textcite{grabner18arxiv}, the former predicting semantic keypoints and
the latter bounding box corners, from which object pose follows from a \pnp\ algorithm.
Direct pose estimation methods include \textcite{vpsKpsTulsianiM15} and \textcite{su15iccv} who
formulate it as classification over a quantized viewpoint space.
\textcite{kanezaki16_rotat} train a joint $3$D object classification and pose \cnn\
from multiple views with unknown viewpoints, however the viewpoint sampling is coarse,
providing limited resolution in the estimated pose.
\textcite{mahendran17cvprw} introduces a carefully designed \cnn\ for viewpoint regression,
analyzing different representations and geodesic loss functions,
while \textcite{MousavianCVPR17} introduce a MultiBin orientation regression network.
KeypointNet~\cite{suwajanakorn2018discovery} learns category-specific semantic keypoints
and their detectors using only a geometric loss.
The $3$D keypoints are also useful for determining relative pose,
although the method struggles when exposed to arbitrary $3$D rotations
due to lack of rotation equivariance.

The key ingredient in our approach is a novel method to map $2$D images to
rotation-equivariant $3$D shape embeddings,
essentially encoding an image with $3$D geometric structure.
The choice of geometric representation (spherical embeddings) is intentional
in order to maintain rotation equivariance.
Alternative geometric representations such as volumetric
(e.g., the single-view volumetric reconstruction from \textcite{tulsiani17cvpr})
would not be rotation equivariant,
although \textcite{weiler3dsteerable} could be a reasonable alternative.

\section{Method}
We now detail our image embedding model. We begin by revisiting spherical \cnns\ (\Cref{cross:sec:sphericalcnn}) as a means to learn rich equivariant embeddings for $3$D shapes, and \cref{cross:sec:basic} introduces our cross-domain architecture that learns to map $2$D images into the same embedding space. \Cref{cross:sec:relpose_subsec,cross:sec:nvs_subsec} describe how these image embeddings can be used for relative pose estimation and novel view synthesis.
\subsection{Spherical CNNs}
\label{cross:sec:sphericalcnn}
Recall that the spherical \cnns~\cite{s.2018spherical,esteves18eccv}
and described in \cref{sph:sec:sphcnn}
produce \SO{3}-equivariant feature maps for inputs defined on the sphere,
and were useful for a variety of $3$D shape analysis tasks,
where inputs often appear in arbitrary pose, for which equivariance is particularly helpful.
In this chapter, we use the same spherical convolutional model described in \cref{sph:sec:sphcnn}
due to its efficiency and performance on $3$D shape alignment tasks,
but now we tackle the more challenging problem of relative $3$D pose estimation from $2$D images.

We start by briefly summarizing the spherical \cnns.
For functions $f$ and $k$ defined on the sphere, their convolution is
\begin{align*}
  (f * k)(y) &= \int\limits_{R\in \SO{3}} f(R\nu)k(R^{-1}y)\,dR,
\end{align*}
where $\nu$ is the north pole of the sphere (a stationary point under \SO{2}).
This extends to $K_{in}$ input channels and $K_{out}$ output channels in a straightforward manner
\begin{align}
  (f * k)_{j}(y) &= \sum_{i=1}^{K_{in}} \int\limits_{R\in \SO{3}} f_i(R\nu)k_{ij}(R^{-1}y)\,dR,
\end{align}
where $f_i$ and $(f * k)_{j}$ denote the input, and output channels, respectively.

This convolution is the primary building block of spherical \cnns.
We define $s$ as a spherical \cnn\ that maps $K_{\text{in}}$-channel spherical inputs to
$K_{\text{out}}$-channel spherical feature maps.
Precisely, in the single-channel case we have \fun{s}{L^2(S^2)}{L^2(S^2)} where $L^2(S^2)$ denotes
square-integrability, necessary for evaluation in the spectral domain.

The \SO{3} equivariance of spherical \cnns\ manifests as follows.
For any \fun{x}{S^2}{\R^{K_\text{in}}},
\begin{align}
 s(\lambda_{R} x) =  \lambda_{R} s(x),
 \label{cross:eq:s_equiv}
\end{align}
where $\lambda_R$ is the rotation operator by $R\in \SO{3}$%
\footnote{We use $\lambda_R$ as a generic rotation operator that can be applied to $3$D shapes and spherical functions, scalar or vector-valued. Interpretation should be clear from context.}.
Technically, the equivariance is only approximate as the nonlinear activations (\relus)
and spatial pooling operations break the bandlimiting assumptions which otherwise guarantee
equivariance. However, in practice these errors are negligible.

To use spherical \cnns\ with $3$D shapes, we must provide a map $r$
that converts any $3$D shape $M$ to a spherical representation.
While there are different sensible choices for $r$, we use the simple ray-casting technique
described in \cref{sph:sec:spherical-3d-object}.
Most importantly, the map $r$ is equivariant to $3$D rotations which ensures end-to-end equivariance of our $3$D shape feature maps: $s(r(\lambda_{R} M)) = \lambda_{R}s(r(M))$.

\subsection{Cross-domain spherical embeddings}
\label{cross:sec:basic}
The primary objective of this chapter is to learn an \emph{image} embedding
that can capture similar underlying $3$D shape properties and equivariant structure.
Specifically, we define an \rgb\ image as the projection $c$ of a shape $M$,
where $c$ can be any usual camera projection model, e.g., perspective or orthographic.
We seek a map $f(c(M))$ that captures the shape properties of $M$
and retains an equivariant structure: $f(c(\lambda_R M)) = \lambda_R f(c(M))$.
This is challenging because $c$ is a camera projection which is not $3$D rotation equivariant,
so we cannot have equivariance by construction.
We propose to learn an approximately equivariant embedding model $f$
using a spherical \cnn\ for $3$D shapes, i.e., a pretrained $s(r(M))$, as supervision.
We wish to learn $f$ such that $f(c(M)) = s(r(M))$.
When learned successfully, the equivariance of $f$ follows simply from \cref{cross:eq:s_equiv},
\begin{align}
	f(c(\lambda_R M)) &= s(r(\lambda_R M)) \nonumber \\
                    &= \lambda_R s(r(M))  \\
                    &=\lambda_R f(c(M)). \nonumber
    \label{cross:eq:img_equiv}
\end{align}
Since $c$ and $r$ are fixed and not part of the trainable model, we substitute $y=c(M)$ and $x=r(M)$ going forward to simplify notation.

Learning $f$ involves predicting high dimensional multi-channel spherical maps from a single image.
The two major design challenges are deciding the structure of $f(y)$ and the training loss
$\mathcal{L}(x, y)$ from predicted embedding $f(y)$ to the target ground truth $s(x)$.

\paragraph{Loss function}
We first describe the training loss.
For simplicity, we describe the loss for a single channel
(in general the final loss is aggregated over the channels).
We represent the spherical function $s(x)$ on an equiangular grid;
a discretized $s(x)$ of resolution $N \times N$ %
is indexed by pairs $(\theta_i, \phi_j)$, where $i,j \in \left\{0,1,...,N-1\right\}$.
The set $\{\theta_i\}$ uniformly samples colatitude,
and similarly $\{\phi_j\}$ uniformly samples azimuth.
Since our target embeddings are unbounded, we found crucial to use a robust loss such as Huber%
\footnote{median pose errors are $\approx 10^\circ$ larger with $L_1$ or $L_2$},
and a Huber breakpoint at \num{1} works well in practice.

We define the loss as follows, where $\mathcal{H}$ is the Huber loss,
and a weight is introduced to account for the nonuniform equirectangular spherical sampling
($\sin(\theta)$ is proportional to the sample area):
 \begin{align}
 \mathcal{L}(x, y) &= \frac{1}{N^2} \sum_{i,j = 0}^{N-1} \mathcal{H}(\sin(\theta_i)(f(y)-s(x))(\theta_i, \phi_j)) \\
  \mathcal{H}(\alpha) &= \begin{cases}
    0.5 \alpha^2                   & \text{for } |\alpha| \le 1, \\
    |\alpha| - 0.5 & \text{otherwise.}
  \end{cases}  \label{cross:eq:loss}
\end{align}

\paragraph{Architecture}
We now describe the structure of our cross-domain embedding model $f$.
With $f(y)$, we are predicting spatially dense spherical feature maps from a single $2$D image.
Convolutional encoder-decoder architectures with skip connections such as
U-Net~\cite{ronneberger15miccai} or Stacked Hourglass~\cite{newell2016stacked}
produce excellent results when some pixelwise association can be made between the input and output
domains (e.g., for dense labeling tasks like semantic segmentation~\cite{deeplabv3plus2018}).
In our case, we must learn a cross-domain map from a $2$D image
(a function on the Euclidean $2$D space) to functions on the sphere.
In this setting, architecture features such as skip connections are not only unnecessary but can be
harmful by forcing the network to incorrectly consider associations across topologies.%
\footnote{Although cross-modal learning has been explored in different domains, e.g., \textcite{aytar2016soundnet}, these methods predict representations in $\mathbb{R}^n$ from different modalities which is a simpler application of $1$D and $2$D CNNs.}

We consider an encoder-decoder architecture, with a number of rounds of downsampling from
input image to a $1$D vector, followed by rounds of upsampling
from the $1$D vector to the set of spherical feature maps.
Following the best practices for this kind of architecture proposed by \textcite{dcgan},
we employ a fully convolutional network with strided convolutions for downsampling
and transposed convolutions for upsampling.
We apply azimuthal circular padding after the $1$D bottleneck,
when the feature maps are expected to assume spherical topology.
We also found performance improvements by replacing convolutional layers with residual layers \cite{he2016deep}.
\Cref{cross:fig:architecture} illustrates the architecture.

\paragraph{Target embeddings}
The remaining design choice is how to select the appropriate target feature maps from $s(x)$.
For all our experiments, $s(x)$ is a ten layer residual spherical \cnn\
trained for ModelNet40 $3$D shape classification on $64 \times 64$ inputs
(i.e., $r(M)$ produces a single-channel $64 \times 64$ output).
The decision of which feature maps to use as the target is application-dependent.
For category-based relative pose estimation, we want features that are void of instance level details,
obtained by taking the target embedding from deeper layers.
For view synthesis, we wish that the instance-level details are preserved,
so we embed to an earlier layer.
We employ the same pretrained spherical \cnn\ for all experiments
(on ModelNet40, ObjectNet3D and ShapeNet), which attests to decent generalization performance.

\begin{figure}[htbp]
\centering
\includegraphics[width=0.8\linewidth]{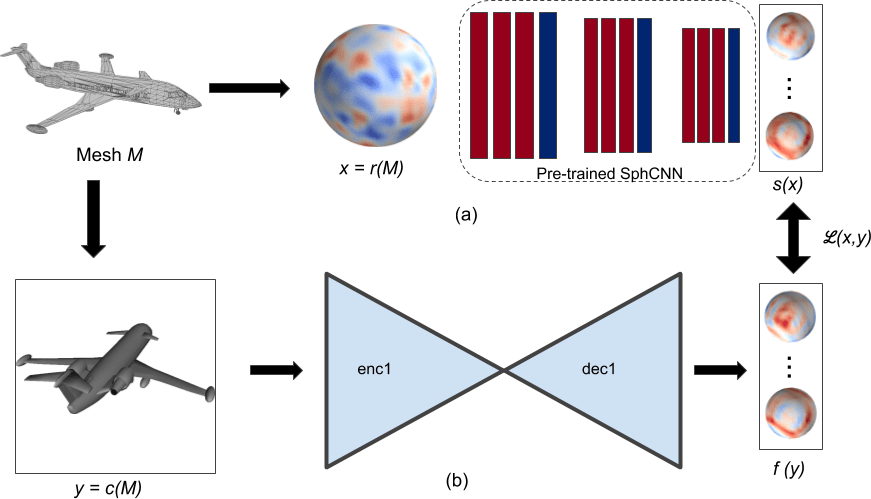}%
\caption{
  \textbf{Cross-domain spherical embeddings.} Given a $3$D mesh,
  (a) we map it to a spherical function, and use a pre-trained spherical CNN to compute its spherical embedding.
  (b) During training, we render a view and learn the transformation to the target spherical embedding using an encoder-decoder.
  For inference, the inputs are $2$D images and only the encoder-decoder part is used.
}
\label{cross:fig:architecture}
\end{figure}

\subsection{Relative pose estimation}
\label{cross:sec:relpose_subsec}
The cross-domain embeddings produced by $f$ are sufficient to recover the relative pose
between pairs of images (even between different instances of the same object category).
Since $f$ is trained to produce \SO{3}-equivariant feature maps,
we can apply $3$D rotations directly to them.
Relative orientation estimation amounts to identifying the rotation
that brings feature maps into alignment.
For alignment we use a simple spherical cross-correlation.
Given two images $y_1$ and $y_2$ we estimate their relative pose as
$\argmax_{R\in \SO{3}} G(R)$ where
\begin{align}
G(R) &= \sum_{k=0}^{K-1} \int\limits_{p\in S^2} f_k(y_1)(p) \cdot f_k(y_2)(R^Tp) dp.
\label{cross:eq:corrargmax}
\end{align}
Here the subscript $k$ denotes the $k$-th spherical channel in the image embedding.
The correlation map $G(R)$ can be evaluated efficiently in the spectral domain
(similar in spirit to spherical convolution, as shown in \cref{h:sec:sphcnns}).
See~\textcite{kostelec2008ffts,makadia2010spherical} for details and implementation.

The resolution of $G(R)$ depends on the resolution of
the input spherical functions $f(y_1)$ and $f(y_2)$.
We set our learned feature maps have a spatial resolution of $16 \times 16$ in this task,
corresponding to a cell width of \ang{22.5} at the equator,
which is too coarse for precise relative pose.
To increase resolution, we upsample the features by a factor of four
using bicubic interpolation prior to evaluating \cref{cross:eq:corrargmax}.

This method also serves to estimate relative pose between an image $y$ and mesh $M$,
by computing the cross-correlation (\cref{cross:eq:corrargmax}) between $f(y)$ and $s(r(M))$.

Recall that during training we take arbitrarily oriented meshes as inputs.
A training example consists of (i) the target embeddings from the pretrained $s(r(M))$
and (ii) a single view rendered from a fixed camera $c(M)$.
No orientation supervision is necessary,
and the model never sees pairs of images together during training.
This reduces the sample complexity and leads to faster convergence.

\subsection{Novel view synthesis}
\label{cross:sec:nvs_subsec}
The spherical embeddings learned by our method can also be applied towards novel view synthesis. 
The rotation equivariant spherical \cnn\ feature maps undergo the same rotation as its inputs,
so if we learn the inverse map that generates an image back from its embedding,
we can rotate the embeddings and generate novel views.

We define the inverse map $g=f^{-1}$ such that $g(f(y)) = y$.
If we let $y_1 = c(M)$ and $y_2 = c(\lambda_R M)$,
i.e., $y_{1}$ and $y_{2}$ are images taken of an object $M$ with a fixed camera $c$,
before and after the object undergoes a $3$D rotation, respectively.
It follows that
\begin{equation}
  g(\lambda_R f(y_1)) = g(f(y_2)) = y_2.
\end{equation}

This gives a way to generate a novel view of the $3$D object under rotation
from the spherical embedding of a single view.
The procedure is as follows; see \cref{cross:fig:pose_synth} for an illustration.
\begin{enumerate}
    \item Obtain the embedding $f(y_1)$ of given view $y_1$,
    \item Rotate the embedding by the desired $R \in \SO{3}$, obtaining $f(y_2) = \lambda_R f(y_1)$,
    \item Apply $g$ to obtain the novel view $y_2=g(f(y_2))$.
\end{enumerate}

Since $g$ is learning the inverse of $f$, we similarly design $g$ as a convolutional encoder-decoder, which is trained from single views enforcing $g(f(y)) = y$ with a pixel-wise $L_2$ loss $\mathcal{L}_s(y) = \norm{g(f(y))  - y}_2^2$  (see \cref{cross:fig:synthtrain} for illustration).
\begin{figure}[htbp]
\centering
\includegraphics[width=\linewidth]{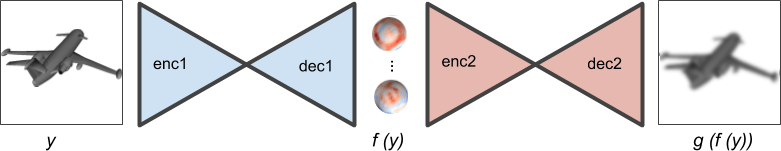}%
\caption{
  \textbf{Novel view synthesis training.}
  We learn the inverse map from spherical embeddings to $2$D views.
  The map from $2$D view to spherical embeddings (in blue) is the same as in \cref{cross:fig:architecture} and is frozen during training.
  The synthesizer network (in red) reconstructs the same input view and is trained with an $L_2$ loss.
}
\label{cross:fig:synthtrain}
\end{figure}

\section{Experiments}
\label{cross:sec:experiments}

\subsection{Architecture details}
\paragraph{Spherical CNN}
We train a 10 layer spherical \cnn\ only once for object classification on ModelNet40
and use it to generate the target embeddings for all experiments in this chapter.
The basic block is the spherical convolutional residual layer as described in \cref{sph:sec:pano},
and training minimizes a cross-entropy loss over 40 classes.
\Cref{cross:fig:layers} shows the architecture.

The network $s$ is trained for 15 epochs with a batch size of 16,
and Adam~\cite{KingmaB14} optimizer with initial learning rate of \num{1e-3},
reduced to \num{2e-4} and \num{4e-5} at steps 5000 and 8500, respectively.
Random anisotropic scaling is used as augmentation. It achieves \num{84.2}{\%} accuracy.
The model from \cref{sph:sec:sphcnn} achieves \num{86.9}{\%} on the same task,
but with a different architecture containing an extra branch to process surface normals;
our inputs here are only the ray lengths from the ray casting procedure.

\paragraph{Embedding  network}
We obtain the embeddings with encoder-decoder residual networks.
Given an input with dimensions $N \times N$,
the encoding step contains one $7\times 7$ convolutional layer followed by
$\log_2 N - 1$ blocks of two residual layers,
followed by a final convolutional layer that produce a $1$D latent vector.
The number of channels double at each residual block, starting at 64 and capped at 256.
Downsampling is through strided convolutions.

The $1$D encoding is then upsampled using a convolutional layer followed by
a sequence of residual blocks and a final $7\times 7$ convolutional layer
up to the desired resolution and number of channels,
which is $16\times 16 \times 32$ for the pose experiments,
and $32\times 32 \times 16$ for novel view synthesis.
Upsampling is through transposed convolutions.

Our targets are spherical \cnn\ features inside the residual bottlenecks,
so the embeddings have four times fewer channels than the actual spherical \cnn\ layer outputs. %
The image inputs are $128\times 128$ and the $1$D encoding has 1024 units.
\Cref{cross:fig:layers} shows more details on resolutions and number of channels per layer.

The embedding network $f$ is trained to minimize a Huber loss.
Training takes \num{200}{k} steps with a batch size of 16, and Adam~\cite{KingmaB14} optimizer with
initial learning rate of \num{2e-4},
reduced to \num{4e-5} and \num{1e-5} at steps \num{80}{k} and \num{180}{k}, respectively.
Random anisotropic scaling of meshes prior to rendering is used as augmentation.

\paragraph{Synthesis  network}
The synthesizer network $g$ follows the same structure as the embedding,
the difference being that the inputs are $32\times 32 \times 16$ and the outputs $128\times 128$.
One question that arises is if the synthesizer should be trained with the target spherical \cnn\
embeddings as inputs, $s(x)$,
or from the embeddings obtained from single views by our network, $f(y)$.
We found that the latter is slightly better.

The synthesis network is trained to minimize an $L_2$ loss for \num{200}{k} steps
with a batch size of eight, and Adam~\cite{KingmaB14} optimizer with
initial learning rate of \num{2e-4}, reduced to \num{4e-5} and \num{1e-5}
at steps \num{80}{k} and \num{180}{k}, respectively.
Random anisotropic scaling of meshes prior to rendering is used as augmentation.

\begin{figure}[htbp]
\centering
\includegraphics[width=0.6\linewidth]{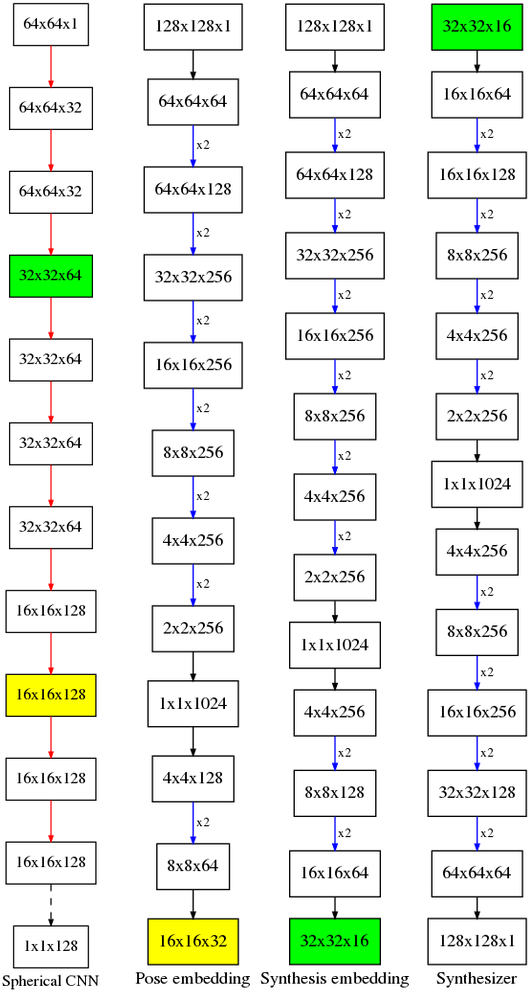}%
\caption{
  \textbf{Network architectures used in this chapter.}
  Rectangles indicate data dimensions (width x height x channels).
  \textit{Red arrow:} spherical convolutional residual bottleneck layer;
  \textit{dashed arrow:} global average pooling;
  \textit{blue arrow:} residual bottleneck layer;
  \textit{black arrow:} convolutional layer.
  Nodes with yellow and green backgrounds are the target embeddings for pose and synthesis, respectively.
}
\label{cross:fig:layers}
\end{figure}

\subsection{Datasets}
We utilize the popular large datasets of $3$D shapes
ModelNet40~\cite{wu20153d} and ShapeNet~\cite{shapenet2015} for most of our experiments.

Some approaches must explicitly deal with the symmetries present in shape categories~\cite{saxena09icra,rad17iccv}.
Our method is immune to this problem by not requiring pose annotations.
However, pose annotations are still necessary for evaluation,
therefore we limit some experiments to categories which are largely
free of symmetry and thus have unique relative orientations.

One example of problematic evaluation due to symmetries is the ShapeNet category \emph{airplanes}.
Some of the instances (e.g., spaceships and flying wings) are fully symmetric around one axis,
resulting in non-injective embeddings and two possible correct alignments that differ by \ang{180}.
For meaningful evaluation we compute the errors up to symmetry for this category.

Recall that we are not estimating pose relative to a canonical object frame
but rather relative object orientation from a pair of images.
Thus, for training, the dataset models need not come aligned per category,
and in fact we introduce random rotations at training time.
For evaluation, in order to quantify our inter-instance performance,
aligned shapes are necessary to determine the ground truth (see \cref{cross:sec:rel-pose-experiments});
for ModelNet40 we use the aligned version from \textcite{sedaghat15iccv}.

There are multiple datasets for object pose estimation, such as Pascal3D+~\cite{xiang14wacv},
KITTI~\cite{geiger12cvpr}, and Pix3D \cite{pix3d},
but they do not exhibit large variation in viewpoints, especially in camera elevation.
For example, Pascal3D+ has most elevations concentrated within $\left[-10^\circ, 10^\circ\right]$
and the official evaluation only considers azimuthal accuracy.
In our setting, we explore geometric embeddings that can capture more challenging arbitrary viewpoints.
Our results show that the problem of relative orientation from two views is
difficult even for synthetic uncluttered rendered images from ModelNet40 and ShapeNet.
Our experiments with real images are limited to the
\emph{airplane} and \emph{cars} categories of ObjectNet3D \cite{xiang2016objectnet3d},
which have the largest variety of viewpoints among all categories.
To increase the difficulty, we augment the \emph{cars} category with in-plane rotations.

\subsection{Relative pose estimation}
\label{cross:sec:rel-pose-experiments}
For training, we render views in arbitrary poses sampled from \SO{3}.
We have two modes of evaluation, \textit{instance} and \textit{category} based.
For category-based, we measure the relative pose error between each instance
and three randomly sampled instances from the test set.
For instance-based, we measure the error between each instance from the test set
and three randomly rotated versions of itself.
The error is the angle between the estimated and ground truth relative poses;
given input ground truth poses $R_1$ and $R_2$ and estimated pose $R$,
the error is $\arccos\left({(\tr{(R_2^\top R_1 R)}-1)/2}\right).$
We compare with the following methods.
\paragraph{Regression}
We consider a method based on \textcite{mahendran17cvprw},
which formulates pose estimation as regression.
To keep the comparison fair, we use the same  architecture for the encoder as our model,
shown in the middle columns of \cref{cross:fig:layers}.
The architecture is exactly the same up to the 1024 dimensional bottleneck,
which is then followed by the pose network from \textcite{mahendran17cvprw}.
We train for \num{200}{k} steps, with a batch size of 16, and Adam~\cite{KingmaB14} optimizer,
with initial learning rate of \num{1e-4}, reduced to \num{5e-5}, \num{2e-5},
and \num{8e-6} at steps \num{40}{k}, \num{75}{k} and \num{125}{k}, respectively.
The \mse\ and geodesic loss scheduling is similar to \textcite{mahendran17cvprw} --
the first \num{100}{k} steps use \mse\ loss, followed by geodesic loss.
When training the $3$DOF model, we found that the performance improves when
warm starting from a network pre-trained on the $2$DOF training set.
\textcite{mahendran17cvprw} require the ground truth pose with respect to
a canonical orientation during training,
whereas our method is self-supervised and can operate on unaligned meshes.
We still outperform it even when allowing extra information,
especially in the presence of $3$DOF rotations.

Since our method does not require aligned meshes, a more fair comparison would be to train the
regression model on pairs of views where the regression target is the relative pose.
We experimented with numerous variations of this approach
and the performance was always worse than the regression to a canonical orientation.
We report all results in the condition that most favors~\textcite{mahendran17cvprw},
using regressed canonical orientations.

\paragraph{KeypointNet}
\textcite{suwajanakorn2018discovery} introduce an unsupervised method of learning keypoints
that are applicable for pose estimation by solving a Procrustes problem.
Similarly to our method, it generates training data by rendering different views from meshes.
However, it requires consistently oriented meshes for dominant direction supervision,
whereas our method makes no assumptions about mesh orientation.
While they show results for $2$DOF rotations,
only viewpoints on a hemisphere are considered, whereas we sample the whole sphere.
We retrain and evaluate KeypointNet with full $2$DOF and $3$DOF rotations.
We utilize the publicly available code and default parameters with minor modifications.
The required changes are because \textcite{suwajanakorn2018discovery} distribute the training,
which allows a larger batch size of 256, while we train only on a single GPU with a batch size of 24.
With a smaller batch size, the default orientation prediction annealing steps
(\num{30}{k}-\num{60}{k}) prevents convergence; we changed it to \num{120}{k}-\num{150}{k}
and increased the number of steps from \num{200}{k} to \num{300}{k} to be able to reproduce
(and slightly improve) the numbers reported in \textcite{suwajanakorn2018discovery}
(see \cref{cross:tab:kptnet}).
We also modify the rendering procedure to generate the $2$DOF and $3$DOF datasets,
as the original paper only considers a $2$DOF hemisphere.

\begin{table}[htbp]
  \caption{Median angular error in degrees for instance based $2$DOF hemisphere alignment on ShapeNet.
    Our hyperparameter selection slightly outperforms the original results from \textcite{suwajanakorn2018discovery}.}
  \label{cross:tab:kptnet}
  \centering
  \begin{tabular}{lS[table-format=1.2]S[table-format=1.2]S[table-format=1.2]}
    \toprule
    & {airplane} & {car} & {chair}\\
    \midrule
    Our parameters & 6.06 & 3.31 & 4.94\\
    Original parameters  & 5.72 & 3.37 & 5.42\\
    \bottomrule
  \end{tabular}
\end{table}

\paragraph{Results for synthetic images}
\Cref{cross:tab:shapenet} shows ShapeNet relative pose estimation results.
\Cref{cross:fig:align} shows the $3$DOF alignment quality on ShapeNet by rendering views using the estimated relative poses.
We show extra results for ModelNet40 in \cref{cross:sec:m40}
and an experiment aligning meshes to images in \cref{cross:sec:im2mesh}.

\begin{table}[htbp]
  \caption{\textbf{ShapeNet relative pose estimation results.}
    We show median angular error in degrees  (\emph{err}), accuracy \emph(a@) at \ang{15} and \ang{30} for instance and category-based, 2 and 3 degrees of freedom relative pose estimation from single views on ShapeNet.
    Comparison is against \textcite{mahendran17cvprw} (\emph{Regr.}) and \textcite{suwajanakorn2018discovery} (\emph{KpNet}).
    KeypointNet does not converge on the full $3$DOF setting; we limit the viewpoints to a hemisphere when evaluating it.
  Note that we still outperform it.}
  \label{cross:tab:shapenet}
  \centering
  \scriptsize
  \begin{tabular}{l
    S[table-format=2.1,table-auto-round]S[table-format=2.1,table-auto-round]S[table-format=2.1,table-auto-round] c
    S[table-format=2.1,table-auto-round]S[table-format=2.1,table-auto-round]S[table-format=2.1,table-auto-round] c
    S[table-format=2.1,table-auto-round]S[table-format=2.1,table-auto-round]S[table-format=2.1,table-auto-round] c
    S[table-format=2.1,table-auto-round]S[table-format=2.1,table-auto-round]S[table-format=2.1,table-auto-round]}
    \toprule
      & \multicolumn{3}{c}{airplane} & \phantom{} & \multicolumn{3}{c}{car} & \phantom{} & \multicolumn{3}{c}{chair} & \phantom{} & \multicolumn{3}{c}{sofa} \\
    \cmidrule{2-4} \cmidrule{6-8} \cmidrule{10-12} \cmidrule{14-16}
    & {err}                         & {a@15}       & {a@30}                    &            & {err}                      & {a@15}       & {a@30}      &           & {err}      & {a@15} & {a@30}      &           & {err}      & {a@15} & {a@30}                              \\
    \midrule
    \multicolumn{16}{l}{{$2$\textbf{DOF, instance based}}} \\
    \; Ours       & \bgd 5.17               & \bgd 85.3  & \bgd 91.9                 &            & \bgd 3.70 & \bgd 92.2 & 92.5      &      & \bgd 5.07 & \bgd 90.6 & \bgd 94.1 &      & \bgd 4.59 & \bgd 93.6 & \bgd 95.2 \\
    \; Regr.      & 16.9                    & 46.3       & 68.7                      &            & 6.55      & 83.5      & \bgd 93.1 &      & 13.7      & 53.9      & 78.3      &      & 17.3      & 43.2      & 69.4      \\
    \; KpNet      & 6.95                    & 79.4       & 91.5                      &            & {div.}      & {div.}      & {div.}      &      & 6.34      & 84.7      & 91.8      &      & 9.20      & 71.3      & 85.4      \\
    \multicolumn{16}{l}{{$2$\textbf{DOF, category based}}} \\
    \; Ours       & \bgd 6.24               & 79.0       & 88.2                      &            & \bgd 4.73 & 73.2      & 73.3      &      & 12.1      & 59.3      & 74.4      &      & \bgd 10.8 & \bgd 58.7 & 70.5      \\
    \; Regr.      & 20.6                    & 38.7       & 63.7                      &            & 7.06      & \bgd 82.4 & \bgd 92.5 &      & 16.8      & 43.7      & 72.0      &      & 19.6      & 37.8      & 66.5      \\
    \; KpNet      & 9.07                    & \bgd 79.4  & \bgd 91.5                 &            & {div.}      & {div.}      & {div.}      &      & \bgd 8.07 & \bgd 79.5 & \bgd 90.2 &      & 15.1      & 49.8      & \bgd 71.8 \\
    \midrule
    \multicolumn{16}{l}{{$3$\textbf{DOF, instance based}}} \\
    \; Ours       & \bgd 6.64               & \bgd 80.9  & \bgd 91.9                 &            & \bgd 3.84 & \bgd 97.3 & \bgd 98.8 &      & \bgd 5.55 & \bgd 89.1 & \bgd 95.7 &      & \bgd 5.21 & \bgd 90.4 & \bgd 94.8 \\
    \; Regr.      & 45.4                    & 12.6       & 31.3                      &            & 9.83      & 69.0      & 86.5      &      & 21.7      & 31.3      & 64.3      &      & 22.2      & 34.8      & 61.4      \\
    \; KpNet      & 14.9                    & 50.3       & 76.6                      &            & 9.12      & 70.4      & 80.9      &      & 10.8      & 66.7      & 85.3      &      & 25.0      & 27.4      & 57.3      \\
    \multicolumn{16}{l}{{$3$\textbf{DOF, category based}}} \\
    \; Ours       & \bgd 7.27               & \bgd 76.4  & \bgd 89.4                 &            & \bgd 4.59 & \bgd 92.1 & \bgd 93.3 &      & \bgd 12.3 & \bgd 59.5 & 77.3      &      & \bgd 9.66 & \bgd 63.9 & \bgd 76.0 \\
    \; Regr.      & 44.4                    & 14.1       & 32.1                      &            & 10.5      & 66.5      & 85.6      &      & 25.6      & 25.1      & 57.2      &      & 24.5      & 30.9      & 58.1      \\
    \; KpNet      & 16.3                    & 46.0       & 75.0                      &            & 10.7      & 64.4      & 77.6      &      & 13.6      & 55.4      & \bgd 81.6 &      & 37.4      & 12.7      & 39.8      \\
    \bottomrule
  \end{tabular}%
\end{table}

\begin{figure}[htbp]
\centering
\includegraphics[width=\linewidth]{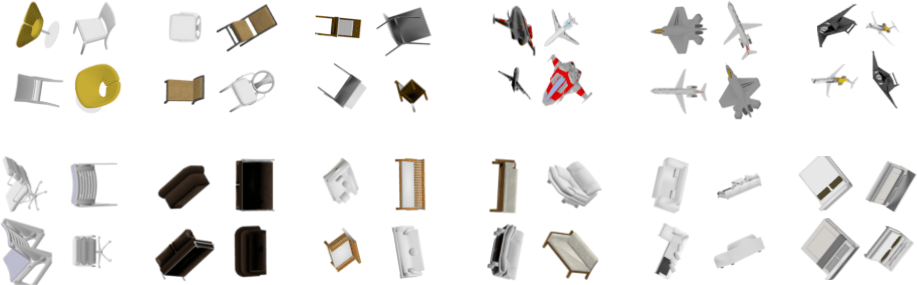}%
\caption{
\textbf{Category-based relative pose estimation.}
We render one object in the pose of the other using our estimated relative pose.
\textit{For each block, top:} Inputs 1 and 2, from the test set.
\textit{Bottom:} Mesh 2 rotated into pose 1, and mesh 1 rotated into pose 2.
We render from the ground truth meshes for visualization purposes only; the inputs to our method are solely the $2$D views and the output is the relative pose.
Note how the alignment is possible even under large appearance variation.
}
\label{cross:fig:align}
\end{figure}

\subsection{Extension to natural images}
Most labeled real-world object pose estimation datasets have restricted pose variations.
The airplane class in ObjectNet3D~\cite{xiang2016objectnet3d} is an exception
with sufficient variation of $3$D poses.
We assume object instance bounding boxes are given
(e.g. using an object detection network~\cite{cvpr17_object}).
We also experiment with the \emph{cars} category by augmenting it
with in-plane rotations to increase the pose variation.
We train our model on image-mesh pairs and significantly outperform the method based on regression. The numbers for \emph{airplanes} are up to a \ang{180} rotation due to symmetry as explained in
\cref{cross:sec:discuss} (see bottom right of \cref{cross:fig:o3d-align} for an example).
\Cref{cross:tab:o3d} shows the comparison while
\Cref{cross:fig:o3d-align} exemplifies some alignment results for \emph{airplanes}.

\begin{table}[htbp]
  \caption{Relative pose estimation results for real images from ObjectNet3D.
    We show median angle error in degrees and accuracy at \ang{15} and \ang{30}.
    We outperform the regression method based on \textcite{mahendran17cvprw} by large margins.}
  \label{cross:tab:o3d}
  \centering
  \begin{tabular}{l
    S[table-format=2.2]S[table-format=2.2]S[table-format=2.2]}
    \toprule
                  & {med err.} & {acc@15} & {acc@30} \\
    \midrule
    \multicolumn{4}{l}{\textbf{Airplane}}            \\
    \; Ours       & 13.75      & 53.40    & 76.60    \\
    \; Regression & 36.52      & 16.70    & 40.40    \\
    \midrule
    \multicolumn{4}{l}{\textbf{Car}}                 \\
    \; Ours       & 8.22       & 72.51    & 78.00    \\
    \; Regression & 16.16      & 46.87    & 74.35    \\
    \bottomrule
  \end{tabular}%
\end{table}

\begin{figure}[htbp]
\centering
\includegraphics[width=\linewidth]{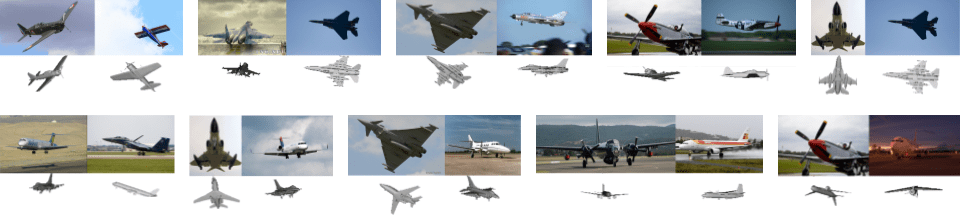}%
\caption{
\textbf{Relative pose estimation for real images.}
We render the mesh corresponding to one input in the pose of the other using the estimated relative pose.
\textit{For each $4\times 4$ block, top:} Inputs 1 and 2, from the test set.
\textit{Bottom:} Mesh 2 rotated into pose 1, and mesh 1 rotated into pose 2.
Image pairs on the top row map to the same mesh in the dataset;
on the bottom row they map to different meshes.
The bottom-right block shows a typical failure case due to symmetry. 
Meshes are used for visualization purposes only; the inputs to our method are the $2$D images and the relative pose is estimated directly from their embeddings via cross-correlation.
}
\label{cross:fig:o3d-align}
\end{figure}

\subsection{Novel view synthesis}
We evaluate novel view synthesis qualitatively.%
\footnote{We attempted a method similar to \textcite{tatarchenko16eccv} as baseline, with and without
  adversarial losses, but results were poor for the large space of rotations considered.}
\Cref{cross:fig:novel} shows the results for multiple generated views in different poses,
with a single $2$D image as input.
We do not expect to generate realistic images here, since the embeddings do not capture
color or texture and the generator is trained with a simple $L_2$ loss.
Our goal is to show that the learned embeddings  naturally capture the geometry,
which is demonstrated by this example, where a simple $3$D rotation of the spherical embeddings
obtained from a single $2$D image produces a novel view of the corresponding $3$D object rotation.
Adversarial and perceptual losses for refining the novel
views~\cite{karras2018progressive,wang2018pix2pixHD} could improve results,
when used in conjunction with our approach.
See \cref{cross:sec:extranvs} for results from other categories.

\begin{figure}[htbp]
\centering
\includegraphics[width=\linewidth]{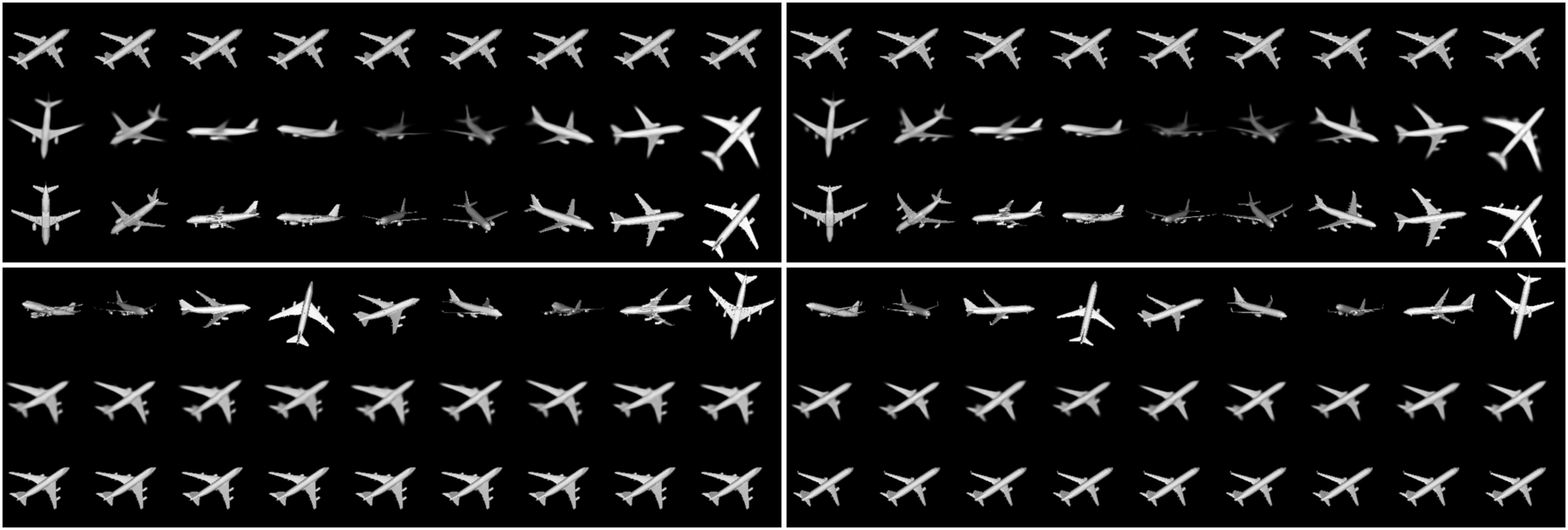}%
\caption{
  \textbf{Novel view synthesis.}
  Our embeddings are category based, capture both geometry and appearance,
  can be rotated as spheres, and can be inverted through another neural network.
  We can generate any new viewpoint from any given viewpoint.
  For each block: \textit{top row:} inputs; \textit{middle row:} novel views generated using our method; \textit{bottom row:} ground truth views rendered from the original mesh.
  Top two blocks show different views generated from a single image; bottom two blocks show a single view generated from different images.
  Each block shows a different instance from the test set.
}
\label{cross:fig:novel}
\end{figure}

\subsection{Discussion}
\label{cross:sec:discuss}
Our image to spherical cross-domain embeddings show
quantitative improvements in relative $3$D object pose estimation.
Most existing literature shows results on a restricted set of rotations,
and our numbers on $2$DOF rotations are comparable to the state of the art.
For full $3$DOF rotations, relative pose estimation from $2$D images is especially
challenging for approaches that attempt to predict the pose directly, since it requires mapping from
a large space of (possibly pairs of) visually disparate rotated images to some pose representation.
In contrast, our method learns the map from image to the spherical embeddings,
which does not require seeing pairs of inputs, and the pose estimation comes naturally
from a spherical cross-correlation.

KeypointNet~\cite{suwajanakorn2018discovery} training failed to converge or converged to a bad model for
cars $2$DOF (noted 'div' in \cref{cross:tab:shapenet}) and for the challenging $3$DOF rotations.
We found that KeypointNet converges if we limit the $3$DOF setting
to views on a hemisphere (instead of the full sphere).
Our numbers for the full $3$DOF space of rotations are still superior
to KeypointNet's results for the limited $3$DOF hemisphere.

Evaluation of the \emph{airplane} class is problematic on ShapeNet due to
the presence of symmetric instances (flying wings and some spaceships),
which admit two possible alignments that differ by a \ang{180} rotation.
We also observe problems on ObjectNet3D, but in this case it's an approximate symmetry that sometimes is
not captured by the low resolution spherical CNN feature maps.
In both cases we consider the symmetry when evaluating the errors by making
$\text{err}_{\text{sym}} = \min(\text{err}, \pi-\text{err})$, in radians.
This metric is used for all methods on \emph{airplanes}.
Note that \textcite{suwajanakorn2018discovery} also observe errors around \ang{180}
and benefit from this metric. ModelNet40 \emph{airplanes} do not suffer from this issue.

Our method is capable of synthesizing any new viewpoint from any other given viewpoint
for any instance of the category it was trained on.
The categories with less appearance variation are easier to learn and produce sharper images.
For all classes, nevertheless, we can verify that the embeddings capture full $3$D information.
\section{Conclusion}
\label{cross:sec:conclusion}
In this chapter, we explored the problem of learning expressive \SO{3}-equivariant embeddings
for $2$D images. We proposed a novel cross-domain embedding that maps $2$D images
to spherical feature maps generated by spherical \cnns\ trained on $3$D shape datasets.
In this way, our cross-domain embeddings encode images with sufficient shape properties
and an equivariant structure that together are directly useful for different tasks,
including relative pose estimation and novel view synthesis. %

We highlight two important areas for future work.
First, the cross-domain embedding architecture uses a large encoder-decoder structure.
The model complexity can be greater than what would be necessary for training traditional
task-specific models (e.g., a relative pose regression network).
This is because we are solving a much higher dimensional problem,
our model must learn an expressive feature representation that generalizes to different applications.
Nonetheless, in future work, it will be useful to explore ways to make this component more compact.

Second, by construction, our model is tied to the spherical \cnns\ that supervise the embeddings.
Another future direction is to explore different
rotation equivariant models to play this role.
Alternatively, improvements to the spherical \cnns\ themselves can also translate
to more powerful embeddings.
For example, incorporating texture/normals on their training could improve our results
and also allow more challenging tasks such as textured view synthesis.

\section{Extra experiments and visualizations}
We evaluate image to mesh alignment on ShapeNet and relative pose estimation on ModelNet40.
For completeness, we also include regression results to estimate error to a canonical pose.
\Cref{cross:tab:m403d} shows the results for ModelNet40 alignment.

\subsection{Image to mesh alignment}
\label{cross:sec:im2mesh}
Although we focus on tasks where the inputs are $2$D images,
our method produces a common equivariant representation for images and meshes that
is suitable for image to mesh alignment.
\Cref{cross:tab:im-mesh} shows the results.
The accuracy is similar whether we align image to image or image to mesh.
\begin{table}[htbp]
  \caption{Image to mesh alignment experiment on ShapeNet.
    We show the category based median relative pose error in deg for image to image (\emph{im-im}) and image to mesh (\emph{im-mesh}).}
  \label{cross:tab:im-mesh}
  \centering
  \begin{tabular}{lS[table-format=2.2]S[table-format=2.2]S[table-format=2.2]S[table-format=2.2]}
    \toprule
    & {airplane} & {car} & {chair} & {sofa}\\
    \midrule
    \multicolumn{4}{l}{$2$\textbf{DOF}} \\
    \; im-mesh & 5.65 & 4.95 & 13.28 & 12.34\\
    \; im-im & 6.24 & 4.73 & 12.10 & 10.80\\
    \midrule
    \multicolumn{4}{l}{$3$\textbf{DOF}} \\
    \; im-mesh & 5.98 & 4.24 & 13.21 & 11.43\\
    \; im-im & 7.27 & 4.59 & 12.30 & 9.66\\
    \bottomrule
  \end{tabular}
\end{table}

\subsection{ModelNet40 relative pose}
\label{cross:sec:m40}
We apply the experimental settings of \cref{cross:sec:rel-pose-experiments}
to categories of the ModelNet40 dataset.
\Cref{cross:tab:m403d} show relative pose estimation results for ModelNet40.
Results and conclusions are similar to the ones in \cref{cross:tab:shapenet}.
\begin{table}[htbp]
  \caption{Median angular error in degrees for instance and category-based, $2$DOF and $3$DOF alignment on ModelNet40.}
  \label{cross:tab:m403d}
  \centering  
  \begin{tabular}{l
    S[table-format=2.1,table-auto-round]S[table-format=2.1,table-auto-round]
    S[table-format=2.1,table-auto-round]S[table-format=2.1,table-auto-round]
    S[table-format=2.1,table-auto-round]S[table-format=2.1,table-auto-round]}
    \toprule
               & {airplane} & {bed}  & {chair} & {car}  & {sofa} & {toilet}      \\
    \midrule
    \multicolumn{7}{l}{$2$\textbf{DOF, instance-based}} \\
    \; Regression & 6.29 & 12.7  & 25.5 & 6.84 & 12.5 & 9.76 \\
    \; Ours       & 3.33 & 4.46  & 7.07 & 4.12 & 4.52 & 4.88 \\
    \multicolumn{7}{l}{$2$\textbf{DOF, category-based}} \\
    \; Regression & 7.13 & 15.8  & 32.2 & 7.00 & 13.3 & 10.4 \\
    \; Ours       & 4.80 & 6.60  & 10.2 & 4.82 & 9.56 & 10.8 \\
    \midrule
    \multicolumn{7}{l}{$3$\textbf{DOF, instance-based}}                      \\
    \; Regression & 11.8     & 26.0 & 43.7  & 16.5 & 25.3 & 17.8        \\
    \; Ours       & 7.23     & 4.93 & 7.79  & 3.95 & 6.51 & 5.17        \\
    \multicolumn{7}{l}{$3$\textbf{DOF, category-based}}                      \\
    \; Regression & 12.9     & 29.9 & 52.5  & 15.2 & 34.5 & 17.8        \\
    \; Ours       & 8.81     & 8.55 & 15.3  & 5.12 & 11.0 & 10.9        \\

    \bottomrule
  \end{tabular}   
\end{table}

\subsection{Novel view synthesis}
\label{cross:sec:extranvs}
We show novel view synthesis results for ShapeNet emph{cars} and \emph{chairs},
including a failure case in \cref{cross:fig:supp_synth}.

\begin{figure}[htbp]
\centering
\includegraphics[width=\linewidth]{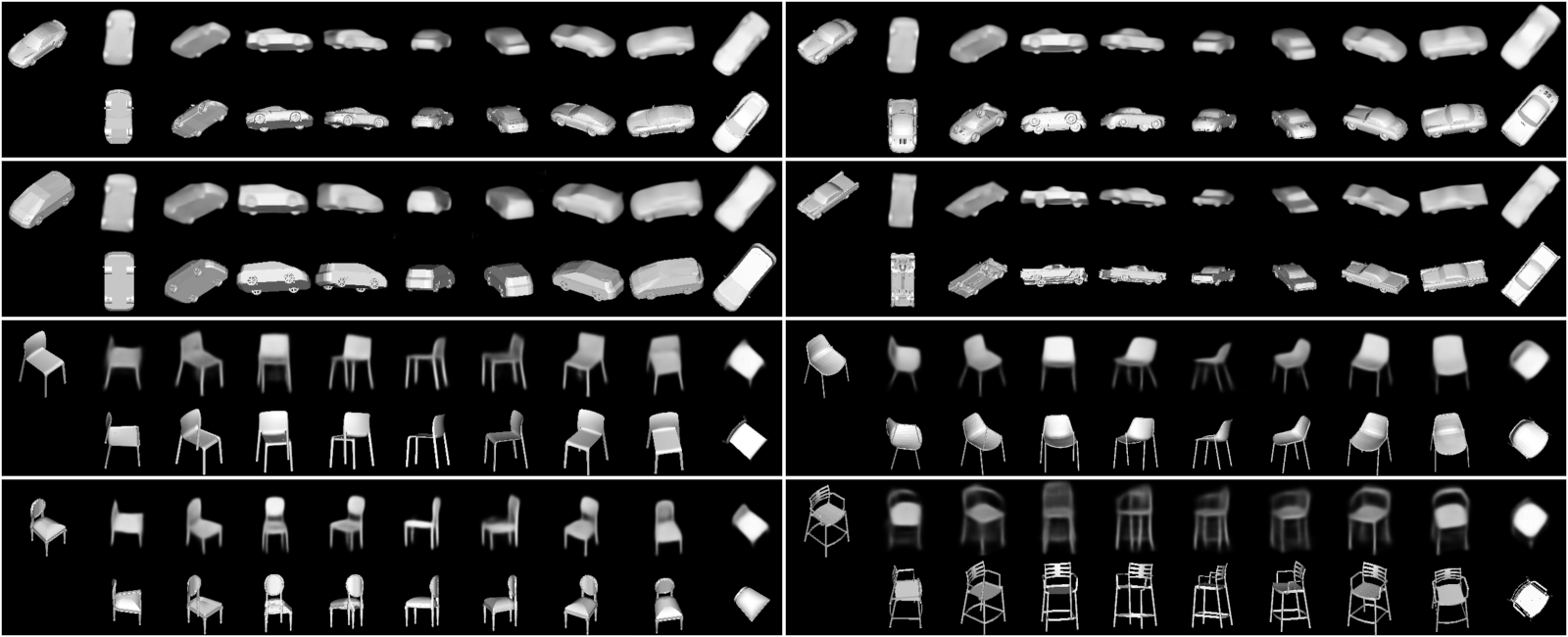}%
\caption{
  \textbf{More novel view synthesis results.}
  \textit{Top-left:} inputs, which are $2$D images from the test set. 
  \textit{Top row:} novel views generated using our method.
  \textit{Bottom row:} ground truth views rendered from the original mesh.
  The bottom right shows a failure case due to a chair with uncommon appearance.
}
\label{cross:fig:supp_synth}
\end{figure}

\subsection{Visualization}
In this section, we visualize some embeddings along with inputs and outputs of different tasks.
We randomly select three channels of the predicted embeddings
and plot them on the sphere for different input orientations.

\Cref{cross:fig:synth_anim_frames} shows inputs, embedding channels,
rotated embedding channels and outputs from novel view synthesis.
\Cref{cross:fig:pose_anim_frames} shows inputs, embedding channels, and alignment visualization.

\begin{figure}[htbp]
\centering
\includegraphics[width=\linewidth]{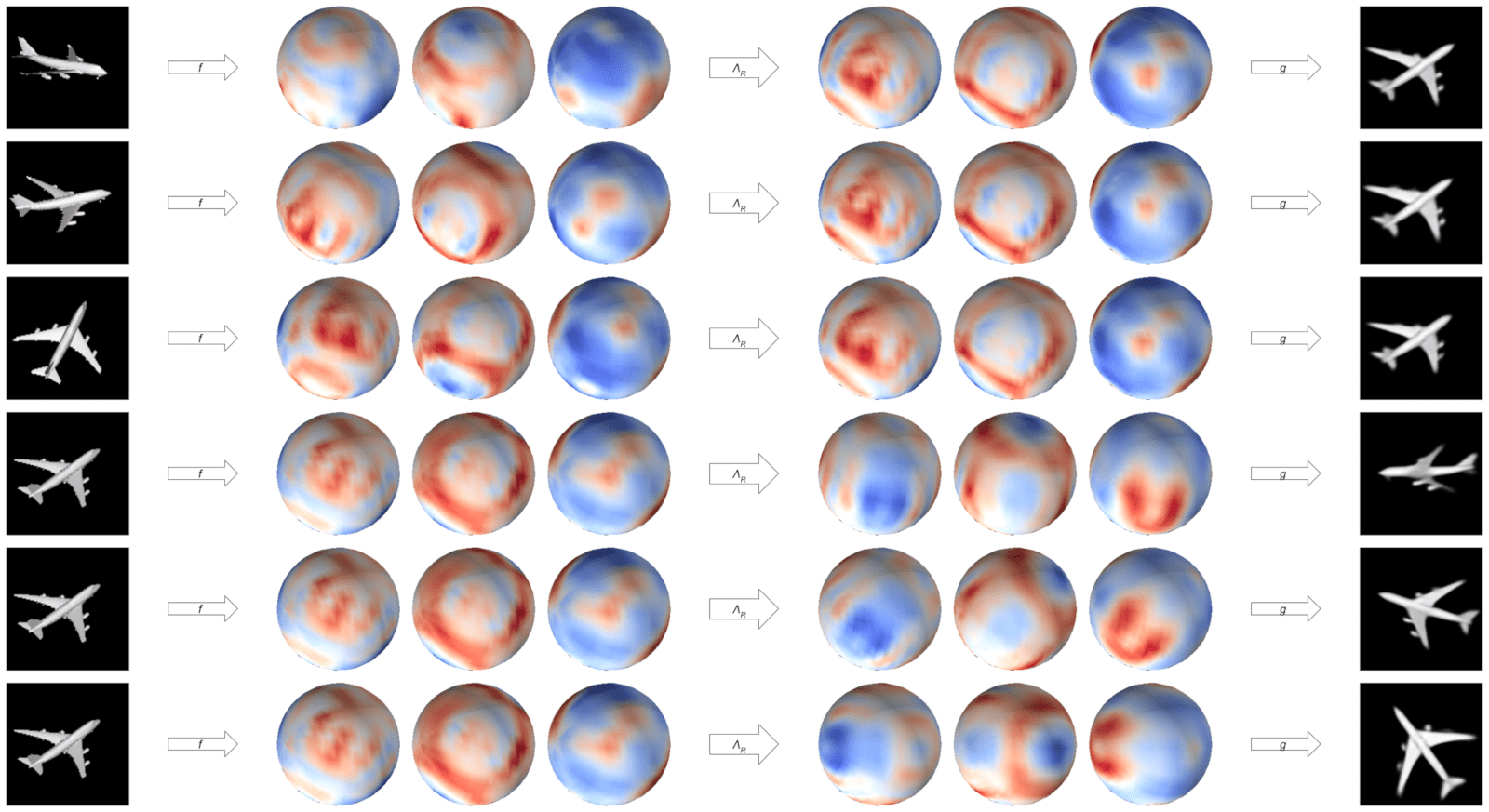}%
\caption{
  \textbf{Novel view synthesis visualization.}
  \textit{Each row:} inputs, three embedding channels, rotated embedding channels, outputs.
  Top three rows show generation of a canonical view from arbitrary views.
  Bottom three rows show generation of arbitrary views from a canonical view.
}
\label{cross:fig:synth_anim_frames}
\end{figure}

\begin{figure}[htbp]
\centering
\includegraphics[width=\linewidth]{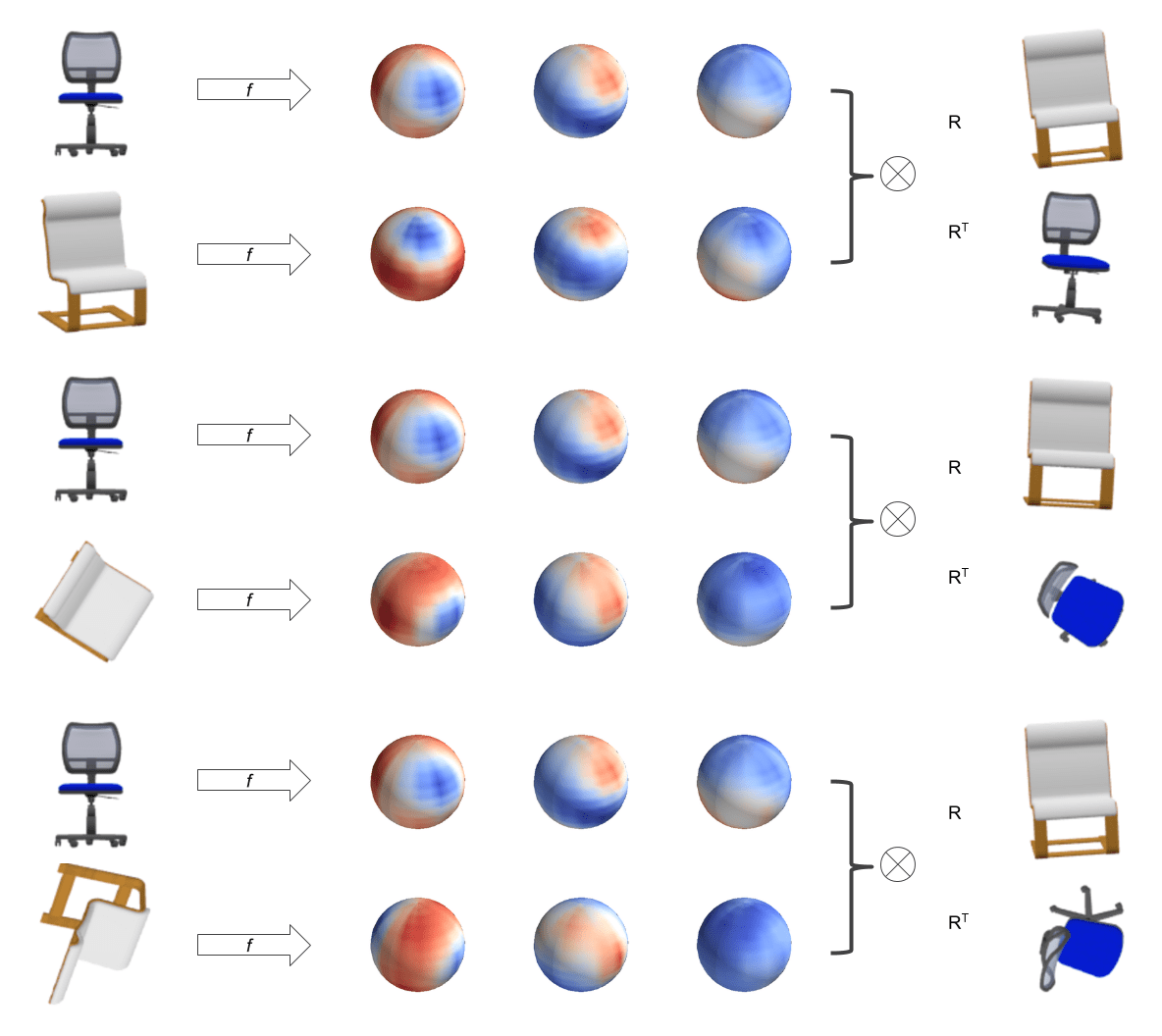}%
\caption{
  \textbf{Relative pose estimation visualization.}
  \textit{Each block of two rows:} pair of inputs, three embedding channels per input,
  mesh 2 rotated into pose 1, and mesh 1 rotated into pose 2.
  We render from the ground truth meshes for visualization purposes only;
  our inputs are solely the $2$D views and output is the relative pose.
}
\label{cross:fig:pose_anim_frames}
\end{figure}

\glsresetall
\chapter{Equivariance of spherical vector fields}
\label{spin:sec:spin}
\chaptersubtitle{The Spin-Weighted Spherical CNNs}
\section{Introduction}
Rotation equivariant \cnns\ are the natural way to learn feature representations on spherical data.
There are two prevailing designs, (a) convolution between spherical functions and zonal (isotropic; constant per latitude) filters, as presented in \cref{sph:sec:sphcnn},
and (b) convolutions on \SO{3} after lifting spherical functions to the rotation group~\cite{s.2018spherical}.
There is a clear distinction between these two designs: (a) is more efficient allowing
to build representational capacity through deeper networks, and (b) has more expressive filters but
is computationally expensive and thus is constrained to shallower networks.
The question we consider in this chapter is: how can we achieve the expressivity/representation capacity of \SO{3} convolutions with the efficiency and scalability of spherical convolutions?

We propose to leverage \swsfs,
introduced by \textcite{newman1966note} in the study of gravitational waves.
These are complex-valued functions on the sphere that, upon rotation,
suffer a phase change besides the usual spherical translation.

Our key observation is that a combination of \swsfs\ allows
more expressive representations than scalar spherical functions,
avoiding the need to lift features to the higher dimensional \SO{3}.
It also enables anisotropic filters, removing the filter constraint of purely spherical \cnns.

We define convolutions and cross-correlations of \swsfs.
For bandlimited inputs, the operations can be computed exactly in the spectral domain,
and are equivariant to the continuous group \SO{3}.
We build a \cnn\ where filters and features are sets of \swsfs,
and adapt nonlinearities, batch normalization, and pooling layers as necessary.

Besides more expressive and efficient representations, we can interpret the spin-weighted
features as equivariant vector fields on the sphere,
enabling applications where the inputs or outputs are vector fields.
Current  spherical \cnns~\cite{s.2018spherical,esteves18eccv,kondor2018clebsch,perraudin2019deepsphere}
cannot achieve equivariance in this sense, as illustrated in \cref{fig:sphvecfield}.

\begin{figure}[ht]
  \centering
  \includegraphics[width=0.6\linewidth]{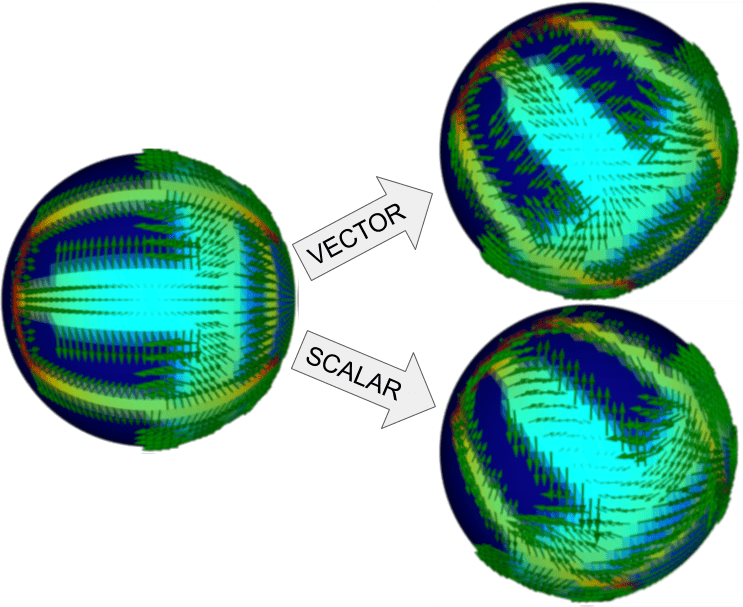}
  \caption{
    Colors represent a scalar field, and
    the green vectors represent a vector field.
    Upon rotation, scalar fields transform by simply moving values to another position,
    while vector fields move and also rotate.
    Treating vector fields as multi-channel scalars (bottom-right) results in incorrect behavior.
    The \acrlongpl{swscnn}
    equivariantly handle vector fields as inputs or outputs.
  }
  \label{fig:sphvecfield}
\end{figure}

To evaluate vector field equivariance, we introduce a variation of MNIST where the images and their
gradients are projected to the sphere.
We propose three tasks on this dataset:
1) vector field classification,
2) vector field prediction from scalar fields,
3) scalar field prediction from vector fields.
We also evaluate our model on spherical image classification, $3$D shape classification, and
semantic segmentation of spherical panoramas.

To summarize the contributions of this chapter,
\begin{enumerate}
\item We define convolution and cross-correlation between sets of spin-weighted
  spherical functions. These are \SO{3} equivariant operations that respect \swsfs\ properties.
\item We build a CNN based on these operations and adapt usual CNN components for
  sets of \swsfs\ as features and filters. This is, to the best of our knowledge, the first
  spherical CNN that operates on vector fields.
\item We demonstrate the efficacy of the \swscnns\ on a variety of tasks
  including spherical image and vector field classification,
  predicting vector field from images and conversely,
  $3$D shape classification and spherical image segmentation.
\end{enumerate}

Most of the content in this chapter appeared originally in~\textcite{esteves20_spin_weigh_spher_cnns}.
Source code is available at \url{https://github.com/daniilidis-group/swscnn}.
\section{Related work}
\paragraph{Equivariant CNNs}
The first equivariant CNNs were applied to images on
the plane~\cite{gens2014deep,cyclicsym}.
\textcite{cohen2016group} formalized these models and named them \gcnns.
While initial methods were constrained to small discrete groups of
rotations on the plane, they were later extended to
larger groups~\cite{WeilerHS18},
continuous rotations~\cite{worrall2017harmonic},
rotations and scale~\cite{esteves2018polar},
$3$D rotations of voxel grids~\cite{worrall2018cubenet,weiler3dsteerable},
and point clouds~\cite{tensor}.

\paragraph{Spherical CNNs}
\gcnns\ can be extended to homogeneous spaces of groups of symmetries~\cite{kondor18icml};
the quintessential example is the sphere $S^2$ as a homogeneous space of
the group \SO{3}, the setting of spherical CNNs.
There are two main branches.
The first branch, introduced by \textcite{s.2018spherical}, lifts the spherical inputs
to functions on \SO{3}, and its filters and features are functions on the group \SO{3},
which is higher dimensional and thus more computationally expensive to process.
\textcite{kondor2018clebsch} is another example.
The second branch, introduced by \textcite{esteves18eccv} (and described in \cref{sph:sec:sphcnn}),
is purely spherical and has filters and features on $S^2$,
using spherical convolution as the main operation.
In this case, the filters are constrained to be zonal (isotropic), which limits
the representational power.
\textcite{perraudin2019deepsphere} also uses isotropic filters, but with graph convolutions
instead of spherical convolutions.

This chapter's approach lies between these two branches; it is not restricted
to isotropic filters but it does not have to lift features to \SO{3};
we employ sets of \swsfs\ as filters and features.

A separate line of work developed spherical \cnns\ that are not rotation-equivariant
\cite{JiangHKPMN19,zhang2019orientation}, which rely on the strong
assumption that the inputs are aligned.

\paragraph{Equivariant vector fields}
Our approach can equivariantly handle spherical vector fields as inputs or outputs.
\textcite{marcos16_rotat_equiv_vector_field_networ} introduced a planar CNN
whose features are vector fields obtained from rotated filters.
\textcite{CohenW17} formalized the concept of feature types that are vectors in a group
representation space.
This was extended to $3$D Euclidean space by \textcite{weiler3dsteerable}.
\textcite{worrall2017harmonic} introduced complex-valued features on $\R^2$ whose phases change
upon rotation; this is similar in spirit to our method, but our features live on the sphere,
requiring different machinery.

\textcite{CohenWKW19} introduced a framework that produces vector field
features on general manifolds; it was specialized to the sphere by \textcite{kicanaoglu2020gauge}.
The major differences are that our implementation is fully spectral and we demonstrate it
on tasks requiring vector field equivariance.
\textcite{cohen2019general} alluded to the possibility
of building spherical CNNs that can process vector fields;
we materialize these networks.

\section{Background}
In this section, we provide the mathematical background that guides our contributions.
We first recall the more commonly encountered spherical harmonics
(described in more detail in \cref{sec:prelim}),
then generalize them to the \swshs.
We also recall the definitions convolutions between spherical functions,
which we later generalize to convolutions between spin-weighted functions.

\paragraph{Spherical Harmonics}
The spherical harmonics \fun{Y_m^\ell}{S^2}{\C} form an orthonormal basis for the space $L^2(S^2)$
of square integrable functions on the sphere.
Any function $\fun{f}{S^2}{\C}$ in $L^2(S^2)$ can be decomposed in this basis via the \sft\  (\cref{spin:eq:sft}), and
synthesized back exactly via its inverse (\cref{spin:eq:isft}),
\begin{align}
  \hat{f}_m^{\ell} &= \int\limits_{S^2} f(x) \overline{Y_m^{\ell}}(x)\, dx \label{spin:eq:sft},
\end{align}
\begin{align}
  f(x) &= \sum_{\ell=0}^\infty \sum_{|m| \le \ell}\hat{f}_m^{\ell}Y_m^{\ell}(x) \label{spin:eq:isft}.
\end{align}
We interchangeably use latitudes and longitudes $(\theta, \phi)$
or points $x \in \R^3,\, \norm{x} = 1$ to index the sphere, and
we use the hat to denote Fourier coefficients.
A function has bandwidth $B$ when
only components of order $\ell \le B$ appear in the expansion.

The spherical harmonics are related to irreducible representations of the group \SO{3} as follows,
\begin{align}
  D_{m,0}^\ell(\alpha, \beta, \gamma) = \sqrt{\frac{4\pi}{2\ell+1}} \overline{Y_m^\ell(\beta, \alpha)},
  \label{spin:eq:wig2sph}
\end{align}
where $\alpha$, $\beta$ and $\gamma$ are ZYZ Euler angles and $D^\ell$ is a Wigner-D
matrix.\footnote{The subscripts $m,\,n$ refer to rows and columns of the matrix, respectively.}
Since $D^\ell$ is a group representation and hence a group homomorphism,
we obtain a rotation formula,
\begin{align}
  \label{spin:eq:sphharmrot}
  Y_m^\ell(g x) &= \sum_{n=-\ell}^{\ell} \overline{D_{m,n}^{\ell}(g)} Y_n^\ell(x),
\end{align}
where we interchangeably use an element $g \in \SO{3}$ or Euler angles $\alpha$, $\beta$ and $\gamma$
to refer to rotations.

Consider the rotation of a function represented by its coefficients
by combining \cref{spin:eq:isft,spin:eq:sphharmrot},
\begin{align}
  f(gx)
  &= \sum_{\ell=0}^\infty \sum_{n=-\ell}^{\ell} \left(\sum_{m=-\ell}^\ell \hat{f}_m^{\ell} \overline{D_{m,n}^{\ell}(g)}\right) Y_n^\ell(x).
\end{align}
This shows that when $f(x) \mapsto f(gx)$, its Fourier coefficients transform as
\begin{align}
  \hat{f}_n^\ell \mapsto \sum_{m}  \overline{D_{m,n}^{\ell}(g)}\hat{f}_m^{\ell}
  \label{spin:eq:sphcoeffrot}
\end{align}

Finally, we recall how convolutions and cross-correlations of spherical functions
are computed in the spectral domain.
\textcite{esteves18eccv} define the convolution between two spherical functions $f$ and $k$
as \cref{spin:eq:sphconv}
while \textcite{makadia2006,s.2018spherical} define the spherical cross-correlation as \cref{spin:eq:sphcorr},

\begin{align}
  (\widehat{k * f})_m^\ell = 2\pi\sqrt{\frac{4\pi}{2\ell+1}} \hat{f}_m^\ell\hat{k}_0^\ell,
  \label{spin:eq:sphconv}
\end{align}
\begin{align}
  (\widehat{k \star f})_{m,n}^{\ell} = \hat{f}_m^\ell\overline{\hat{k}_n^{\ell}},
  \label{spin:eq:sphcorr}
\end{align}

Both are shown to be equivariant through \cref{spin:eq:sphcoeffrot}.
The left-hand side of \cref{spin:eq:sphconv} correspond to the Fourier coefficients of a spherical function,
while the left-hand side of \cref{spin:eq:sphcorr} correspond to the Fourier coefficients of a function on \SO{3}.
This section laid the foundation for the spin-weighted generalization.
See \cref{sec:prelim} and \textcite{Vilenkin_1991,folland2016course} for the full details.

\paragraph{Spin-Weighted Spherical Harmonics}
The \acrfullpl{swsf} are complex-valued functions on the sphere
whose phases change upon rotation.
They have different types determined by the spin weight.

Let \fun{\li{s}{f}}{S^2}{\C} be a \swsf\ with spin weight $s$,
$\lambda_\alpha$ a rotation by $\alpha$ around the polar axis,
and $\nu$ the north pole.
In a conventional spherical function, $\nu$ is fixed by the rotation,
so $(\lambda_\alpha (f))(\nu) = f(\nu)$.
In a spin-weighted function, however, the rotation results in a phase change,
\begin{align}
  (\lambda_\alpha(\li{s}{f}))(\nu) = \li{s}{f}(\nu) e^{-is\alpha}.
\end{align}
If the spin weight is $s=0$, this is equivalent to the conventional spherical functions.

The \acrfullpl{swsh} form a basis of the space of square-integrable
spin-weighted spherical functions; for all square-integrable $\li{s}{f}$, we can write
\begin{align}
  \li{s}{f}(\theta, \phi) &= \sum_{\ell \in \N}\sum_{m=-\ell}^{\ell}\li{s}{Y}_m^\ell(\theta, \phi) \li{s}{\hat{f}}_{m}^\ell,
\end{align}
where $\li{s}{\hat{f}}_{m}^\ell$ are the expansion coefficients,
and the decomposition is defined similarly to \cref{spin:eq:sft}.
For $s=0$, the \swshs\ are exactly the spherical harmonics;
we have $\li{0}{Y}_m^\ell = Y_m^\ell$.

The \swshs\ are related to the matrix elements $D_{mn}^\ell$ of \SO{3} representations as follows,
\begin{align}
  D_{m,-s}^\ell(\alpha, \beta, \gamma) = (-1)^s \sqrt{\frac{4\pi}{2\ell + 1}}
  \overline{\li{s}{Y}_m^\ell(\beta, \alpha)}e^{-is\gamma}.
\end{align}
Note how different spin-weights are related to different columns of $D^\ell$,
while the standard spherical harmonics are related to a single column as in \cref{spin:eq:wig2sph}.
This shows that the \swshs\ can be seen as functions on \SO{3} with sparse spectrum,
a point of view that is advocated by \textcite{Boyle_2016}.

The \swshs\ do not transform among themselves upon rotation as the
spherical harmonics (\cref{spin:eq:sphharmrot}) due to the extra phase change.
Fortunately, the coefficients of expansion of a \swsf\ into the \swshs\ do transform
among themselves according to \cref{spin:eq:sphcoeffrot}.
When $\li{s}{f}(x) \mapsto \li{s}{f}(gx)$,
\begin{align}
  \li{s}{\hat{f}}_n^\ell \mapsto \sum_{m}  \overline{D_{m,n}^{\ell}(g)}\li{s}{\hat{f}}_m^{\ell}.
  \label{spin:eq:spincoeffrot}
\end{align}
This is crucial for defining equivariant convolutions between combinations of \swsfs\
as we will do in \cref{spin:sec:swconv}.
We refer to \textcite{del20123,boyle2013angular,Boyle_2016}
for more details about \swsfs.
\section{Method}
We introduce a fully convolutional network, the \acrfull{swscnn},
where layers are based on spin-weighted convolutions,
and filters and features are combinations of \swsfs.
We define spin-weighted convolutions and cross-correlations,
show how to efficiently implement them,
and adapt common neural network layers to work with combinations of \swsfs.

\subsection{Spin-Weighted Convolutions and Cross-Correlations}
\label{spin:sec:swconv}
We define and evaluate the convolutions and cross-correlations in the spectral domain.
Consider a set of spin weights $W_F,\,W_K$ and sets of functions
$F=\{\fun{\li{s}{f}}{S^2}{\C}\ \mid s \in W_F\}$ and filters $K=\{\fun{\li{s}{k}}{S^2}{\C}\ \mid s \in W_K\}$ to be convolved.

\paragraph{Spin-weighted convolution}
We define the convolution between $F$ and $K$ as follows,
\begin{align}
  \li{s}{(\widehat{F * K})}_m^\ell = \sum_{i \in W_F} \li{i}{\hat{f}}_m^\ell \, \li{s}{\hat{k}}_i^\ell,
  \label{spin:eq:spinconv}
\end{align}
where $s \in W_K$ and $-\ell \le m \le \ell$.
Only coefficients $\li{s}{\hat{k}}_i^\ell$ where $i \in W_F$ influence
the output, imposing sparsity in the spectra of $K$.
The convolution $F*K$ is also a set of \swsfs\ with $s \in W_K$,
the same spin weights as $K$;
we leverage this to specify the desired sets of spins at each layer.

We show this operation is \SO{3} equivariant by applying the rotation formula
from \cref{spin:eq:spincoeffrot}.
Let $\lambda_g$ denote a rotation of each $\li{s}{f}(x) \in F$ by $g \in \SO{3}$.
We have,
\begin{align}
  \li{s}{(\widehat{\lambda_gF * K})}_n^\ell
  &= \sum_{i \in W_F} \sum_{m}  \overline{D_{m,n}^{\ell}(g)}\li{i}{\hat{f}}_m^{\ell} \, \li{s}{\hat{k}}_i^\ell \nonumber \\
  &= \sum_{m}  \overline{D_{m,n}^{\ell}(g)} \sum_{i \in W_F}  \li{i}{\hat{f}}_m^{\ell} \, \li{s}{\hat{k}}_i^\ell \nonumber \\
  &= \sum_{m}  \overline{D_{m,n}^{\ell}(g)} \li{s}{(\widehat{F * K})}_m^\ell \nonumber \\
  &= \lambda_g(\li{s}{(\widehat{F * K})}_n^\ell).
    \label{spin:eq:swconv_proof}
\end{align}

Now consider the spherical convolution defined in \cref{spin:eq:sphconv}.
It follows immediately that it is, up to a constant, a special case of the spin-weighted convolution,
where $F$ and $K$ have only one element with $s=0$, and only the filter coefficients
of form $\li{0}{\hat{k}}_0^\ell$ are used.

\paragraph{Spin-weighted cross-correlation}
We define the cross-correlation between $F$ and $K$ as follows,
\begin{align}
  \li{s}{(\widehat{F \star K})}_m^\ell = \sum_{i \in W_F \cap W_K} \li{i}{\hat{f}}_m^\ell \, \overline{\li{i}{\hat{k}}_s^\ell},
  \label{spin:eq:spincorr}
\end{align}
In this case, only the spins that are common to $F$ and $K$ are used,
but all spins may appear in the output, so it can be seen as a
function on \SO{3} with dense spectrum.
To ensure a desired set of spins in $F \star K$, we can sparsify the spectra in $K$
by eliminating some orders.
A procedure similar to \cref{spin:eq:swconv_proof} proves the \SO{3} equivariance of this operation.

The spin-weighted cross-correlation generalizes the spherical cross-correlation.
When $F$ and $K$ contain only a single spin weight $s=0$,
the summation in \cref{spin:eq:spincorr} will contain only one term and we recover the
spherical cross-correlation defined in \cref{spin:eq:sphcorr}.

\paragraph{Examples}
To visualize the convolution and cross-correlations, we use the phase of the
complex numbers and define local frames %
to obtain a vector field.
We visualize combinations of \swsfs\ by associating pixel intensities with the spin-weight $s=0$
and plotting vector fields for each $s > 0$.

Consider an input $F=\{\li{0}{f},\, \li{1}{f}\}$ and filter $K=\{\li{0}{k},\, \li{1}{k}\}$,
both with spin weights $0$ and $1$.
Their convolution also has spins $0$ and $1$,
as shown on the left side of \cref{fig:spinconvcorr}.
Now consider a scalar valued (spin $s=0$) input $F=\{\li{0}{f}\}$ and filter $K=\{\li{0}{k}\}$.
The cross-correlation will have components of every spin, but we only take spin weights $0$ and $1$
to visualize (this is equivalent to eliminating all orders larger than $1$ in the spectrum of $k$);
see \cref{fig:spinconvcorr} (right).

\begin{figure}[thbp]
  \centering
  \includegraphics[width=\linewidth]{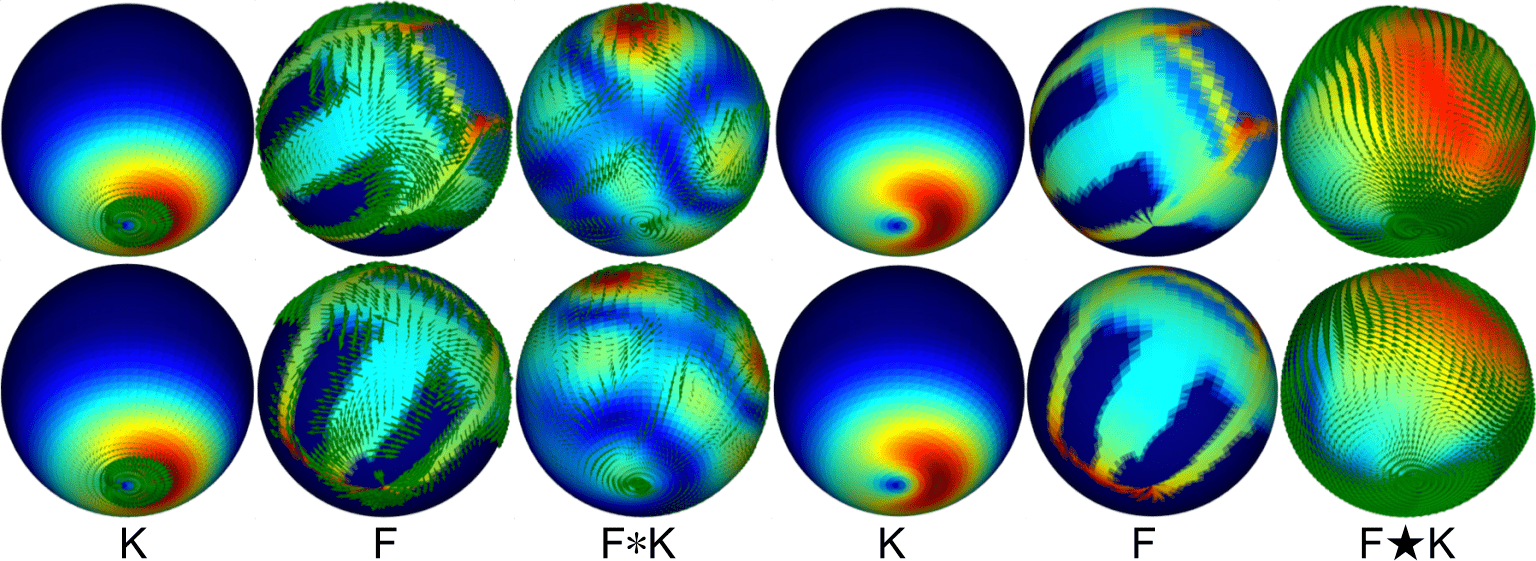}
  \caption{
    Left block ($2\times 3$): convolution between sets of functions of
    spins 0 and 1. The operation is equivariant as a vector field and
    outputs carry the same spins.
    Right block ($2\times 3$): spin-weighted cross-correlation between scalar spherical functions.
    The operation is also equivariant and we show outputs
    corresponding to spins $0$ and $1$.
    The second row shows the effect of rotating the input $F$.
  }
\label{fig:spinconvcorr}
\end{figure}

\subsection{Implementation}
\label{sec:supp:swsh}
Our implementation of the \swsh\ decomposition and its inverse follows~\textcite{Huffenberger_2010}.
The basic idea is to leverage the relation between the \swshs\ and the Wigner-D matrices.
Recall that we can write the Wigner-D matrices as
\begin{align}
D_{m,n}^\ell(\alpha,\beta,\gamma)= e^{-im\alpha}d_{m,n}^\ell(\beta)e^{-in\gamma},
\end{align}
where $d^\ell$ is a Wigner-d matrix.

We define $\Delta_{m,n}^\ell$ as
\begin{align}
  \Delta_{m,n}^\ell = d_{m,n}^\ell(\pi/2),
\end{align}
then the following relation holds~\cite{risbo1996fourier},
\begin{align}
  d_{m,n}^\ell(\theta) = i^{m-n}\sum_{k=-\ell}^\ell \Delta_{k,m}^\ell e^{-ik\theta}\Delta_{k,n}^\ell.
\end{align}

Now we rewrite the \swsh\ forward transform,
\begin{align*}
  \li{s}{\hat{f}}_m^\ell
  &= \int\limits_{\theta, \phi} \li{s}{f}(\theta,\phi) \overline{\li{s}{Y}_m^{\ell}(\theta,\phi)}\, \sin\theta \,d\theta\,d\phi \\
  &= \int\limits_{\theta, \phi} \li{s}{f}(\theta,\phi)
    (-1)^s \sqrt{\frac{2\ell + 1}{4\pi}}
    e^{is\psi}
    D_{m,-s}^\ell(\phi, \theta, \psi)
    \, \sin\theta \,d\theta\,d\phi \\
  &= (-1)^s \sqrt{\frac{2\ell + 1}{4\pi}}
    \int\limits_{\theta, \phi} \li{s}{f}(\theta,\phi)
    e^{-im\phi} d_{m,-s}^\ell(\theta)
    \, \sin\theta \,d\theta\,d\phi \\
  &= (-1)^s \sqrt{\frac{2\ell + 1}{4\pi}}
    \int\limits_{\theta, \phi} \li{s}{f}(\theta,\phi)
    e^{-im\phi}
    i^{m+s}\sum_{k=-\ell}^\ell \Delta_{k,m}^\ell e^{-ik\theta}\Delta_{k,-s}^\ell
    \, \sin\theta \,d\theta\,d\phi \\
  &= (-1)^si^{m+s} \sqrt{\frac{2\ell + 1}{4\pi}}
    \sum_{k=-\ell}^\ell \Delta_{k,m}^\ell \Delta_{k,-s}^\ell
    \int\limits_{\theta, \phi} \li{s}{f}(\theta,\phi)
    e^{-im\phi} e^{-ik\theta}
    \, \sin\theta \,d\theta\,d\phi \\
  &= (-1)^si^{m+s} \sqrt{\frac{2\ell + 1}{4\pi}}
    \sum_{k=-\ell}^\ell \Delta_{k,m}^\ell \Delta_{k,-s}^\ell
    I_{k,m}.
\end{align*}
Since the $\Delta_{m,n}^\ell$ are constants, they are pre-computed.
We still need to compute
\begin{align}
  I_{k,m}=\int\limits_{\theta, \phi} \li{s}{f}(\theta,\phi)e^{-im\phi} e^{-ik\theta}\, \sin\theta \,d\theta\,d\phi,
  \label{eq:Ikm}
\end{align}
which can be done efficiently with an FFT.
There is a problem because $\li{s}{f}$ is defined on the sphere so it is not
periodic in both directions; we then define $\li{s}{f'}$ as the periodic
extension of $\li{s}{f}$ which is a function on the torus.
See \textcite{mcewen2008fast,Huffenberger_2010} for more details about this extension.
We can then express $\li{s}{f'}$ by its Fourier coefficients,
\begin{align}
  \li{s}{f'}(\theta, \phi) = \sum_{p,q}  \li{s}{\hat{f}'_{p,q}} e^{ip\theta}e^{iq\phi}.
\end{align}
Substituting this in \cref{eq:Ikm} yields,
\begin{align*}
  I_{k,m}
  &=\int\limits_{\theta=0}^{\pi}\int\limits_{\phi=0}^{2\pi}
    \sum_{p,q}  \li{s}{\hat{f}'_{p,q}} e^{ip\theta}e^{iq\phi}
    e^{-im\phi} e^{-ik\theta}\, \sin\theta \,d\theta\,d\phi \\
  &=\sum_{p,q}  \int\limits_{\theta=0}^{\pi}\int\limits_{\phi=0}^{2\pi}
    \li{s}{\hat{f}'_{p,q}} e^{i(p-k)\theta}e^{i(q-m)\phi}\,
    \sin\theta \,d\theta\,d\phi \\
  &=\sum_{p}  2\pi \int\limits_{0}^{\pi}
    \li{s}{\hat{f}'_{p,m}} e^{i(p-k)\theta}\, \sin\theta \,d\theta \\
  &=2\pi \sum_{p} \li{s}{\hat{f}'_{p,m}} \hat{w}(p-k),
\end{align*}
where $\hat{w}$ can be obtained analytically.
Note that the last expression is a $1$D discrete convolution;
if we see $\hat{w}$ as the Fourier transform of some $w$,
the convolution can be evaluated as the FFT of the multiplication in the spatial domain,
\begin{align}
  I_{k,m} = \frac{2\pi}{N^2} \sum_{\theta, \phi}\li{s}{f'(\theta, \phi)} {w}(\theta) e^{-ik\theta}e^{-im\phi},
\end{align}
for $N$ uniformly sampled $\theta,\,\phi$.
Here, $w$ can be pre-computed,
so computing $I_{k,m}$ amounts to
1) extend the function to the torus,
2) apply the weights $w$,
3) compute a $2$D FFT.

\subsection{Spin-weighted spherical CNNs}
Our main operation is the convolution defined in \cref{spin:sec:swconv}.
Since components with the same spin can be added,
the generalization to multiple channels is immediate.
Convolution combines features of different spins, so we enforce the
same number of channels per spin per layer.
Each feature map then consists of a set of \swsfs\ of different spins,
$F = \{\fun{\li{s}{f}}{S^2}{\C^k} \mid s \in W_F\}$, where $k$ is the number
of channels and $W_F$ the set of spins.

\paragraph{Filter localization}
We compute the convolutions in the spectral domain but apply nonlinearities, batch normalization
and pooling in the spatial domain.
This requires expanding the feature maps into the \swshs\ basis and back at every layer,
but the filters themselves are parameterized by their spectrum.
We follow the idea of \cref{sph:sec:locfilters} to enforce
filter localization with spectral smoothness.
The filters there are of the form $\li{0}{\hat{k}}_0^\ell$, so the spectrum is $1$D
and can be interpolated from a few anchor points,
smoothing it out and reducing the number of parameters.
In the current case, filters take the general form $\li{s}{\hat{k}}_m^\ell$
where $s \in W_{F*K}$ are the output spin weights and $m \in W_F$ are the input
spin weights.
We then interpolate the spectrum of each component along the degrees $\ell$,
resulting in a factor of $\abs{W_{F*K}} \abs{W_F}$ more parameters per layer.

\paragraph{Batch normalization and nonlinearity}
We force features with spin weight $s=0$ to be real
by taking their real part after every convolution.
Then we can use the common \relu\ as the nonlinearity
and the standard batch normalization from \textcite{IoffeS15}.

For $s>0$, we have complex-valued feature maps.
Since values move and change phase upon rotation,
equivariant operations must commute with this behavior.
Pointwise operations on magnitudes satisfy this requirement.
Similarly to \textcite{worrall2017harmonic}, we employ a variation of
the \relu\  to the complex values $z = ae^{i\theta}$ as follows,
where $a \in \R^+$ and $b\in \R$ is a learnable scalar,
\begin{align}
  z \mapsto \max(a + b, 0) e^{i\theta}.
\end{align}
Batch normalization is also applied pointwise, but it does not commute with spin-weighted rotations
because of the mean subtraction and offset addition steps.
We adapt it by removing these steps,
where $\sigma^2$ is the channel variance,
$\gamma \in \C$ is a learnable factor
and $\eps \in \R^+$ is a constant added for stability,
\begin{align}
  z \mapsto \frac{z}{\sqrt{\sigma^2+ \eps}} \gamma.
\end{align}
As usual, the variance is computed along the batch during training and along the
whole dataset during inference.
The variance of a set of complex numbers is real and only depends on their magnitudes;
we use a spherical quadrature rule to compute it.

\paragraph{Complexity analysis}
We follow \textcite{Huffenberger_2010} for the \swsft\ implementation,
whose complexity for bandwidth $B$ is $\mathcal{O}(B^3)$.
While it is asymptotically slower than the $\mathcal{O}(B^2 \log^2{B})$ of the standard \sft\ from \textcite{driscoll1994computing},
the difference is small for bandwidths typically needed in practice~
\cite{s.2018spherical,esteves18eccv,kondor2018clebsch}.
The \soft\ implementation from \textcite{kostelec2008ffts} is $\mathcal{O}(B^4)$.
Our final model requires $\abs{W}$ transforms per layer, so
it is asymptotically a factor $\nicefrac{\abs{W}B}{\log^2{B}}$ slower than
using SFT as in \textcite{esteves18eccv},
and a factor $\nicefrac{B}{\abs{W}}$ faster than using the SOFT as in \textcite{s.2018spherical}.
Typical values in our experiments are $B=32$ and $\abs{W}=2$.

\section{Experiments}
We start with experiments on image and vector field classification,
image prediction from a vector field, and vector field from an image,
where all images and vector fields are on the sphere.
Next, we show applications to $3$D shape classification and semantic segmentation
of spherical panoramas.

We use only spin weights $0$ and $1$.
When inputs do not have both spins, the first layer
is designed so that its outputs have.
Following features and filters also have spins $0$ and $1$.

Every model is trained with different random seeds five times and
averages and standard deviations (within parenthesis) are reported.

\subsection{Spherical Image Classification}
\label{spin:sec:sphmnist}
Our first experiment is on the Spherical MNIST dataset introduced by \textcite{s.2018spherical}.
This is an image classification task where the handwritten digits from MNIST
are projected on the sphere.
Three modes are evaluated depending on whether the training/test set are rotated (R)
or not (NR).

We simplify the architecture in \cref{sph:sec:arch} to have a single branch.
The spherical baseline has six layers with $16,16,32,32,58,58$ channels per layer,
and $8$ filter parameters per layer.
The \swscnn\ follows the same topology,
switching from spherical to spin-weighted convolutions.
Since the filters now have richer spectra, they need more parameters.
In order to keep similar capacity between competing models,
we set the number parameters per spin-order pair $(s,m)$%
\footnote{We use spins 0 and 1 throughout: $M_F=M_K=\{0, 1\}$.
  This amounts to four spin-order pairs per filter per degree:
  $\li{0}k_0^\ell,\, \li{0}k_1^\ell,\, \li{1}k_0^\ell,\, \li{1}k_1^\ell$.}
to $6,6,4,4,3,3$ at each layer.
We also cut the number of channels per layer, so
while we have the same number of parameters, we have significantly fewer feature maps.
The final architecture has $16,16,20,24,28,32$ channels per layer, with
pooling every two layers, and our custom batch normalization applied at every layer.
The planar baseline has the same number of layers and uses $2$D convolutions with
$3\times 3$ kernels. We set the number of channels per layer to $16,16,32,32,54,54$.
to match the number of parameters of the other models.

Training lasts $12$ epochs using the Adam optimizer~\cite{KingmaB14},
optimizing the usual cross-entropy loss.
We set the initial learning rate to \num{1e-3}
and decay it to \num{2e-4} epoch $6$ and \num{4e-5} at epoch $10$.
The mini-batch size is set to $32$ and input resolution is $64\times 64$

\Cref{tab:sphmnist} shows the results; we outperform previous spherical CNNs in every mode.

\begin{table}[htbp]
  \caption{Spherical MNIST results.
    Our model is more expressive than the isotropic
    and more efficient than the previous anisotropic spherical CNNs,
    allowing deeper models and improved performance.
  }
  \label{tab:sphmnist}
  \centering
  \begin{tabular}{@{}
    l
    S[table-format=2.2(1)]
    S[table-format=2.2(1)]
    S[table-format=2.2(1)]
    r
    Z
    @{}}
    \toprule
                                    & {NR/NR}            & {R/R}              & {NR/R}             & {params}   & {time [s]} \\
    \midrule
    Planar CNN                      & \bgl 99.07 +- 0.04 & 81.07 +- 0.63      & 17.23 +- 0.71      & \SI{59}{k} & {-}        \\
    \textcite{s.2018spherical}      & 95.59              & 94.62              & 93.4               & \SI{58}{k} & {-}        \\
    \textcite{kondor2018clebsch}    & 96.4               & 96.6               & 96.0               & {-}        & {-}        \\
    \textcite{esteves18eccv}        & 98.75 +- 0.08      & \bgl 98.71 +- 0.05 & \bgl 98.08 +- 0.24 & \SI{57}{k} & 294        \\
    Ours                            & \bgd 99.37 +- 0.05 & \bgd 99.37 +- 0.01 & \bgd 99.08 +- 0.12 & \SI{58}{k} & 548        \\
    \bottomrule
  \end{tabular}
\end{table}

\subsection{Spherical Vector Field Classification}
One crucial advantage of the \swscnns\ is that they
are equivariant as vector fields.
To demonstrate this, we introduce a spherical vector field dataset and a classification task.
\paragraph{Dataset}
We start from MNIST~\cite{lecun2010mnist}, compute image gradients
with Sobel kernels and project the vectors to the sphere.
To increase the challenge, we follow \textcite{larochelle2007empirical}
and swap the train and test sets so there are \SI{10}{k} images for training and
\SI{50}{k} for test.
We call this dataset \svfmnist; \cref{spin:fig:dset_cls} shows some samples.
\begin{figure}[thbp]
  \centering
  \includegraphics[width=0.9\linewidth]{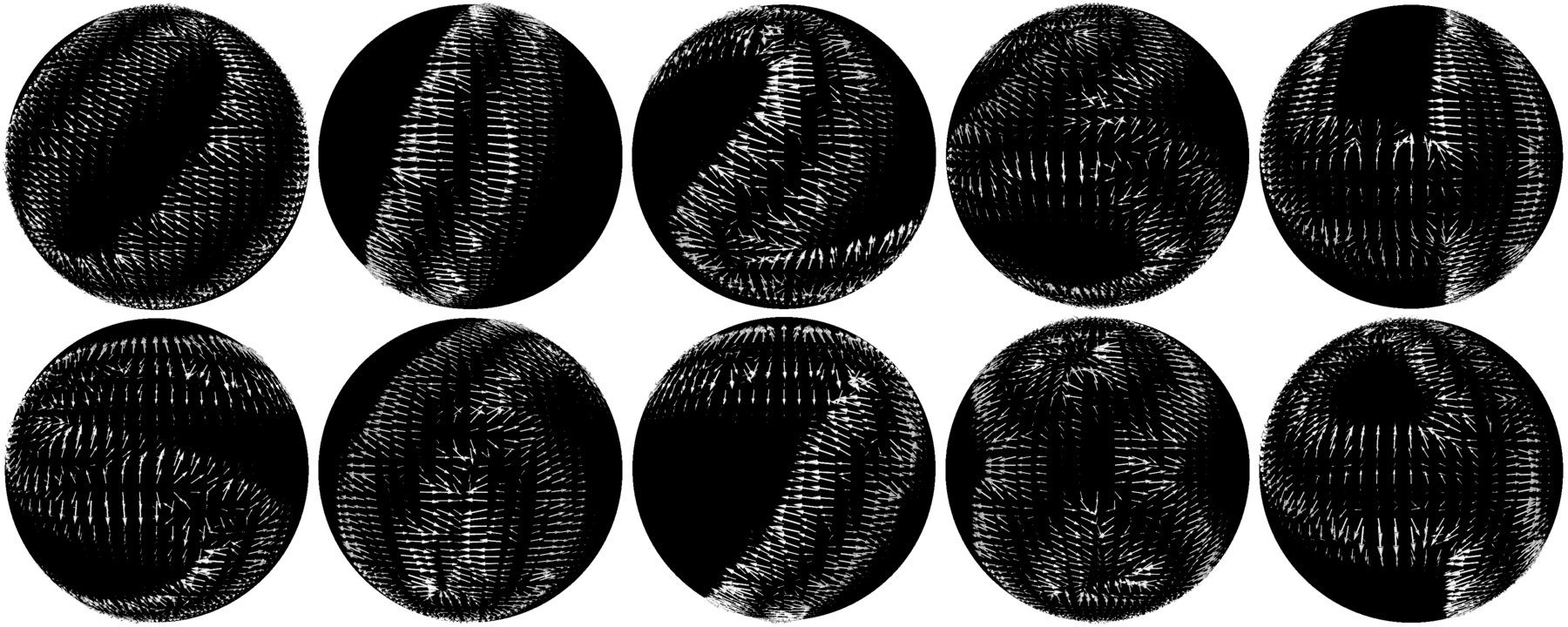}
  \caption{
    Samples from \svfmnist, classification task.
    We show one sample for each category in canonical orientation
    for easy visualization.
  }
\label{spin:fig:dset_cls}
\end{figure}

The vector field is converted to a spin weight $s=1$ complex-valued function
using a predefined local tangent frame per point on the sphere.
The inverse procedure converts $s=1$ features to output vector fields.%

\paragraph{Architecture} We use the same architecture and training protocol as
described in \cref{spin:sec:sphmnist},
the only difference is that now the first layer maps from
spin 1 to spins 0 and 1.
The planar and spherical CNN models take the vector field as
a 2-channel input.

\paragraph{Results} \Cref{tab:sphvecmnistcls} shows the results.
The NR/R column clearly shows the advantage of vector field equivariance;
the baselines cannot generalize to unseen vector field rotations,
even when they are equivariant in the scalar sense as \cite{esteves18eccv}.
\begin{table}[htbp]
  \caption{\Acrlong{svfmnist} classification results.
    When vector field equivariance is required,
    the gap between our method and the spherical and planar baselines is larger.}
  \label{tab:sphvecmnistcls}
  \centering
  \begin{tabular}{@{}
    l
    S[table-format=2.1(1)]
    S[table-format=2.1(1)]
    S[table-format=2.1(1)]
    }%
    \toprule
                           & {NR/NR}            & {R/R}              & {NR/R}                                    \\
    \midrule
    Planar                 & 97.7 +- 0.2      & 50.0 +- 0.8      & 14.6 +- 0.9                             \\
    \textcite{esteves18eccv}   & \bgd 98.4 +- 0.1 & \bgl 94.5 +- 0.5 & \bgl 24.8 +- 0.8                        \\
    Ours                   & \bgl 98.2 +- 0.1 & \bgd 97.8 +- 0.2 & \bgd 98.2 +- 0.7                        \\
    \bottomrule
  \end{tabular}
\end{table}
\subsection{Spherical Vector Field Prediction}
The \swscnns\ can also be used for dense prediction.
We introduce two new tasks on \svfmnist, 1) predicting a vector field from an image
and 2) predicting an image from a vector field.
For these tasks, we implement a fully convolutional U-Net architecture~\cite{ronneberger15miccai}
with spin-weighted convolutions, and use the same training procedure as in \cref{spin:sec:sphmnist},
but minimizing the mean squared error.

\paragraph{Datasets} When the image is a grayscale digit
and the vector field comes from its gradients,
both tasks can be easily solved via discrete integration and differentiation.
We call this case ``easy''.
It highlights a limitation of isotropic spherical CNNs; the
results show that the constrained filters cannot approximate a simple image gradient operator.

We also experiment with a more challenging scenario, denoted ``hard'', where the digits are colored
and the vector fields are rotated based on the digit category.

When predicting an image from a vector field,
the color is determined in HSV space, where the value is the original grayscale value,
the hue is $c/10$ for category $c$, and the saturation is set to one.
The target is then converted back to RGB.
\Cref{spin:fig:dset_dense_scalar} shows a few input/target pairs.
\begin{figure}[thbp]
  \centering
  \includegraphics[width=0.9\linewidth]{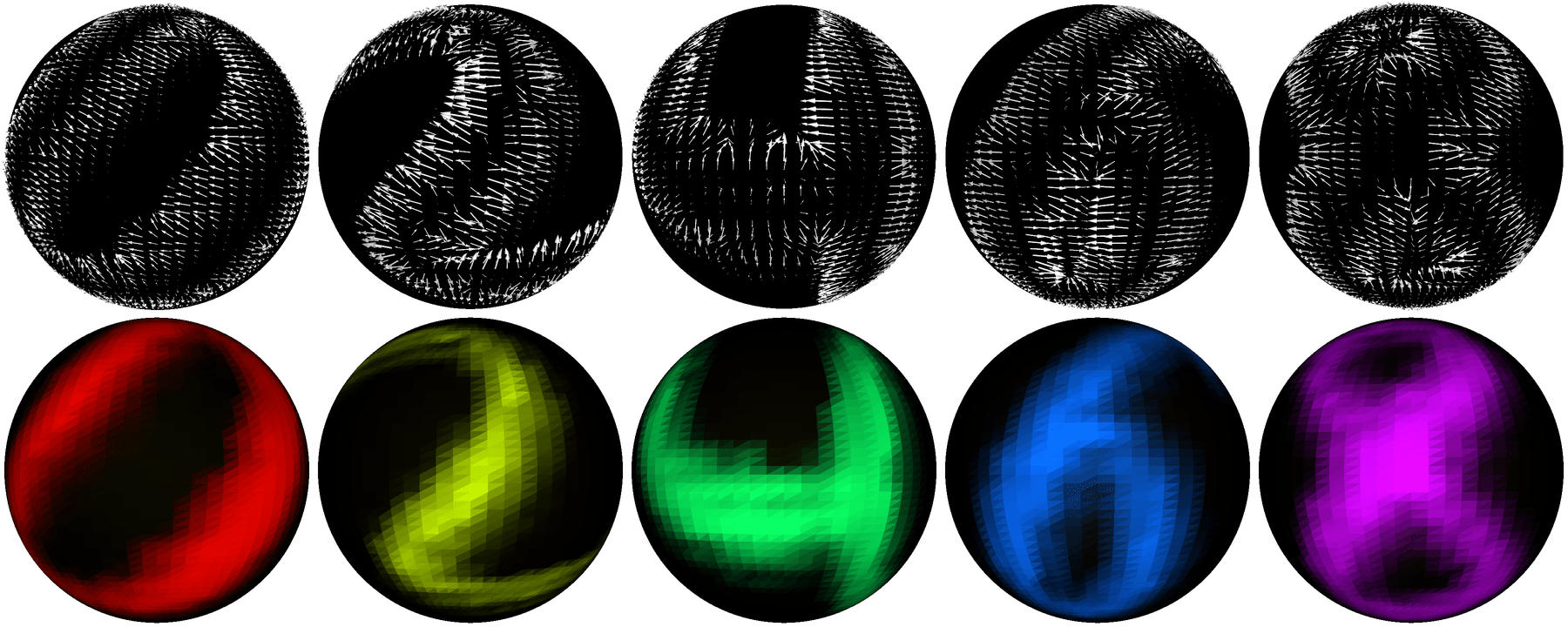}
  \caption{
    Samples from \svfmnist, image from vector field prediction ``hard'' task.
    Top shows input vector fields, bottom the target spherical images.
    Note that the targets have different colors based on the category,
    so the task cannot be solved via simple gradient integration.
    Samples are in canonical orientation for easy visualization.
  }
\label{spin:fig:dset_dense_scalar}
\end{figure}

When predicting a vector field from an image, the angular
offset on all vectors depends on the target category.
The offset for category $c$ is given by $\exp(2\pi i c/10)$.
\Cref{spin:fig:dset_dense_vector} shows a few input/target pairs.
\begin{figure}[thbp]
  \centering
  \includegraphics[width=0.9\linewidth]{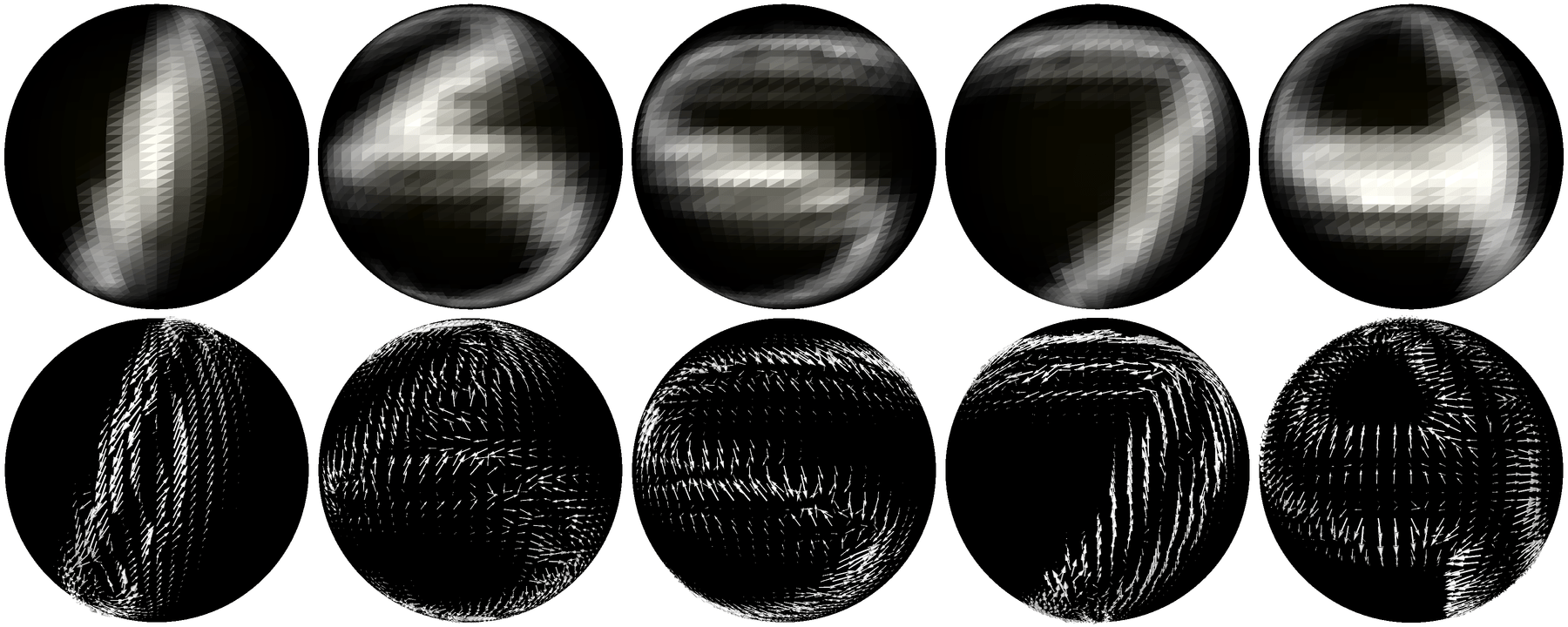}
  \caption{
    Samples from \svfmnist, vector field from image prediction ``hard'' task.
    Top shows input spherical images, bottom the target vector fields.
    The targets have different angular offsets based on the category so
    the task cannot be solved via simple image gradient estimation.
    Samples are in canonical orientation for easy visualization.
  }
\label{spin:fig:dset_dense_vector}
\end{figure}

The ``hard'' tasks involve semantics and require the network to implicitly classify the input
in order to correctly predict output color and vector directions.

\paragraph{Architectures} The architecture for dense prediction
is a fully convolutional U-Net~\cite{ronneberger15miccai} with spin-weighted convolutions.
We use $16,32,32,32,32,16$ channels per layer, with
pooling in the first two layers and nearest neighbors upsampling in the last two.
The number of filter parameters chosen per spin-order per layer is $6,4,3,3,4,6$.

The spherical CNN baseline uses spherical convolutions
and sets the numbers of filter parameters to $8$ per layer and the number
of channels to $20,40,78,78,40,20$.

The planar baseline again uses $2$D convolutions with $3\times 3$ kernels and
of channels to $18,36,72,72,36,18$ channels.

\paragraph{Results} \Cref{tab:sphvecmnistdense} shows the results.
While the planar baseline does well in the ``easy'' tasks that can be solved with
simple linear operators, our model still outperforms it when generalization
to unseen rotations is demanded (NR/R).
In the ``hard'' task, the \swscnns\ are clearly superior by large margins.
\begin{table}[htbp]
  \caption{Vector field to image and image to vector field  results on \svfmnist.
    The \swscnns\ show superior performance, especially on the more challenging tasks.
    The metric is the mean-squared error $\times 10^{3}$ (lower is better).
    All models have around {112k} parameters.}
  \label{tab:sphvecmnistdense}
\centering
{
  \begin{tabular}{@{}l
    l %
    S[table-format=2.1(1)]S[table-format=2.1(1)]S[table-format=2.2(1)]
    l
    S[table-format=2.1(1)]S[table-format=2.1(1)]S[table-format=2.1(1)]
    l %
    Z %
    Z %
    @{}}
    \toprule
                             &  &                 & {easy}          &                 &  &                  & {hard}          &                  &  &          &        \\
    \cmidrule{3-5} \cmidrule{7-9}
                             &  & {NR/NR}         & {R/R}           & {NR/R}          &  & {NR/NR}          & {R/R}           & {NR/R}           &  & {params} & {time} \\
    \midrule
    \multicolumn{12}{l}{\textbf{Image to Vector Field}} \\
    Planar                   &  & \bgd 0.3 +- 0.1 & 5.0 +- 0.1      & 9.3 +- 0.1      &  & 16.9 +- 0.5      & 26.0 +- 0.1     & 32.9 +- 0.2      &  & {112k}   & 5      \\
    \textcite{esteves18eccv} &  & 9.7 +- 0.3      & 31.0 +- 0.2     & 45.6 +- 0.7     &  & 13.3 +- 0.6      & 28.5 +- 0.4     & 41.6 +- 0.4      &  & {112k}   & 36     \\
    Ours                     &  & 2.9 +- 0.2      & \bgd 3.4 +- 0.1 & \bgd 4.3 +- 0.1 &  & \bgd 11.6 +- 0.6 & \bgd 9.2 +- 0.4 & \bgd 10.2 +- 0.6 &  & {112k}   & 67     \\
    \multicolumn{12}{l}{\textbf{Vector Field to Image}} \\
    Planar                   &  & \bgd 1.4 +- 0.1 & \bgd 3.2 +- 0.1 & 6.9 +- 0.4      &  & 3.3 +- 0.2       & 13.4 +- 0.2     & 21.1 +- 0.3      &  & {112k}   & 5      \\
    \textcite{esteves18eccv} &  & 3.8 +- 0.1      & 4.9 +- 0.2      & 15 +- 2         &  & \bgd 2.6 +- 0.1  & 6.4 +- 0.2      & 20.3 +- 0.9      &  & {112k}   & 37     \\
    Ours                     &  & 3.5 +- 0.1      & 3.8 +- 0.1      & \bgd 4.0 +- 0.1 &  & \bgd 2.6 +- 0.1  & \bgd 2.7 +- 0.1 & \bgd 2.9 +- 0.1  &  & {112k}   & 73     \\
    \bottomrule
  \end{tabular}
}
\end{table}

We show examples of inputs and outputs for the dense prediction tasks;
\Cref{spin:fig:out_dense_scalar} shows the vector field to image task while
\cref{spin:fig:out_dense_vector} shows the image to vector field task.
Models are trained on the R mode, so they have access to rotated samples
at training time.
Nevertheless, the standard CNN and spherical CNN models are not equivariant
in the vector field sense and cannot achieve the same accuracy as
the \swscnns.
\begin{figure}[thbp]
  \centering
  \includegraphics[width=\linewidth]{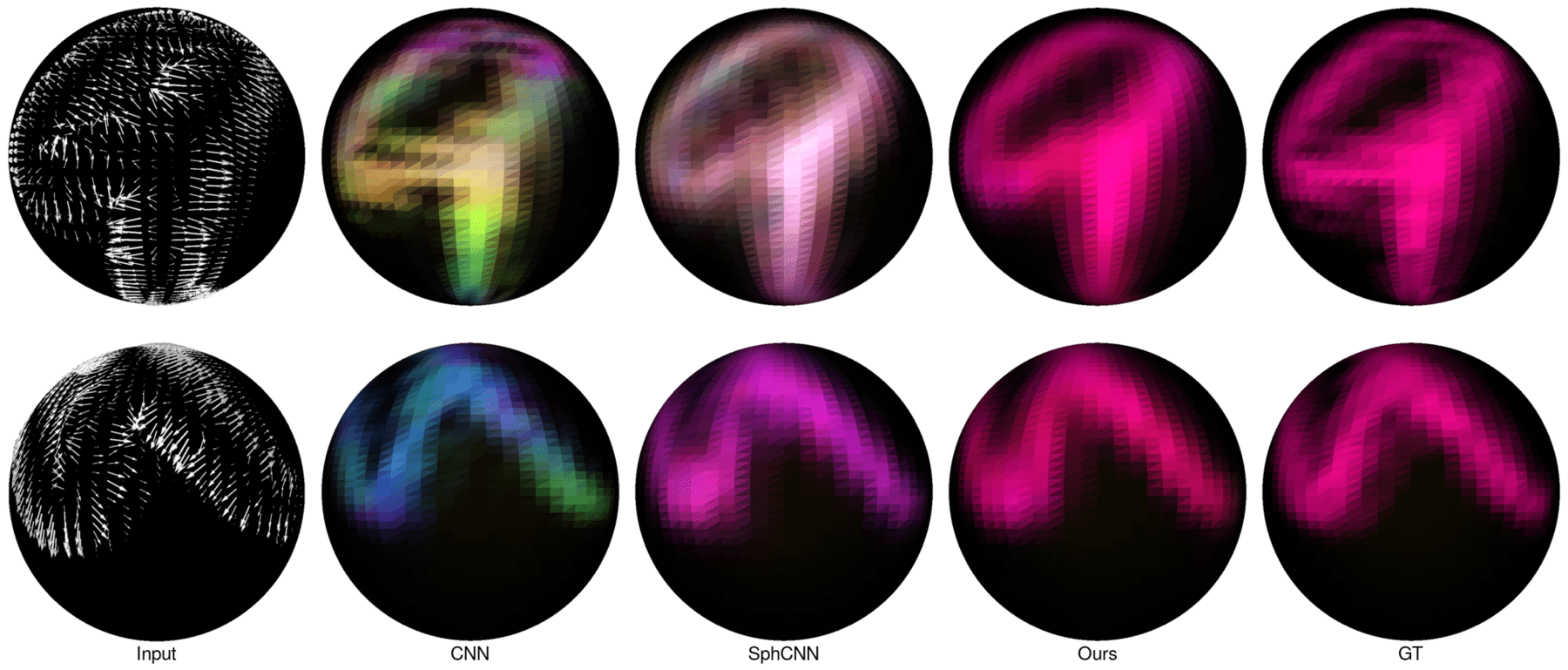}
  \caption{
    Input/output samples for the spherical vector field to image task.
    We show two rotated instances of the same input to highlight
    that standard CNNs and spherical CNNs do not respect the spherical vector
    field equivariance, while the \swscnns\ do.
  }
\label{spin:fig:out_dense_scalar}
\end{figure}
\begin{figure}[thbp]
  \centering
  \includegraphics[width=\linewidth]{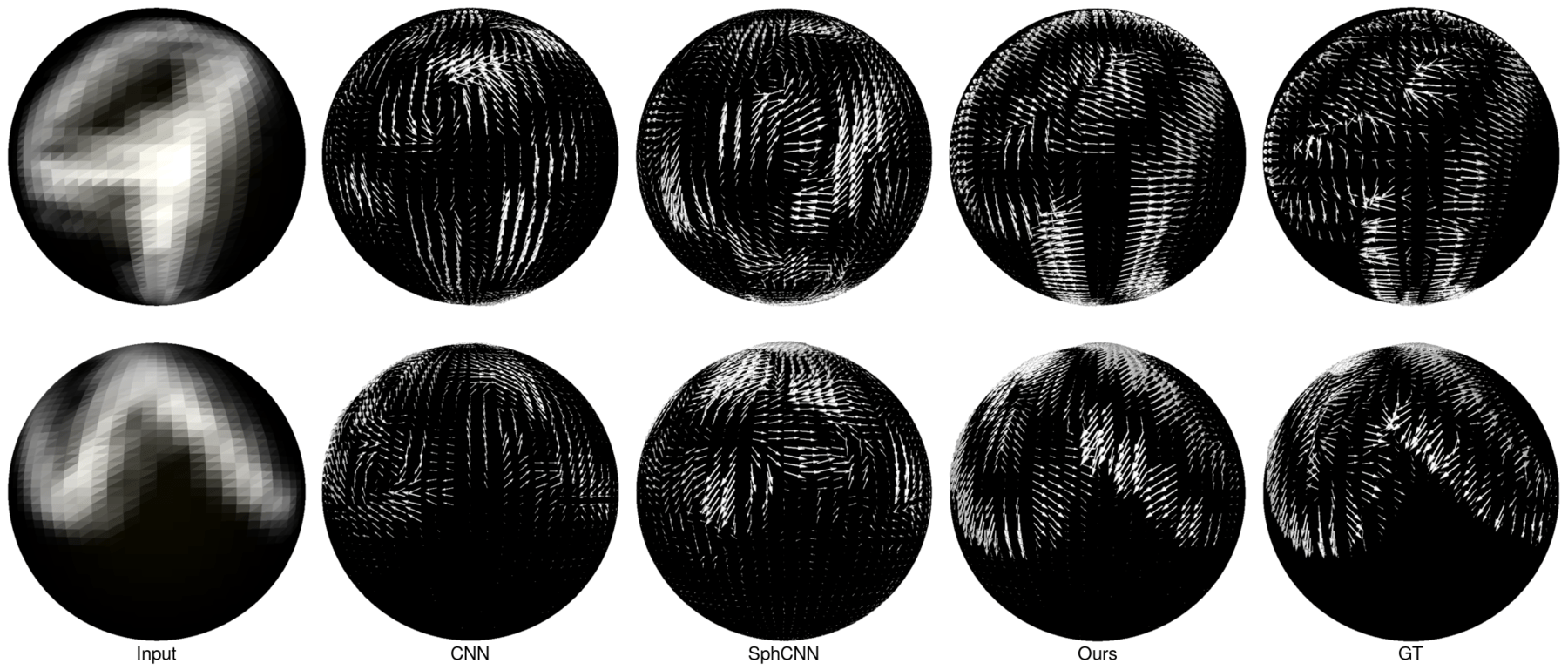}
  \caption{
    Input/output samples for the spherical image to vector field task.
  }
\label{spin:fig:out_dense_vector}
\end{figure}
\subsection{Classification of $3$D shapes}
\label{spin:sec:m40}
We tackle $3$D object classification on ModelNet40~\cite{wu20153d},
following the protocol from \cref{sph:sec:experiments} which considers
azimuthally  and arbitrarily rotated shapes.

Besides more expressive filters,
our method also represents the shapes more faithfully on the sphere.
\textcite{esteves18eccv,s.2018spherical} cast rays from the shape's center
and assign the intersection distance and angle between normal and ray to points on the sphere.
Normals are not uniquely determined by a single angle
but this limitation was necessary to preserve equivariance as a scalar field.

By using \swscnns, we can represent any normal direction uniquely,
without breaking equivariance.
We split the vector in radial and tangent components, where
the radial is represented with spin $s=0$ and the tangent has $s=1$.
Since the intersection distance has also $s=0$,
our $3$D shape representation has two spherical channels with $s=0$ and one of $s=1$.
Following \textcite{s.2018spherical}, we also use the convex hull for extra channels.

When inputs have limited orientations, a globally equivariant model
can be undesirable, even though equivariance in the local sense is still useful.
We can keep the benefits while still having access to the global pose
by breaking equivariance on the final layers, which we do by simply replacing them with regular $2$D convolutions.
We call this model ``Ours + BE''; it results in better performance on ``upright''
but worse on ``rotated'', as expected.

\paragraph{Training} We train for 48 epochs using the Adam optimizer~\cite{KingmaB14},
with learning rate linearly increasing from 0 to \num{5e-3} during the first epoch
then decayed by a factor of 5 at epochs 32 and 44.
The mini-batch size is 32 and input resolution is $64\times 64$.
The cross-entropy loss is optimized and we found that label smoothing
regularization~\cite{szegedy2016rethinking} with $\eps=0.2$ is beneficial.

\paragraph{Architecture} The basic block is residual~\cite{he2016deep} with a bottleneck
halving the number of channels when input and output number of channels match.
Our custom batch normalization and nonlinearity is applied to the complex feature maps.
We use $32,32,64,64,128,128$, $256,256$ channels per layer where average pooling is applied before
each increase in the number of channels,
and $6,6,4,4,3,3,3,3$ filter parameters are learned per spin-order per layer,
with a total of \SI{1.2}{M} parameters.
When breaking equivariance (``Ours + BE''), we replace the last two layers by 3 blocks of
$2$D convolution with $3\times 3$ kernels.

\paragraph{Baselines}
The same training procedure and architecture are used for the SphCNN~\cite{esteves18eccv} baseline,
which explains the superior numbers we report when comparing with the original paper.

We evaluate the baseline from \textcite{JiangHKPMN19} following the recipe in the paper.
The only difference is that we randomly rotate the training and test sets.
Each training set object is rotated multiple times to serve as augmentation.
The numbers we obtain differ from the \SI{90.5}{\%} accuracy reported in the original paper
because our results are for azimuthally and arbitrarily rotated datasets while
the original has all objects in a canonical pose.

\paragraph{Results} \Cref{tab:m40} compares with previous spherical \cnns.
The ``upright'' mode has only azimuthal rotations while
``rotated'' is unrestricted.
EMVN~\cite{Esteves_2019_ICCV} is state-of-the-art on this task
with \num{94.4}{\%} accuracy on ``upright'' and \num{91.1}{\%} on ``rotated'', but
it requires \num{60} images as input and much larger model.
\begin{table}[htbp]
  \caption{ModelNet40 shape classification accuracy [\%].
    Our model outperforms previous spherical CNNs
    while requiring small input size and low parameter count.}
  \label{tab:m40}
\centering
{
  \begin{tabular}{
    @{}l
    S[table-format=2.1(1)]S[table-format=2.1(1)]
    }
\toprule
                                & {upright} & {rotated} \\
\midrule
UGSCNN \cite{JiangHKPMN19}      & 87.3 +- 0.3      & 81.9 +- 0.9 \\
SphCNN \cite{esteves18eccv}     & 89.3 +- 0.5   & \bgl 88.4 +- 0.3  \\
Ours                            & \bgl 89.6 +- 0.3   & \bgd 88.8 +- 0.1  \\
Ours + BE                       & \bgd 90.1 +- 0.3   & 88.2 +- 0.2 \\    
\bottomrule
  \end{tabular}
}
\end{table}
\subsection{Semantic segmentation of spherical panoramas}
We evaluate our method on the Stanford $2$D$3$DS dataset~\cite{ArmeniSZS17},
which contains 1,413 RGB-D panoramas with corresponding pixelwise semantic labels and normals.
We follow the usual protocol of reporting the average performance over the three official folds,
and we use the same weights per class as \textcite{JiangHKPMN19} to mitigate the class imbalance.

As in \cref{spin:sec:m40}, our model is able uniquely represent surface normals.
In this task, representing the normals with respect to local tangent frames is also more realistic,
as they could be estimated from a depth sensor without knowledge of global orientation.
Note that competing methods don't usually leverage the normals, so we also show results
without them for comparison.

\paragraph{Training}
We train for 48 epochs using the Adam optimizer~\cite{KingmaB14},
with the learning rate linearly increasing from \num{0} to \num{1e-2} during the first epoch
then decayed by a factor of 10 at epoch 40.
The mini-batch size is 8 and input resolution is $128\times 128$.
The pixelwise cross-entropy loss is optimized with label smoothing
regularization~\cite{szegedy2016rethinking} with $\eps=0.2$.

\paragraph{Architecture}
A fully convolutional U-Net~\cite{ronneberger15miccai} architecture
is used with same residual block described in \cref{spin:sec:m40}.
We use $16,64,128,128,256,256,128,128,64,16$ channels per layer
where average pooling/nearest neighbor upsampling
is applied before each increase/decrease in the number of channels, and
$8,6,6,4,4,3,3,4,4,6,6,8$ filter parameters are learned per spin-order per layer,
with a total of \SI{2.5}{M} parameters.
When breaking equivariance in ``Ours + BE'', we replace the last layer by six blocks of
$2$D convolutions with $3\times 3$ kernels and 32 channels.

\paragraph{Results} \Cref{tab:2d3ds} shows the results in terms of pixelwise accuracy and \miou.
Inputs are upright so global \SO{3} equivariance is not required;
nevertheless, our method matches the state-of-the-art performance,
which demonstrates the representational power of the \swscnns.
\begin{table}[htbp]
  \caption{Semantic segmentation on Stanford $2$D$3$DS.
    Our model clearly outperforms previous equivariant models
    and matches the state-of-the-art non-equivariant model.
}
  \label{tab:2d3ds}
\centering
{
  \begin{tabular}{
    @{}
    l
    S[table-format=2.1(1)]S[table-format=2.1(1)]
    }
    \toprule
                                              & {acc [\%]}    & {\miou}         \\
    \midrule
    UGSCNN \cite{JiangHKPMN19}                & 54.7         & 38.3           \\
    Gauge CNN \cite{CohenWKW19}               & 55.9         & 39.4           \\
    HexRUNet \cite{zhang2019orientation}      & \bgl 58.6    & \bgl 43.3      \\
    SphCNN \cite{esteves18eccv}               & 52.8(6)      & 40.2(3)        \\
    Ours                                      & 55.6(5)      & 41.9(5)        \\
    Ours + normals                            & 57.5(6)      & \bgd 43.4(4)   \\
    Ours + normals + BE                       & \bgd 58.7(5) & \bgd 43.4(4)   \\
    \bottomrule
  \end{tabular}
}
\end{table}
\section{Conclusion}
In this chapter, we introduced the \acrlongpl{swscnn}, which use sets of
\acrlongpl{swsf} as features and filters, and employ layers of a newly introduced
spin-weighted spherical convolution to process spherical images or spherical vector fields.
Our model achieves superior performance on the tasks attempted, at a reasonable
computational cost.
We foresee further applications of the \swscnns\ to $3$D shape analysis,
climate/atmospheric data analysis
and other tasks where inputs or outputs can be represented as spherical images or vector fields.

\glsresetall
\chapter{Conclusion and future work}
\label{sec:conclusion}
\section{Conclusion}
This thesis presented different methods of learning equivariant representations with \cnns, and demonstrated promising results in multiple tasks.
The methods introduced leverage symmetries in the data to reduce sample and model complexity
and improve generalization performance.
We conclude by reiterating the main ideas and their applications.
\begin{enumerate}
\item The \ptns\ provide equivariance to the group of similarities on
the plane via a transformation to canonical coordinates, and were
applied to image classification.  This was one of the first models
equivariant to scale, and equivariant to a continuous group of
transformations other than translation.

\item The \emvns\ achieve equivariance to the icosahedral group of discrete rotations
through discrete group convolutions.
They were applied to $3$D shape classification, retrieval, and panoramic image classification.
These models leverage image descriptors from multiple views
to construct a function on the group that is input to a \gcnn.
The descriptors can come from any other model.

\item The spherical \cnns\ achieve equivariance to \SO{3}, the continuous group of $3$D rotations,
through spherical convolutions evaluated in the spectral domain.
Applications to $3$D shape classification, retrieval, and shape alignment were demonstrated.
This was the first model based on spherical convolutions with inputs, filters,
and features on the sphere, and
also one of the first to achieve \SO{3}-equivariance.

\item The cross-domain embeddings were introduced to obtain equivariant spherical
representations from a $2$D view of a $3$D object.
They enable computation of the relative $3$D pose between two views of the object through spherical
correlation between their spherical embeddings, and also generation of novel views by rotating
and inverting the embeddings.
This was the first learned model for pose estimation with no regression, classification,
or keypoints involved.

\item The \swscnns\ generalize the spherical \cnns.
They remove the isotropic filter constraint in an efficient way and lead to more expressive models,
also allowing equivariant processing of vector fields on the sphere.
Applications to spherical image classification, semantic segmentation,
and spherical vector field classification and generation were shown.
This model was the first demonstrate \SO{3}-equivariance for spherical vector fields.
\end{enumerate}
\section{Future work}

\subsection{Invertible mesh to sphere mapping}
Applications of our spherical and spin-weighted spherical \cnns\ to $3$D shape analysis
rely on the procedure to convert a mesh to a spherical function described in
\cref{sph:sec:spherical-3d-object}.
There, we cast rays that intersect the mesh and
construct a spherical function based on the intersection point properties.
When there are multiple intersections per ray, only a single one is used,
which makes the process non-invertible.

Besides the loss of information that may hurt tasks like classification and retrieval,
the non-invertibility forbids applications that require dense predictions
such as $3$D object part segmentation and mesh generation,
so finding an invertible mesh to sphere map would benefit several fronts.

One way to obtain such a map is through \emph{mean curvature flow},
a surface evolution process where points move
at velocities proportional to the local mean curvature, in the direction of the local surface normal.
Intuitively, points move inwards where the curvature is positive and outwards where it is negative,
so watertight genus-zero surfaces evolving under this rule tend to approach a sphere and shrink.
\textcite{mantegazza2011lecture} describes these flows in full detail.
There are different ways of computing curvature flows \cite{Kazhdan_2012,crane2013robust}
and properties like conformality or authalicity may be enforced.
Neural ODEs~\cite{ChenRBD18} may also enable learning
the flow instead of following the mean curvature rule;
\textcite{gupta20_neural_mesh_flow} applies it to learn a mapping from sphere to mesh.

Some complicating factors when computing the flow occur when inputs are not
(i) manifold meshes, (ii) watertight, and (iii) genus-zero,
so some form of pre-processing should be used to handle these cases.

\textcite{Sinha2016,maron2017convolutional} present similar ideas that map meshes to
the sphere and torus, respectively.
While \textcite{Sinha2016} also obtain a spherical representation of a mesh, they use
it just as an intermediate and flatten it to a square image to be processed by a \cnn.
\textcite{maron2017convolutional} identify this flattening as a limitation and propose a map
from mesh to the torus instead.

Since our spherical \cnns\ naturally handle spherical signals, we can operate on
spherical representations of meshes directly.
Moreover, the \swscnns\ also handle spherical vector fields,
so the flow process itself can serve as input shape features
(besides the usual features such as curvatures).
This is similar in spirit to the \hks~\cite{SunOG09},
which samples values of the heat kernel at different times to construct vertex descriptors.
In our case, flow iterations move the vertices,
so we can collect the displacements per time step and assign them
to the corresponding final position of the vertex on the sphere,
resulting in a multi-channel spherical vector field,
which we can call ``curvature flow signature''.
Similarly to the \hks, the flow is determined by intrinsic shape properties, so the representation
is equivariant to isometries.

The idea of encoding shape and transferring features through
flows of diffeomorphisms is central in the field of computational anatomy~\cite{Miller_2015}.
Another useful idea from this field is the definition of distance between shapes through
a metric on the group of diffeomorphisms, which
could be leveraged to construct losses for shape inference tasks.
Refer to \textcite{younes2010shapes} for more details about these ideas.

\subsection{Large scale computer vision problems}
The interest in equivariant representations has grown considerably since the
research for this thesis began, and it continues to grow.
Currently, most successful equivariant models are on tasks with limited data
(e.g., medical imaging), on non-Euclidean manifolds (e.g., spherical images),
or where inputs are heavily perturbed (e.g., rotated $3$D shape retrieval).

While \textcite{CohenW17,weiler2019general} demonstrated that rotation/reflection equivariant
models can improve classification on CIFAR\num{10/100}, which are upright natural image datasets,
no improvements have been achieved on popular large scale computer vision tasks such as
ImageNet image classification~\cite{Russakovsky2015imagenet} and
COCO~\cite{LinMBHPRDZ14} instance segmentation and object detection.

In theory, equivariant models can reduce sample and model complexity,
and improve generalization performance even when global perturbations are not present.
In practice, data augmentation, architecture
and optimizer choices seem to have a larger effect on performance
on these large scale tasks.

At least part of this gap is due to engineering challenges;
significant engineering effort has been put on optimizing the standard deep learning
operations (e.g., the cuDNN library~\cite{ChetlurWVCTCS14}),
while the equivariant counterparts are still mostly in a research stage.
It is possible that, with some engineering,
models equivariant to planar rotation, reflections, and/or scaling
can make progress on popular large scale computer vision tasks.

Even tasks where equivariant models do excel, such as the ones we tackled in
\cref{sph:sec:sphcnn,spin:sec:spin} could be improved with more efficient implementations
and further exploration of architecture design.
Some of the models introduced in this thesis implement fairly complicated operations with many
steps using high-level TensorFlow~\cite{tensorflow2015-whitepaper} operations; it is likely that
lower-level implementations would be more efficient, enabling larger models and
higher resolutions.
The e3nn library~\cite{e3nn_2020_3723557} is a notable related endeavor.

\subsection{Unsupervised learning of symmetries}
Every equivariant model presented in this thesis and the vast majority of
related literature assume that the symmetries are known beforehand;
we design models to be equivariant to a specific group of symmetries.
What if the symmetries are unknown?
While some symmetries are ubiquitous, such as the shift-invariance in natural images,
which justifies the use of $2$D convolutions,
other types of symmetries may be unknown,
so an unsupervised way of finding them in the data can be useful.

There is some recent exploration in this direction.
\textcite{krippendorf20_detec_symmet_with_neural_networ} search for
invariant orbits on the last layer of a pre-trained network, while
\textcite{zhou20_meta_learn_symmet_by_repar} use meta-learning to enforce
filter constraints that reveal underlying symmetries.

\appendix
\clearpage
\phantomsection
\addcontentsline{toc}{chapter}{List of Tables}

\singlespacing
\listoftables
\clearpage
\phantomsection
\renewcommand{\listfigurename}{List of Illustrations}
\addcontentsline{toc}{chapter}{List of Illustrations}
\listoffigures

\doublespacing
\clearpage
\renewcommand{\abbreviationsname}{List of Acronyms}
\printnoidxglossaries

\clearpage
\phantomsection
\addcontentsline{toc}{chapter}{Bibliography} %
\printbibliography

\end{document}